\documentclass[a4paper,oneside,english,british,australian,12pt,a4,round]{book}
\usepackage{lmodern}
\usepackage[T1]{fontenc}
\usepackage{fancyhdr}
\usepackage{array}
\usepackage{float}
\usepackage{textcomp}
\usepackage{mathrsfs}
\usepackage{mathtools}
\usepackage{url}
\usepackage{multirow}
\usepackage{amsmath}
\usepackage{amsthm}
\usepackage{amssymb}
\usepackage{stackrel}
\usepackage{graphicx}
\usepackage{tablefootnote}
\usepackage{wasysym}
\usepackage[authoryear]{natbib}
\usepackage{nomencl}
\providecommand{\printnomenclature}{\printglossary}
\providecommand{\makenomenclature}{\makeglossary}
\makenomenclature

\makeatletter

\pdfpageheight\paperheight
\pdfpagewidth\paperwidth

\providecommand{\tabularnewline}{\\}
\floatstyle{ruled}
\newfloat{algorithm}{tbp}{loa}[chapter]
\providecommand{\algorithmname}{Algorithm}
\floatname{algorithm}{\protect\algorithmname}

\theoremstyle{plain}
    \ifx\thechapter\undefined
	    \newtheorem{thm}{\protect\theoremname}
	  \else
      \newtheorem{thm}{\protect\theoremname}[chapter]
    \fi
\theoremstyle{plain}
    \ifx\thechapter\undefined
      \newtheorem{lem}{\protect\lemmaname}
    \else
      \newtheorem{lem}{\protect\lemmaname}[chapter]
    \fi

\usepackage{graphicx}

\usepackage{colortbl} 
\definecolor{header_color}{rgb}{0.74,0.88,0.91}
\definecolor{even_color}{rgb}{0.9,0.9,0.9}
\definecolor{subheader_color}{rgb}{0.85,0.93,0.95}
\definecolor{childheader_color}{rgb}{1.0,0.93,0.87}

\definecolor{ccolor_best}{rgb}{1.0,0.9,0.9}
\definecolor{ccolor_wrong}{rgb}{1.0,0.85,0.85}

\usepackage[dvipsnames]{xcolor}
\usepackage{textcomp}
\usepackage{setspace}

\usepackage{graphicx}
\usepackage{amsmath}
\usepackage{latexsym}
\usepackage{colortbl}
\usepackage{algpseudocode}
\usepackage[latin1]{inputenc}

\usepackage{flushend}

\usepackage{multirow}

\usepackage{fancyhdr}
\fancyheadoffset{\marginparwidth-2.39cm}
\usepackage[hidelinks]{hyperref}
\hypersetup{
     colorlinks   = true,
     citecolor    = Blue, %BlueViolet,
     linkcolor    = Blue, %Maroon,
     urlcolor     = Brown,
}

\renewcommand{\citet}{\citep}
\DeclareMathOperator{\sigmoid}{sigmoid}
\DeclareMathOperator{\relu}{relu}
\DeclareMathOperator{\sign}{sign}
\DeclareMathOperator{\softmax}{softmax}
\algnewcommand\algorithmicforeach{\textbf{for each}}
\algdef{S}[FOR]{ForEach}[1]{\algorithmicforeach\ #1\ \algorithmicdo}

\usepackage{chngcntr}

\usepackage{capt-of}

\makeatother

\usepackage{babel}
\addto\captionsaustralian{%
	\def\nomname{Nomenclature}%
}
\addto\captionsaustralian{\renewcommand{\lemmaname}{Lemma}}
\addto\captionsaustralian{\renewcommand{\theoremname}{Theorem}}
\addto\captionsbritish{%
	\def\nomname{Nomenclature}%
}
\addto\captionsbritish{\renewcommand{\algorithmname}{Algorithm}}
\addto\captionsbritish{\renewcommand{\lemmaname}{Lemma}}
\addto\captionsbritish{\renewcommand{\theoremname}{Theorem}}
\addto\captionsenglish{%
	\def\nomname{Nomenclature}%
}
\addto\captionsenglish{\renewcommand{\algorithmname}{Algorithm}}
\addto\captionsenglish{\renewcommand{\lemmaname}{Lemma}}
\addto\captionsenglish{\renewcommand{\theoremname}{Theorem}}

\def\nomname{Nomenclature}
\providecommand{\lemmaname}{Lemma}
\providecommand{\theoremname}{Theorem}

\begin{document}
\selectlanguage{british}%
\selectlanguage{english}%
\pagenumbering{roman} 
\setcounter{page}{1}
\pagestyle{empty}

\begin{center}
\vspace*{4cm}
\textbf{\LARGE{}Memory and attention }\\
\textbf{\LARGE{}in deep learning}{\LARGE\par}
\par\end{center}

\vspace{2cm}

\begin{center}
by
\par\end{center}

\vspace{-3cc}

\begin{center}
\textbf{\large{}Hung Thai Le}{\large\par}
\par\end{center}

\vspace{-3cc}

\begin{center}
BSc. (Honours)
\par\end{center}

\vspace{1.5cm}

\begin{center}
Submitted in fulfilment of the requirements for the degree of\\
 Doctor of Philosophy
\par\end{center}

\vspace{3cm}

\begin{center}
{\large{}Deakin University}{\large\par}
\par\end{center}

\vspace{-3cc}

\begin{center}
\emph{\small{}August 2019}{\small\par}
\par\end{center}\selectlanguage{australian}%

\selectlanguage{australian}%

\newpage{}

\pagestyle{fancy} 
\selectlanguage{english}%

\chapter*{Acknowledgements}

\markboth{}{Acknowledgements}

\selectlanguage{australian}%
\addcontentsline{toc}{chapter}{Acknowledgements}

\selectlanguage{english}%
I would like to thank my principal supervisor A/Prof. Truyen Tran
for his continual guidance and support. I have been lucky to have
an outstanding supervisor with deep insight and great vision, who
has taught me valuable lessons for both my work and personal life.
I would also like to express my appreciation to my co-supervisor Prof.
Svetha Venkatesh for giving me the opportunity to undertake research
at PRaDA and for her valuable advice and inspirational talks. Thanks
to my friends Kien Do, Tung Hoang, Phuoc Nguyen, Vuong Le, Romelo,
Tin Pham, Dung Nguyen, Thao Le, Duc Nguyen and everyone else at PRaDA
for making it an original and interesting place to do research. Most
of all, I would like to thank my parents, my sister and my wife for
their encouragement, love and support.\selectlanguage{australian}%

\newpage{}

\selectlanguage{english}%
\tableofcontents{}

\listoffigures

\listoftables

\listof{algorithm}{List of Algorithms}

\newpage{}

\chapter*{Abstract}

\markboth{}{Abstract}

\selectlanguage{australian}%
\addcontentsline{toc}{chapter}{Abstract}

\selectlanguage{english}%
Intelligence necessitates memory. Without memory, humans fail to perform
various nontrivial tasks such as reading novels, playing games or
solving \foreignlanguage{australian}{maths}. As the ultimate goal
of machine learning is to derive intelligent systems \foreignlanguage{british}{that}
learn and act automatically just like human, memory construction for
machine is inevitable. 

Artificial neural networks model neurons and synapses in the brain
by interconnecting computational units via weights, which is a typical
class of machine learning algorithms that resembles memory structure.
Their descendants with more complicated modeling techniques (a.k.a
deep learning) have been successfully applied to many practical problems
and demonstrated the importance of memory in the learning process
of machinery systems. 

Recent progresses on modeling memory in deep learning have revolved
around external memory constructions, which are highly inspired by
computational Turing models and biological neuronal systems. Attention
mechanisms are derived to support acquisition and retention operations
on the external memory. Despite the lack of theoretical foundations,
these approaches have shown promises to help machinery systems reach
a higher level of intelligence. The aim of this thesis is to advance
the understanding on memory and attention in deep learning. Its contributions
include: (i) presenting a collection of taxonomies for memory, (ii)
constructing new memory-augmented neural networks (MANNs) that support
multiple control and memory units, (iii) introducing variability via
memory in sequential generative models, (iv) searching for optimal
writing operations to maximise the memorisation capacity in slot-based
memory networks, and (v) simulating the Universal Turing Machine via
Neural Stored-program Memory\textendash a new kind of external memory
for neural networks. 

The simplest form of MANNs consists of a neural controller operating
on an external memory, which can encode/decode one stream of sequential
data at a time. Our proposed model called Dual Controller Write-Protected
Memory Augmented Neural Network extends MANNs to using dual controllers
executing the encoding and decoding process separately, which is essential
in some healthcare applications. One notable feature of our model
is the write-protected decoding for maintaining the stored information
for long inference. To handle two streams of inputs, we propose a
model named Dual Memory Neural Computer that consists of three controllers
working with two external memory modules. These designs provide MANNs
with more flexibility to process structural data types and thus expand
the range of application for MANNs. In particular, we demonstrate
that our architectures are effective for various healthcare tasks
such as treatment recommendation and disease progression. 

Learning generative models for sequential discrete data such as utterances
in conversation is a challenging problem. Standard neural variational
encoder-decoder networks often result in either trivial or digressive
conversational responses. To tackle this problem, our second work
presents a novel approach that models variability in stochastic sequential
processes via external memory, namely Variational Memory Encoder-Decoder.
By associating each read head of the memory with a mode in the mixture
distribution governing the latent space, our model can capture the
variability observed in natural conversations. 

The third work aims to give a theoretical explanation on optimal memory
operations. We realise that the scheme of regular writing in current
MANN is suboptimal in memory utilisation and introduces computational
redundancy. A theoretical bound on the amount of information stored
in slot-based memory models is formulated and our goal is to search
for optimal writing schemes that maximise the bound. The proposed
solution named Uniform Writing is proved to be optimal under the assumption
of equal contribution amongst timesteps. To balance between maximising
memorisation and overwriting forgetting, we modify the original solution,
resulting in a solution dubbed Cached Uniform Writing. The proposed
solutions are empirically demonstrated to outperform other recurrent
architectures, claiming the state-of-the-arts in various sequential
tasks.

MANNs can be viewed as a neural realisation of Turing Machines and
thus, can learn algorithms and other complex tasks. By leveraging
neural network simulation of Turing Machines to neural architecture
for Universal Turing Machines, we develop a new class of MANNs that
uses Neural Stored-program Memory to store the weights of the controller,
thereby following the stored-program principle in modern computer
architectures. By validating the computational universality of the
approach through an extensive set of experiments, we have demonstrated
that our models not only excel in classical algorithmic problems,
but also have potential for compositional, continual, few-shot learning
and question-answering tasks.

\newpage{}

\chapter*{Relevant Publications}

\markboth{}{Relevant Publications}

\selectlanguage{australian}%
\addcontentsline{toc}{chapter}{Relevant Publications}

\selectlanguage{english}%
Part of this thesis has been published or documented elsewhere. The
details of these publications are as follows: 

Chapter 4: 
\begin{itemize}
\item Le, H., Tran, T., \& Venkatesh, S. (2018). Dual control memory augmented
neural networks for treatment recommendations. In Pacific-Asia Conference
on Knowledge Discovery and Data Mining (pp. 273-284). Springer, Cham.
\item Le, H., Tran, T., \& Venkatesh, S. (2018). Dual memory neural computer
for asynchronous two-view sequential learning. In Proceedings of the
24th ACM SIGKDD International Conference on Knowledge Discovery \&
Data Mining (pp. 1637-1645). ACM.
\end{itemize}
Chapter 5:
\begin{itemize}
\item Le, H., Tran, T., Nguyen, T., \& Venkatesh, S. (2018). Variational
memory encoder-decoder. In Advances in Neural Information Processing
Systems (pp. 1508-1518).
\end{itemize}
Chapter 6:
\begin{itemize}
\item Le, H., Tran, T., \& Venkatesh, S. (2019). Learning to Remember More
with Less Memorization. In International Conference on Learning Representations.
2019.
\end{itemize}
Chapter 7:
\begin{itemize}
\item Le, H., Tran, T., \& Venkatesh, S. (2019). Neural Stored-program Memory.
In International Conference on Learning Representations. 2020.
\end{itemize}
Although not the main contributions, the following collaborative work
is the application of some work in the thesis:
\begin{itemize}
\item Khan, A., Le, H., Do, K., Tran, T., Ghose, A., Dam, H., \& Sindhgatta,
R. (2018). Memory-augmented neural networks for predictive process
analytics. arXiv preprint arXiv:1802.00938. 
\end{itemize}

\newpage{}

\markboth{}{Abbreviations}

\selectlanguage{australian}%
\addcontentsline{toc}{chapter}{Notation}

\selectlanguage{english}%
\renewcommand{\nomname}{Abbreviations} 

\printnomenclature{}

\newpage{}

\selectlanguage{australian}%
\pagenumbering{arabic}
\selectlanguage{english}%

\chapter{Introduction\label{chap:Introduction}}

\section{Motivations}

In a broad sense, memory is the ability to store, retain and then
retrieve information on request. In human brain, memory is involved
in not just remembering and forgetting but also reasoning, attention,
insight, abstract thinking, appreciation and imagination. Modern machine
learning models find and transfer patterns from training data into
some form of memory that will be utilised during inference. In the
case of neural networks, long-term memories on output-input associations
are stored in the weights on the connections between processing units.
These connections are a simple analogy of synapses between neurons
and this form of memory simulates the brain\textquoteright s neocortex
responsible for gradual acquisition of data patterns. Learning in
such scenario is slow since the signal from the output indicating
how to adjust the connecting weights will be both noisy and weak \citet{kumaran2016learning}.
While receiving training data samples, the learning algorithm performs
small update per sample to reach a global optimisation for the whole
set of data.

It is crucial to keep in mind that memory in neural networks does
not limit to the concept of storing associations in the observed data.
For example, in sequential processes, where the individual data points
are no longer independent and identically distributed (i.i.d.), some
form of short-term memory must be constructed across sequence before
the output is given to the network for weight updating. Otherwise,
the long-term memory on associations between the output and inputs,
which are given at different timestamps, will never be achieved. Interestingly,
both forms of memory are found in Recurrent Neural Networks (RNNs)
\citet{elman1990finding,jordan1997serial,rumelhart1988learning}\textendash{}
a special type of neural network capable of modeling sequences. The
featured short-term memory, also referred to as working memory, has
been known to relate with locally stable points \citet{hopfield1982neural,sussillo2014neural}
or transient dynamics \citet{maass2002real,jaeger2004harnessing}
of RNNs. Although these findings shed light into the formation of
the working memory, the beneath memory mechanisms and how they affect
the learning process remain unclear. With the rise of deep learning,
more layers with complicated interconnections between neurons have
been added to neural networks. These complications make it harder
to understand and exploit the working memory mechanisms. Worse still,
due to its short-term capacity, the working memory in RNNs struggles
to cope with long sequences. These challenges require new interpretations
and designs of memory for deep learning in general and RNNs in particular. 

In recent years, memory-augmented neural networks (MANNs) emerge as
a new form of memory construction for RNNs. They model external memories
explicitly and thus, overcome the short-term limitation of the working
memory. Known as one of the first attempts at representing explicit
memory for RNNs, the Long Short-Term Memory (LSTM) \citet{hochreiter1997long}
stores the \textquotedblleft world states\textquotedblright{} in a
cell memory vector, which is determined after a single exposure of
input at each timestep. By referring to the cell memory, LSTM can
bridge longer time lags between relevant input and output events,
extending the range of RNN's working memory. Recent advances have
proposed new external memory modules with multiple memory vectors
(slots) supporting attentional retrieval and fast-update \citet{2014arXiv1410.5401G,graves2016hybrid,weston2014memory}.
The memory slots are accessed and computed fast by a separated controller
whose parameters are slowly learnt weights. Because these memories
are external and separated, it is convenient to derive theoretical
explanations on memorisation capacity \citet{gulcehre2017memory,le2018learning}.
Nonetheless, with bigger memory and flexible read/write operators,
these models significantly outperform other recurrent counterparts
in various long-term sequential testbeds such as algorithmic tasks
\citet{2014arXiv1410.5401G,graves2016hybrid}, reasoning over graphs
\citet{graves2016hybrid}, continual learning \citet{lopez2017gradient},
few-shot learning \citet{santoro2016meta,Le2020Neural}, healthcare
\citet{le2018cdual,prakash2017condensed,Le:2018:DMN:3219819.3219981},
process analytics \citet{khan2018memory}, natural language understanding
\citet{le2018variational,le2018learning} and video question-answering
\citet{gao2018motion}. 

In this thesis, we focus on external memory of MANNs by explaining
and promoting its influence on deep neural architectures. In the original
formulation of MANNs, one controller is allowed to operate on one
external memory. This simple architecture is suitable for supervised
sequence labeling tasks where a sequence of inputs with target labels
are provided for supervised training. However, single controller/memory
design is limited for tasks involving sequence-to-sequence and especially,
multi-view sequential mappings. For example, an electronic medical
record (EMR) contains information on patient\textquoteright s admissions,
each of which consists of various views such as diagnosis, medical
procedure, and medicine. The complexity of view interactions, together
with the unalignment and long-term dependencies amongst views poses
a great challenge for classical MANNs. One important aspect of external
memory is its role in imagination or generative models. Sequence generation
can be supported by RNNs \citet{graves2013generating,chung2015recurrent},
yet how different kinds of memory in RNNs or MANNs cooperate in this
process has not been adequately addressed. Another underexplored problem
is to measure memorisation capacity of MANNs. There is no theoretical
analysis or clear understanding on optimal operations that a memory
should have to maximise its capacity. Finally, the current form of
external memory is definitely not the ultimate memory mechanism for
deep learning.  Current MANNs are equivalent to neural simulations
of Turing Machines \citet{graves2014neural}. Hence, in terms of computational
capacity, MANNs are not superior to RNNs, which are known to be Turing-complete
\citet{siegelmann1995computational}. This urges new designs of external
memory for MANNs that express higher computational power and more
importantly, reach the capacity of human memory. 

\section{Aims and Scope }

This thesis focuses on expanding the capacity of MANNs. Our objectives
are:
\begin{itemize}
\item To construct a taxonomy for memory in RNNs. 
\item To design novel MANN architectures for modeling different aspects
of memory in solving complicated tasks, which include multiple processes,
generative memory, optimal operation, and universality.
\item To apply such architectures to a wide range of sequential problems,
especially those require memory to remember long-term contexts.
\end{itemize}
We study several practical problems that require memory:
\begin{itemize}
\item \emph{Sequence to sequence mapping and multi-view sequential learning}.
The former can be found in treatment recommendation where given time\textendash ordered
medical history as input, we predict a sequence of future clinical
procedures and medications. The problem is harder than normal supervised
sequence labeling tasks because there are dual processes: input encoding
and output decoding. The latter is even more complicated as the input-output
relations not only extend throughout the sequence length, but also
span across views to form long-term intra-view and inter-view interactions,
which is common in drug prescription and disease progression in healthcare.
We aim to extend MANNs to handle these complexities, introducing generic
frameworks to solve multi-view sequence to sequence mapping problems. 
\item \emph{Learning generative models for sequential discrete data}. Tasks
such as translation, question-answering and dialog generation would
benefit from stochastic models that can produce a variety of outputs
for an input. Unfortunately, current approaches using neural encoder-decoder
models and their extensions using conditional variational \foreignlanguage{australian}{autoencoder}
often compose short and dull sentences. As memory plays an important
role in human imagination, we aim to use memory as a main component
that blends uncertainty and variance into neural encoder-decoder models,
thereby introducing variability while maintaining coherence in conversation
generation. 
\item \emph{Ultra-long sequential learning given limited memory resources}.
Current RAM-like memory models maintain memory accessing every timesteps,
thus they do not effectively leverage the short-term memory held in
the controller. Previous attempts try to learn ultra-long sequences
by expanding the memory, which is not always feasible and do not aim
to optimise the memory by some theoretical criterion. It is critical
to derive a theoretical bound on the amount of stored information
and formulate an optimisation problem that maximises the bound under
limited memory size constraint. Our theoretical analysis on this problem
results in novel writing mechanisms that exploit the short-term memory
and approximate the optimal solution. 
\item \emph{Universal sequential learning}. We focus on long-life learning
scenarios where sequences of tasks (subtasks) are handled by an agent,
which requires a memory for tasks to avoid catastrophic forgetting.
Similar situations occur when a Universal Turing Machine simulates
any other Turing Machines to perform universal tasks. Inspired by
the stored-program principle in computer architectures, we aim to
build a Neural Stored-program Memory that enables MANNs to switch
tasks through time, adapt to variable contexts and thus fully resemble
the Universal Turing Machine or Von Neumann Architecture. 
\end{itemize}

\section{Significance and Contribution }

The significance of this thesis is \foreignlanguage{australian}{organised}
around three central lines of work: (i) presenting taxonomy of memory
in RNNs that arise under distinct roles and relations to human memory
(ii) introducing novel MANN designs to model different aspects of
memory and (iii) applying these designs to a wide range of practical
problems in healthcare, dialog, natural language processing, few-shot,
continual learning, etc. In particular, our contributions are:
\begin{itemize}
\item A survey for various types of memory studied for RNNs. The survey
involves different forms of memory in the brain, popular memory constructions
in neural networks and a taxonomy of external memory based on  operational
mechanisms as well as relations to computational models. Several examples
of implementations by modern neural networks are also studied.
\item A generic deep learning model using external memory dubbed Dual Controller
Write-Protected Memory Augmented Neural Network for sequence to sequence
mapping. In the encoding phase, the memory is updated as new input
is read; at the end of this phase, the memory holds the history of
the inputs. During the decoding phase, the memory is write\textendash protected
and the decoding controller generates one output at a time. The proposed
model is demonstrated on the MIMIC-III dataset on two healthcare tasks:
procedure prediction and medication prescription. 
\item A novel MANN architecture named Dual Memory Neural Computer (DMNC)
that can model both synchronous and asynchronous dual view processes.
In the modeling facet, DMNC\textquoteright s contributions are three-fold:
(i) introducing a memory-augmented architecture for modeling multi-view
sequential processes, (ii) capturing long-term dependencies and different
types of interactions \foreignlanguage{british}{amongst} views including
intra-view, late and early inter-view interactions, and (iii) modeling
multiple clinical admissions by employing a persistent memory. In
the application facet, we contribute to the healthcare analytic practice
by demonstrating the efficacy of DMNC on drug prescription and disease
progression. 
\item A Variational Memory Encoder-Decoder (\foreignlanguage{british}{VMED})
framework for sequence generation. VMED introduces variability into
encoder-decoder architecture via the use of external memory as mixture
model. By modeling the latent temporal dependencies across timesteps,
our model produces a Mixture of Gaussians representing the latent
distribution. We form a theoretical basis for our model formulation
using mixture prior for every step of generation and apply our proposed
model to conversation generation problem. The results demonstrate
that VMED outperforms recent advances both quantitatively and qualitatively.
\item A theory driven approach for optimising memory operations in slot-based
MANNs. We contribute a meaningful measurement on MANN memory capacity.
Moreover, we propose Uniform Writing (UW) and Cached Uniform Writing
(CUW) as faster and optimal writing mechanisms for longer-term memorisation
in MANNs. Our models are grounded in theoretical analysis on the optimality
of the introduced measurement. With a comprehensive suite of synthetic
and practical experiments, we provide strong evidences that our simple
writing mechanisms are crucial to MANNs to reduce computation complexity
and achieve competitive performance in sequence modeling tasks.
\item A new type of external memory for neural networks that paves the way
for a new class of MANNs that simulate Universal Turing Machines.
The memory, which takes inspirations from the stored-program memory
in computer architecture, gives memory-augmented neural networks a
flexibility to change their control programs through time while maintaining
differentiability. The mechanism simulates modern computer behavior,
where CPU continually reads different instructions from RAM to execute
different functions, potentially making MANNs truly neural computers.
\end{itemize}

\section{Thesis Structure}

This thesis contains 8 chapters with supplementary materials in the
Appendix. The rest of the thesis is arranged in the following order:
\begin{itemize}
\item Chapter \ref{chap:Tax} presents our survey on taxonomy of memory
in RNNs. The chapter first reviews various memory definitions from
cognitive science. A brief introduction on the most basic neural network\textendash Feedforward
Neural Networks and their fundamental form of memory are then presented.
We process to the main part that covers Recurrent Neural Networks
(RNNs) and memory categories for RNNs based on their formations. Further
interpretations on memory taxonomy based on operational mechanisms
and automata simulations are also investigated. 
\item Chapter \ref{chap:MANN} reviews a special branch of memory in RNNs
and also the main focus of this thesis: memory-augmented neural networks
(MANNs). We first describe the Long Short-term Memory (LSTM) and its
variants. Next, we also spend a section for attention mechanism\textendash a
featured operation commonly exploited in accessing external memory
in MANNs. We then introduce several advanced developments that empower
RNNs with multiple memory slots, especially generic slot-based memory
architectures such as Neural Turing Machine and Differentiable Neural
Computer. 
\item Chapter \ref{chap:multiple} introduces Dual Control Memory-augmented
Neural Network (DC-MANN), an extension of MANN to model sequence to
sequence mapping. Our model supports write-protected decoding (DCw-MANN),
which is empirically proved suitable for sequence-to-sequence task.
We further extend our DC-MANN to a broader range of problems where
the input can come from multiple channels. To be specific, we propose
a general structure Dual Memory Neural Computer (DMNC) that can capture
the correlations between two views by exploiting two external memory
units. We conduct the experiments to validate the performance of these
models on applications in healthcare. 
\item Chapter \ref{chap:Variational-Memory-Encoder} presents a novel memory-augmented
generation framework called Variational Memory Encoder-Decoder. Our
external memory plays a role as a mixture model distribution generating
the latent variables to produce the output and take part in updating
the memory for future generation steps. We adapt Stochastic Gradient
Variational Bayes framework to train our model by minimising variational
approximation of KL divergence to accommodate the Mixture of Gaussians
in the latent space. We derive theoretical analysis to backup our
training protocol and evaluate our model on two open-domain and two
closed-domain conversational datasets.
\item Chapter \ref{chap:Optimal-Writing-in} suggests a meaningful measurement
on MANN's memory capacity. We then formulate an optimisation problem
that maximises the bound on the proposed measurement. The proposed
solution dubbed Uniform Writing is optimal under the assumption of
equal timestep contributions. To relax this assumption, we introduce
modifications to the original solution, resulting in a new solution
termed Cached Uniform Writing. This method aims to balance between
memorising and forgetting via allowing overwriting mechanism. To validate
the effectiveness of our solutions, we conduct experiments on six
ultra-long sequential learning problems given a limited number of
memory slots. 
\item Chapter \ref{chap:Neural-Stored-program-Memory} interprets MANNs
as neural realisations of Turing Machines. The chapter points out
a missing component\textendash the stored-program memory, that is
potential for making current MANNs truly neural computers. Then, a
design of Neural Stored-program Memory (NSM) is proposed to implement
stored-program principle, together with new MANN architectures that
materialise Universal Turing Machines. The significance of NSM lies
in its formulation as a new form of memory, standing in between slow-weight
and fast-weight concepts. NSM not only induces Universal Turing Machine
realisations, which imply universal artificial intelligence, but also
defines another type of adaptive weights, from which other neural
networks can also reap benefits. 
\item Chapter \ref{chap:Conclusions} \foreignlanguage{british}{summarises}
the main content of the thesis and outlines future directions.
\end{itemize}

\chapter{Taxonomy for Memory in RNNs\label{chap:Tax}}

\section{Memory in Brain}

Memory is a crucial part of any cognitive model studying the human
mind. This section briefly reviews memory types studied throughout
the cognitive and neuroscience literature. Fig. \ref{fig:Types-of-memory}
\foreignlanguage{australian}{shows} a taxonomy of cognitive memory
\citet{kotseruba201840}.

\subsection{Short-term Memory}

\paragraph*{Sensory memory}

Sensory memory caches impressions of sensory information after the
original stimuli have ended. It can also preprocess the information
before transmitting it to other cognitive processes. For example,
echoic memory keeps acoustic stimulus long enough for perceptual binding
and feature extraction processes. Sensory memory is known to associate
with temporal lope in the brain. In the neural network literature,
sensory memory can be designed as neural networks without synaptic
learning \citet{johnson2013robust}.

\paragraph*{Working memory}

Working memory holds temporary storage of information related to the
current task such as language comprehension, learning, and reasoning
\citet{baddeley1992working}. Just like computer that uses RAM for
its computations, the brain needs working memory as a mechanism to
store and update information to perform cognitive tasks such as attention,
reasoning and learning. Human neuroimaging studies show that when
people perform tasks requiring them to hold short-term memory, such
as the location of a flash of light, the prefrontal cortex becomes
active \citet{curtis2003persistent}. As we shall see later, recurrent
neural networks must construct some form of working memory to help
the networks learn the task at hand. As working memory is short-term
\citet{goldman1995cellular}, the working memory in RNNs also tends
to vanish quickly and needs the support from other memory mechanisms
to learn complex tasks that require long-term dependencies. 

\subsection{Long-term Memory}

\paragraph*{Motor/procedural memory}

The procedural memory, which is known to link to basal ganglia in
the brain, contains knowledge about how to get things done in motor
task domain. The knowledge may involve co-coordinating sequences of
motor activity, as would be needed when dancing, playing sports or
musical instruments. This procedural knowledge can be implemented
by a set of if-then rules learnt for a particular domain or a neural
network representing perceptual-motor associations \citet{salgado2012procedural}. 

\paragraph*{Semantic memory}

Semantic memory contains knowledge about facts, concepts, and ideas.
It allows us to identify objects and relationships between them. Semantic
memory is a highly structured system of information learnt gradually
from the world. The brain's neocortex is responsible for semantic
memory and its processing is seen as the propagation of activation
amongst neurons via weighted connections that slowly change \citet{kumaran2016learning}.

\paragraph*{Episodic memory}

Episodic memory stores specific instances of past experience. Different
from semantic memory, which does not require temporal and spatial
information, episodic remembering restores past experiences indexed
by event time or context \citet{tulving1972episodic}. Episodic memory
is widely acknowledged to depend on the hippocampus, acting like an
autoassociate memory that binds diverse inputs from different brain
areas that represent the constituents of an event \citet{kumaran2016learning}.
It is conjectured that the experiences stored in hippocampus transfer
to neocortex to form semantic knowledge as we sleep via consolidation
process. Recently, many attempts have been made to integrate episodic
memory into deep learning models and achieved promising results in
reinforcement \citet{mnih2015human,blundell2016model,pritzel2017neural}
and supervised learning \citet{graves2016hybrid,lopez2017gradient,Le:2018:DMN:3219819.3219981}.

\begin{figure}
\begin{centering}
\includegraphics[width=0.95\textwidth]{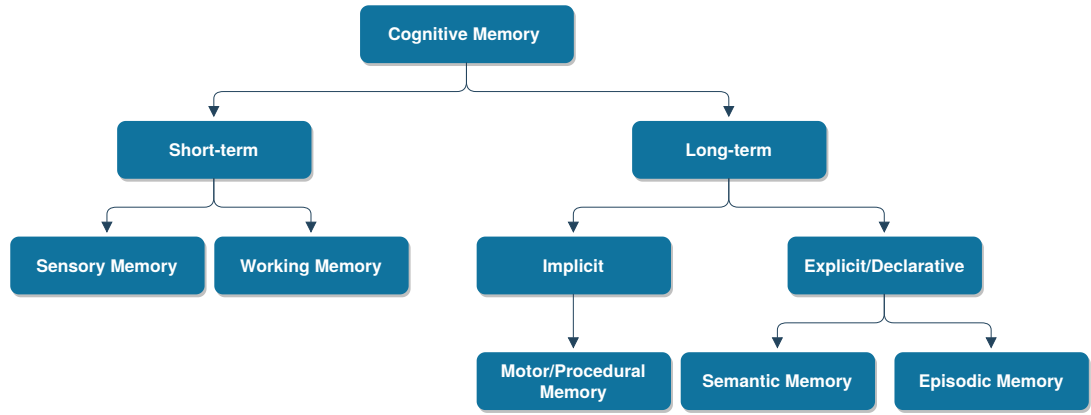}
\par\end{centering}
\caption{Types of memory in cognitive models\label{fig:Types-of-memory}}
\end{figure}

\section{Neural Networks and Memory}

\subsection{Introduction to Neural Networks}

\subsubsection*{Feed-forward neural networks}

A feed-forward neural network arranges neurons in layers with connections
going forward from one layer to another, creating a directed acyclic
graph. That is, connections going backwards or between nodes within
a layer are prohibited. Each neuron in the network is a computation
unit, which takes inputs from outputs of other neurons, then applies
a weighted sum followed by a nonlinear transform, and produces an
output. The multilayer perceptron (MLP) \nomenclature{MLP}{Multilayer Perceptron}
is a commonly used feed-forward neural network for classifying data
or approximating an unknown function. An example MLP is shown in Fig.
\ref{fig:A-multilayer-perceptron}, with three layers: input, output
and a single \textquotedblleft hidden\textquotedblright{} layer. In
order to distinguish linearly inseparable data points, the activation
function must be nonlinear. The weight of a connection, which resembles
synapse of the neocortex, is simply a coefficient by which the output
of a neuron is multiplied before being taken as the input to another
neuron. Hence, the total input to a neuron $j$ is

\begin{equation}
y_{j}=\underset{i}{\sum}w_{ij}x_{i}+b_{j}
\end{equation}
where $x_{i}$ is the output of a neuron $i$, $w_{ij}$ is the weight
of the connection from neuron $i$ to neuron $j$, and $b_{j}$ is
a constant offset or bias. The output of neuron $j$, or $x_{j}$,
is the result of applying an  activation function to $y_{j}$. The
following lists common activation functions used in modern neural
networks,

\begin{equation}
\sigmoid\left(z\right)=\frac{1}{1+e^{-z}}
\end{equation}

\begin{equation}
\tanh\left(z\right)=\frac{e^{z}-e^{-z}}{e^{z}+e^{-z}}
\end{equation}

\begin{equation}
\relu\left(z\right)=\max(z,0)
\end{equation}

\begin{figure}
\begin{centering}
\includegraphics[width=0.7\textwidth]{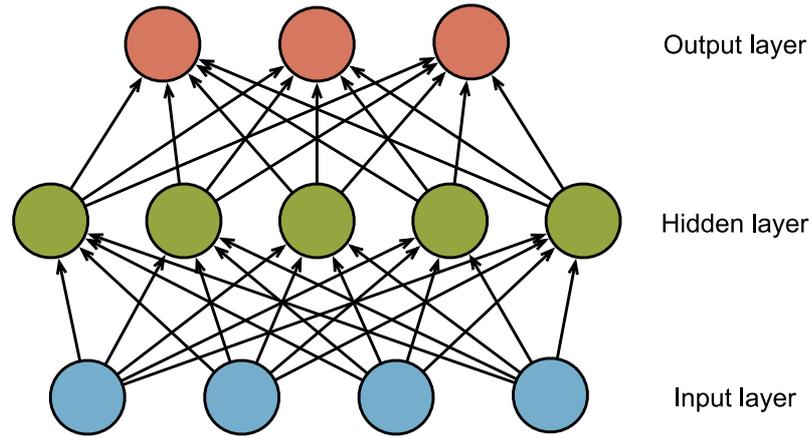}
\par\end{centering}
\caption{A multilayer perceptron with a single hidden-layer.\label{fig:A-multilayer-perceptron}}
\end{figure}

Given a set of training data with ground truth label for each data
points, the network is typically trained with gradient-based optimisation
algorithms, which estimate the parameters by minimising a loss function.
A popular loss function is the average negative log likelihood

\begin{equation}
\mathcal{L}=-\frac{1}{N}\stackrel[i=1]{N}{\sum}\log P\left(\hat{y}_{i}=y_{i}|x_{i}\right)
\end{equation}
where $N$ is the number of training samples, $x_{i}$ and $y_{i}$
is the $i$-th data sample and its label, respectively, and $\hat{y_{i}}$
is the predicted label. During training, forward propagation outputs
$\hat{y_{i}}$ and calculates the loss function. An algorithm called
back-propagation, which was first introduced in \citet{rumelhart1988learning},
computes the gradients of the loss function $\mathcal{L}$ with respect
to (w.r.t) the parameters $\theta=\left\{ w_{ij},b_{j}\right\} $.
Then, an \foreignlanguage{australian}{optimisation} algorithm such
as stochastic gradient descent updates the parameters based on their
gradients $\left\{ \frac{\partial\mathcal{L}}{\partial w_{ij}},\frac{\partial\mathcal{L}}{\partial b_{j}}\right\} $
as follows,

\begin{align}
w_{ij} & \coloneqq w_{ij}-\text{\ensuremath{\lambda\frac{\partial\mathcal{L}}{\partial w_{ij}}}}\\
b_{j} & \coloneqq b_{j}-\lambda\frac{\partial\mathcal{L}}{\partial b_{j}}
\end{align}
where $\lambda$ is a small learning rate. 

\subsubsection*{Recurrent neural networks}

A recurrent neural network (RNN)\nomenclature{RNN}{Recurrent Neural Network}
is an artificial neural network where connections between nodes form
a directed graph with self-looped feedback. This allows the network
to capture the hidden states calculated so far when activation functions
of neurons in the hidden layer are fed back to the input layer at
every time step in conjunction with other input features. The ability
to maintain the state of the system makes RNN especially useful for
processing sequential data such as sound, natural language or time
series signals. So far, many varieties of RNN have been proposed such
as Hopfield Network \citet{hopfield1982neural}, Echo State Network
\citet{jaeger2004harnessing} and Jordan Network \citet{jordan1997serial}.
Here, for the ease of analysis, we only discuss Elman\textquoteright s
RNN model \citet{elman1990finding} with single hidden layer as shown
in Fig. \ref{fig:A-typical-Recurrent}. 

\begin{figure}
\begin{centering}
\includegraphics[width=0.9\textwidth]{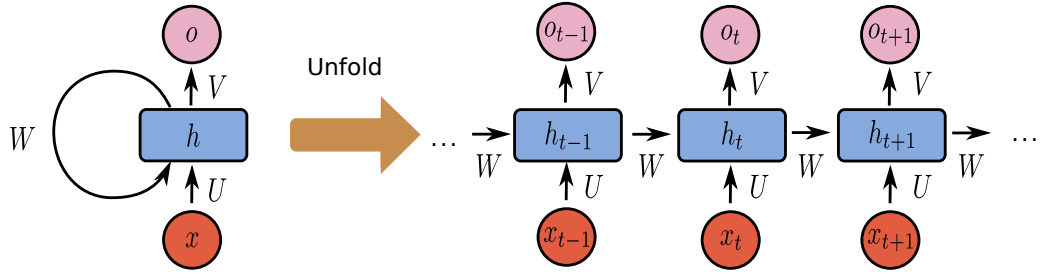}
\par\end{centering}
\caption{A typical Recurrent Neural Network (Left) and its unfolded representation
(Right). Each neuron at timestep $t$ takes into consideration the
current input $x_{t}$ and previous hidden state $h_{t-1}$ to generate
the $t$-th output $o_{t}$. $W$, $U$ and $V$ are learnable weight
matrices of the model.\label{fig:A-typical-Recurrent}}
\end{figure}

An Elman RNN consists of three layers, which are input ($x\in\mathbb{R^{\mathrm{N}}}$),
hidden ($h\in\mathbb{R}^{\mathrm{D}}$) and output ($o\in\mathbb{R}^{\mathrm{M}}$)
layer. At each timestep, the feedback connection forwards the previous
hidden state $h_{t-1}$ to the current hidden unit, together with
the values from input layer $x_{t}$, to compute the current state
$h_{t}$ and output value $o_{t}$. The forward pass begins with a
specification of the initial state $h_{0}$, then we apply the following
update equations
\begin{alignat}{1}
h_{t} & =f\left(h_{t-1}W+x_{t}U+b\right)\label{eq:h_rnn}\\
o_{t} & =g\left(h_{t}V+c\right)\label{eq:o_rnn}
\end{alignat}
where $b\in\mathbb{R}^{\mathrm{D}}$ and $c\in\mathbb{R}^{\mathrm{M}}$
are the bias parameters. $U\in\mathbb{R}^{\mathrm{N\times D}}$, $V\in\mathbb{R}^{\mathrm{D\times M}}$
and $W\in\mathbb{R}^{\mathrm{D\times D}}$ are weight matrices for
input-to-hidden, hidden-to-output and hidden-to-hidden connections,
respectively. $f$ and $g$ are functions that help to add non-linearity
to the transformation between layers. For classification problems,
$g$ is often chosen as the softmax function and the output $o_{t}$
represents the conditional distribution of $t$-th output given previous
inputs. The final output $\hat{y_{t}}$ is the label whose probability
score is the highest. By repeating the updates, one can map the input
sequence $x=\{x_{1},x_{2},...,x_{T}\}$ to an output sequence $\hat{y}=\{\hat{y}_{1},\hat{y}_{2},...,\hat{y}_{T}\}$.
The total loss for a given sequence $x$ paired with a ground-truth
sequence $y=\left\{ y_{1},y_{2},...,y_{T}\right\} $ would then be
the sum of the losses over all the timesteps

\[
\mathcal{L}\left(y|x\right)=\sum_{t=1}^{T}\mathcal{L}_{t}\left(y_{t}|x_{1},x_{2},...,x_{t}\right)=-\sum_{t=1}^{T}\log P\left(\hat{y}_{t}=y_{t}|x_{1},x_{2},...,x_{t}\right)
\]

The loss function can be minimised by using gradient descent approach.
The derivatives w.r.t the parameters can be determined by the Back-Propagation
Through Time algorithm \citet{werbos1990backpropagation}. RNNs are
widely used in sequential tasks such as language modeling \citet{mikolov2010recurrent},
handwriting generation \citet{graves2013generating} and speech recognition
\citet{graves2013speech}. RNNs demonstrate better performance than
other classical approaches using Hidden Markov Model (HMM) or Conditional
Random Fields (CRFs). 

\subsection{Semantic Memory in Neural Networks}

Neural networks learn structured knowledge representation from the
data by adjusting connection weights amongst the units in the network
under supervised training paradigms \citet{hinton1986learning,rumelhart1988learning,plunkett1992connectionism}.
The connection weights capture the semantic structure of the domain
under modeling \citet{mcclelland1995there,rogers2004semantic}. The
trained model generalises to novel examples rather than just naively
memorising training items. However, modern deep learning models are
often massively over-parameterised and thus prone to overfitting,
even to noise \citet{zhang2016understanding}. Further investigations
indicate that although deep networks may employ brute-force memorising
strategy, they should operate in a fashion that can perform inductive
generalisation \citet{arpit2017closer,krueger2017deep}. Unfortunately,
since all of these arguments are validated empirically or via simulations,
no theoretical principles governing semantic knowledge extraction
were given. 

The lack of theoretical guarantee remained until recently when\emph{
}Saxe et al. (2019) confirmed the existence of semantic memory in
neural network by theoretically describing the trajectory of knowledge
acquisition and organisation of neural semantic representations. The
paper is restricted to a simple linear neural network with one hidden
layer. The network is trained to correctly output the associated properties
or features of the input items (e.g., dog $\rightarrow$bark, horse
$\rightarrow$big). Each time a training sample $i$ is presented
as $\left\{ x_{i},y_{i}\right\} $, the weights of the network $W_{1}$
and $W_{2}$ are adjusted by a small amount to gradually minimise
the squared error loss $\mathcal{L}=\left\Vert y_{i}-\hat{y_{i}}\right\Vert ^{2}$.
The parameter update rule is derived via standard back propagation
as follows,

\begin{align}
\Delta W_{1} & =\lambda W_{2}^{\top}\left(y_{i}-\hat{y_{i}}\right)x_{i}^{\top}\\
\Delta W_{2} & =\lambda\left(y_{i}-\hat{y_{i}}\right)\left(W_{1}x_{i}\right)^{\top}
\end{align}
where $\lambda$ is the learning rate. We are interested in estimating
the total weight change after epoch $t$, which can be approximated,
when $\lambda\ll1$, as the following,

\begin{align}
\Delta W_{1}\left(t\right) & \approx\lambda PW_{2}\left(t\right)^{\top}\left(\varSigma^{yx}-W_{2}\left(t\right)W_{1}\left(t\right)\varSigma^{x}\right)\\
\Delta W_{2}\left(t\right) & \approx\lambda P\left(\varSigma^{yx}-W_{2}\left(t\right)W_{1}\left(t\right)\varSigma^{x}\right)W_{1}\left(t\right)^{\top}
\end{align}
where $P$ is the number of training samples; $\varSigma^{x}=E\left[xx^{\top}\right]$
and $\varSigma^{yx}=E\left[yx^{\top}\right]$ are input and input-output
correlation matrices, respectively. We can take the continuum limit
of this difference equation to obtain the following system of differential
equations

\begin{align}
\tau\frac{d}{dt}W_{1} & =W_{2}^{\top}\left(\varSigma^{yx}-W_{2}W_{1}\varSigma^{x}\right)\\
\tau\frac{d}{dt}W_{2} & =\left(\varSigma^{yx}-W_{2}W_{1}\varSigma^{x}\right)W_{1}^{\top}
\end{align}
where $\tau=\frac{1}{P\lambda}$. To simplify the equations, we assume
$\varSigma^{x}=I$ and apply reparametrisation trick to obtain

\begin{align}
\tau\frac{d}{dt}\overline{W}_{1} & =\overline{W}_{2}^{\top}\left(S-\overline{W}_{2}\overline{W}_{1}\right)\label{eq:w1de}\\
\tau\frac{d}{dt}\overline{W}_{2} & =\left(S-\overline{W}_{2}\overline{W}_{1}\right)\overline{W}_{1}^{\top}\label{eq:w2de}
\end{align}
where $S$ is the diagonal matrix in the singular value decomposition
of $\varSigma^{yx}=USV^{\top}$; $\overline{W}_{1}$ and $\overline{W}_{2}$
are new variables such that $W_{1}=R\overline{W}_{1}V^{\top}$ and
$W_{2}=U\overline{W}_{2}R$ with an arbitrary orthogonal matrix $R$.
When $\overline{W}_{1}\left(0\right)$ and $\overline{W}_{2}\left(0\right)$
are initialised with small random weights, we can approximate them
with diagonal matrices of equal modes. A closed form solution of the
scalar dynamic corresponding to each mode of Eqs. (\ref{eq:w1de})
and (\ref{eq:w2de}) can be derived as follows,

\begin{equation}
a_{\alpha}\left(t\right)=\frac{s_{\alpha}e^{2s_{\alpha}t/\tau}}{e^{2s_{\alpha}t/\tau}-1+s_{\alpha}/a_{\alpha}\left(0\right)}
\end{equation}
where $a_{\alpha}$ is a diagonal element of the time-dependent diagonal
matrix $A\left(t\right)$ such that $A\left(t\right)=\overline{W}_{2}\left(t\right)\overline{W}_{1}\left(t\right)$
. Inverting the change of variables yields

\begin{align}
W_{1}\left(t\right) & =Q\sqrt{A\left(t\right)}V^{\top}\\
W_{2}\left(t\right) & =U\sqrt{A\left(t\right)}Q^{-1}
\end{align}
where $Q$ is an arbitrary invertible matrix. If the initial weights
are small, then the matrix $Q$ will be close to a rotation matrix.
Factoring out the rotation, the hidden representation of item $i$
is

\begin{equation}
h_{i}^{\alpha}\left(t\right)=\sqrt{a_{\alpha}\left(t\right)}v_{i}^{\alpha}\label{eq:evol_ht}
\end{equation}
where $v_{i}^{\alpha}=V^{\top}\left[\alpha,i\right]$. Hence, we obtain
a temporal evolution of internal representations $h$ of the deep
network. By using multi-dimensional scaling (MDS) visualisation of
the evolution of internal representations over developmental time,
Saxe et al. (2019) demonstrated a progressive differentiation of hierarchy
in the evolution, which matched the data's underlying hierarchical
structure. When we have the explicit form of the evolution (Eq. (\ref{eq:evol_ht})),
this matching can be proved as an inevitable consequence of deep learning
dynamics when exposed to hierarchically structured data \citet{saxe2019mathematical}. 

\subsection{Associative Neural Networks}

Associative memory is used to store associations between items. It
is a general concept of memory that spans across episodic, semantic
and motor memory in the brain. We can use neural networks (either
feed-forward or recurrent) to implement associative memory. There
are three kinds of associative networks:
\begin{itemize}
\item Heteroassociative networks store $Q$ pair of vectors $\left\{ x^{1}\in\mathcal{X},y^{1}\in\mathcal{Y}\right\} $,
..., $\left\{ x^{Q}\right.\in\mathcal{X},$ $\left.y^{Q}\in\mathcal{Y}\right\} $
such that given some key $x^{k}$, they return value $y^{k}$. 
\item Autoassociative networks are a special type of the heteroassociative
networks, in which $y^{k}=x^{k}$ (each item is associated with itself).
\item Pattern recognition networks are also a special case where $x^{k}$
is associated with a scalar $k$ representing the item's category. 
\end{itemize}
Basically, these networks are used to represent associations between
two vectors. After two vectors are associated, one can be used as
a cue to retrieve the other. In principle, there are three functions
governing an associative memory:
\begin{itemize}
\item Encoding function $\otimes:\mathcal{X}\times\mathcal{\mathcal{Y}}\to\mathcal{M}$
associates input items into some form of memory trace $\mathcal{M}$.
\item Trace composition function $\mathcal{\oplus:M}\times\mathcal{\mathcal{M}}\to\mathcal{\mathcal{M}}$
combines memory traces to form the final representation for the whole
dataset.
\item Decoding function $\bullet:\mathcal{X}\times\mathcal{\mathcal{M}}\to\mathcal{\mathcal{Y}}$
produces a (noisy) version of the item given its associated.
\end{itemize}
Different models employ different kinds of functions (linear, non-linear,
dot product, outer product, tensor product, convolution, etc.). Associative
memory concept is potential to model memory in the brain \citet{marr1991theory}.
We will come across some embodiment of associative memory in the form
of neural networks in the next sections. 

\section{The Constructions of Memory in RNNs}

\subsection{Attractor dynamics }

Attractor dynamics denotes neuronal network dynamics which is dominated
by groups of persistently active neurons. In general, such a persistent
activation associates with an attractor state of the dynamics, which
for simplicity, can take the form of fixed-point \citet{amit1992modeling}.
This kind of network can be used to implement associative memory by
allowing the network's attractors to be exactly those vectors we would
like to store \citet{rojas2013neural}. The approach supports memory
for the items per se, and thus differs from semantic memory in the
sense that the items are often stored quickly and what being stored
cannot represent the semantic structure of the data. Rather, attractor
dynamics resembles working and episodic memory. Like episodic memory,
it acts as an associative memory, returning stored value when triggered
with the right clues. The capacity of attractor dynamics is low, which
reflects the short-term property of working memory. In the next part
of the sub-section, we will study these characteristics through one
embodiment of attractor dynamics.

\subsubsection*{Hopfield network}

The Hopfield network, originally proposed in 1982 \citet{hopfield1982neural},
is a recurrent neural network that implements associative memory using
fix-points as attractors. The function of the associative memory is
to recognise previously learnt input vectors, even in the case where
some noise has been added. To achieve this function, every neuron
in the network is connected to all of the others (see Fig. \ref{fig:Hopfield-network-and}
(a)). Each neuron outputs discrete values, normally $1$ or $-1$,
according to the following equation

\begin{equation}
x_{i}\left(t+1\right)=\sign\left(\stackrel[j=1]{N}{\sum}w_{ij}x_{j}\left(t\right)\right)\label{eq:hopfield_update}
\end{equation}
where $x_{i}\left(t\right)$ is the state of $i$-th neuron at time
$t$ and $N$ is the number of neurons. Hopfield network has a scalar
value associated with the state of all neurons $x$, referred to as
the \textquotedbl energy\textquotedbl{} or Lyapunov function,

\begin{equation}
E\left(x\right)=-\frac{1}{2}\stackrel[i=1]{N}{\sum}\stackrel[j=1]{N}{\sum}w_{ij}x_{i}x_{j}
\end{equation}

If we want to store $Q$ patterns $x^{p}$, $p=1,2,...,Q$, we can
use the Hebbian learning rule \citet{hebb1962organization} to assign
the values of the weights as follows,

\begin{equation}
w_{ij}=\stackrel[p=1]{Q}{\sum}x_{i}^{p}x_{j}^{p}
\end{equation}
which is equivalent to setting the weights to the elements of the
correlation matrix of the patterns\footnote{As an associative memory, Hopfield network implements $\otimes$,
$\oplus$, $\bullet$ by outer product, addition and nonlinear recurrent
function, respectively. }. 

Upon presentation of an input to the network, the activity of the
neurons can be updated (asynchronously) according to Eq. (\ref{eq:hopfield_update})
until the energy function has been minimised \citet{hopfield1982neural}.
Hence, repeated updates would eventually lead to convergence to one
of the stored patterns. However, the network will possibly converge
to spurious patterns (different from the stored patterns) as the energy
in these spurious patterns is also a local minimum. 

\subsubsection*{The capacity problem }

The memorisation of some pattern can be retrieved when the network
produces the desired vector $x^{p}$ such that $x\left(t+1\right)=x\left(t\right)=x^{p}$.
This happens when the crosstalk computed by

\begin{equation}
\stackrel[q=1,q\neq p]{Q}{\sum}x^{q}\left(x^{p}\cdot x^{q}\right)
\end{equation}
is less than $N$. If the crosstalk term becomes too large, it is
likely that previously stored patterns are lost because when they
are presented to the network, one or more of their bits are flipped
by the associative computation. We would like to keep the probability
that this could happen low, so that stored patterns can always be
recalled. If we set the upper bound for one bit failure at 0.01, the
maximum capacity of the network is $Q\thickapprox0.18N$ \citet{rojas2013neural}.
With this low capacity, RNNs designed as attractor dynamics have difficulty
handling big problems with massive amount of data. 

\begin{figure}
\begin{centering}
\includegraphics[width=0.95\textwidth]{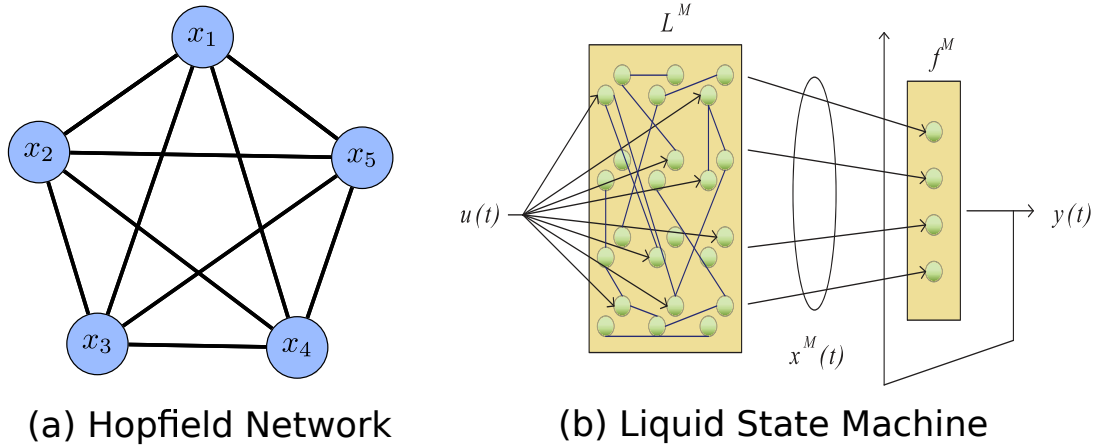}
\par\end{centering}
\caption{(a) Hopfield network with five neurons. (b) Structure of a Liquid
State Machine $M$. The machine wants to transform input stream $u(\cdot)$
into output stream $y(\cdot)$ using some dynamical system $L^{M}$
(the liquid). \label{fig:Hopfield-network-and}}
\end{figure}

\subsection{Transient Dynamics }

One major limitation of memorising by attractor mechanisms is the
incapability of remembering sequences of past inputs. This demands
a new paradigm to explain the working memory mechanism that enable
RNNs to capture sequential dependencies and memorise information between
distance external stimuli. Within this new paradigm, the trajectories
of network states should become the main carriers of information about
external sensory stimuli. Recent proposals \citet{maass2002real,maass2011liquid,jaeger2004harnessing}
have suggested that an arbitrary recurrent network could store information
about recent input sequences in its transient dynamics despite the
presence of attractors (the pattern might or might not converge to
the attractors). A useful analogy is the surface of a liquid. Transient
ripples on the surface can encode information about past objects that
were thrown in even though the water surface has no attractors \citet{ganguli2008memory}.
In the light of transient dynamics, RNNs carry past information to
serve a given task as a working memory. 

\subsubsection*{Liquid State Machines}

Liquid State Machines (LSMs) \nomenclature{LSM}{Liquid State Machine}
\citet{maass2002real} use a dynamic reservoir/liquid ($L^{M}$),
which consists of nodes randomly connected to each other, to handle
time-series data. The purpose is to map an input function of time
$u\left(t\right)$\textendash a continuous sequence of disturbances,
to an output function $y\left(t\right)$ that provides a real-time
analysis of the input sequence. In order to achieve that, we assume
that at every time $t$, $L^{M}$ generates an internal \textquotedblleft liquid
state\textquotedblright{} $x^{M}\left(t\right)$, which constitutes
its current response to preceding perturbations $u(s)$ for $s\leq t$.
After a certain time-period, the state of the liquid $x^{M}\left(t\right)$
is read as input for a readout network $f^{M}$, which by assumption,
has no temporal integration capability of its own. This readout network
learns to map the states of the liquid to the target outputs as illustrated
in Fig. \ref{fig:Hopfield-network-and} (b).

All information about the input $u(s)$ from preceding time points
$s\leq t$ that is needed to produce a target output $y(t)$ at time
$t$ has to be contained in the current liquid state $x^{M}\left(t\right)$.
LSMs allow realisation of large computational power on functions of
time even if all memory traces are continuously decaying. Instead
of worrying about the code and location where information about past
inputs is stored, the approach focuses on addressing the separation
question: for which later time point $t$ will any two significantly
different input functions of time $u\left(t\right)$ and $v\left(t\right)$
cause significantly different liquid states $x_{u}^{M}(t)$ and $x_{v}^{M}(t)$
\citet{maass2011liquid}. 

Most implementations of LSMs use the reservoir of untrained neurons.
In other words, there is no need to train the weights of the RNN.
The recurrent nature of the connections fuses the input sequence into
a spatio-temporal pattern of neuronal activation in the liquid and
computes a large variety of nonlinear functions on the input. This
mechanism is theoretically possible to perform universal continuous
computations. However, separation and approximation properties must
be fulfilled for the system to work well. Similar neural network design
can be found in Echo state networks \citet{jaeger2004harnessing}.
A Liquid State Machine is a particular kind of spiking neural networks
that more closely mimics biological neural networks \citet{maass1997networks}. 

\subsubsection*{Memory trace of recurrent networks\label{subsec:Memory-trace-of}}

When viewing recurrent networks as transient dynamics, one may want
to measure the lifetimes of transient memory traces in the networks.
Ganguli et al. (2018) studied a discrete time network whose dynamics
is given by

\begin{equation}
x_{i}\left(n\right)=f\left(\left[Wx\left(n-1\right)\right]_{i}+v_{i}s\left(n\right)+z_{i}\left(n\right)\right),\,i=1,...,N
\end{equation}
Here, a scalar time-varying signal $s(n)$ drives an RNN of $N$ neurons.
$x(n)$ is the network state at $n$-th timestep, $f(\cdot)$ is a
general sigmoidal function, $W$ is an $N\times N$ recurrent connectivity
matrix, and $v$ is a vector of feed-forward connections encoding
the signal into the network. $z(n)$ denotes a zero mean Gaussian
white noise with covariance $\left\langle z_{i}(k_{1}),z_{j}(k_{2})\right\rangle =\text{\ensuremath{\varepsilon\delta_{k_{1},k_{2}}}\ensuremath{\delta_{i,j}}}$.

The authors built upon Fisher information to construct useful measures
of the efficiency with which the network state $x(n)$ encodes the
history of the signal $s\left(n\right)$, which can be derived as

\begin{equation}
J\left(k\right)=v^{\top}W^{k\top}\left(\varepsilon\stackrel[k=0]{\infty}{\sum}W^{k}W^{k\top}\right)^{-1}W^{k}v\label{eq:fisher_jk}
\end{equation}
where $J\left(k\right)$ measures the Fisher information that $x(n)$
retains about a signal entering the network at $k$ time steps in
the past. For a special case of normal networks having a normal connectivity
matrix $W$, Eq. (\ref{eq:fisher_jk}) simplifies to

\begin{equation}
J\left(k\right)=\stackrel[i=1]{N}{\sum}v_{i}^{2}\left|\lambda_{i}\right|^{2k}\left(1-\left|\lambda_{i}\right|^{2}\right)
\end{equation}
where $\lambda_{i}$ is the $i$-th eigenvalue of $W$. For large
k, the decay of the Fisher information is determined by the magnitudes
of the largest eigenvalues and it decays exponentially. Similar findings
with different measurements on the memory trace in modern recurrent
networks are also found in a more recent work \citet{le2018learning}.

\section{External Memory for RNNs\label{sec:External-Memory-for}}

Recurrent networks can in principle use their feedback connections
to store representations of recent input events in the form of implicit
memory (either attractor or transient dynamics). Unfortunately, from
transient dynamics perspective, the implicit memory tends to decay
quickly \citet{ganguli2008memory,le2018learning}. This phenomenon
is closely related to gradient vanishing/exploding problems \citet{bengio1994learning,hochreiter1997long,pascanu2013difficulty}
which often occur when training RNNs with gradient-based algorithms
such as Back-Propagation Through Time \citet{williams1989learning,werbos1990backpropagation}.
A solution is to equip RNNs with external memory to cope with exponential
decay of the implicit short-term memory. The external memory enhances
RNNs with stronger working \citet{hochreiter1997long,cho2014gru,graves2014neural,graves2016hybrid}
or even episodic-like memory \citet{graves2014neural,santoro2016meta}.
We will spend the next sections to analyse different types of external
memory and their memory operation mechanisms. Examples of modern recurrent
neural networks that utilise external memory are also discussed. 

\subsection{Cell Memory\label{subsec:Cell-memory}}

Despite the fact that RNNs offer working memory mechanisms to handle
sequential inputs, learning what to put in and how to utilise the
memory is challenging. Back-Propagation Through Time \citet{williams1989learning,werbos1990backpropagation}
is the most common learning algorithm for RNNs, yet it is inefficient
in training long sequences mainly due to insufficient or decaying
backward error signals. This section will review the analysis of this
problem and study a group of methodologies that overcome the problem
through the use of cell memory and gated operation. 

\begin{figure}
\begin{centering}
\includegraphics[width=0.95\textwidth]{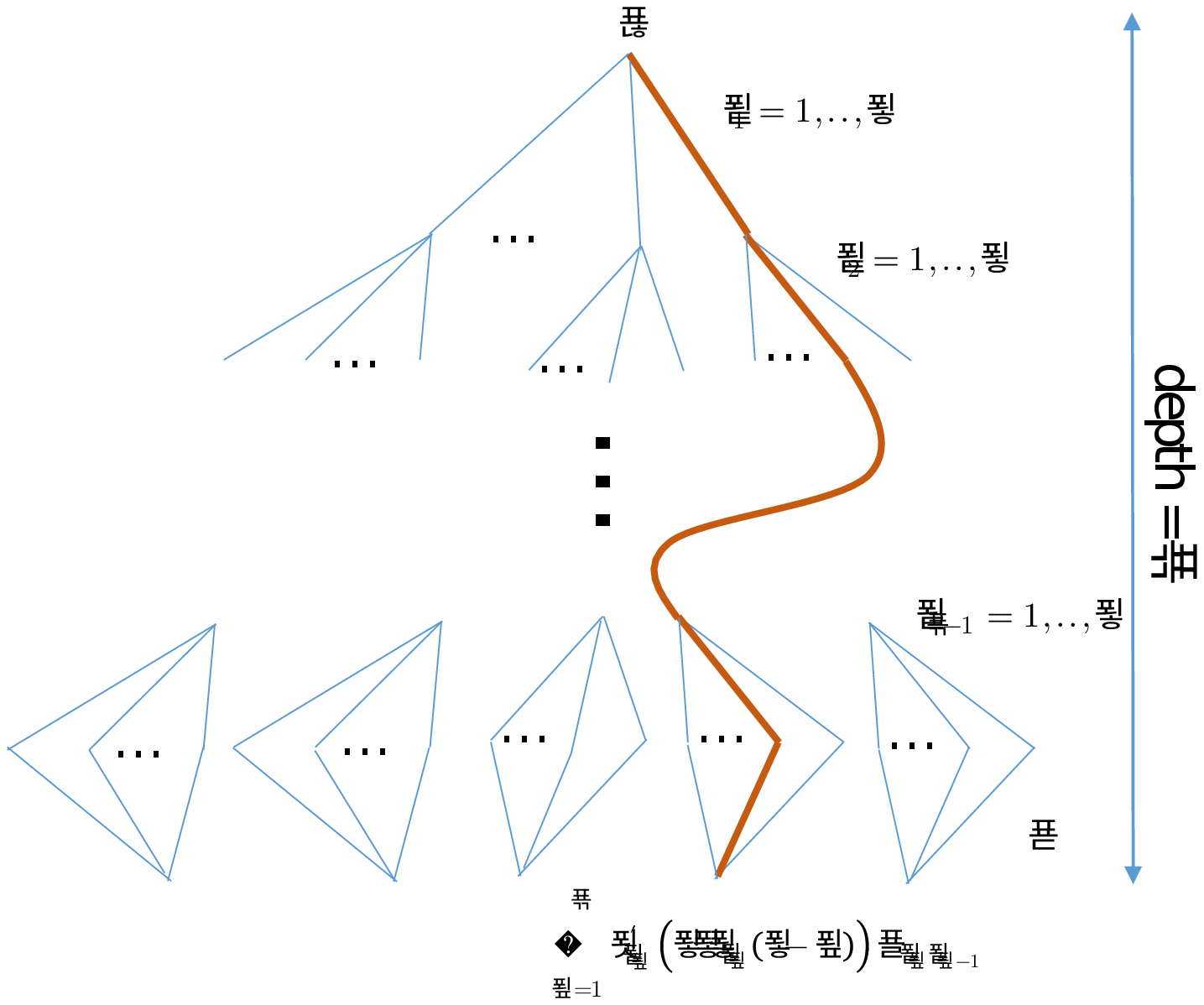}
\par\end{centering}
\caption{Error back flow from $\vartheta_{u}\left(t\right)$ to $\vartheta_{v}\left(t-q\right)$
in the computation graph. Each computation node has $n$ children.
Each product term corresponds to a computation path of depth $q$
from node $u$ to $v$. The sum of $n^{q-1}$ products is the total
error. \label{fig:Error-back-flow}}
\end{figure}

\subsubsection*{Hochreiter's analysis on gradient vanishing/exploding problems}

Let us assume that the hidden layer of an RNN has $n$ neurons. With
differentiable activation function $f_{i}$, the activation of a neuron
$i$ at step $t$ of the recurrent computation is as follow,

\begin{align}
y^{i}\left(t\right) & =f_{i}\left(\underset{j}{\sum}w_{ij}y^{j}\left(t-1\right)\right)\\
 & =f_{i}\left(net_{i}\left(t\right)\right)
\end{align}
The backpropagated error signal for neuron $j$ at step $t$ is

\begin{equation}
\vartheta_{j}\left(t\right)=f_{j}^{\prime}\left(net_{j}\left(t\right)\right)\underset{i}{\sum}w_{ij}\vartheta_{i}\left(t+1\right)
\end{equation}
The error occurring at an arbitrary neuron $u$ at time step $t$
($\vartheta_{u}\left(t\right)$) is backpropagated through time for
$q$ timesteps to an arbitrary neuron $v$ ($\vartheta_{v}\left(t-q\right)$).
We can measure the contribution of the former to the latter as the
following,

\begin{equation}
\frac{\partial\vartheta_{v}\left(t-q\right)}{\partial\vartheta_{u}\left(t\right)}=\begin{cases}
f_{v}^{\prime}\left(net_{v}\left(t-1\right)\right)w_{uv} & ;q=1\\
f_{v}^{\prime}\left(net_{v}\left(t-q\right)\right)\sum_{l=1}^{n}\frac{\partial\vartheta_{l}\left(t-q+1\right)}{\partial\vartheta_{u}\left(t\right)}w_{lv} & ;q>1
\end{cases}\label{eq:rec_bperr}
\end{equation}

By induction, we can obtain expressive form of the recursive Eq. (\ref{eq:rec_bperr})
as

\begin{equation}
\frac{\partial\vartheta_{v}\left(t-q\right)}{\partial\vartheta_{u}\left(t\right)}=\sum_{l_{1}=1}^{n}...\sum_{l_{q-1}=1}^{n}\prod_{m=1}^{q}f_{l_{m}}^{\prime}\left(net_{l_{m}}\left(t-m\right)\right)w_{l_{m}l_{m-1}}
\end{equation}
where $l_{q}=v$ and $l_{0}=u$. The computation can be visually explained
through a drawing of the computation graph as in Fig. \ref{fig:Error-back-flow}.

It is obvious to realise that if $\left|f_{l_{m}}^{\prime}\left(net_{l_{m}}\left(t-m\right)\right)w_{l_{m}l_{m-1}}\right|$
is greater (smaller) than $1$ for all $m$, then the largest product
increases (decreases) exponentially with $q$, which represents the
exploding and vanishing gradient problems in training neural networks.
These problems are critical since they prevent proper update on the
weights of the model, and thus freeze or disturb the learning process.
With nonlinear activation functions such as $\sigmoid$, the term
$\left|f_{l_{m}}^{\prime}\left(net_{l_{m}}\right)w_{l_{m}l_{m-1}}\right|$
goes to zero when $w_{l_{m}l_{m-1}}\to\infty$ and is less than $1$
when $\left|w_{l_{m}l_{m-1}}\right|<4$, which implies vanishing gradient
tends to occur with nonlinear activation function. 

We can also rewrite Eq. (\ref{eq:rec_bperr}) in matrix form for $q>1$
as follows,

\begin{equation}
\frac{\partial\vartheta_{v}\left(t-q\right)}{\partial\vartheta_{u}\left(t\right)}=W_{u}^{\top}F^{\prime}\left(t-1\right)\prod_{m=2}^{q-1}\left(WF^{\prime}\left(t-m\right)\right)W_{v}f_{v}^{\prime}\left(net_{v}\left(t-q\right)\right)
\end{equation}
where the weight matrix $W$ have its elements $W_{ij}=w_{ij}$. $W_{u}$
and $W_{v}$ are $u$'s incoming weight vector and $v$'s outgoing
weight vector, respectively, such that $\left[W_{u}\right]_{i}=w_{ui}$
and $\left[W_{v}\right]_{i}=w_{vi}$. $F^{\prime}\left(t-m\right)$
is a diagonal matrix whose diagonal elements $F^{\prime}\left(t-m\right)_{ii}=f_{i}^{\prime}\left(net_{i}\left(t-m\right)\right)$.
Using a matrix norm $\left\Vert \cdot\right\Vert _{A}$ compatible
with vector norm $\left\Vert \cdot\right\Vert _{p}$, we define

\begin{equation}
f_{max}^{\prime}\coloneqq\max_{m=1,...,q}\left\{ \left\Vert F^{\prime}\left(t-m\right)\right\Vert _{A}\right\} 
\end{equation}
By applying norm sub-multiplicativity and using the inequality 

\[
\left|x^{T}y\right|\leq n\left\Vert x\right\Vert _{\infty}\left\Vert y\right\Vert _{\infty}\leq n\left\Vert x\right\Vert _{p}\left\Vert y\right\Vert _{p},
\]
we obtain a weak upper bound for the contribution

\begin{equation}
\left|\frac{\partial\vartheta_{v}\left(t-q\right)}{\partial\vartheta_{u}\left(t\right)}\right|\leq n\left(f_{max}^{\prime}\left\Vert W\right\Vert _{A}\right)^{q}
\end{equation}

This result confirms the exploding and vanishing gradient problems
since the error backprob contribution decays (when $f_{max}^{\prime}\left\Vert W\right\Vert _{A}$
< 1) or grows (when $f_{max}^{\prime}\left\Vert W\right\Vert _{A}$
>1) exponentially with $q$. More recent analyses on the problems
are presented by Bengio et al., (1994) and Pascanu et al., (2013). 

\subsubsection*{Problem with naive solution}

When analysing a single neuron $j$ with a single connection to itself,
avoiding the exploding and vanishing gradient problems requires

\begin{equation}
f_{j}^{\prime}\left(net_{j}\left(t\right)\right)w_{jj}=1
\end{equation}
In this case, the constant error flow is enforced by using linear
function $f_{j}$ and constant activation (e.g., $f_{j}\left(x\right)=x$
with $\forall x$ and setting $w_{jj}=1$). These properties are known
as the \emph{constant error carousel} (CEC)\nomenclature{CEC}{Constant Error Carousel}.
The strict constraint makes this solution unattractive because it
limits computation capacity of RNNs with linear activation. Even worse,
neuron $j$ is connected to other neurons as well, which makes thing
complicated. Let us consider an additional input weight $w_{ji}$
connecting neuron $i$ to $j$. $w_{ji}$ is learnt to keep relevant
external input from $i$ such that $w_{ji}y_{i}>0$ when the input
signal $y_{i}$ is relevant. Assume that the loss function is reduced
by keeping neuron $j$ active ($>0$) for a while between two occurrences
of two relevant inputs. During that period, activation of neuron $j$
is possibly disturbed since with a fixed $w_{ji}$, $w_{ji}y_{i}<0$
with irrelevant inputs. Since $y^{j}\left(t\right)=f_{j}\left(w_{jj}y^{j}\left(t-1\right)+w_{ji}y^{i}\left(t-1\right)\right)$
where $f_{j}$ is linear, $y^{j}\left(t-1\right)$ is kept constant
and $y^{i}\left(t-1\right)$ scales with the external input, it is
likely to deactivate neuron $j$. Hence, if naively following CEC,
learning a $w_{ji}$ to capture relevant inputs while protecting neuron
$j$ from disturbances of irrelevant inputs is challenging (input
weight conflict \citet{hochreiter1997long}). Similar problem happens
with the output weight (output weight conflict). These conflicts make
the learning hard, and require a more flexible mechanism for controlling
input/output weight impact conditioned on the input signal. 

\subsubsection*{The original Long Short-Term Memory (LSTM)}

Hochreiter and Schmidhuber (1997) originally proposed LSTM\nomenclature{LSTM}{Long Short-Term Memory}
using multiplicative gate units and a memory cell unit to overcome
the weight conflicts while following CEC. The idea is to apply CEC
to neurons specialised for memorisation, each of which has an internal
state independent from the activation function. This separation between
memorisation and computation is essential for external memory concept.
Besides, to control input/output weight impact, gate units conditioned
on the inputs are multiplied with the incoming/outgoing connections,
modifying the connection value through time. In particular, if a neuron
$c_{j}$ becomes cell memory, its output is computed as

\begin{equation}
y^{c_{j}}\left(t\right)=y^{out_{j}}\left(t\right)h\left(s_{c_{j}}\left(t\right)\right)\label{eq:outcell}
\end{equation}
where $y^{out_{j}}\left(t\right)$ is the output gate, $h$ is a differentiable
function for scaling down the neuron's output, and $s_{c_{j}}$ captures
past information by using the dynamics

\begin{align}
s_{c_{j}}\left(0\right) & =0\\
s_{c_{j}}\left(t\right) & =y^{fg_{j}}\left(t\right)s_{c_{j}}\left(t-1\right)+y^{in_{j}}\left(t\right)f\left(net_{c_{j}}\left(t\right)\right)\,\mathrm{for}\,t>0\label{eq:scj}
\end{align}
where $y^{in_{j}}\left(t\right)$ is the input gate, $y^{fg_{j}}\left(t\right)$
is the (optional) forget gate and $f$ is the activation function,
which can be nonlinear. Without forget gate, $c_{j}$ can be viewed
as a neuron with an additional fixed self-connection. The computation
paths that mainly pass through this special neuron preserve the backward
error. The remaining problem is to protect this error from disturbance
from other paths. The gates are calculated as

\begin{equation}
y^{g_{j}}\left(t\right)=f_{g_{j}}\left(\sum_{u}w_{g_{j}u}y^{u}\left(t-1\right)\right)
\end{equation}
where $g$ can represent input, output and forget gate. The gates
are adaptive according to the input from other neurons, hence, it
is possible to learn $\left\{ w_{g_{j}u}\right\} $ to resolve the
input/output weight conflict problem. 

Although the cell memory provides a potential solution to cope with
training RNN over long time lag, unfortunately, in practice, the multiplicative
gates are not good enough to overcome a fundamental challenge of LSTM:
the gates are not coordinated at the start of training, which can
cause $s_{c_{j}}$ to explode quickly (internal state drift). Various
variants of LSTM have been proposed to tackle the problem \citet{greff2016lstm}.
We will review some of them in Chapter \ref{chap:MANN}. 

\subsubsection*{Cell memory as external memory}

From Eq. (\ref{eq:scj}), we can see the internal state of the cell
memory holds two types of information: (i) the previous cell state
and (ii) the normal state of RNN, which is the activation of current
computation. Therefore, the cell state contains a new form of external
memory for RNNs. The size of the memory is often equal the number
of hidden neurons in RNNs and thus, cell memory is also known as vector
memory. The memory supports writing and reading mechanisms implemented
as gated operations in $y^{in_{j}}\left(t\right)$ and $y^{out_{j}}$,
respectively. They control how much to write to and read from the
cell state. With the cell state, which is designed to keep information
across timesteps, the working memory capacity of LSTM should be greater
than that of RNNs. The memory reading and writing are also important
to determine the memory capacity. For instance, if writing irrelevant
information too often, the content in the cell state will saturate
and the memory fails to hold much information. Later works make use
of the gating mechanism to build skip-connections between inputs (a
source of raw memory) and neurons in higher layers \citet{srivastava2015training,he2016deep},
opening chance to ease the training of very deep networks. 

\subsection{Holographic Associative Memory}

The holographic associative memory (HAM\nomenclature{HAM}{Holographic Associative Memory})
roots its operation on the principle of optical holography, where
two beams of light are associated with one another in a holograms
such that reconstruction of one original beam can be made by presenting
another beam. Recall that the capacity of associative memory using
attractor dynamics is low. To maintain $Q$ pairs of key-value (in
Hopfield network, value is also key), it requires $N^{2}$ weight
storage where $Q\approx0.18N$. HAM presents a solution to compress
the key-values into a fixed size vector via Holographic Reduced Representation
(HRR) without substantial loss of information \citet{plate1995holographic}.
This can be done in real or complex domain using circular convolution
or element-wise complex multiplication for the encoding function ($\otimes$),
respectively. The compressed vector ($\mathcal{M}$), as we shall
see, can be used as external memory for RNNs. 

\subsubsection*{Holographic Reduced Representation }

Consider a complex-valued vector key $x\in\mathbb{C}^{N}$,

\begin{equation}
x=\left[x_{a}\left[1\right]e^{ix_{\phi}\left[1\right]},...,x_{a}\left[N\right]e^{ix_{\phi}\left[N\right]}\right]
\end{equation}
The association encoding is computed by

\begin{align}
m & =x\circledast y\\
 & =\left[x_{a}\left[1\right]y_{a}\left[1\right]e^{i\left(x_{\phi}\left[1\right]+y_{\phi}\left[1\right]\right)},...,x_{a}\left[N\right]y_{a}\left[N\right]e^{i\left(x_{\phi}\left[N\right]+y_{\phi}\left[N\right]\right)}\right]
\end{align}
where $\circledast$ is element-wise complex multiplication, which
multiplies the moduli and adds the phases of the elements. Trace composition
function is simply addition

\begin{equation}
m=x^{1}\circledast y^{1}+x^{2}\circledast y^{2}+...+x^{Q}\circledast y^{Q}
\end{equation}
Although the memory $m$ is a vector with the same dimension as that
of stored items, it can store many pairs of items since we only need
to store the information that discriminates them. The decoding function
is multiplying an inverse key $x^{-1}=\left[x_{a}\left[1\right]^{-1}e^{-ix_{\phi}\left[1\right]},...,x_{a}\left[N\right]^{-1}e^{-ix_{\phi}\left[N\right]}\right]$
with the memory as follows,

\begin{align}
\tilde{y} & =x^{-1}\circledast m\\
 & =x^{-1}\circledast\left(\sum_{\forall k}x^{k}\circledast y^{k}\right)\\
 & =y+x^{-1}\circledast\left(\sum_{\forall k:x^{k}\neq x}x^{k}\circledast y^{k}\right)\label{eq:holo_decode}
\end{align}
The second term in Eq. (\ref{eq:holo_decode}) is noise and should
be minimised. Under certain conditions, the noise term has zero mean
\citet{plate1995holographic}. One way to reconstruct better is to
pass the retrieved vector through an auto-associative memory to correct
any errors. 

\subsubsection*{Redundant Associative Long Short-Term Memory }

One recent attempt to apply HRR\nomenclature{HRR}{Holographic Reduced Representation}
to LSTM is the work by Danihelka et al. (2016). The authors first
propose Redundant Associative Memory, an extension of HRR with multiple
memory traces for multiple transformed copies of each key vector.
In particular, each key vector will be transformed $S$ times using
$S$ constant random permutation matrix $P_{s}$. Hence, we obtain
the memory trace $c_{s}$ for the $s$-th copy

\begin{equation}
c_{s}=\sum_{\forall k}\left(P_{s}x^{k}\right)\circledast y^{k}
\end{equation}
The $k$-th value is retrieved as follows,

\begin{align}
\tilde{y}^{k} & =\frac{1}{S}\sum_{s=1}^{S}\left(\overline{P_{s}x^{k}}\right)\circledast c_{s}\\
 & =y^{k}+\sum_{k^{\prime}\neq k}y^{k^{\prime}}\circledast\frac{1}{S}\sum_{s=1}^{S}P_{s}\left[\overline{x^{k}}\circledast x^{k^{\prime}}\right]
\end{align}
where $\overline{P_{s}x^{k}}$ and $\overline{x^{k}}$ are the complex
conjugates of $P_{s}x^{k}$ and $x^{k}$, respectively, which are
equal to the inverses if the modulus $x_{a}^{k}=1$. Since permuting
the key decorrelates the retrieval noise, the noise term has variance
$O\left(\frac{Q}{S}\right)$ and increase the number of copies will
enhance retrieval quality. 

Applying the idea to LSTM, we can turn the cell memory to a holographic
memory by encoding the term containing input activation in Eq. (\ref{eq:scj})
before added up to the cell memory. The network learns to generate
the key $x^{k}$ and the inverse key $\left(x^{-1}\right)^{k}$ for
$k$-th timestep. It should be noted that the inverse key at $k$-th
timestep can associate to some preceding key. Following Redundant
Associative Memory extension, multiple copies of cell memory are employed.
The cell state will be decoded to retrieve some past input activation
necessary for current output \citet{danihelka2016associative}. Then
the decoded value will be multiplied with the output gate as in Eq.
(\ref{eq:outcell}).

\subsection{Matrix Memory}

\subsubsection*{Correlation matrix memory}

Correlation Matrix Memory (CMM) stores associations between pairs
of vectors using outer product as the encoding function. Although
the purpose looks identical to that of attractor dynamics, CMM\nomenclature{CMM}{Correlation Matrix Memory}
is arranged differently using feed-forward neural network without
self-loop connections. The memory construction ($\otimes+\oplus$)
follows Hebbian learning

\begin{equation}
M=\sum_{i=1}^{Q}y_{i}x_{i}^{\top}
\end{equation}
where $Q$ is the number of stored patterns, $x_{i}$ and $y_{i}$
are the $i$-th key-value pair. The memory retrieval ($\bullet$)
is simply dot product

\begin{align}
\tilde{y_{j}} & =Mx_{j}\\
 & =\left(\sum_{i=1}^{Q}y_{i}x_{i}^{\top}\right)x_{j}\\
 & =\sum_{i=1,i\neq j}^{Q}y_{i}x_{i}^{\top}x_{j}+y_{j}\left\Vert x_{j}\right\Vert ^{2}
\end{align}
If the keys are orthonormal, then the retrieval is exact. Actually,
linear independence is enough for exact retrieval. In this case, WidrowHoff
learning rule should be used. 

When the stored values are binary vectors, a threshold function is
applied. The capacity for binary CMM is heavily dependent on the sparsity
of the patterns (the sparser the better). In general, CMM offers a
capacity that is at least comparable to that of the Hopfield model
\citet{baum1988internal}.

\subsubsection*{Fast-weight}

Fast-weights refer to synapses that change slower than neuronal activities
but much faster than the standard slow weights. These fast weights
form temporary memories of the recent past that support the working
memory of RNNs \citet{hinton1987using,schmidhuber1992learning,ba2016using}.
In a recent fast-weight proposal \citet{ba2016using}, the memory
is similar to a correlation matrix memory with decaying factor to
put more weight on the recent past. In particular, the fast memory
weight matrix $A$ is computed as follows,

\begin{equation}
A\left(t\right)=\lambda A\left(t-1\right)+\eta h\left(t\right)h\left(t\right)^{\top}
\end{equation}
where $\lambda$ and $\eta$ are the decay and learning rate, respectively.
$h\left(t\right)$ is the hidden state of the RNN and also the pattern
being stored in the associative memory. The memory is used to iteratively
refine the next hidden state of RNN as the following,

\begin{equation}
h_{s+1}\left(t+1\right)=f\left(\left[Wh\left(t\right)+Cx\left(t\right)\right]+A\left(t\right)h_{s}\left(t+1\right)\right)
\end{equation}
where $h_{0}\left(t+1\right)=f\left(Wh\left(t\right)+Cx\left(t\right)\right)$,
following the ordinary dynamics of RNNs and $h_{s}\left(t+1\right)$
is the hidden state at $s$-th step of refinement. 

\subsubsection*{Tensor product representation}

Tensor product representation (TPR) is a mechanism to store symbolic
structures. It shares common properties with CMM when the tensor is
of order 2, in which tensor product is equivalent to outer product.
In TPR\nomenclature{TPR}{Tensor Product Representation}, relations
between concepts are described by the set of filler-role bindings.
The vector space of filler and role are denoted as $V_{\mathcal{F}}$
and $V_{\mathcal{R}}$, respectively. The TPR is defined as a tensor
$T$ in a vector space $V_{\mathcal{F}}\otimes V_{\mathcal{R}}$,
where $\otimes$ is the tensor product operator, which is computed
as

\begin{equation}
T=\sum_{i}f_{i}\otimes r_{i}
\end{equation}
where $f_{i}$ and $r_{i}$ are vectors representing some filler and
role, respectively. The tensor dot product $\bullet$ is used to decode
the memory as follows,

\begin{equation}
f_{j}=T\bullet r_{j}
\end{equation}

For example, the following 4 concepts have relations: \emph{dog}(\emph{bark})
and \emph{horse}(\emph{big}) in which the set of filler is $\mathcal{F}=\left\{ bark,horse\right\} $
and the set of role is $\mathcal{R}=\left\{ bark,big\right\} $. The
TPR of these concepts is 

\begin{equation}
T=f_{dog}\otimes r_{bark}+f_{horse}\otimes r_{big}
\end{equation}
Or we can encode a tree structure as in Fig. \ref{fig:SDM's-memory-write}
(a) by the following operations:

\begin{align}
T & =A\otimes r_{0}\otimes+\left(B\otimes r_{0}+C\otimes r_{1}\right)\otimes r_{1}\\
 & =A\otimes r_{0}\otimes+B\otimes r_{0}\otimes r_{1}+C\otimes r_{1}\otimes r_{1}\\
 & =A\otimes r_{0}\otimes+B\otimes r_{01}+C\otimes r_{11}
\end{align}
This mechanism allows storing symbolic structures and grammars and
thus supports reasoning. For further details, we refer readers to
the original work \citet{smolensky1990tensor} and recent application
to deep learning \citet{schlag2018learning,le2020self}.

\begin{figure}
\begin{centering}
\includegraphics[width=0.9\textwidth]{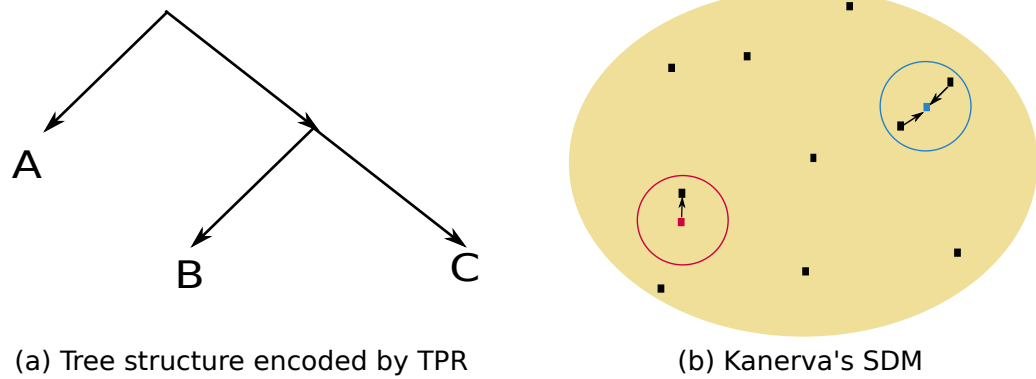}
\par\end{centering}
\caption{(a) Example of a tree encoded by TPR. (b) SDM's memory write (red)
and read (blue) access. The read and write involve all memory locations
around the queried points.\label{fig:SDM's-memory-write}}
\end{figure}

\subsection{Sparse Distributed Memory}

Matrix memory is a direct extension to vector memory for RNNs. There
are two ways to build a matrix memory: correlation matrix memory (or
tensor memory) and sparse distributed memory. While the former focuses
on storing the associations amongst items (e.g., Hopfield network,
Holographic memory and CMM), the latter aims to store each item as
a high-dimensional vector, which is closer to Random Access Memory
in computer architecture. Because each vector is physically stored
in a memory slot, we also refer to this model as slot-based memory.
Sparse distributed memory (SDM\nomenclature{SDM}{Sparse Distributed Memory})
can represent correlation matrix memory, computer memory, feed-forward
artificial neural networks and associative-memory models of the cerebellum.
Such a versatility naturally results in SDM's applications to RNN
as one form of external memory. 

\subsubsection*{Kanerva memory model }

In 1988, Pentti Kanerva introduced the SDM as a new approach to model
human long-term memory \citet{kanerva1988sparse}. The model revolves
around a simple idea that the distances between concepts in our minds
correspond to the distances between points of a high-dimensional space.
As we, when hinted by key signals, tend to remember specific things
such as individual, object, scene and place, the brain must make the
identification nearly automatic, and high-dimensional vectors as internal
representations of things do that. Another important property of high
dimensional spaces is that distance between two random points should
be far, which allows inexact representation of the point of interest.
In other words, using long vectors to store items enables a fault-tolerant
and robust memory. 

The SDM stores items (binary vectors) in a large number of hard locations
or memory slots whose addresses ($m_{a}$) are given by binary strings
of length $D$, randomly distributed throughout the address space
$\left\{ 0,1\right\} ^{D}$. Input to the memory consists of two binary
patterns, an address pattern (location to be accessed) and a content
pattern (item to be stored). The pattern is called self-addressing
when its content is also its address. Furthermore, in SDM, each memory
slot $m$ is armed with a vector of counters $m_{c}$ initialised
to $0$ with the same length of the content. The memory operations
are based on similarity between the addresses. 

\begin{algorithm}[t]
\begin{algorithmic}[1]
\Require{input $x$ and SDM}
\State{Find a set of chosen locations $M(x)$ using Eq. ($\ref{eq:sdm_r}$)}
\ForEach{$m$ in $M(x)$}
\For{$i=1,D$}
\If{$x_c[i]==1$}
\State{$m_c[i]\mathrel{+}=1$}
\Else
\State{$m_c[i]\mathrel{-}=1$}
\EndIf
\EndFor
\EndFor
\end{algorithmic} 

\caption{Memory writing in SDM\label{alg:kanerva_w}}
\end{algorithm}

\paragraph{Memory writing}

When storing input item $x=\left(x_{a},x_{c}\right)$ to the SDM,
the address pattern $x_{a}$ is compared against all memory location
addresses. Relevant physical locations to consider are those which
lie within a hypersphere of radius $r$ centered on the address pattern
point$ $

\begin{equation}
M\left(x\right)=\left\{ m:d\left(m_{a},x_{a}\right)<r\right\} \label{eq:sdm_r}
\end{equation}
where $d$ is some similarity measure between 2 vectors. In the original
model, Kanerva used Hamming distance. The content is distributed in
the set of locations $M\left(x\right)$ as in Algo. \ref{alg:kanerva_w}.

\paragraph{Memory reading}

Basically, reading from any point in the memory space pools the data
of all nearby locations. Given a cue address $x_{a}^{\prime}$, contents
of the counters at locations near $x_{a}^{\prime}$ are summed and
thresholded at zero to return the binary content. The proximity criteria
still follows Eq. (\ref{eq:sdm_r}). The reading mechanism allows
SDM to retrieve data from imprecise or noisy cues. Fig. \ref{fig:SDM's-memory-write}
(b) visualises the memory access behaviors. 

The assumption underlying the original SDM are: (i) the location addresses
are fixed, and only the contents of the locations are modifiable,
(ii) the locations are sparse and distributed across the address space
$\left\{ 0,1\right\} ^{D}$ (e.g., randomly sample $10^{6}$ addresses
from an address space of $1000$ dimensions ). These assumptions make
the model perform well on storing random input data.

\paragraph{SDM as an associative matrix memory}

We can implement SDM by using three operations of associative memory.
The minimum setting for this implementation includes:
\begin{itemize}
\item A hard-address matrix $A\in\mathbb{B}^{N\times D}$ where $N$ and
$D$ are the numbers of memory locations and the dimension of the
address space, respectively. 
\item A counter (content) matrix $C\in\mathbb{B}^{N\times D}$.
\item Cosine similarity is used to measure proximity.
\item Threshold function $\boldsymbol{y}$ that maps distances to binary
values:$\boldsymbol{y}\left(d\right)=1$ if $d\geq r$ and vice versa.
\item Threshold function $\boldsymbol{z}$ that converts a vector to binary
vector: $\boldsymbol{z}\left(x\right)=1$ if $x\geq0$ and vice versa.
\end{itemize}
Then, the memory writing ($\otimes+\oplus$) and reading ($\bullet$)
become

\begin{align}
C & \coloneqq C+\boldsymbol{y}\left(Ax_{a}\right)x_{c}^{\top}\\
x_{c}^{\prime} & =\boldsymbol{z}\left(C^{\top}\boldsymbol{y}\left(Ax_{a}^{\prime}\right)\right)
\end{align}
These expressions are closely related to attention mechanisms commonly
used nowadays (Sec. \ref{subsec:Attention-mechanism}).

In general, SDM overcomes limitations of correlation matrix memory
such as Hopfield network since the number of stored items in SDM is
not limited by the number of processing elements. Moreover, one can
design SDM to store a sequence of patterns. Readers are referred to
Keeler (1988) for a detailed comparison between SDM and Hopfield network
\citet{keeler1988comparison}. 

\subsubsection*{Memory-augmented neural networks and attention mechanisms}

The current wave of deep learning has leveraged the concept of \foreignlanguage{australian}{SDM}
to external neural memory capable of supporting the working memory
of RNNs \citet{weston2014memory,graves2014neural,graves2016hybrid,miller2016key}.
These models enhance the SDM with real-valued vectors and learnable
parameters. For example, the matrices $A$ and $C$ can be automatically
generated by a learnable neural network. To make whole architecture
learnable, differentiable functions and flexible memory operations
must be used. Attention mechanisms are the most common operations
used in MANNs\nomenclature{MANN}{Memory-augmented Neural Network}
to facilitate the similarity-based memory access of SDM. Through various
ways of employing attentions, RNNs can access the external memory
in the same manner as one accesses SDM. Details on neural distributed
(slot-based) memory and attention mechanisms will be provided in Chapter
\ref{chap:MANN}. 

\section{Relation to Computational Models\label{sec:Relation-to-Computational}}

Automatons are abstract models of machines that perform computations
on an input by moving through a series of states \citet{sipser2006introduction}.
Once the computation reaches a finish state, it accepts and possibly
produces the output of that input. In terms of computational capacity,
there are three major classes of automaton:
\begin{itemize}
\item Finite-state machine 
\item Pushdown automata 
\item Turing machine 
\end{itemize}
Pushdown automata and Turing machine can be thought of as extensions
of finite-state machines (FSMs\nomenclature{FSM}{Finite-state Machine})
when equipped with an external storage in the form of stack and memory
tape, respectively. With stored-program memory, an even more powerful
class of machines, which simulates any other Turing machines, can
be built as universal Turing machine \citet{turing1937computable}.
As some Turing machines are also universal, they are usually regarded
as one of the most general and powerful automata besides universal
Turing machines. 

One major objective of automata theory is to understand how machines
compute functions and measure computation power of models. For example,
RNNs, if properly wired, are Turing-complete \citet{siegelmann1995computational},
which means they can compute arbitrary sequences if they have unlimited
memory. Nevertheless, in practice, RNNs struggle to learn from the
data to predict output correctly given simple input sequence \citet{bengio1994learning}.
This poses a question on the effective computation power of RNNs. 

Another way to measure the capacity of RNNs is via simulations of
operations that they are capable of doing. The relationship between
RNNs and \foreignlanguage{australian}{FSMs} has been discovered by
many \citet{giles1992learning,casey1996dynamics,omlin1996constructing,tivno1998finite},
which suggest that RNNs can mimic FSMs by training with data. The
states of an RNN must be grouped into partitions representing the
states of the generating automation. Following this line of thinking,
we can come up with neural architectures that can simulate pushdown
automata, Turing machine and universal Turing machine. Neural stack
is an example which arms RNN with a stack as its external memory \citet{mozer1993connectionist,joulin2015inferring,grefenstette2015learning}.
By simulating push and pop operations, which are controlled by the
RNN, neural stack mimics the working mechanism of pushdown automata.
Neural Turing Machine and its extension Differentiable Neural Computer
\citet{graves2014neural,graves2016hybrid} are prominent neural realisations
of Turing machine. They use an RNN controller to read from and write
to an external memory in a manner resembling Turing machine's operations
on its memory tape. Since the memory access is not limited to the
top element as in neural stack, these models have more computational
flexibility. Until recently, Le et al.\emph{ }(2020) extended the
simulation to the level of universal Turing machine \citet{Le2020Neural,le2020neurocoder}
by employing the stored-program principle \citet{turing1937computable,vonNeumann:1993:FDR:612487.612553}.
We save a thorough analysis on the correspondence between these MANNs
and Turing machines for Chapter \ref{chap:Neural-Stored-program-Memory}.
Here, we briefly draw a correlation between models of recurrent neural
networks and automata (see Fig. \ref{fig:Relation-between-external}
). 

It should be noted that the illustration is found on the organisation
of memory in the models rather than the computational capacity. For
example, some Turing machines are equivalent to universal Turing machine
in terms of capacity; RNNs are on par with other MANNs because they
are all Turing-complete. Having said that, when neural networks are
organised in a way that simulates powerful automata, their effective
capacity is often greater and thus, they tend to perform better in
complicated sequential learning tasks \citet{graves2014neural,graves2016hybrid,Le2020Neural}.
A similar taxonomy with proof of inclusion relation amongst models
can be found in the literature \citet{ma2018taxonomy}. 

\begin{figure}
\begin{centering}
\includegraphics[width=0.95\textwidth]{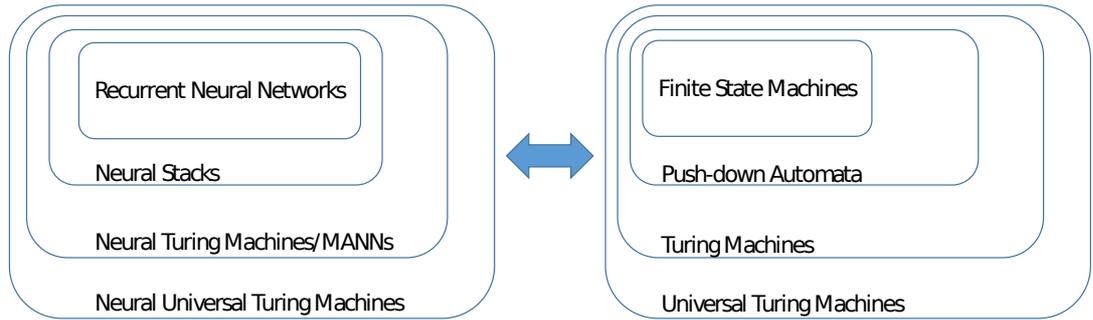}
\par\end{centering}
\caption{Relation between external memory and computational models\label{fig:Relation-between-external}}

\selectlanguage{australian}%
\selectlanguage{australian}%
\end{figure}

\section{Closing Remarks}

We have briefly reviewed different kinds of memory organisations in
the neural network literature. In particular, we described basic neural
networks such as Feed-forward and Recurrent Neural Networks and their
primary forms of memory constructions, followed by a taxonomy on mathematical
models of well-known external memory designs based on memory operational
mechanisms and relations to automation theory. In the next chapter,
we narrow the scope of literature review to the main content of this
thesis: Memory-augment Neural Networks and their extensions.

\chapter{Memory-augmented Neural Networks\label{chap:MANN}}

\section{Gated RNNs}

\subsection{Long Short-Term Memory\label{subsec:Long-Short-Term-Memory}}

Despite its ability to model temporal dependencies in sequential data,
RNNs face a big mathematical challenge of learning long sequences.
The basic problem is that gradients propagated over many steps tend
to either vanish or explode. Although the explosion can be prevented
with the use of activation functions (i.e., $\tanh$ or sigmoid) that
restrict the range of update values, the vanishing problem remains
with these nonlinear activation functions (Sec. \ref{subsec:Cell-memory}).
The difficulty with long-term dependencies arises from the exponentially
smaller weights given to long-term interactions compared to short-term
ones. In practice, experiments have shown that RNNs might find it
hard to learn sequences of only length 10 or 20 \citet{bengio1994learning}. 

Long Short-Term Memory (LSTM) \citet{hochreiter1997long} is introduced
as a simple yet clever way to alleviate the problem. The core idea
is to produce paths where the gradient can flow for long duration
by adding a linear self-loop memory cell to the computation of the
hidden unit. Notably, the weight of the linear self-loop is gated
(controlled by another hidden unit) and dependent on the input. This
enables the network to dynamically moderate the amount of information
passed by the hidden unit. In LSTM, there is a system of gating units
that controls the flow of information, as illustrated in Fig. \ref{fig:Block-diagram-of}.
The modern LSTM model is slightly different from the original LSTM
presented in Sec. \ref{subsec:Cell-memory}, in which we move from
neuronal to vector representation with additional parameters. 

\begin{figure}
\begin{centering}
\includegraphics[width=0.9\textwidth]{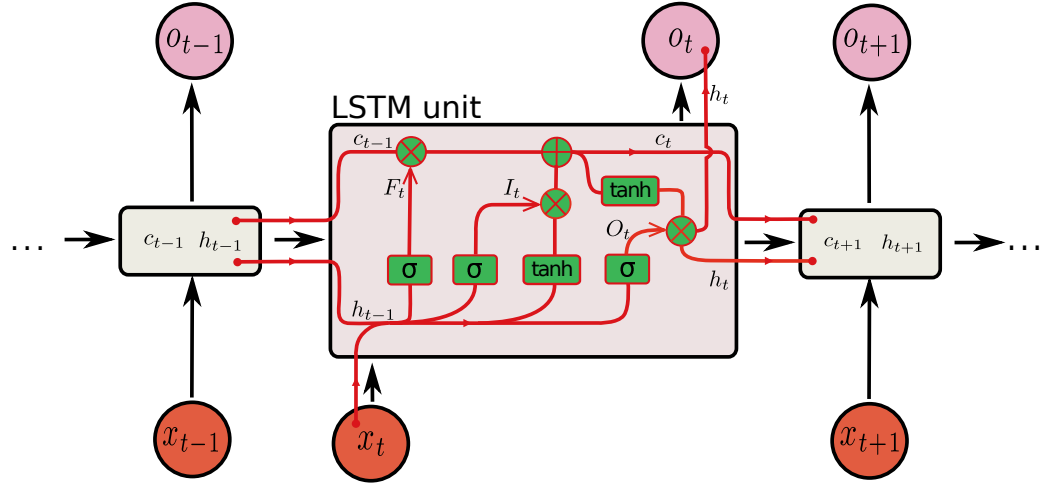}
\par\end{centering}
\caption{Block diagram of a modern LSTM unit\label{fig:Block-diagram-of}.
$\times$ and $+$ are element-wise product and add operators, respectively.
$\sigma$ and $\tanh$ are sigmoid and tanh functions, respectively. }
\end{figure}

The most important component is the cell memory $c_{t}$, which has
a linear self-loop formulation

\begin{equation}
c_{t}=f_{t}\ast c_{t-1}+i_{t}\ast\tilde{c_{t}}
\end{equation}
where $f_{t}$ is the forget gate, $c_{t-1}$ is the previous cell
value, $i_{t}$ is the input gate, $\tilde{c_{t}}$ is the candidate
value for current cell memory and $\ast$ denotes element-wise multiplication.
Similar to RNN's hidden state computation (Eq. (\ref{eq:h_rnn})),
$\tilde{c_{t}}$ is calculated as the following,

\begin{equation}
\tilde{c_{t}}=\tanh\left(W_{c}h_{t-1}+U_{c}x_{t}+b_{c}\right)
\end{equation}
The gates are also functions of previous hidden state and current
input with different parameters

\begin{align}
f_{t} & =\sigma\left(W_{f}h_{t-1}+U_{f}x_{t}+b_{f}\right)\\
i_{t} & =\sigma\left(W_{i}h_{t-1}+U_{i}x_{t}+b_{i}\right)\\
o_{t} & =\sigma\left(W_{o}h_{t-1}+U_{o}x_{t}+b_{o}\right)
\end{align}
where $\sigma$ is the sigmoid function that keeps the gate values
in range $\left[0,1\right]$. The final hidden state $h_{t}$ is computed
based on the cell memory $c_{t}$, gated by the output gate $o_{t}$
as follows,

\begin{equation}
h_{t}=o_{t}\ast\tanh\left(c_{t}\right)
\end{equation}
Given the hidden state $h_{t}$, other computations for the output
$o_{t}$ are the same as in Elman's RNN (Eq. (\ref{eq:o_rnn})).

In LSTM, the forget gate $f_{t}$ plays a crucial role in enabling
the network to capture long-term dependencies. If $f_{t}\rightarrow1$,
the previous memory will be preserved and thus, the product of derivatives
associated with a distant input is close to one. This allows a distant
input to take part in the backpropagation update and slow down the
gradient vanishing process. If $f_{t}\rightarrow0$, the path to previous
cells is disconnected and the model tends to remember only short-term
events.

Empirical results have shown that LSTM networks learn long-term dependencies
more easily than the simple RNNs. State-of-the-art performances were
obtained in various challenging sequence processing tasks \citet{graves2005framewise,vinyals2015neural}.
Other simpler alternatives to LSTM have been studied including Highway
Networks \citet{srivastava2015training} and GRUs \citet{cho2014gru}. 

\subsection{Gated Recurrent Unit}

One simplified variant of LSTM is Gated Recurrent Unit (GRU\nomenclature{GRU}{Gated Recurrent Unit })
\citet{cho2014gru}, which uses two multiplicative gates to harness
the vanishing gradients problem and capture longer dependencies in
the sequence. Unlike LSTM, GRU does not require a separate memory
cell. At each timestep, using a reset gate $r_{t}$, the model computes
a candidate hidden state $\tilde{h_{t}}$ as follows,
\begin{align}
r_{t} & =\sigma\left(W_{r}x_{t}+U_{r}h_{t-1}+b_{r}\right)\\
\tilde{h_{t}} & =\tanh\left(W_{h}x_{t}+U_{h}\left(r_{t}\ast h_{t-1}\right)+b_{h}\right)
\end{align}
The candidate hidden state is determined by current input and previous
hidden state. When $r_{t}$ is close to 0, the candidate hidden state
is reset with the current input, allowing the model to delete any
irrelevant information from the past. The hidden state is then updated
by linear interpolation between the previous hidden state and the
candidate hidden state

\begin{equation}
h_{t}=z_{t}\ast h_{t-1}+\left(1-z_{t}\right)\ast\tilde{h_{t}}
\end{equation}
where an update gate $z_{t}$ decides how much the hidden state should
update its content. The removal of input gate prevents the amount
of information in the hidden states from exploding. $z_{t}$ is computed
by 

\begin{align}
z_{t} & =\sigma\left(W_{z}x_{t}+U_{z}h_{t-1}+b_{z}\right)
\end{align}

A main advantage of GRU compared with LSTM is that GRU can run faster
while maintaining comparable performance \citet{chung2014empirical}.
The reduction of parameters also helps GRU less overfit to training
data as LSTM does. 

\section{Attentional RNNs\label{sec:Attentional-RNNs}}

\subsection{Encoder-Decoder Architecture\label{subsec:Encoder-decoder-architecture}}

Intuitively, attention mechanism is motivated by human visual attention
where our eyes are able to focus on a certain region of an image/language
with \textquotedblleft high resolution\textquotedblright{} while perceiving
the surrounding context in \textquotedblleft low resolution\textquotedblright .
This focus is adjusted dynamically overtime and directly contributes
to our decision making process. Before going into details, we will
briefly review sequence-to-sequence model\textendash a recurrent architecture
that is often used with attention mechanism. 

Amongst sequential modeling tasks, sequence-to-sequence mapping is
one of the most challenging one whose practical applications may include
machine translation, document \foreignlanguage{australian}{summarisation}
and dialog response generation. To solve such tasks, we may use an
RNN-like encoder to model the input sequence and then an RNN-like
decoder to model the output sequence. To link the two models, the
final hidden state of the encoder (thought vector) is passed to the
decoder as the latter's initial hidden state (see Fig. \ref{fig:seq2seq}
(a)). This encoder-decoder architecture, often referred to as Seq2Seq,
is firstly introduced by Cho et al. (2014) and has demonstrated superior
performance over LSTM in machine translation \citet{cho2014properties,sutskever2014sequence}. 

\begin{figure}
\begin{centering}
\includegraphics[width=1\textwidth]{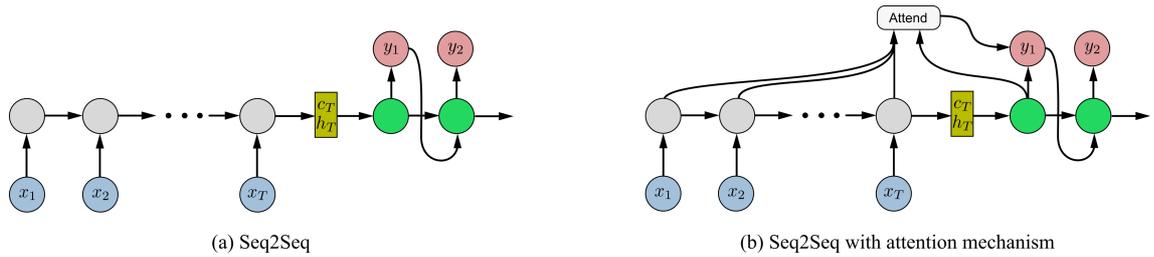}
\par\end{centering}
\caption{(a) Seq2Seq Model. Gray and green denote the LSTM encoder and decoder,
respectively. In this architecture, the output at each decoding step
can be fed as input for the next decoding step. (b) Seq2Seq Model
with attention mechanism. The attention computation is repeated across
decoding steps. \label{fig:seq2seq}}
\end{figure}

\subsection{Attention Mechanism\label{subsec:Attention-mechanism}}

Even when applying LSTM to Seq2Seq helps to ease the gradient vanishing
in general, the decoder in Seq2Seq is likely to face this problem
when the number of decoding steps becomes larger. Given that the decoder
receives a fixed-size though vector representing the whole input sequence,
it is hard to recover the contribution of distant encoding input in
predicting decoder's outputs. To overcome this, Bahdanau et al. (2015)
proposed using attention mechanism in encoder-decoder architecture.
The key idea is to let the decoder look over every piece of information
that the original input sequence holds at every decoding step, which
is equivalent to creating a direct connection from a decoder unit
to any encoder unit (see Fig. \ref{fig:seq2seq} (b)). Each connection
then will be weighted by an attention score, which is a function of
hidden states from both encoder and decoder. The weight $\alpha_{ij}$
between the $i$-th decoding step and the $j$-th encoding step is
defined as

\begin{align}
e_{ij} & =v^{T}\tanh\left(Ws_{i-1}+Uh_{j}\right)\\
\alpha_{ij} & =\frac{\exp\left(e_{ij}\right)}{\stackrel[k=1]{L}{\sum}\exp\left(e_{ik}\right)}\label{eq:attend_aij}
\end{align}
where $e_{ij}$ is the unnormalised weight, $v$ is a parametric vector
and $W$, $U$ are parametric matrices. $s$ and $h$ are used to
denote the hidden state of the decoder and the encoder, respectively.
Eq. (\ref{eq:attend_aij}) is the well-known softmax function to make
the weights sum to one over $L$ encoding steps. Then, a context vector
for the $i$-th decoding step is computed using a weighted summation
of all encoder's hidden states as follows, 

\begin{equation}
c_{i}=\stackrel[j=1]{L}{\sum}\alpha_{ij}h_{j}
\end{equation}
Finally, the context vector $c_{i}$ is combined with the decoder
hidden state $s_{i}$ to compute the $i$-th decoder's output and
next state \citet{Bahdanau2015a}. Attention mechanism has several
modifications such as hard attention \citet{xu2015show} and pointer
network \citet{vinyals2015pointer}. 

\subsection{Multi-Head Attention \label{subsec:Multi-head-Attention}}

Traditional RNNs read a sequence step by step to extract sequential
dependencies, which is slow and hard to capture far apart relations.
Attention helps link two distant timesteps quickly and thus, shows
potential to replace completely RNNs in modeling sequential data.
However, the vanilla attention mechanism is shallow with one step
of computation per timestep, which relies on the hidden state of RNNs
for richer representation. In an effort to replace RNNs with attention,
Vaswani et al. (2017) proposed a deeper attention mechanism with multiple
heads implemented efficiently using dot-product operation. The model
reads all timesteps in the sequence at once like Feed-forward Neural
Networks, which utilises parallel computing. Moreover, multiple keys,
values and queries are packed into matrices $K$, $V$ and $Q$, respectively.
Then, the multi-head attention operation is computed as follows,

\begin{equation}
Attention\left(Q,K,V\right)=\softmax\left(\frac{QK^{T}}{\sqrt{d_{k}}}\right)V
\end{equation}
where $d_{k}$ is the number of key dimension. The multi-head attention
lies at the core of self-attention mechanism, in which, relational
features are encoded from the input sequence (Fig. \ref{fig:self-att}
(a)). Similarly, the output sequence features can be extracted and
combined with the encoded input to form an encoder-decoder architecture
called The Transformer. (Fig. \ref{fig:self-att} (b)). 

The Transformer has empirically demonstrated that attention alone
can replace recurrent models in solving sequential tasks including
machine translation and language parsing \citet{vaswani2017attention}.
This opens a new research direction in deep learning where attention
can be used to extract relations between time-ordered events. The
limitation of self-attention is its quadratic complexity. However,
this can be compensated with parallel computation ability. Detailed
discussion of this new research angle is beyond the scope of this
thesis. Instead, we will focus on slot-based memory networks, another
approach with attention that is built upon a readable/writable external
memory. The approach resembles closely SDM as well as human associative
memory.

\begin{figure}
\begin{centering}
\includegraphics[width=0.95\textwidth]{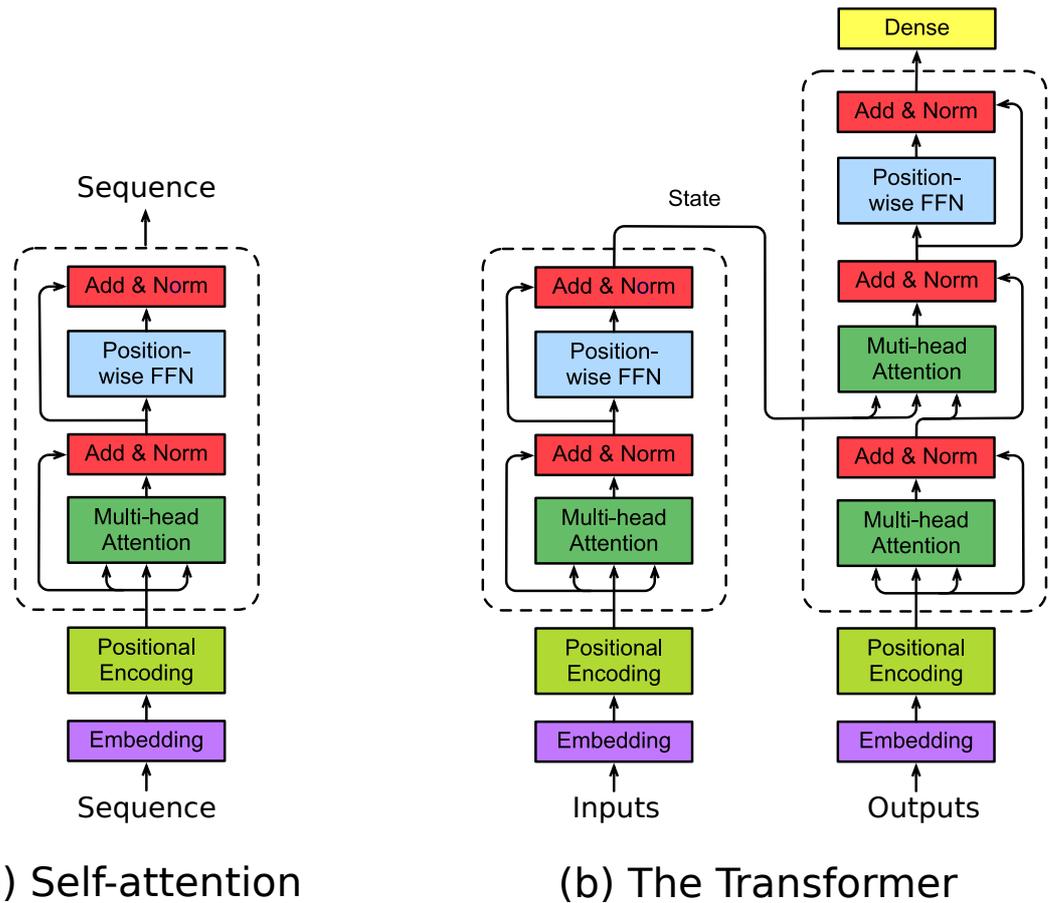}
\par\end{centering}
\caption{Computation stages of the encoding using self-attention (a) and encoding-decoding
architecture\textendash The Transformer (b). Embedding layers convert
input/output tokens to vectors of fix dimension, followed by Positional
Encoding layers that add temporal information to each vector. The
main block of computation combines multi-head attention, residual
connection, layer normalisation and Feed-forward layers, which can
be repeated multiple times.\label{fig:self-att}}
\end{figure}

\section{Slot-Based Memory Networks\label{sec:Slot-based-Memory-Networks}}

\subsection{Neural Stack}

Traditional stack is a storage of elements that works on the principle
of last-in-first-out, which describes the order in which the elements
come off a stack. In general, stack supports two operations: push,
which adds an element to the stack, and pop, which removes the most
recently added element (the top one). Additionally, a peek operation
may give access to the value of the top element without modifying
the stack. Stack is a convenient memory for solving problems with
hierarchical structures because it stores the temporary results in
a way that supports backtracking and tree traversal. Recently, researchers
have tried to implement continuously differentiable prototype of traditional
stacks using deep networks \citet{joulin2015inferring,grefenstette2015learning}.
We briefly review the implementations proposed by Grefenstette et
al. (2015) that aim to mimic Stack, Queue and Deque on solving natural
language transduction problems. 

In the implementations, a row-expandable matrix $V$ is used to store
the data. The $i$-th row $V\left[i\right]$ is associated with a
strength scalar $s\left[i\right]$. When $v_{t}$\textendash the item
at timestep $t$\textendash{} is presented, its value is added to
the matrix $V$ and never be modified, which yields,

\begin{equation}
V_{t}\left[i\right]=\begin{cases}
V_{t-1}\left[i\right]=v_{i} & \mathrm{if}\,1\leq i<t\\
v_{t} & \mathrm{if}\,i=t
\end{cases}\label{eq:stack_v}
\end{equation}
To modify the stack content under push and pop operations, we modify
the strength vector instead as the following,

\begin{equation}
s_{t}\left[i\right]=\begin{cases}
\max\left(0,s_{t-1}\left[i\right]-\max\left(0,u_{t}-\sum_{j=i+1}^{t-1}s_{t-1}\left[j\right]\right)\right) & \mathrm{if}\,1\leq i<t\\
d_{t} & \mathrm{if}\,i=t
\end{cases}
\end{equation}
where $u_{t}$ and $d_{t}$ are the pop and push signals generated
by a neural network, respectively. Basically, the strength for the
top item is set to the push signal. Then, we want to subtract the
strength of stored items ($s_{t}\left[i\right]$) by an amount of
the pop signal $\left(u_{t}\right)$ from the top (highest index)
to the bottom (lowest index) of the stack. If the pop signal is greater
than the strength, the strength of the item is set to $0$ (totally
popped out of the stack) and the remainder of the pop signal is passed
to lower items until we run out of pop signal. The peek or read operation
is carried out by 

\begin{equation}
r_{t}=\stackrel[i=1]{t}{\sum}\left(\min\left(s_{t}\left[i\right],\max\left(0,1-\sum_{j=i+1}^{t}s_{t}\left[j\right]\right)\right)\right)V_{t}\left[i\right]\label{eq:stack_r}
\end{equation}
The output $r_{t}$ of the read operation is the weighted sum of the
rows of $V_{t}$, scaled by the temporary strength values created
during the traversal. Intuitively, items with zero strength do not
contribute to read value and items on the bottom contribute less than
those near the top. Neural Queue and DeQue can be implemented in similar
manners by modifying Eqs. (\ref{eq:stack_v})-(\ref{eq:stack_r}). 

A controller implemented as RNN is employed to control stack operations.
The current input $i_{t}$ from the sequence and the previous read-out
$r_{t-1}$ will be concatenated as input for the RNN to produce the
current hidden state $h_{t}$ and the controller output $o_{t}^{\prime}$.
The controller output will be used to generate the item, control signals
and final output of the whole network as follows,

\begin{align}
d_{t} & =\sigma\left(W_{d}o_{t}^{\prime}+b_{d}\right)\\
u_{t} & =\sigma\left(W_{u}o_{t}^{\prime}+b_{u}\right)\\
v_{t} & =\tanh\left(W_{v}o_{t}^{\prime}+b_{v}\right)\\
o_{t} & =\tanh\left(W_{o}o_{t}^{\prime}+b_{o}\right)
\end{align}
Experiments have demonstrated that the proposed models are capable
of solving transduction tasks for which LSTM-based models falter \citet{grefenstette2015learning}. 

\subsection{Memory Networks}

One solution to ensure a model will not forget is to create a slot-based
memory module and store every piece of information into the memory
slots. The memory can be implemented as a matrix $M\in\ensuremath{\mathbb{R}}^{N\times D}$
whose rows contain vectors representing the considering piece of information.
Here, $N$ is the number of slots and $D$ is the dimension of the
representation vector (word size). Following this principle, Memory
Network (MemNN\nomenclature{MemNN}{Memory Network}) \citet{weston2014memory}
stores all information (e.g., knowledge base or background context)
into an external memory. When there is a retrieval request, it assigns
a relevance probability to each memory slot using content-based attention
scheme, and reads contents from each memory slot by taking their weighted
sum. Since the model is designed for language understanding, each
slot of the memory often associates with a document or a sentence.
When a query/question about facts related to the stored documents
is presented, MemNN will perform content-based attention as follows,

\begin{equation}
p_{i}=\softmax\left(u^{T}m_{i}\right)
\end{equation}
where $u$ is the feature and $m_{i}$ is the memory's $i$-th row
vector, which represent the query and the stored document, respectively.
$p_{i}$ is the attention score to the $i$-th memory slot, normalised
by softmax function. 

The output of the memory, given query $u$, is the read vector

\begin{equation}
r=\stackrel[i=1]{N}{\sum}p_{i}c_{i}
\end{equation}
where $c_{i}$ is the output vector corresponding to the $i$-th slot.
In MemNN, it is trainable parameter while in key-value memory network
\citet{miller2016key}, it comes from the data. Then, the model can
make prediction by feeding the read values to another feed-forward
neural network. 

A multi-hop version MemN2N\nomenclature{MemN2N}{End-to-End Memory Network}
has also been studied and outperforms LSTM and MemNN in question-answering
tasks \citet{sukhbaatar2015end}. MemN2N extends MemNN by adding refinement
updates on the query and the read-out. The refinement reads

\begin{equation}
u_{k+1}=Hu_{k}+r_{k}
\end{equation}
where $H$ is a parametric matrix and $k$ is the refinement step. 

Although memory networks have big advantages over LSTM due to the
use of external matrix memory, it is hard to scale to big dataset
since the number of memory slots grows linearly with the number of
data. Some tricks such as hashing have been proposed but they have
a trade-off between capacity and accuracy. More importantly, it is
unlikely that we tend to store everything in our brain. We have the
ability to forget the old memory and update with new knowledge, which
is ignored by memory network designs.

\subsection{Neural Turing Machine}

In contrast to MemNN, Neural Turing Machine (NTM\nomenclature{NTM}{Neural Turing Machine})
\citet{2014arXiv1410.5401G} introduces a slot-based read/write mechanism
to the memory module. The memory size does not need to equal the number
of considering pieces of information. The model learns to overwrite
obsolete or unimportant memory slots with recent and useful information
to \foreignlanguage{australian}{optimise} a final goal. This writing
scheme fits with sequential task where the prediction goal can be
achieved without paying attention to all timestep inputs. To control
the memory operations, NTM uses a neural controller network whose
parameters are slow-learning weights. The controller is responsible
for determining instantly after each timestep the content to be read
from and written to the memory. An illustration of NTM components
is described in Fig. \ref{fig:General-architecture-of} (a).

\begin{figure}
\begin{centering}
\includegraphics[width=0.9\textwidth]{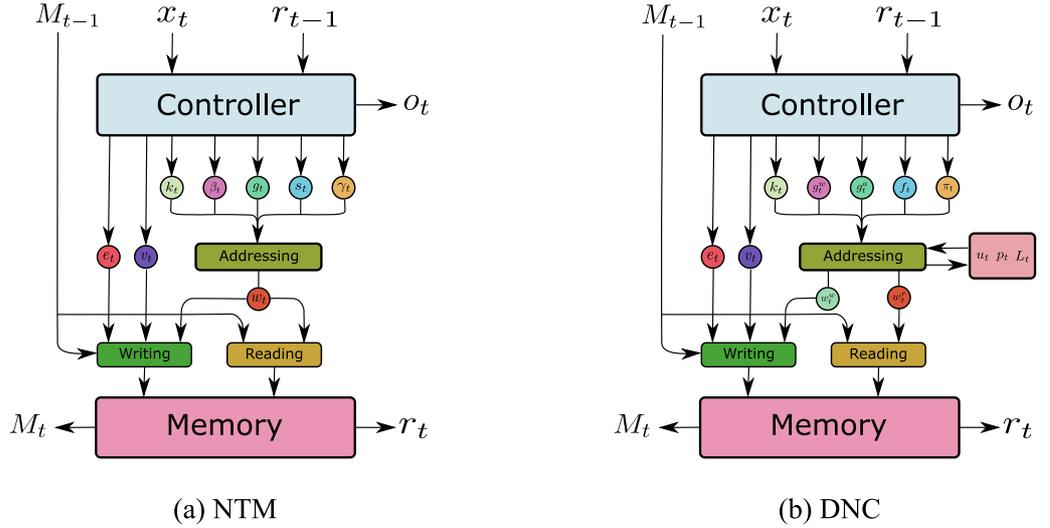}
\par\end{centering}
\caption{(a) Architecture of NTM. Circles denote intermediate variables computed
by the controller. The controller takes the current timestep data
$x_{t}$ and the previous read value $r_{t-1}$ as the input and produces
$r_{t}$, updates memory $M_{t}$ and predict output $o_{t}$. (b)
Architecture of DNC. The operation is similar to NTM's with extra
modules to keep track of memory usage $u_{t}$, precedence $p_{t}$
and link matrix $L_{t}$. \label{fig:General-architecture-of}}

\selectlanguage{australian}%
\selectlanguage{australian}%
\end{figure}

In NTM, both reading and writing locations are determined by the address,
which is a weight over the memory slots. The weight is initially computed
by the content-based attention,

\begin{equation}
w_{t}^{c}\left(i\right)=\frac{\exp\left(\beta_{t}m\left(k_{t},M_{t}\left(i\right)\right)\right)}{\stackrel[j=1]{D}{\sum}\exp\left(\beta_{t}m\left(k_{t},M_{t}\left(j\right)\right)\right)}\label{eq:cbaseatt}
\end{equation}
Here, $w_{t}^{c}\in\ensuremath{\mathbb{R}}^{N}$ is the content-based
weight, $\beta_{t}$ is a strength scalar, $m$ is a matching function
that measures the similarity between a key $k_{t}\in\ensuremath{\mathbb{R}}^{D}$
and the $i$-th memory slot $M_{t}\left(i\right)$. In practice, $m$
is implemented as cosine similarity

\begin{equation}
m\left(k_{t},M_{t}(i)\right)=\frac{k_{t}\cdot M_{t}(i)}{||k_{t}||\cdot||M_{t}(i)||}
\end{equation}

Besides the content-based addressing, NTM supports location-based
addressing started with an interpolation between content-based weight
and the previous weight

\begin{equation}
w_{t}^{g}=g_{t}w_{t}^{c}+\left(1-g_{t}\right)w_{t}
\end{equation}
where $g_{t}$ is the interpolation gate. This allows the system to
learn when to use (or ignore) content-based addressing. Also, the
model is able to shift focus to other rows by performing convolution
shift modulo $R$ as the following,

\begin{equation}
\tilde{w_{t}}\left(i\right)=\stackrel[j=0]{R}{\sum}w_{t}^{g}\left(i\right)s_{t}\left(i-j\right)
\end{equation}
where $s_{t}$ is the shift weighting. Finally, sharpening is used
to prevent the shifted weight from blurring, which results in the
final weight

\begin{equation}
w_{t}\left(i\right)=\frac{\tilde{w_{t}}\left(i\right)^{\gamma}}{\underset{j}{\sum}\tilde{w_{t}}\left(j\right)^{\gamma}}
\end{equation}

Given the weight calculated, the memory update is defined by using
these bellowing equations

\begin{align}
M_{t}^{erased}\left(i\right) & =M_{t-1}\left(i\right)\left[1-w_{t}\left(i\right)e_{t}\right]\\
M_{t}\left(i\right) & =M_{t}^{erased}\left(i\right)+w_{t}\left(i\right)v_{t}
\end{align}
where $e_{t}\in\ensuremath{\mathbb{R}}^{D}$ and $v_{t}\in\ensuremath{\mathbb{R}}^{D}$
are erase vector and update vector, respectively. The read value is
computed using the same address weight as follows,

\begin{equation}
r=\stackrel[i=1]{N}{\sum}w_{t}\left(i\right)M_{t}\left(i\right)
\end{equation}
The controller can be implemented as a feed-forward network or LSTM
fed with an concatenation of the read-out $r_{t}$ and the timestep
data $x_{t}$. The computation of the output $o_{t}$ follows the
same computing mechanism of the controller network (see Sec. \ref{subsec:Long-Short-Term-Memory}). 

With a fixed size external memory, NTM can scale well when dealing
with very long sequence while maintaining better remembering capacity
than other recurrent networks such as RNN, GRU and LSTM. Experiments
have shown NTM outperforms LSTM by a huge margin in memorisation testbeds
including copy, repeat copy, associative recall and priority sort
\citet{2014arXiv1410.5401G}.

\subsection{Differentiable Neural Computer\label{subsec:Differentiable-Neural-Computer}}

In this subsection, we briefly review DNC\nomenclature{DNC}{ Differentiable Neural Computer}
\citet{graves2016hybrid}, a powerful extension of the NTM. A DNC
consists of a controller, which accesses and modifies an external
memory module using a number of read heads and one write head. Given
some input $x_{t}$, and a set of $R$ read values from memory $r_{t-1}=\left[r_{t-1}^{1},...,r_{t-1}^{k},...,r_{t-1}^{R}\right]$,
the controller produces the output $o_{t}$ and the interface which
consists of intermediate variables, as depicted in Fig. \ref{fig:General-architecture-of}
(b). DNC also uses the content-based attention in Eq. (\ref{eq:cbaseatt})
to determine the content-based write-weight $w_{t}^{cw}$ and read-weights
$w_{t}^{cr,k}$. However, different from NTM, DNC does not support
location-based attention. Instead, DNC introduces dynamic memory allocation
and temporal memory linkage for computing the final write-weight $w_{t}^{w}$
and read-weights $w_{t}^{r,k}$ separately.\textbf{ }

\textbf{Dynamic memory allocation \& write weightings:} DNC maintains
a differentiable free list tracking the usage $u_{t}\in\left[0,1\right]^{N}$
for each memory location. Usage is increased after a write and optionally
decreased after a read, determined by the free gates $f_{t}^{k}$
as follows,

\begin{equation}
u_{t}=\left(u_{t-1}+w_{t-1}^{w}-u_{t-1}\circ w_{t-1}^{w}\right)\stackrel[k=1]{R}{\prod}\left(1-f_{t}^{k}w_{t}^{r,k}\right)
\end{equation}
The usage is sorted and then the allocation write-weight is defined
as

\begin{equation}
a_{t}\left[\varPhi_{t}\left[j\right]\right]=\left(1-u_{t}\left[\varPhi_{t}\left[j\right]\right]\right)\stackrel[i=1]{j-1}{\prod}u_{t}\left[\varPhi_{t}\left[i\right]\right]
\end{equation}
in which, $\varPhi_{t}$ contains elements from $u_{t}$ sorted by
ascending order from least to most used. Given the write gate $g_{t}^{w}$
and allocation gate $g_{t}^{a}$, the final write-weight then can
be computed by interpolating between the content-based write-weight
and the allocation write-weight,

\begin{equation}
w_{t}^{w}=g_{t}^{w}\left[g_{t}^{a}a_{t}+\left(1-g_{t}^{a}\right)w_{t}^{cw}\right]
\end{equation}
Then, the memory is updated by the following rule

\begin{equation}
M_{t}=M_{t-1}\circ\left(E-g_{t}^{w}w_{t}^{w}e_{t}^{\top}\right)+g_{t}^{w}w_{t}^{w}v_{t}^{\top}\label{eq:dnc_w}
\end{equation}

\textbf{Temporal memory linkage \& read weightings:} DNC uses a temporal
link matrix $L_{t}\in\left[0,1\right]{}^{N\times N}$ to keep track
of consecutively modified memory locations, and $L_{t}\left[i,j\right]$
represents the degree to which location $i$ was the location written
to after location $j$. Each time a memory location is modified, the
link matrix is updated to remove old links to and from that location,
and add new links from the last-written location. To compute the link
matrix, DNC maintains a precedence weighting to keep track of which
locations were most recently written by using the following equation

\begin{equation}
p_{t}=\left(1-\stackrel[i]{N}{\sum}w_{t}^{w}\left(i\right)\right)p_{t-1}+w_{t}^{w}
\end{equation}
Then $p_{t}$ is used to update the link matrix as follows,

\begin{equation}
L_{t}\left(i,j\right)=\left(1-w_{t}^{w}\left(i\right)-w_{t}^{w}\left(j\right)\right)L_{t-1}\left(i,j\right)+w_{t}^{w}\left(i\right)p_{t-1}\left(j\right)
\end{equation}
Given the link matrix, the final read-weight is given as follow,

\begin{equation}
w_{t}^{r,k}=\pi_{t}^{k}\left(1\right)L_{t}^{\top}w_{t-1}^{r,k}+\pi_{t}^{k}\left(2\right)w_{t}^{cr,k}+\pi_{t}^{k}\left(3\right)L_{t}w_{t-1}^{r,k}
\end{equation}
The read mode weight $\pi_{t}^{k}$ is used to balance between the
content-based read-weight and the forward $L_{t}w_{t-1}^{r,k}$ and
backward $L_{t}^{\top}w_{t-1}^{r,k}$ of the previous read. Then,
the $k$-th read value $r_{t}^{k}$ is retrieved using the final read-weight
vector

\begin{equation}
r_{t}^{k}=\stackrel[i]{N}{\sum}w_{t}^{r,k}(i)M_{t}(i)\label{eq:dncread}
\end{equation}

The performance of DNC is better than that of NTM, LSTM and other
variants of memory networks. It achieves state-of-the-art results
under various experimental settings such as copy/recall, question-answering
and graph reasoning \citet{graves2016hybrid}. 

\subsection{Memory-augmented Encoder-Decoder Architecture\label{subsec:Memory-augmented-Encoder-Decoder}}

A memory-augmented encoder-decoder (MAED\nomenclature{MAED}{Memory-augmented Encoder-Decoder Architecture})
consists of two neural controllers linked via external memory. This
is a natural extension to read-write MANNs to handle sequence-to-sequence
problems, which has been investigated in Sec. \ref{sec:Attentional-RNNs}.
MAED has recently demonstrated promising results in machine translation
\citet{britz2017efficient,wang2016memory}, OCR \citet{nguyen2019improving}
and healthcare \citet{le2018cdual,Le:2018:DMN:3219819.3219981,prakash2017condensed}.
In this thesis, MAED is the basic framework on which some of our proposed
architectures are built upon. Here, we briefly analyse the operations
of MAED. More details can be found in Chapter \ref{chap:multiple}.

As mentioned in Sec. \ref{subsec:Encoder-decoder-architecture}, Seq2Seq
only allows information transfer from the encoder to the decoder via
``thought vector''. Attention mechanism creates shortcut connections
via directly attending to other timesteps (see Sec. \ref{subsec:Attention-mechanism}
and \ref{subsec:Multi-head-Attention}). Unlike these approaches,
MAED uses an external memory as a reservoir containing past information
to which the controllers attend. That is, instead of directly attending
to other timesteps, the controllers (both encoder and decoder) attend
to the memory to store and retrieve relevant information The information
then participates in other decision-making processes such as output
predictions, feature extractions and memory operations. If the number
of memory slots is finite, the method keeps the computation resource
linear to the length of the sequence, which is feasible for life-long
learning. Furthermore, selectively attending to a fixed number of
memory slots requires the model to learn to compress the information
to the external memory. On one hand, it may make the learning harder
since the controllers must learn to write timestep information to
the memory efficiently. On the other hand, attending to all timesteps
is biologically implausible and information compression is a powerful
skill that mimics the capacity of memory in the brain. Hence, this
thesis has chosen MAED as the key framework to study and develop several
memory and attention techniques for neural networks (Chapter \ref{chap:multiple}
and \ref{chap:Variational-Memory-Encoder}). 

\section{Closing Remarks}

We have reviewed MANNs, a family of recurrent neural networks with
external memory, which is useful for a wide range of sequential learning
tasks. Armed with operations such as gating and attention mechanisms,
MANNs are able to control and update external vector or matrix memory
module. The power of MANNs has been verified in learning long sequential
data, in which their performances are superior to that of vanilla
RNNs. 

We hypothesise that MANNs are also effective in open challenges in
multi-process/multi-view data, data with uncertainty, ultra-long sequences
and problems required universal computation. We address these challenges
in Chapters \ref{chap:multiple}, \ref{chap:Variational-Memory-Encoder},
\ref{chap:Optimal-Writing-in} and \ref{chap:Neural-Stored-program-Memory},
respectively. A recurring theme across all of our proposed MANNs is
found on slot-based memory architecture.

\chapter{Memory Models for Multiple Processes \label{chap:multiple}}

\section{Introduction}

\subsection{Multi-Process Learning}

Traditional sequential learning focuses on modeling single processes,
in which the sequence of outputs shares the same domain with the input
sequence. We can extend the problem to broader scenarios where input
and output sequences are from different processes (sequence to sequence)
or there are multiple sequences acting as inputs and outputs (multi-view
sequential mapping). For example, in healthcare setting, there are
at least three processes that are executed: the disease progression,
the treatment protocols, and the recording rules. 

Let us start with a generic formulation of the multi-process learning.
Let $S^{i_{1}}$, ...,$S^{i_{M}}$ and $S^{o_{1}}$, ...,$S^{o_{N}}$
denote $M$ input and $N$ output view spaces, respectively. Each
sample of the multi-process problem $\left(\left\{ X^{i_{k}}\right\} _{k=1}^{M},\left\{ Y^{o_{k}}\right\} _{k=1}^{N}\right)$
consists of $M$ input views: $X^{i_{1}}=\left\{ x_{1}^{i_{1}},...,x_{t_{1}}^{i_{1}},...,x_{L^{i_{1}}}^{i_{1}}\right\} $,
..., $X^{i_{M}}=\left\{ x_{1}^{i_{M}},...,x_{t_{M}}^{i_{M}},...,x_{L^{i_{M}}}^{i_{M}}\right\} $
and $N$ output views $Y^{o_{1}}=\left\{ y_{1}^{o_{1}},...,y_{t_{1}}^{o_{1}},...,y_{L^{o_{1}}}^{o_{1}}\right\} ,...,Y^{o_{N}}=\left\{ y_{1}^{o_{N}},...,y_{t_{N}}^{o_{N}},...,y_{L^{o_{N}}}^{o_{N}}\right\} $.
Each view has a particular length ($L^{i_{1}}$, ..., $L^{i_{M}}$
and $L^{o_{1}}$, ..., $L^{o_{N}}$) and can be treated as a set/sequence
of events that belongs to different spaces ($x_{t_{1}}^{i_{1}}\in S^{i_{1}}$,
..., $x_{t_{M}}^{i_{M}}\in S^{i_{M}}$, $y_{t_{1}}^{o_{1}}\in S^{o_{1}}$,
..., $y_{t_{N}}^{o_{N}}\in S^{o_{N}}$). Each event then can be represented
by an one-hot vector $v\in\left[0,1\right]^{\left\Vert C\right\Vert }$,
where $C$ can be $S^{i_{1}}$, ..., $S^{i_{M}}$ or $S^{o_{1}}$,
..., $S^{o_{N}}$. In single process, $M=N=1$ and $S^{i_{1}}=S^{o_{1}}$.
Practical problems belonging to this setting can be solved efficiently
using RNNs or LSTM as in language modeling \citet{mikolov2010recurrent},
speech and optical character recognition \citet{graves2013speech,graves2013generating,graves2005framewise}.
We focus more on other complicated cases, which are dual processes
(sequence to sequence) and dual-view mapping, where $M=N=1$, $S^{i_{1}}\neq S^{o_{1}}$
and $M=2,N=1$, $S^{i_{1}}\neq S^{i_{2}}\neq S^{o_{1}}$, respectively.
These cases represent two classes of problems that traditional RNNs
may struggle to deal with. Thus, we aim to solve them using MAED (Sec.
\ref{subsec:Memory-augmented-Encoder-Decoder}) by proposing: (i)
separate controls for two separate sub-processes (encoding, decoding)
and (ii) multiple memories for multiple parallel, asynchronous processes.
To account for these settings, we need more powerful MAEDs to handle
the complexity of long-term dependencies and view interactions, which
will be presented in the next sections using healthcare as a practical
domain. 

\subsection{Real-World Motivation}

A main motivation for our work in this chapter is modeling healthcare
processes. In healthcare, a hospital visit is documented as one admission
record consisting of diagnosis and treatment codes for the admission.
The collection of these records are electronic medical records (EMRs)
of a patient. Recently, using EMRs as the data for scientists to analyse
and make treatment predictions has become the key for improving healthcare
\citet{Williams2008}. A typical EMR contains information about a
sequence of admissions for a patient and a wide range of information
can be stored in each admission, such as detailed records of symptoms,
data from monitoring devices, clinicians\textquoteright{} observations.
Diagnoses, procedures or drug prescriptions are the most important
information stored in EMRs and are typically coded in standardised
formats, each of which represents a medical process in EMR data.

In particular, diagnoses are coded using WHO\textquoteright s ICD
(International Classification of Diseases) coding schemes. For example,
in the ICD10 scheme, E10 encodes Type 1 diabetes mellitus, E11 encodes
Type 2 diabetes mellitus and F32 indicates depressive episode. The
treatment can be procedure or drug. The procedures are typically coded
in CPT (Current Procedural Terminology) or ICHI (International Classification
of Health Interventions) schemes. The drugs are often coded in ATC
(Anatomical Therapeutic Chemical) or NDC (National Drug Code).

It is important to note that there are order dependencies amongst
medical processes such as the diagnosis codes as well as procedure
or drug codes. For example, diagnosis codes are often sequenced under
some strict rules such as: (i) condition codes must be sequenced first,
followed by the manifestation codes, (ii) primary diagnosis that describe
the nature of the sequela must be coded before the secondary diagnosis
that describes the original injury, (iii) single conditions that require
more than one code have clear instructions that indicate which must
be coded first. Similarly, procedure or drug codes often follow specific
order, often corresponding to the order of diagnosis codes or the
order of prescriptions. Moreover, the dependency in EMR clinical codes
can be very long-term. Due to the fact that EMR data is temporally
sequenced by patient medical visits, clinical codes at current admission
may be related to other codes appearing in previous admissions. Since
it is normal for patients to periodically visit hospital for regular
health check, the number of admissions for some person should be very
large and thus the dependency amongst clinical codes should be very
long. Another characteristic of EMR data is its rarity. Although the
number of EMR records are increasing day by day, it cannot cover many
rarely-seen symptoms and treatments which appear sparsely throughout
the history of EMR documentation. More seriously, rare diseases are
deadly and rare treatments are expensive, so it is compulsory to make
use of these rare information in order to make significant predictions.

Unfortunately, it is not easy for prediction models to capture the
long-term dependencies and rarity of EMR\nomenclature{EMR}{ Electronic Medical Record}
data. Recent researches dealing with medical prediction have largely
focused on modeling the admission's diagnosis and treatments as two
set of codes and only capture sequential dependencies from one admission
to another \citet{nguyen2016deepr,pham2017predicting}. Also, most
of these methods avoid using rare data by only keeping codes that
appear frequently. This approach exposes three limitations. First,
using set representation ignores the internal internal sequential
dependencies and thus fail to discover sequential relations amongst
codes from the same admission. Second, refusing to use rare clinical
events makes the contribution less significant because in healthcare,
rare events are the more important ones. Third, outputs of this approach
are often set of codes, which again detaches from realistic need where
treatments and diagnoses have to follow strict orders. These challenges
motivate new approaches that treat EMR history as sequences of correlated
processes. 

\subsubsection*{Sequence to sequence mapping}

Treatment recommendation can be cast as sequence to sequence problem.
Once diagnoses are confirmed, we want to predict the output sequence
(treatment codes of the current visit) given the input sequence (all
diagnoses followed by treatments from the first visit to the previous
visit plus the diagnoses of the current visit). More formally, we
denote all the unique medical codes (diagnosis, procedures and drugs)
from the EMR data as $c_{1},c_{2},..c_{\left|C\right|}\in C$, where
$\left|C\right|$ is the number of unique medical codes. A patient's
$n$-th admission's input is represented by a sequence of codes: 
\begin{equation}
\left[c_{d_{1}}^{1},c_{d_{2}}^{1},...,\textrm{\ensuremath{\permil}},c_{p_{1}}^{1},c_{p_{2}}^{1}...,\diameter\allowbreak,...,c_{d_{1}}^{n-1},c_{d_{2}}^{n-1},...,\textrm{\ensuremath{\permil}},c_{p_{1},}^{n-1},c_{p_{2}}^{n-1},....,\diameter\allowbreak,c_{d_{1}}^{n},c_{d_{2}}^{n},...,\permil\right]
\end{equation}
 Here, $c_{d_{j}}^{k}$ and $c_{p_{j}}^{k}$ are the $j$-th diagnosis
and treatment code of the $k$-th admission, respectively. $\textrm{\ensuremath{\permil}}$
, $\diameter$ are special characters that informs the model about
the change from diagnosis to treatment codes and the end of an admission,
respectively. This reflects the natural structure of a medical history,
which is a sequence of clinical visits, each of which typically includes
a subsequence of diagnoses, and a subset of treatments. A diagnosis
subsequence usually starts with the primary condition followed by
secondary conditions. In a subset of treatments, the order is not
strictly enforced, but it may reflect the coding practice. The output
of the patient's $n$-th admission is : $\left[c_{p_{1}}^{n},c_{p_{2}}^{n},...,c_{p_{L_{out}}}^{n},\diameter\right]$,
in which $L_{out}$ is the length of the treatment sequence we want
to predict and $\diameter$ is used to inform the model to stop predicting.
Finally, each code is represented by one-hot vector $v_{c}\in\left[0,1\right]^{\left\Vert C\right\Vert }$,
where $v_{c}=\left[0,...,0,1,0..,0\right]$ ($v_{c}[i]=1$ if and
only if $v_{c}$ represents $c_{i}$). Unlike set encoding of each
admission, representing the data in this way preserves the admission's
internal order information allowing sequence-based methods to demonstrate
their power of capturing sequential events.

In Sec. \ref{sec:Dual-Control-Architecture}, we present a novel treatment
recommendation model using memory network to remember long-term dependencies
and rare events from EMR data. This model is built upon Differential
Neural Computer (DNC) \citet{graves2016hybrid}, a fully differentiable
implementation of memory-augmented neural network (MANN) (see Sec.
\ref{subsec:Differentiable-Neural-Computer}). In question-answer
bAbI task \citet{weston2015towards}, DNC treats the story, question
as sequences of words and perform well on the task of predicting sequence
of answers, which suggests the power of DNC in handling sequence input
and solving sequence prediction task. Despite of its successes, DNC
have never been applied to realistic domain such as healthcare, especially
in clinical treatment sequence prediction. This realisation motivates
us to design a DNC-based architecture that fits and works well with
healthcare domain. In our design, we make use of two controllers instead
of one to handle dual processes: diagnoses and treatments. Each controller
will employ different remembering strategies for each process and
thus increase the robustness of prediction and the speed of learning.
Besides, we apply a write-protected policy for our controller to direct
the model toward reasonable remember strategies. 

\subsubsection*{Two-view sequential learning \label{subsec:Asynchronous-Two-View-Sequential}}

For two-view problems, the generic notations simplify to $S^{i_{1}}$,
$S^{i_{2}}$ denoting input view spaces and $S$ the output view space.
Each sample of the two-view problem $\left(X^{i_{1}},X^{i_{2}},Y\right)$
consists of two input views: $X^{i_{1}}=\left\{ x_{1}^{i_{1}},...,x_{t_{1}}^{i_{1}},...,x_{L^{i_{i}}}^{i_{1}}\right\} $,
$X^{i_{2}}=\left\{ x_{1}^{i_{2}},...,x_{t_{2}}^{i_{2}},...,x_{L^{i_{2}}}^{i_{2}}\right\} $
and one output view $Y=\left\{ y_{1},...,y_{t},...,y_{L}\right\} $.
Each view has a particular length ($L^{i_{1}}$, $L^{i_{2}}$ or $L$),
each of which is a set/sequence of events that belongs to different
spaces ($x_{t_{1}}^{i_{1}}\in S^{i_{1}}$, $x_{t_{2}}^{i_{2}}\in S^{i_{2}}$,
$y_{t}\in S$). It should be noted that this formulation can be applied
to many situations including video-audio understanding, image-captioning
and other two-channel time-series signals. Here we focus effort on
solving the two-view problems in healthcare.

For example, in drug prescription, doctors prescribe drugs after considering
diagnoses and procedures administered to patients. In modeling disease
progression, doctor may refer to patient's history of admissions to
help diagnoses the current diseases or to predict the future disease
occurrences of the patient. There are clinical recording rules applying
to EMR codes such that diagnoses are \textquotedblleft ordered by
priority\textquotedblright{} or procedures follow the order that \textquotedblleft the
procedures were performed\textquotedblright \footnote{\url{https://mimic.physionet.org/mimictables/}}.
Besides, although medical codes from different views are highly correlated,
they are not aligned. For instances, some diagnoses may correspond
to one procedure or one diagnosis may result in multiple medicines.
Hence, these problems can be treated as asynchronous two-view sequential
learning.

In the drug prescription context, $S^{i_{1}}$ and $S^{i_{2}}$ represent
the diagnosis and procedure spaces, respectively and $S$ corresponds
to the medicine space. The drug prescription objective is to select
an optimal subset of medications from $S$ based on diagnosis and
procedure codes. Similarly, we can formulate the disease progression
problem as two input sequences (diagnoses and interventions) and one
output set (next diagnoses). Although our architecture can model sequential
output, the choice of representing output as set is to follow a common
practice in healthcare where the order of medical suggestions is specified.
Because a patient may have multiple admission records for different
hospital visits, a patient record can be represented as $\left\{ \left(X_{a}^{i_{1}},X_{a}^{i_{2}},Y_{a}\right)\right\} _{a=1}^{A}$,
where $A$ is the number of admissions this patient commits. In order
to predict $Y_{a}$, we may need to exploit not only $\left(X_{a}^{i_{1}},X_{a}^{i_{2}}\right)$
but also $\left\{ \left(X_{pa}^{i_{1}},X_{pa}^{i_{2}}\right)\right\} _{pa=1}^{a-1}$.
More details on how our work makes use of previous admissions and
handles long-term dependencies will be given in Sec. \ref{subsec:Persistent-Memory-for}.

In Sec. \ref{sec:Dual-Mem-Architecture}, we propose a novel memory
augmented neural network model solving the problem of asynchronous
interactions and long-term dependencies at the same time. Our model
makes use of three neural controllers and two external memories constituting
a dual memory neural computer. In our architecture, each input view
is assigned to a controller and a memory to model the intra-view interactions
in that particular view. At each time step, the controller reads an
input event, updates the memory, and generates an output based on
its current hidden state and read vectors from the memory. Corresponding
to the two types of inter-view interactions, there are two modes in
our architecture: late-fusion and early-fusion memories. In the late-fusion
mode, the memory space for each view is separated and independent,
that is, there is no information exchange between the two memories
during the encoding process. The memories' read values are only synthesised
to generate inter-view knowledge in the decoding phase. Contrast to
the late-fusion mode, the memory addressing space in the early-fusion
mode is shared amongst views. That is, the encoder from one view can
access and modify the contents of the other view's memory. This design
ensures the information is shared across views via memories accessing.
In order to facilitate this asynchronous sharing, we design novel
cache components that temporarily hold the write values of every timestep.
This enables related information at different time steps to be written
to the memories together. Finally, we apply memory write-protected
mechanism in the decoding process to make the inference of our model
more efficient. 

\section{Background}

\subsection{Multi-View Learning}

In multi-view learning, data can be naturally partitioned into channels
presenting different views of the same data. Multi-view sequential
learning is a sub-class of multi-view learning where each view data
is in the form of sequential events, which can be synchronous or asynchronous.
In the synchronous setting, all views share the same time step and
view length. Some problems of this type include video consisting of
visual and audio streams; and text as a joint sequence of words and
part-of-speech tags. Synchronous multi-view sequential learning is
an active area \citet{rajagopalan2016extending,zadeh2017tensor,zadeh2018memory}.
These works make assumptions on the time step alignment and thus they
are constrained by the scope of synchronous multi-view problems. 

In this work, we relax these assumptions and focus more on asynchronous
settings, that is, there is no alignment amongst views and the sequence
lengths vary across views. These occur when the data is collected
from channels having different time scales or we cannot infer the
precise time information when extracting data. In healthcare, for
instance, an electronic medical record (EMR) contains information
on patient's admissions, each of which consists of various views such
as diagnosis, medical procedure, and medicine. Although an admission
is time-stamped, medical events from each view inside the admission
are not synchronous and different in length. 

Asynchronous multi-view data often demonstrates three types of view
interactions. The first type is intra-view interactions, those involving
only one view, representing the internal dynamics. For example, each
EMR view has specific rules for coding its events, forming distinctive
correlations amongst medical events inside a particular view. The
second type is late inter-view interactions, those that span from
input views to output, representing the mapping function between the
inputs and the outputs. We call it ``late'' because the interaction
across input views is considered only in the inference process. The
third type is early inter-view interactions, those that account for
relations covering multiple input views and happening before the inference
process. For example, in drug prescription, the diagnosis view is
the cause of the medical procedure view, both of which affect the
output which are medicines prescribed for patient. The interactions
in sequential views not only span across views but also extend throughout
the length of the sequences. One example involves patients whose diseases
in current admission are related to other diseases or treatments from
distant admissions in the past. The complexity of view interactions,
together with the unalignment and long-term dependencies amongst views
poses a great challenge in asynchronous multi-view sequential problems.

\subsection{Existing Approaches }

\textbf{Deep learning for healthcare:} The recent success of deep
learning has drawn board interest in building AI systems to improve
healthcare. Several studies have used deep learning methods to better
categorise diseases and patients: denoising autoencoders, an unsupervised
approach, can be used to cluster breast cancer patients \citet{tan2014unsupervised},
and convolutional neural networks (CNNs) can help count mitotic divisions,
a feature that is highly correlated with disease outcome in histological
images \citet{cirecsan2013mitosis}. Another branch of deep learning
in healthcare is to solve biological problems such as using deep RNN
to predict gene targets of microRNAs \citet{zurada1994end}. Despite
these advances, a number of challenges exist in this area of research,
most notably how to make use of other disparate types of data such
as electronic medical records (EMRs). Recently, more efforts have
been made to utilise EMR data in disease prediction \citet{pham2017predicting},
unplanned admission and risk prediction \citet{nguyen2016deepr} problems.
Other works apply LSTMs, both with and without attention to clinical
time series for heart failure prediction \citet{choi2016retain} or
diagnoses prediction \citet{lipton2016learning}. Treatment recommendation
is also an active research field with recent deep learning works that
model EMR codes as sequence such as using sequence of billing codes
for medicine suggestions \citet{bajor2016predicting} or using set
of diagnoses for medicine sequence prediction \citet{zhang2017leap}.
Differing from these approaches, our works focus on modeling both
the admission data and the treatment output as two sequences to capture
order information from input codes and ensure dependencies amongst
output codes at the same time. 

\textbf{Multi-view learning in healthcare:} Multi-view learning is
a well-studied problem, where methods often exploit either the consensus
or the complementary principle \citet{xu2013survey}. A straightforward
approach is to concatenate all multiple views into one single view
making it suitable for conventional machine learning algorithms, both
for vector inputs \citet{gonzalez2015multiview,zadeh2016multimodal}
or sequential inputs \citet{morency2011towards,song2012multi}. Another
approach is co-training \citet{blum1998combining,nigam2000analyzing},
aiming to maximise the mutual agreement on views. Other approaches
either establish a latent subspace shared by multiple views \citet{quadrianto2009estimating}
or perform multiple kernel learning \citet{rakotomamonjy2007more}.
These works are typically limited to non-sequential views.

More recently, deep learning is increasingly applied for multi-view
problems, especially with sequential data. For example, LSTM \citet{hochreiter1997long}
is extended for multi-view problems \citet{rajagopalan2016extending}
or multiple kernel learning is combined with convolution networks
\citet{poria2015deep}. More recent methods focus on building deep
networks to extract features from each view before applying different
late-fusion techniques such as tensor products \citet{zadeh2017tensor},
contextual LSTM \citet{poria2017context} and gated memory \citet{zadeh2018memory}.
All of these deep learning methods are designed only for synchronous
sequential input views. Hence, the applications of these methods mostly
fall into tagging problems where the output is aligned with the input
views. As far as we know, the only work that can apply to asynchronous
inputs is in Chung et al., (2017), in which the authors construct
a dual LSTM for feature extraction and use attention for late-fusion
\citet{chung2017lip}. Without multiple memories for multi-view sequences,
LSTM-based models fail short in capturing long-term dependencies in
the sequences. LSTM encoder-decoder with attention can only support
late-fusion modeling, which may be insufficient for cases that require
early-fusion. More importantly, attentional LSTMs do not assume fixed
size memory, which may be impractical for ultra-long sequences, online
and life-long learning.

In healthcare, there are only few works that make use of multi-view
data. A multi-view multi-task model is proposed to predict future
chronic diseases given multi-media and multi-model observations \citet{nie2015beyond}.
However, this model is only designed for single-instance regression
problems. DeepCare \citet{pham2017predicting} solves the disease
progression problem by combining diagnosis and intervention views.
It treats medical events in each admission as a bag and uses pooling
to compute the feature vectors for the two views in an admission.
The sequential property of events inside each admission is ignored
and there is no mechanism to model inter-view interactions at event
level. There are many other works using deep learning such as RETAIN
\citet{choi2016retain}, Dipole \citet{ma2017dipole} and LEAP \citet{zhang2017leap}
that attack different problems in healthcare. However, they are designed
for single input view.

\textbf{MANNs for healthcare:} Memory augmented neural networks (MANNs)
have emerged as a new promising research topic in deep learning. Memory
Networks (MemNNs) \citet{weston2014memory} and Neural Turing Machines
(NTMs) \citet{2014arXiv1410.5401G} are the two classes of MANNs that
have been applied to many problems such as meta learning \citet{santoro2016meta}
and question answering \citet{sukhbaatar2015end}. In healthcare,
there is limited work applying MemNN-based models to handle medical-related
problems such as clinical textual QA \citet{hasan2016clinical} or
diagnosis inference \citet{prakash2017condensed}. However, these
works have been using clinical documents as input, rather than just
using medical codes stored in EMRs. Our work, on the other hand, learns
end-to-end from raw medical codes in EMRs by leveraging Differentiable
Neural Computer (DNC) \citet{graves2016hybrid}, the latest improvement
over the NTM. In practice, DNC and other NTM variants have been used
for various domains such as visual question answering \citet{ma2017visual},
and one-shot learning \citet{santoro2016meta}, yet it is the first
time DNC is adapted for healthcare tasks.

\section{Dual Control Architecture\label{sec:Dual-Control-Architecture}}

\begin{figure*}
\begin{centering}
\includegraphics[width=0.7\textwidth]{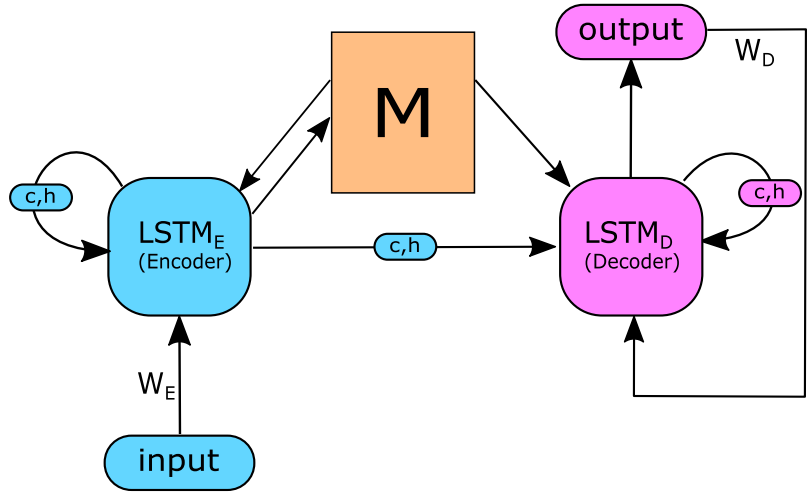}
\par\end{centering}
\centering{}\caption{Dual Controller Write-Protected Memory Augmented Neural Network. $LSTM_{E}$
is the encoding controller. $LSTM_{D}$ is the decoding controller.
Both are implemented as LSTMs. \label{fig:Dual-Controller-Write} }
\end{figure*}

We now present our first contribution\textendash a deep neural architecture
called Dual Controller Write-Protected Memory Augmented Neural Network
(DCw-MANN) (see Fig. \ref{fig:Dual-Controller-Write}). Our DCw-MANN\nomenclature{DCw-MANN}{Dual Controller Write-Protected Memory Augmented Neural Network}\nomenclature{DC-MANN}{Dual Controller Memory Augmented Neural Network}
introduces two simple but crucial modifications to the original DNC:
(i) using two controllers to handle dual processes of encoding and
decoding, respectively; and (ii) applying a write-protected policy
in the decoding phase. 

In the encoding phase, after going through embedding layer $W_{E}$,
the input sequence is fed to the first controller (encoder) $LSTM_{E}$.
At each time step, the controller reads from and writes to the memory
information necessary for the later decoding process. In the decoding
phase, the states of the first controller is passed to the second
controller (decoder) $LSTM_{D}$. The use of two controllers instead
of one is important in our setting because it is harder for a single
controller to learn many strategies at the same time. Using two controllers
will make the learning easier and more focused. Also different from
the encoder, the decoder can make use of its previous prediction (after
embedding layer $W_{D}$) as the input together with the read values
from the memory. Another important feature of DCw-MANN is its write-protected
mechanism in the decoding phase. This has an advantage over the writing
strategy used in the original DNC since at decoding step, there is
no new input that is fed into the system. Of course, there remains
dependencies amongst codes in the output sequence. However, as long
as the dependencies amongst output codes are not too long, they can
be well-captured by the cell memory $c_{t}$ inside the decoder's
LSTM. Therefore, the decoder in our design is prohibited from writing
to the memory. To be specific, at time step $t+1$ we have the hidden
state and cell memory of the controllers calculated as:

\begin{equation}
h_{t+1},c_{t+1}=\begin{cases}
LSTM_{E}\left(\left[W_{E}v_{d_{t}},r_{t}\right],h_{t},c_{t}\right); & t\leq L_{in}\\
LSTM_{D}\left(\left[W_{D}v_{p_{t}},r_{t}\right],h_{t},c_{t}\right); & t>L_{in}
\end{cases}
\end{equation}
where $v_{d_{t}}$ is the one-hot vector representing the input sequence's
code at time $t\leq L_{in}$ and $v_{p_{t}}$ is the predicted one-hot
vector output of the decoder at time $t>L_{in}$, defined as $v_{p_{t}}=onehot\left(o_{t}\right)$,
i.e.,: 
\begin{equation}
v_{p_{t}}\left[i\right]=\begin{cases}
1 & ;i=\underset{1\leqslant j\leqslant\left|C_{p}\right|}{argmax}(o_{t}\left[j\right])\\
0 & ;\mathrm{otherwise}
\end{cases}.
\end{equation}
We propose a new memory update rule to enable the write-protected
mechanism:

\begin{equation}
M_{t}=\begin{cases}
M_{t-1}\circ\left(E-w_{t}^{w}e_{t}^{\top}\right)+w_{t}^{w}v_{t}^{\top}; & t\leq L_{in}\\
M_{t-1}; & t>L_{in}
\end{cases}
\end{equation}
where $E$ is an $N\times D$ matrix of ones , $w_{t}^{w}\in\left[0,1\right]^{N}$
is the write-weight, $e_{t}\in\left[0,1\right]^{D}$ is an erase vector,
$v_{t}\in\ensuremath{\mathbb{R}}^{D}$ is a write vector, $\circ$
is point-wise multiplication, and $L_{in}$ is the length of input
sequence.

\section{Dual Memory Architecture\label{sec:Dual-Mem-Architecture}}

We now present the second contribution to solve the generic asynchronous
two-view sequential learning: a new deep memory augmented neural network
called Dual Memory Neural Computer (DMNC\nomenclature{DMNC}{Dual Memory Neural Computer}). 

\begin{figure}
\begin{centering}
\includegraphics[width=0.7\textwidth]{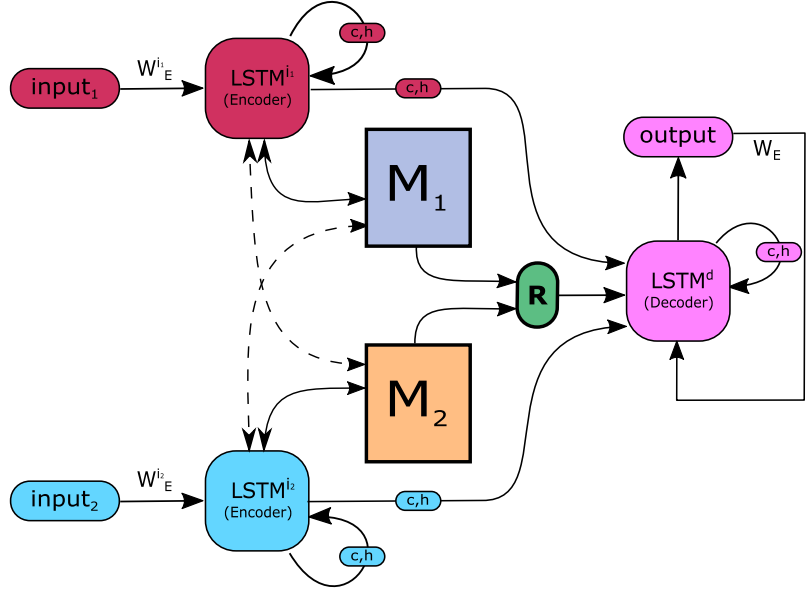}
\par\end{centering}
\caption{Dual Memory Neural Computer. $LSTM^{i_{1}}$, $LSTM^{i_{2}}$ are
the two encoding controllers implemented as LSTMs. $LSTM^{d}$ is
the decoding controller. The dash arrows represent cross-memory accessing
in early-fusion mode. \label{fig:Dual-Architecture-Neural}}
\end{figure}

\subsection{Dual Memory Neural Computer}

Our architecture consists of three neural controllers (two for encoding
and one for decoding), each of which interacts with two external memory
modules (see Fig. \ref{fig:Dual-Architecture-Neural}). Each of the
two memory modules is similar to the external memory module in DNC
\citet{graves2016hybrid}, that is, it is equipped with temporal linkage
and dynamic allocation. The three controllers have their own embedding
matrices $W_{E}^{i_{1}}$, $W_{E}^{i_{2}}$, $W_{E}$ which project
the one-hot representation of events to a unified $d$-dimensional
space. We use $\mathbf{x}_{t_{1}}^{i_{1}}$, $\mathbf{x}_{t_{2}}^{i_{2}}$,
$\mathbf{y}_{t}$ $\in\mathbb{R}^{d}$ to denote the embedding vector
of $x_{t_{1}}^{i_{1}}$, $x_{t_{2}}^{i_{2}}$, $y_{t}$, respectively,
in which $\mathbf{x}_{t_{1}}^{i_{1}}=W_{E}^{i_{1}}x_{t_{1}}^{i_{1}},\mathbf{x}_{t_{2}}^{i_{2}}=W_{E}^{i_{2}}x_{t_{2}}^{i_{2}},\mathbf{y}_{t}=W_{E}y_{t}$.
The embedding vectors $\mathbf{x}_{t_{1}}^{i_{1}}$, $\mathbf{x}_{t_{2}}^{i_{2}}$
are always used as inputs of the encoders while the embedding vector
$\mathbf{y}_{t}$ will only be used as input of the decoder if the
output view is a sequence. 

Each encoder will transform the embedding vectors to $h$-dimensional
hidden vectors. The current hidden vectors and outputs of the encoders
are computed as:

\begin{equation}
h_{t_{1}}^{i_{1}},o_{t_{1}}^{i_{1}}=LSTM{}^{i_{1}}\left(\left[\mathbf{x}_{t_{1}}^{i_{1}},r_{t_{1}-1}^{i_{1}}\right],h_{t_{1}-1}^{i_{1}}\right),1\leq t_{1}<L^{i_{1}}\label{eq:7}
\end{equation}

\begin{equation}
h_{t_{2}}^{i_{2}},o_{t_{2}}^{i_{2}}=LSTM{}^{i_{2}}\left(\left[\mathbf{x}_{t_{2}}^{i_{2}},r_{t_{2}-1}^{i_{2}}\right],h_{t_{2}-1}^{i_{2}}\right),1\leq t_{2}<L^{i_{2}}\label{eq:8}
\end{equation}
where $r_{t_{1}-1}^{i_{1}}$, $r_{t_{2}-1}^{i_{2}}$ are read vectors
at previous time step of each encoder and $L^{i_{1}},$$L^{i_{2}}$
are the lengths of input views. It should be noted that the time step
in each view may be asynchronous and the lengths may be different.
In our applications, since we treat input views as sequences, we use
$LSTM$ as the core of the encoders\footnote{For inputs as sets, we can replace the $LSTM$s with $MLP$s}.
Using separated encoder for each view naturally encourages the intra-view
interactions. To model inter-view interactions, we use two modes of
memories, late-fusion and early-fusion.

\textbf{Late-fusion memories:} In this mode, our architecture only
models late inter-view interactions. In particular, $r_{t_{1}}^{i_{1}}$
and $r_{t_{2}}^{i_{2}}$ are computed separately:

\begin{equation}
r_{t_{1}}^{i_{1}}=\left[r_{t_{1}}^{i_{1,1}},...,r_{t_{1}}^{i_{1,R}}\right]=m{}_{read}^{e_{1}}\left(o_{t_{1}}^{i_{1}},M_{1}\right)\label{eq:vr1}
\end{equation}

\begin{equation}
r_{t_{2}}^{i_{2}}=\left[r_{t_{2}}^{i_{2,1}},...,r_{t_{2}}^{i_{2,R}}\right]=m{}_{read}^{e_{2}}\left(o_{t_{2}}^{i_{2}},M_{2}\right)\label{eq:vr2}
\end{equation}
where $M_{1}$, $M_{2}$ are the two memory matrices containing view-specific
contents and $m_{read}^{e_{1}}$, $m_{read}^{e_{2}}$ are two read
functions of the encoders with separated set of parameters. Given
the encoder output vectors, the read functions produce the keys $k_{t_{1}}^{i_{1}}$,
$k_{t_{2}}^{i_{2}}$ in the manner of DNC. The keys are used to address
the corresponding memory and compute the read vectors using Eq.(\ref{eq:dncread}).
This design ensures the dynamics of computation in one view does not
affect the other's and only in-view contents are stored in view-specific
memory. This mode is important because in certain situations, writing
external contents to view-specific memory will interfere the acquired
knowledge and obstruct the learning process. In Section 4.1, we will
show a case study that fits with this setting and the empirical results
will demonstrate that the late-fusion mode is necessary to achieve
better performance. 

\textbf{Early-fusion memories: }When there exists a strong correlation
between the two input views, requiring to model early inter-view interactions,
we introduce another mode of memories: early-fusion mode. In this
mode, the two memories share the same addressing space, that is, the
encoder from one view can access the memory content from another view
and vice versa. Also, the read functions $m_{read}^{e}$ share the
same parameter set:

\begin{equation}
r_{t_{1}}^{i_{1}}=\left[r_{t_{1}}^{i_{1,1}},...,r_{t_{1}}^{i_{1,R}}\right]=m{}_{read}^{e}\left(o_{t_{1}}^{i_{1}},\left[M_{1},M_{2}\right]\right)\label{eq:evr1}
\end{equation}

\begin{equation}
r_{t_{2}}^{i_{2}}=\left[r_{t_{2}}^{i_{2,1}},...,r_{t_{2}}^{i_{2,R}}\right]=m{}_{read}^{e}\left(o_{t_{2}}^{i_{2}},\left[M_{1},M_{2}\right]\right)\label{eq:evr2}
\end{equation}

Since the read vectors for one encoder can come from either memories,
the encoder's next hidden values are dependent on both views' memory
contents, which enables possible early inter-view interactions in
this mode. 

\begin{algorithm}[t]
\begin{algorithmic}[1]
\Require{Training set $\{\{(X^{i_{1}}_a,X^{i_{2}}_a,Y_a\}_{a=1}^{A}\}_{n=1}^{N}$}
\State{Sample $B$ samples from training set}
\ForEach{sample in $B$}
\State{Clear memory $M_1$, $M_2$}
\For{$a=1,A$}
\State{$(X^{i_{1}},X^{i_{2}},Y)=(X^{i_{1}}_a,X^{i_{2}}_a,Y_a)$}
\While{$t_1<L^{i_1}$ or $t_2<L^{i_2}$}
\If{$t_1<L^{i_1}$}
\State{Use Eq.($\ref{eq:7}$) to calculate $h_{t_{1}}^{i_{1}},o_{t_{1}}^{i_{1}} $}
\State{Use Eq.($\ref{eq:dnc_w}$) or Eq.($\ref{eq:cache_w}$) to update $M_1$}
\State{Use Eq.($\ref{eq:vr1}$) or Eq.($\ref{eq:evr1}$) to read $M_1$}
\State{$t_1=t_1+1$}
\EndIf
\If{$t_2<L^{i_2}$}
\State{Use Eq.($\ref{eq:8}$) to calculate $h_{t_{2}}^{i_{2}},o_{t_{2}}^{i_{2}}$}
\State{Use Eq.($\ref{eq:dnc_w}$) or Eq.($\ref{eq:cache_w}$) to update $M_2$}
\State{Use Eq.($\ref{eq:vr2}$) or Eq.($\ref{eq:evr2}$) to read $M_2$}
\State{$t_2=t_2+1$}
\EndIf
\EndWhile
\State{Use Eq.($\ref{eq:read_1}$) and Eq.($\ref{eq:read_2}$) to read $M_1$,$M_2$ }
\State{Use Eq.($\ref{eq:out_y}$) to calculate $\widehat{y}$}
\State{Update parameter $\theta$ using $\nabla_{\theta}Loss_{set}\left(Y,\widehat{y}\right)$}
\EndFor
\EndFor
\end{algorithmic} 

\caption{Training algorithm for healthcare data (set output)\label{alg:Training-algorithm-for}}
\end{algorithm}

\textbf{Memories modification with cache components:} In both modes,
the two memories are updated every timestep by the two encoders. While
in the late-fusion mode, the writings to two memories are independent
and can be executed in parallel using Eq.(\ref{eq:dnc_w}), in the
early-fusion mode, the writings must be executed in an alternating
manner. In particular, the two encoders take turn writing to memories,
allowing the exchange of information at every timestep. Doing this
way is optimal if the two views are synchronous and equal in lengths.
To make it work with variable length input views, we introduce a new
component to our architecture: a cache memory that lies between the
controller and the external memory. Different from the original DNC
which writes directly the event's value to the external memory, in
the early-fusion mode of our architecture, each controller integrates
write values inside its own cache memory $c_{t}$ until an appropriate
moment before committing them to the external memory. We introduce
$g_{t}^{c}$ as a learnable cache gate to control the degree of integration
between current write value and the previous cache's content as follows:
\begin{eqnarray}
g_{t}^{c} & = & f^{c}\left(o_{t}^{i}\right)\\
c_{t} & = & g_{t}^{c}\circ c_{t-1}+\left(1-g_{t}^{c}\right)\circ v{}_{t}
\end{eqnarray}

In these equations, $g_{t}^{c}$ is the cache gate, $o_{t}^{i}$ is
the encoder output, $f^{c}$ is a learnable function\footnote{In this section, all $f$ functions are implemented as single-layer
feed-forward neural networks}, $c_{t}$ is the cache content and $v_{t}$ is the write value. Then,
the cache will be written to the memory using the following formula:

\begin{equation}
M_{t}=M_{t-1}\circ\left(E-g_{t}^{w}w_{t}^{w}e_{t}^{\top}\right)+g_{t}^{w}w_{t}^{w}c_{t}^{\top}\label{eq:cache_w}
\end{equation}

We propose this new writing mechanism for early-fusion mode to enable
one encoder to wait for another while processing input events (in
this context, waiting means the encoder stops writing to memory).
In the original DNC, if the write gate $g_{t}^{w}$ is close to zero,
the encoder does not write to memory and the write value at current
time step will be lost. However, in our design, even when there is
no writing, the write value somehow can be kept in the cache if $g_{t}^{c}<1$.
The cache in a view may choose to hold an event's write value instead
of writing it immediately at the read time step. Thus, the information
of the event is compressed in the cache until appropriate occasion,
which may be after the appearance of another event from the other
view. This mechanism enables two related asynchronous events to simultaneously
involve in building up the memories. 

\textbf{Write-protected memories: }In our architecture, during the
inference process, the decoder stops writing to memories. We add this
feature to our design because the decoder does not receive any new
input when producing output. Writing to memories in this phase may
deteriorate the memory contents, hampering the efficiency of the model. 

\subsection{Inference in DMNC}

In this section, we give more details on the operation of the decoder.
Because the decoder works differently for different output types (set
or sequence), we will present two versions of decoder implementation.

\textbf{Output as sequence: }In this setting, the decoder ingests
the encoders' final states as its initial hidden state $h_{0}=\left[h_{L^{i_{1}}}^{i_{1}},h_{L^{i_{2}}}^{i_{2}}\right]$
. The decoder's hidden and output vectors are given as: $h_{t},\left[o_{t}^{1},o_{t}^{2}\right]=LSTM^{d}\left(\left[\mathbf{y}_{t-1}^{*},r_{t-1}^{i_{1}},r_{t-1}^{i_{2}}\right],h_{t-1}\right)$.
Here, $\mathbf{y}_{t-1}^{*}$ is the embedding of the previous prediction
$y_{t-1}^{*}$. The decoder combines the read vectors from both memories
to produce a probability distribution over the output:

\begin{equation}
r_{t}^{i_{1}}=\left[r_{t}^{i_{1,1}},...,r_{t}^{i_{1,R}}\right]=m_{read}^{d}\left(o_{t}^{1},M_{1}\right)\label{eq:read_1}
\end{equation}

\begin{equation}
r_{t}^{i_{2}}=\left[r_{t}^{i_{2,1}},...,r_{t}^{i_{2,R}}\right]=m_{read}^{d}\left(o_{t}^{2},M_{2}\right)\label{eq:read_2}
\end{equation}

\begin{equation}
P\left(y_{t}|X^{i_{1}},X^{i_{2}}\right)=\pi\left(\left[o_{t}^{1},o_{t}^{2}\right]+f^{d}\left(\left[r_{t}^{i_{1}},r_{t}^{i_{2}}\right]\right)\right)
\end{equation}
where $r_{t}^{i_{1}},r_{t}^{i_{2}}$ are read vector from $M_{1},M_{2}$,
respectively, provided by the read function $m_{read}^{d}$, $f^{d}$
is a learnable function and $\pi$ is softmax function. The current
prediction is $y_{t}^{*}=\underset{y\in S}{argmax}\,P\left(y_{t}=y|X^{i_{1}},X^{i_{2}}\right)$
and the loss function is the cross entropy:

\begin{equation}
Loss_{seq}\left(Y,P\right)=-\stackrel[t=1]{L}{\sum}\log P\left(y_{t}|X^{i_{1}},X^{i_{2}}\right)\label{eq:18}
\end{equation}

\textbf{Output as set: }In this setting, the decoder uses $m_{read}^{d}$
to read from the memories once to get the read vectors $r^{i_{1}},r^{i_{2}}$.
The decoder combines these vectors with the encoders' final hidden
values to produce the output vector $\hat{y}\in\mathbb{R}^{\left|S\right|}$:

\begin{equation}
\hat{y}=\sigma\left(f^{d}(W_{1}r^{i_{1}}+W_{2}r^{i_{2}}+W_{3}\left[h_{L^{i_{1}}}^{i_{1}},h_{L^{i_{2}}}^{i_{2}}\right])\right)\label{eq:out_y}
\end{equation}
Here, the combination is simply the linear weighted summation with
parameter matrices $W_{1}$, $W_{2}$, $W_{3}$. $f^{d}$ is a learnable
function and $\sigma$ is the sigmoid function. For set output, the
loss function is multi-label loss defined as:
\begin{equation}
Loss_{set}\left(Y,\widehat{y}\right)=-\left(\underset{y_{l}\in Y}{\sum}\log\widehat{y}_{l}+\underset{y_{l}\notin Y}{\sum}\log\left(1-\widehat{y}_{l}\right)\right)\label{eq:20}
\end{equation}
For both settings, the decoder makes use of both memories' contents
and encoders' final hidden values to produce the output. While memory
contents represent the long-term knowledge, the encoder's hidden values
represent the short-term information stored inside the controllers.
Both are crucial to model inter-view interactions and necessary for
the decoder to predict the correct outputs.

\subsection{Persistent Memory for Multiple Admissions\label{subsec:Persistent-Memory-for}}

As mentioned earlier in Sec. \ref{subsec:Asynchronous-Two-View-Sequential},
one unique property of healthcare is the long-term dependencies amongst
admissions. Therefore, the output at the current admission $Y_{a}$
is dependent on the current and all previous admission's inputs $\left\{ \left(X_{pa}^{i_{1}},X_{pa}^{i_{1}}\right)\right\} _{pa=1}^{a}$.
There are several ways to model this property. The simplest solution
is to concatenate the current admission with previous ones to make
up single sequence input for the model. This method causes data replication
and preprocessing overhead. Another solution is to use recurrent neural
network to model the dependencies. Some papers use GRU and LSTM where
each time step is fed with an admission. The admission is treated
as a set of medical events and represented by a feature vector \citet{choi2015doctor,pham2017predicting}. 

In our memory-augmented architecture, we can model this dependencies
by using the memories to store information from previous admissions.
In the original DNC, the memory content is flushed every time new
data sample (i.e. new admission) is fed \textendash{} this certainly
loses the information of admission history. We modify this mechanism
by keeping the memories persistent during a patient's admissions processing.
That is, the content of memories is built up and modified during the
whole history of a patient's admissions. The memories are only cleared
prior to reading a new patient's record.

Persistent memories in our architecture play two important roles.
First, because the number of events across admissions are large while
memory sizes are moderate, the memory modules learn to compress efficiently
the input views, keeping only essential information. This makes memory
look-ups in the decoding process only limited to a fixed size of chosen
knowledge. This is more compact and focused than attention mechanisms,
in which the decoder has to attend to all events in the input. Second,
each memory slot can store information of any event in the input views,
which enables skip-connection reference in the decoding process, i.e.,
the decoder can jump to any input event, even the one in the farthest
admission, to look for relevant information. The whole process of
training our dual memory neural computer for healthcare data is summarised
in Algorithm \ref{alg:Training-algorithm-for}.

\section{Applications}

In this section, we perform experiments both on real-world data and
synthetic tasks. The purpose of the synthetic task is to study the
incremental impact of dual control modifications we propose. Moreover,
we demonstrate the effectiveness of our proposed dual memory model
DMNC. We use $DMNC_{l}$ and $DMNC_{e}$ to denote the late-fusion
and early-fusion mode of our model, respectively. The data for real-world
problems are real EMR data sets, some are public accessible. We make
the source code of DCw-MANN and DMNC publicly available at \url{https://github.com/thaihungle/MAED}
and \url{https://github.com/thaihungle/DMNC}, respectively.

\subsection{Synthetic Task: Odd-Even Sequence Prediction}

\begin{figure}
\centering{}%
\begin{minipage}[c][1\totalheight][t]{0.49\linewidth}%
\begin{center}
\includegraphics[width=1\linewidth]{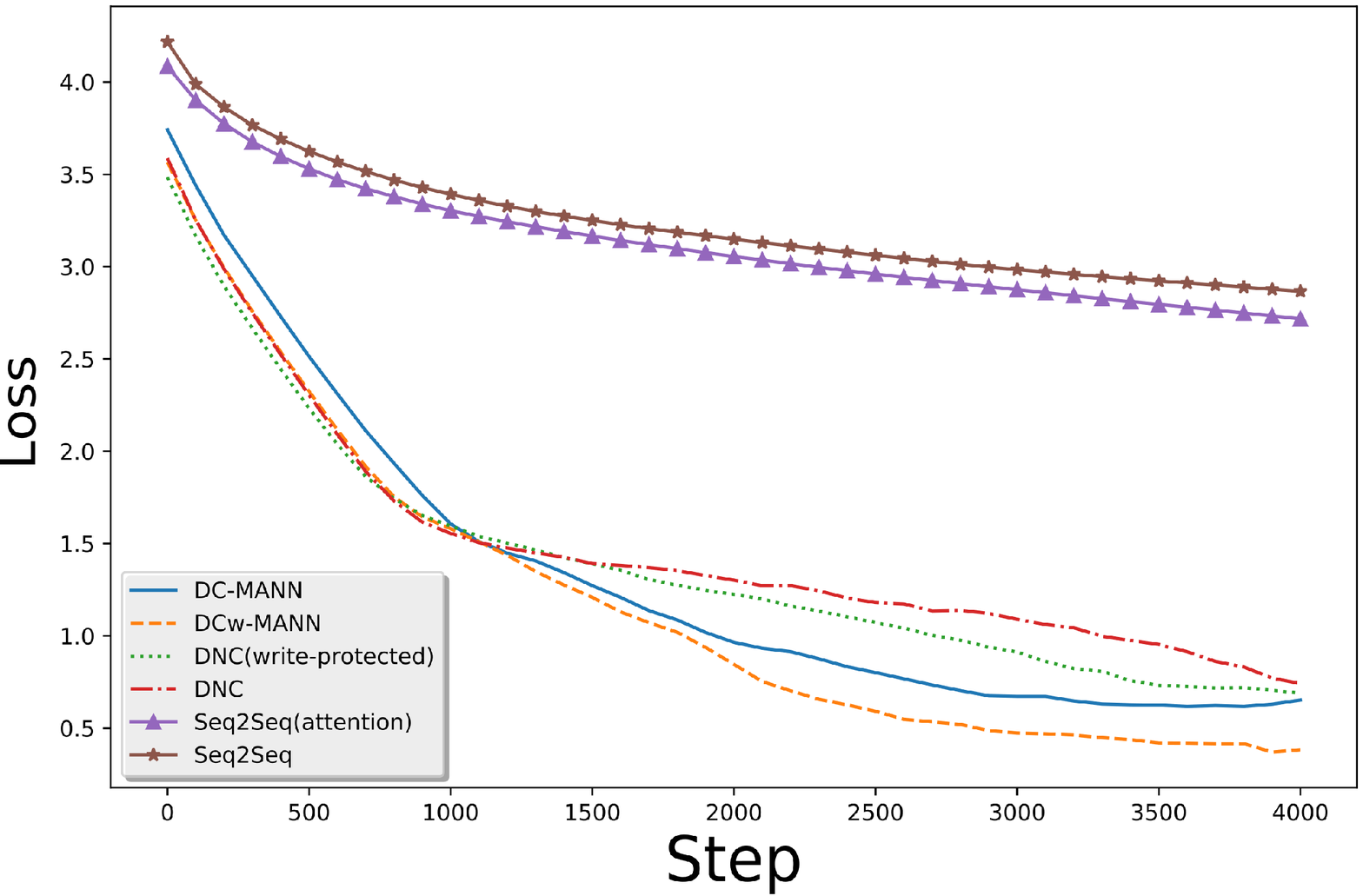}
\par\end{center}
\caption{Training Loss of Odd-Even Task\label{fig:Odd-Even's-Train-Loss}}
\end{minipage}\hfill{}%
\begin{minipage}[c][1\totalheight][t]{0.49\linewidth}%
\begin{center}
\includegraphics[width=1\linewidth]{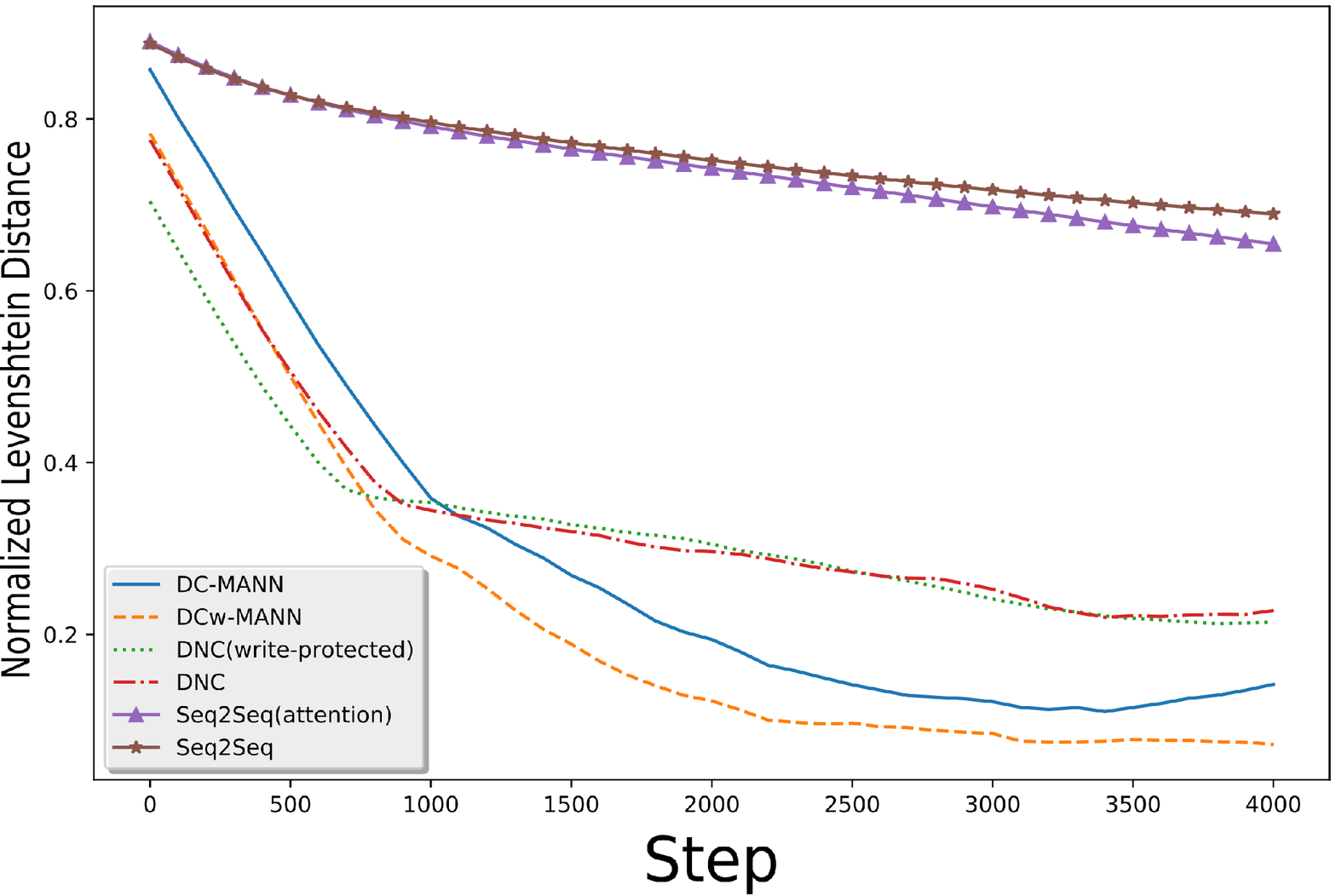}
\par\end{center}
\caption{Training NLD of Odd-Even Task\label{fig:Training-BLEU-score}}
\end{minipage}
\end{figure}

In this task, the input is sequence of random odd numbers chosen without
replacement from the set $S_{o}=\left\{ 1,3,5,...,49\right\} $ and
the output is sequence of even numbers from the set $S_{e}=\left\{ 2,4,6,..98\right\} .$
The $n$-th number $y_{n}$ in the output sequence is computed as:

$y_{n}=\begin{cases}
2x_{n} & n\leq\left\lfloor \frac{L}{2}\right\rfloor \\
y_{n-1}+2 & n>\left\lfloor \frac{L}{2}\right\rfloor 
\end{cases}$. $x_{n}$ is the $n$-th number in the input sequence and $L$ is
the length of both input and output sequence chosen randomly from
the range $\left[1,20\right]$. The formula is designed to reflect
healthcare situations where treatment options depend both on diagnoses
in the input sequence and other treatments in the same output sequence.
Here is an example of an input-output sequence pair with $L=7$: $input\coloneqq\left[11,7,25,39,31,1,13\right]$
and $output\coloneqq\left[22,14,50,52,54,56,58\right]$. We want to
predict the even numbers in the output sequence given odd numbers
in the input sequence, hence we name it odd-even prediction task.
In this task, the model has to ``remember'' the first half of the
input sequence to compute the first half of the output sequence, then
it should switch from using input to using previous output at the
middle of the output sequence to predict the second half. 

\textbf{Evaluations:} Our baselines are Seq2Seq \citet{DBLP:journals/corr/SutskeverVL14},
its attention version \citet{Bahdanau2015a} and the original DNC
\citet{graves2016hybrid}. Since we want to analyse the impact of
new modifications, in this task, we explore two other models: DNC
with write-protected mechanism in the decoding phase and dual controller
MANN without write-protected mechanism (DC-MANN). We use the Levenshtein
distance (edit distance) to measure the model's performance. To account
for variable sequence lengths, we \foreignlanguage{australian}{normalise}
this distance over the length of the longer sequence (between 2 sequences).
The predicted sequence is good if its Normalised Levenshtein Distance
(NLD) to the target sequence is small. 

\textbf{Implementation details:} For all experiments, deep learning
models are implemented in Tensorflow 1.3.0. Optimiser is Adam \citet{kingma2014adam}
with learning rate of 0.001 and other default parameters. The hidden
dimensions for LSTM and the embedding sizes for all models are set
to 256 and 64, respectively. Memory's parameters including number
of memory slots and the size of each slot are set to 128 and 128 ,
respectively.

\textbf{Results: }After training with 4000 input-output pair of sequences,
the models will be tested for the next 1000 pairs. The learning curves
of the models are plotted in Figs. \ref{fig:Odd-Even's-Train-Loss}
and \ref{fig:Training-BLEU-score}. The average NLD of the predictions
is summarised in Table \ref{tab:Tabel2}. As is clearly shown, the
proposed model outperforms other methods. Seq2Seq-based methods fail
to capture the data pattern and underperform other methods. The introduction
of two controllers helps boost the performance of DNC significantly.
Additional DNC-variant with write-protected also performs better than
the original one, which suggests the benefit of decoding without writing. 

\begin{figure}
\begin{minipage}[c][1\totalheight][t]{0.5\textwidth}%
\begin{center}
\begin{tabular}{cc}
\hline 
Model & NLD\tabularnewline
\hline 
Seq2Seq & 0.679\tabularnewline
Seq2Seq with attention & 0.637\tabularnewline
DNC & 0.267\tabularnewline
DNC (write-protected) & 0.250\tabularnewline
\hline 
DC-MANN & 0.161\tabularnewline
DCw-MANN & \textbf{0.082}\tabularnewline
\hline 
\end{tabular}
\par\end{center}
\captionof{table}{Test Results on Odd-Even Task (lower is better)}\label{tab:Tabel2}%
\end{minipage}\hfill{}%
\begin{minipage}[c][1\totalheight][t]{0.48\linewidth}%
\begin{center}
\includegraphics[width=1\linewidth]{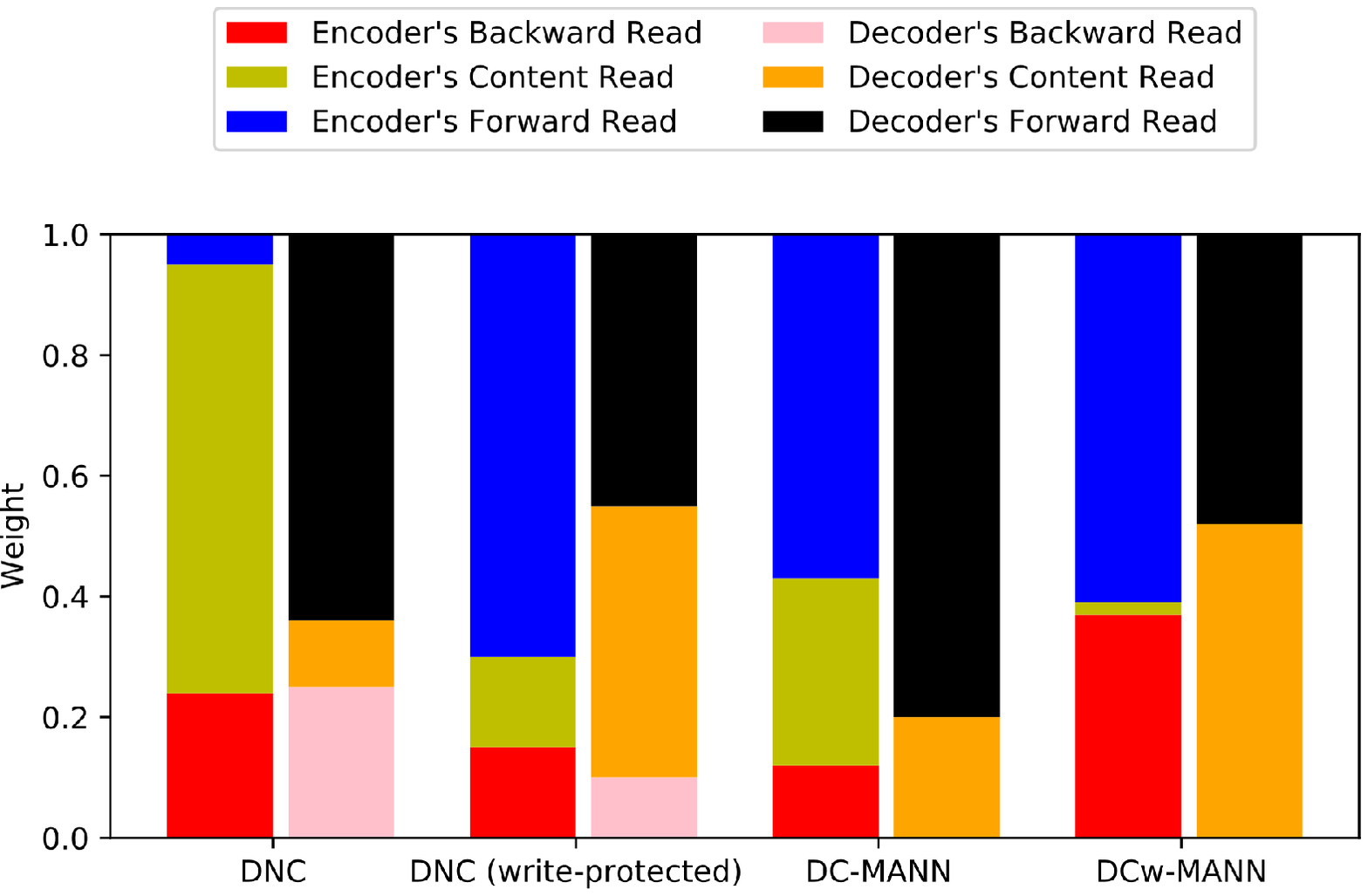}
\par\end{center}
\caption{Read Modes of MANNs on Odd-Even Task\label{fig:Read-Modes-of}}
\end{minipage}
\end{figure}

Fig. \ref{fig:Read-Modes-of} plots read mode weights for three reading
strategies employed in encoding and decoding phases. We can observe
the differences in the way the models prefer reading strategies. The
biggest failure of DNC is to keep using backward read in the decoding
process. This is redundant because in this problem, it is the forward
of the previous read location (if the memory location that corresponds
to $x_{n-1}$ is the previous read, then its forward is the memory
location that corresponds to $x_{n}$) that defines the current output
($y_{n}$). On the other hand, dual controllers with write-protected
mechanism seems help the model avoid bad strategies and focus more
on learning reasonable strategies. For example, using dual controllers
tends to lessen the usage of content-based read in the encoding phase.
This strategy is reasonable in this example since the input at each
time step is not repeated. Write-protected policy helps balance the
forward and content-based read in the decoding phase, which may reflect
the output pattern \textendash{} half-dependent on the input and half-dependent
on the previous output. 

\subsection{Treatment Recommendation Tasks}

The dataset used for this task is MIMIC-III \citet{johnson2016mimic},
which is a publicly available dataset consisting of more than 58k
EMR admissions from more than 46k patients. An admission history in
this dataset can contain hundreds of medical codes, which raises a
great challenge in handling long-term dependencies. In MIMIC-III,
there are both procedure and drug codes for the treatment process
so we consider two separate treatment recommendation tasks: procedure
prediction and drug prescription. In practice, if we use all the drug
codes in an EMR record, the drug sequence can be very long since,
each day in hospital, the doctor can prescribe several types of drugs
for the patient. Hence, we only pick the first drug used in a day
during the admission as the representative drug for that day. We also
follow the previous practice that only focuses on patients who have
more than one visit \citet{ma2017dipole,nguyen2016deepr,pham2017predicting}.
The statistics of the two sub-datasets is detailed in Table \ref{tab:Statistic-of-MIMIC-III}.

\begin{table}
\begin{centering}
\begin{tabular}{lcc}
\hline 
MIMIC-III Dataset (\# of visit >1) & Procedure as output & Drug as output\tabularnewline
\hline 
\# of patients & 6,314 & 5,620\tabularnewline
\# of admissions & 16,317 & 14,656\tabularnewline
\# of unique diagnosis codes & 4,669 & 4,563\tabularnewline
\# of unique treatment codes & 1,439 & 2,446\tabularnewline
Average \# of diagnosis sequence length & 13.3 & 13.8\tabularnewline
Max \# of diagnosis sequence length & 39 & 39\tabularnewline
Average \# of treatment sequence length & 4.7 & 11.4\tabularnewline
Max \# of treatment sequence length & 40 & 186\tabularnewline
Average \# of visits per patient & 2.5 & 2.6\tabularnewline
Max \# of visits per patient & 29 & 29\tabularnewline
\hline 
\end{tabular}
\par\end{centering}
\caption{Statistics of MIMIC-III sub-datasets\label{tab:Statistic-of-MIMIC-III}}
\end{table}

\textbf{Evaluations:} For comprehensiveness, beside direct competitors,
we also compare our methods with classical for healthcare predictions,
which are Logistic Regression and Random Forests. Because traditional
methods are not designed for sequence predictions, we simply pick
the top outputs (ignoring ordering information). In treatment recommendation
tasks, we use precision, which is defined as the number of correct
predicted treatment codes (ignoring the order) divided by the number
of predict treatment codes. More formally, let $S_{p}^{n}$ be the
set of ground truth treatments for the $n$-th admission, $S_{q}^{n}$
be the set of treatments that the model outputs. Then the precision
is: $\frac{1}{N}\stackrel[n=1]{N}{\sum}\frac{\left|S_{p}^{n}\cap S_{q}^{n}\right|}{\left|S_{q}^{n}\right|}$,
where $N$ is total number of test patients. To measure how closely
the generated treatment compares against the real treatment, we use
Mean Jaccard Coefficient\footnote{The metrics actually are at disadvantage to the proposed sequence-to-sequence
model, but we use to make them easy to compare against non-sequential
methods.}, which is defined as the size of the intersection divided by the
size of the union of ground truth treatment set and predicted treatment
set: $\frac{1}{N}\stackrel[n=1]{N}{\sum}\frac{\left|S_{p}^{n}\cap S_{q}^{n}\right|}{\left|S_{p}^{n}\cup S_{q}^{n}\right|}$. 

\textbf{Implementation details:} We randomly divide the dataset into
the training, validation and testing set in a $0.7:0.1:0.2$ ratio,
where the validation set is used to tune model's hyper-parameters.
For the classical Random Forests and Logistic Classifier, the input
is bag-of-words. Also, we apply One-vs-Rest strategy \citet{rifkin2004defense}
to enable these classifiers to handle multi-label output and the hyper-parameters
are found by grid-searching.

\textbf{Results:} Table \ref{tab:MIMIC-III-Procedure-Prediction}
reports the prediction results on two tasks (procedure prediction
and drug prescription). The performance of the proposed DCw-MANN is
higher than that of baselines on the testing data for both tasks,
validating the use of dual controllers with write-protected mechanism.
Without memory, Seq2Seq methods seem unable to outperform classical
methods, possibly because the evaluations are set-based, not sequence-based.
In the drug prescription task, there is a huge drop in performance
of the Seq2Seq-based approaches.
\begin{table}
\begin{centering}
\begin{tabular}{ccccc}
\hline 
\multirow{2}{*}{Model} & \multicolumn{2}{c}{Procedure Output} & \multicolumn{2}{c}{Drug Output}\tabularnewline
\cline{2-5} \cline{3-5} \cline{4-5} \cline{5-5} 
 & Precision & Jaccard & Precision & Jaccard\tabularnewline
\hline 
Logistic Regression & 0.256 & 0.185 & 0.412 & 0.311\tabularnewline
Random Forest & 0.276 & 0.199 & 0.491 & 0.405\tabularnewline
Seq2Seq & 0.263 & 0.196 & 0.220 & 0.138\tabularnewline
Seq2Seq with attention & 0.272 & 0.204 & 0.224 & 0.142\tabularnewline
DNC & 0.285 & 0.214 & 0.577 & 0.529\tabularnewline
\hline 
DCw-MANN  & \textbf{0.292} & \textbf{0.221} & \textbf{0.598} & \textbf{0.556}\tabularnewline
\hline 
\end{tabular}
\par\end{centering}
\caption{Results on MIMIC-III dataset for procedure prediction and drug prescription
(higher is better).\label{tab:MIMIC-III-Procedure-Prediction}}
\end{table}
It should be noted that, in drug prescription, the drug codes are
given day by day; hence, the average length of output sequence are
much longer than the procedure's one. This could be a very challenging
task for Seq2Seq. Memory-augmented models, on the other hand, have
an external memory to store information, so it can cope with long-term
dependencies. Figs. \ref{fig:Training-Loss-of-1} and Fig. \ref{fig:Valid-Loss-of}
show that compared to DNC, DCw-MANN is the faster learner in drug
prescription task. This case study demonstrates that a MANN with dual
controller and write-protected mechanism can significantly improve
the performance of the sequence prediction task in healthcare.

\begin{figure}
\centering{}%
\begin{minipage}[t]{0.48\textwidth}%
\begin{center}
\includegraphics[width=1\linewidth]{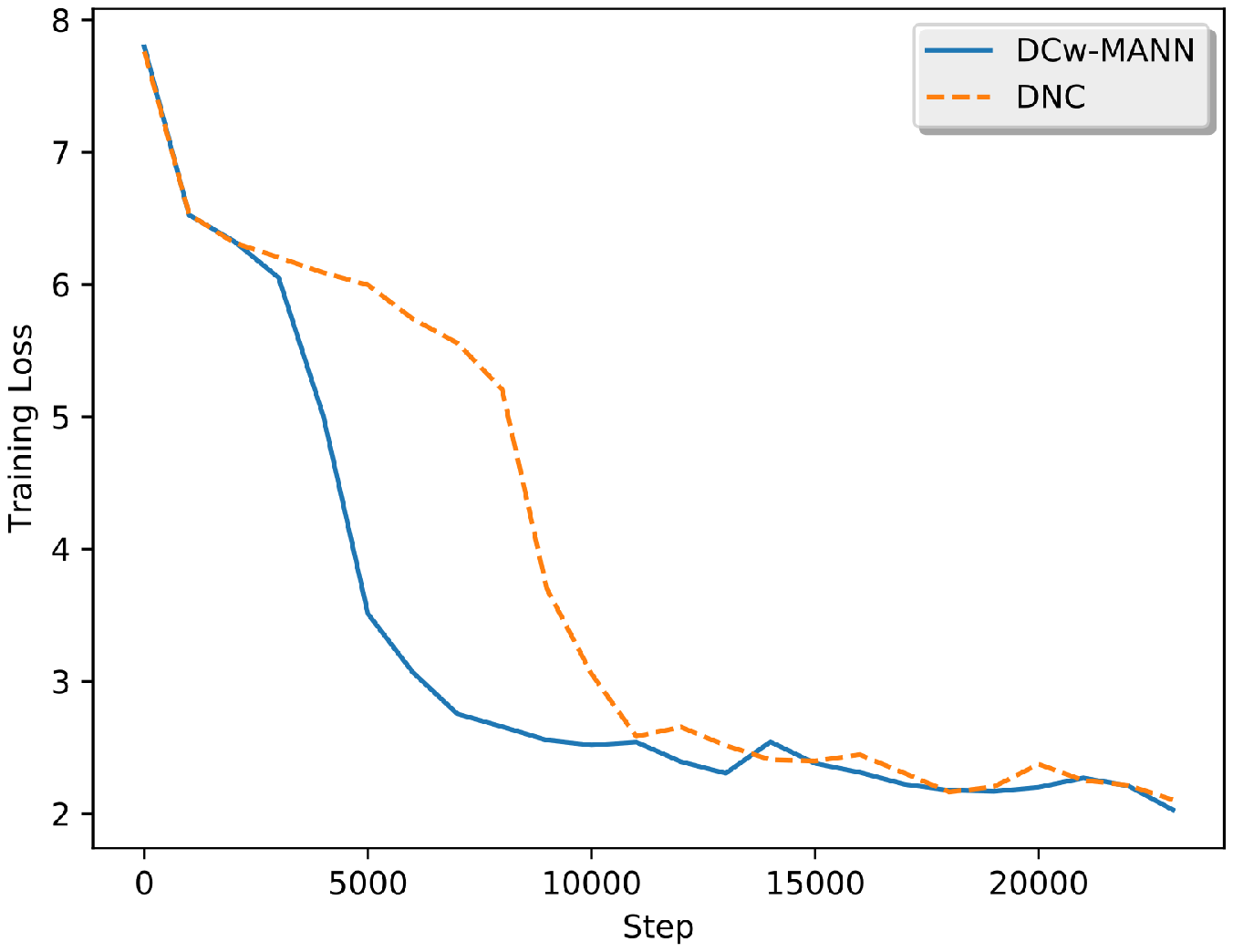}
\par\end{center}
\caption{Training Loss of Drug Prescription Task\label{fig:Training-Loss-of-1}}
\end{minipage}\hfill{}%
\begin{minipage}[t]{0.48\textwidth}%
\begin{center}
\includegraphics[width=1\linewidth]{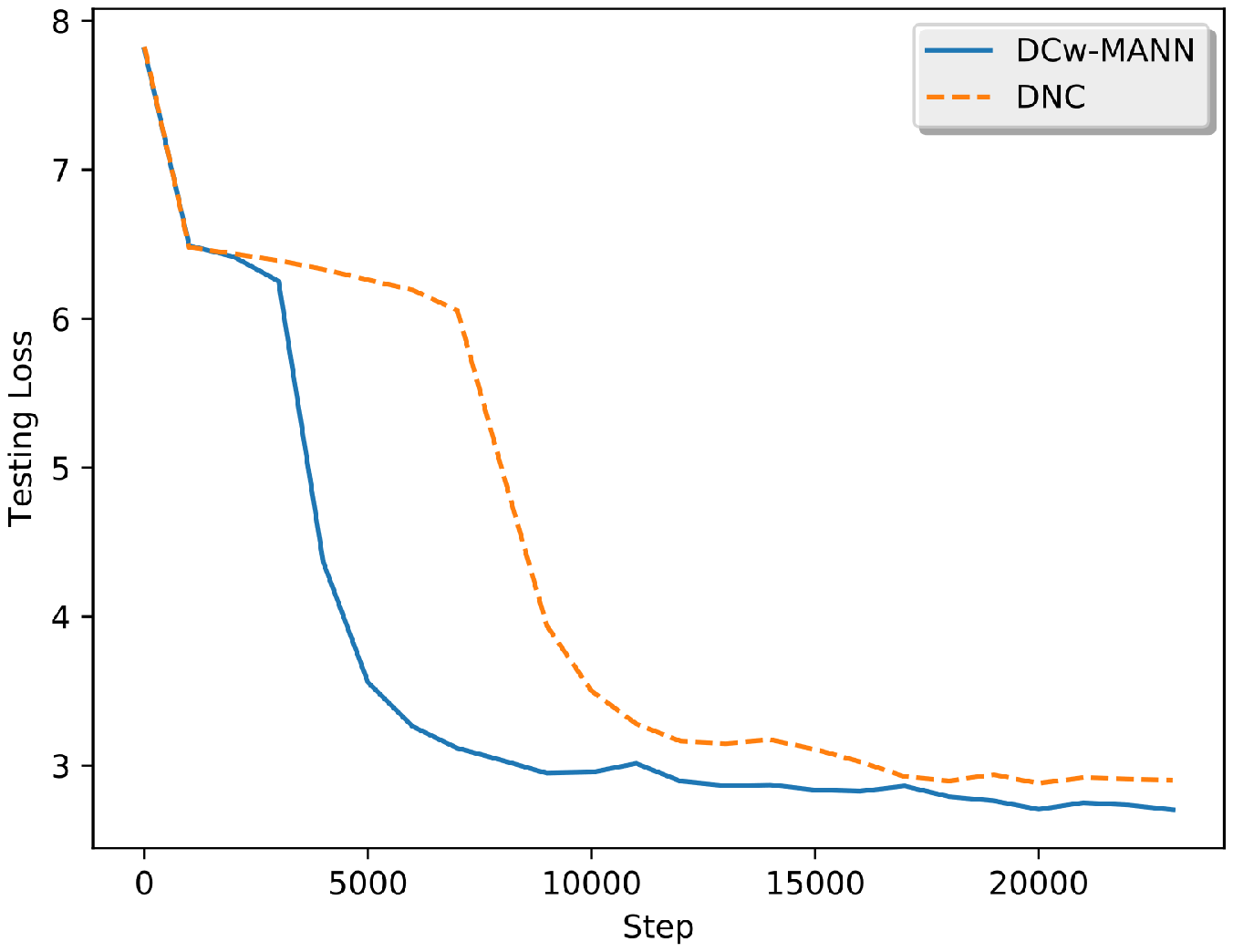}
\par\end{center}
\caption{Testing Loss of Drug Prescription Task\label{fig:Valid-Loss-of}}
\end{minipage}
\end{figure}

\subsection{Synthetic Task: Sum of Two Sequences}

\begin{figure}
\begin{centering}
\includegraphics[width=0.8\linewidth]{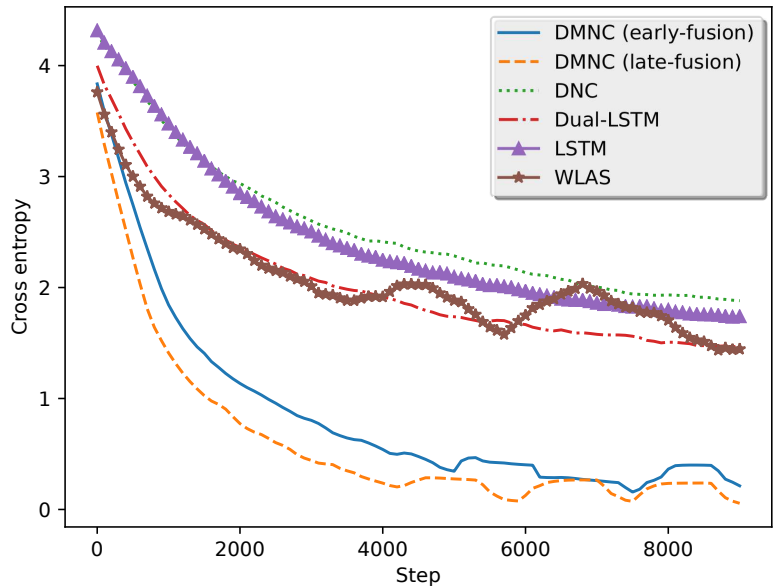}
\par\end{centering}
\caption{Training loss of sum of two sequences task. The training error curves
have similar patterns.\label{fig:Training-Loss-of}}
\end{figure}

\begin{table}
\begin{centering}
\caption{Sum of two sequences task test results. Max train sequence length
is 10.\label{tab:Sum-task-test}}
~
\par\end{centering}
\centering{}%
\begin{tabular}{cccc}
\hline 
\multirow{2}{*}{Model} & \multicolumn{3}{c}{Accuracy (\%)}\tabularnewline
\cline{2-4} \cline{3-4} \cline{4-4} 
 & $L_{max}=10$ & $L_{max}=15$ & $L_{max}=20$\tabularnewline
\hline 
LSTM & 35.17 & 24.12 & 18.64\tabularnewline
DNC & 37.8 & 20.43 & 14.67\tabularnewline
Dual LSTM & 52.41 & 42.57 & 30.47\tabularnewline
WLAS & 55.98 & 43.29 & 32.49\tabularnewline
\hline 
$DMNC_{l}$ & \textbf{99.76} & \textbf{98.53} & \textbf{78.17}\tabularnewline
$DMNC_{e}$ & 98.84 & 93.00 & 69.93\tabularnewline
\hline 
\end{tabular}
\end{table}

We conduct this synthetic experiment to verify our model performance
and behavior. In this problem, the input views are two randomly generated
sequence of numbers: $\left\{ x_{1}^{1},...,x_{L}^{1}\right\} $,
$\left\{ x_{1}^{2},...,x_{L}^{2}\right\} $. Each sequence has $L$
integer numbers. $L$ is randomly chosen from range $\left[1,L_{max}\right]$
and the numbers are randomly chosen from range $\left[1,50\right]$.
The output view is also a sequence of integer numbers defined as $\left\{ y_{i}=x_{i}^{1}+x_{L+1-i}^{2}\right\} _{i=1}^{L}$,
in which $y_{i}\in\left[2,100\right]$. Note that this summation form
is unknown to the model. During training, only the outputs are given.
Because the output's number is the sum of two numbers from the two
input views, we name the task as sum of two sequences. It should be
noted that two input numbers in the summation do not share the same
time step; hence, the problem is asynchronous. To learn and solve
the task, a model has to read all the numbers from the two input sequences
and discover the correct pair that will be used to produce the summation.
Synchronous multi-view models certainly fail this task because they
assume the inputs to be aligned. In the training phase, we choose
$L_{max}=10$, training for 10,000 iterations with mini batch size
$=50$. In the testing phase, we evaluate on 2500 random samples with
$L_{max}=10$, $L_{max}=15$, $L_{max}=20$ to verify the generalisation
of the models beyond the range where they are trained. 

\textbf{Evaluations}: the baselines for this synthetic task are chosen
as follows:
\begin{itemize}
\item View-concatenated sequential models: This concatenates events in input
views to form one long sequence. This technique transforms the two-view
sequential problem to normal sequence-to-sequence problem. We pick
LSTM and DNC as two representative methods for this approach.
\item Attention model WLAS \citet{chung2017lip}: This has a LSTM encoder
per view, and attention is used for decoding, similar to that in machine
translation \citet{cho2014properties,Bahdanau2015a}. The model is
applied successfully in the problem of video sentiment analysis. To
make it suitable for our tasks, we replace the encoders' feature-extraction
layers in the original WLAS by an embedding layer. We choose this
model as baseline since its architecture is somehow similar to ours.
The difference is that we make use of external memories instead of
attention mechanism. 
\item Dual LSTM: This model is the WLAS model without attention, that is,
only the final states of encoders are passed into the decoder. 
\end{itemize}
\textbf{Implementations: }For all models, embedding and hidden dimensions
are 64 and 128, respectively. Word size for memory-based methods are
64. Memory size for the view-concatenated DNC and DMNC are 32 and
16, respectively. We double the memory size for view-concatenated
DNC to account for the fact that the length of the input sequence
is nearly double due to view concatenation. We use Adam optimiser
with default parameters and apply gradient clipping size $=10$ to
train all models. Since output is a sequence, we use the cross-entropy
loss function in Eq.(\ref{eq:18}). The evaluation metric used in
this task is accuracy \textendash{} the number of correct predictions
over the length of output sequence. 

\textbf{Results: }The training loss curves of the models are plotted
in Fig. \ref{fig:Training-Loss-of}. The test average accuracy is
summarised in Table \ref{tab:Sum-task-test}. As clearly shown, overall
the proposed model outperforms other methods by a huge margin of about
45\%. Although dual LSTM and WLAS perform better than view-concatenated
methods, it's too hard for non-memory methods to ``remember'' correctly
pairs of inputs for later output summation. View-concatenated DNC
even with double memory size still fails to learn the sum rule because
storing two views' data in a single memory seems to mess up the information,
making this model perform worst. Between two versions of DMNC, late-fusion
mode is better perhaps due to the independence between two inputs'
number sequences. This is the occasion where trying to model early
cross-interactions damages the performance. The slight drop in performance
when testing with $L_{max}=15$ shows that our model really learns
the sum rule. When $L_{max}=20$, the input length is longer than
the memory size, so even when DMNCs can learn the sum rule, they cannot
store all input pairs for later summation. However, our methods still
manage to perform better than any other baseline. 

\subsection{Drug Prescription Task}

\begin{table}[t]
\begin{centering}
\caption{MIMIC-III data statistics.\label{tab:MIMIC-III-data-statistic}}
~
\par\end{centering}
\centering{}%
\begin{tabular}{lclc}
\hline 
\# of admissions & 42,586 & \# of diag & 6,461\tabularnewline
\# of patients & 34,594 & \# of proc & 1,881\tabularnewline
Avg. view len & 53.86 & \# of drug & 300\tabularnewline
\hline 
\end{tabular}
\end{table}

\begin{table}
\begin{centering}
\caption{Mimic-III drug prescription test results.\label{tab:Mimic-III-test-results}}
~
\par\end{centering}
\centering{}%
\begin{tabular}{cccccc}
\hline 
Model & AUC & F1 & P@1 & P@2 & P@5\tabularnewline
\hline 
\selectlanguage{australian}%
\selectlanguage{australian}%
 & \multicolumn{5}{c}{Diagnosis Only}\tabularnewline
\hline 
\multicolumn{1}{c}{Binary Relevance} & 82.6 & 69.1 & 79.9 & 77.1 & 70.3\tabularnewline
Classifier Chains & 66.8 & 63.8 & 68.3 & 66.8 & 61.1\tabularnewline
LSTM & 84.9 & 70.9 & 90.8 & 86.7 & 79.1\tabularnewline
DNC & 85.4 & 71.4 & 90.0 & 86.7 & 79.8\tabularnewline
\hline 
\selectlanguage{australian}%
\selectlanguage{australian}%
 & \multicolumn{5}{c}{Procedure Only}\tabularnewline
\hline 
Binary Relevance & 81.8 & 69.4 & 82.6 & 80.1 & 73.6\tabularnewline
Classifier Chains & 63.4 & 61.7 & 83.7 & 80.3 & 71.9\tabularnewline
LSTM & 83.9 & 70.8 & 88.1 & 86.0 & 78.4\tabularnewline
DNC & 83.2 & 70.4 & 88.4 & 85.8 & 78.7\tabularnewline
\hline 
\selectlanguage{australian}%
\selectlanguage{australian}%
 & \multicolumn{5}{c}{Diagnosis and procedure}\tabularnewline
\hline 
Binary Relevance & 84.1 & 70.3 & 81.0 & 78.2 & 72.3\tabularnewline
Classifier Chains & 64.6 & 63.0 & 84.6 & 81.5 & 74.2\tabularnewline
LSTM & 85.8 & 72.1 & 91.6 & 86.8 & 80.5\tabularnewline
DNC & 86.4 & 72.4 & 90.9 & 87.4 & 80.6\tabularnewline
Dual LSTM & 85.4 & 71.4 & 90.6 & 87.1 & 80.5\tabularnewline
WLAS  & 86.6 & 72.5 & 91.9 & 88.1 & 80.9\tabularnewline
\hline 
$DMNC_{l}$ & 87.4 & 73.2 & \textbf{92.4} & 88.9 & \textbf{82.6}\tabularnewline
$DMNC_{e}$ & \textbf{87.6} & \textbf{73.4} & 92.1 & \textbf{89.9} & 82.5\tabularnewline
\hline 
\end{tabular}
\end{table}

The data set used for this task is MIMIC-III, which is a publicly
available dataset consisting of more than 52k EMR admissions from
more than 46k patients. In this task, we keep all the diagnosis and
procedure codes and only preprocess the drug code since the raw drug
view's average length can reach hundreds of codes in an admission,
which is too long given the amount of data. Therefore, only top 300
frequently used of total 4781 drug types are kept (covering more than
70\% of the raw data). The final statistics of the preprocessed data
is summarised in Table \ref{tab:MIMIC-III-data-statistic}. 

\textbf{Evaluations: }We compare our model with the following baselines:
\begin{itemize}
\item Bag of words and traditional classifiers: In this approach, each input
view is considered as a set of events. The vector represents the view
is the sum of one-hot vectors representing the events. These view
vectors are then concatenated and passed into traditional classifiers:
SVM, Logistic Regression, Random Forest. To help traditional methods
handle multi-label output, we apply two popular techniques: Binary
Relevance \citet{luaces2012binary} and Classifier Chains \citet{read2011classifier}.
We will only report the best model for each of the two techniques,
which are Logistic Regression and Random Forest, respectively.
\item View-concatenated sequential models (LSTM, DNC), Dual LSTM and WLAS
\citet{chung2017lip}: similar to those described in the synthetic
task.
\item Single-view models: To see the performance gains when making use of
two input views, we also report results when only using one view for
Binary Relevance, Classifier Chains, LSTM and DNC. 
\end{itemize}
\textbf{Implementations:} We randomly divide the dataset into the
training, validation and testing set in a $2/3:1/6:1/6$ ratio. For
traditional methods, we use grid-searching over typical ranges of
hyper-parameters to search for best hyper-parameter values. Deep learning
models' best embedding and hidden dimensions are 64 and 64, respectively.
Optimal word and memory size for DMNC are 64 and 16, respectively.
The view-concatenated DNC shares the same setting except the memory
size is doubled to 32 memory slots. Since the output in this task
is a set, we use the multi-label loss function in Eq.(\ref{eq:20})
for deep learning methods. We measure the relative quality of model
performances by using common multi-label metrics, Area Under the ROC
Curve (AUC) and F1 scores, both of which are macro-averaged. Similar
results can be achieved when using micro-averaged so we did not report
them here. In practice, precision at $k$ (P@$k$) are often used
to judge the treatment recommendation quality. Therefore, we also
include them ($k=1,2,5$) in the evaluation metrics.

\textbf{Results:} Table \ref{tab:Mimic-III-test-results} shows the
performance of experimental models on aforementioned performance metrics.
We can see the benefit of using two input views instead of one, which
helps improve the model performances. Traditional methods clearly
underperform deep learning methods perhaps because these methods are
hard to scale when there are many output labels and the inputs in
our problem are not bag-of-words. Amongst deep learning models, our
proposed ones consistently outperform others in all type of measurements.
Our methods demonstrate 1-2\% improvements over the second runner-up
baseline WLAS. The late-fusion mode seems suitable for certain type
of metrics, but overall, the early-fusion mode is the winner, highlighting
the importance of modeling early inter-view interactions. 

\textbf{Case study: }In Table \ref{tab:Example-Recommend-Treatments},
we show an example of drugs prescribed for a patient given his current
diagnoses and procedures. The patient had serious problems with his
bowel as described in the first four diagnoses. The next three diagnoses
are also severe relating to his heart problems while the remaining
diagnoses are less urgent. It seems that heart-related diagnoses later
led to heart surgeries listed in the procedure codes. Both modes of
DMNC predict correctly the drug Docusate Sodium used to cure urgent
bowel symptoms. Relating to heart diseases and surgeries, our models
predict closely to expert's choices. Potassium Chloride is necessary
for a healthy heart. Acetaminophen and Propofol are commonly used
during surgeries. However, some heart medicine such as Heparin is
missed by the two models. Figs. \ref{fig:Memory-write-gate-1} and
\ref{fig:Memory-write-gate} demonstrate the ``focus'' of the two
memories on diagnosis and procedure view, respectively. The higher
the write gate values, the more information of the medical codes will
be written into the memories. We can see both modes pay less attention
on last diagnoses corresponding to less severe symptoms. Compared
to the late-fusion, the early-fusion mode keeps more information on
procedures, especially the heart-related events. This may help increase
the weight on heart-related medicines and enable it to include Acetylsalicylic
Acid, a common drug used after heart attack in the top recommendations. 

\begin{figure}
\begin{centering}
\includegraphics[width=0.8\linewidth]{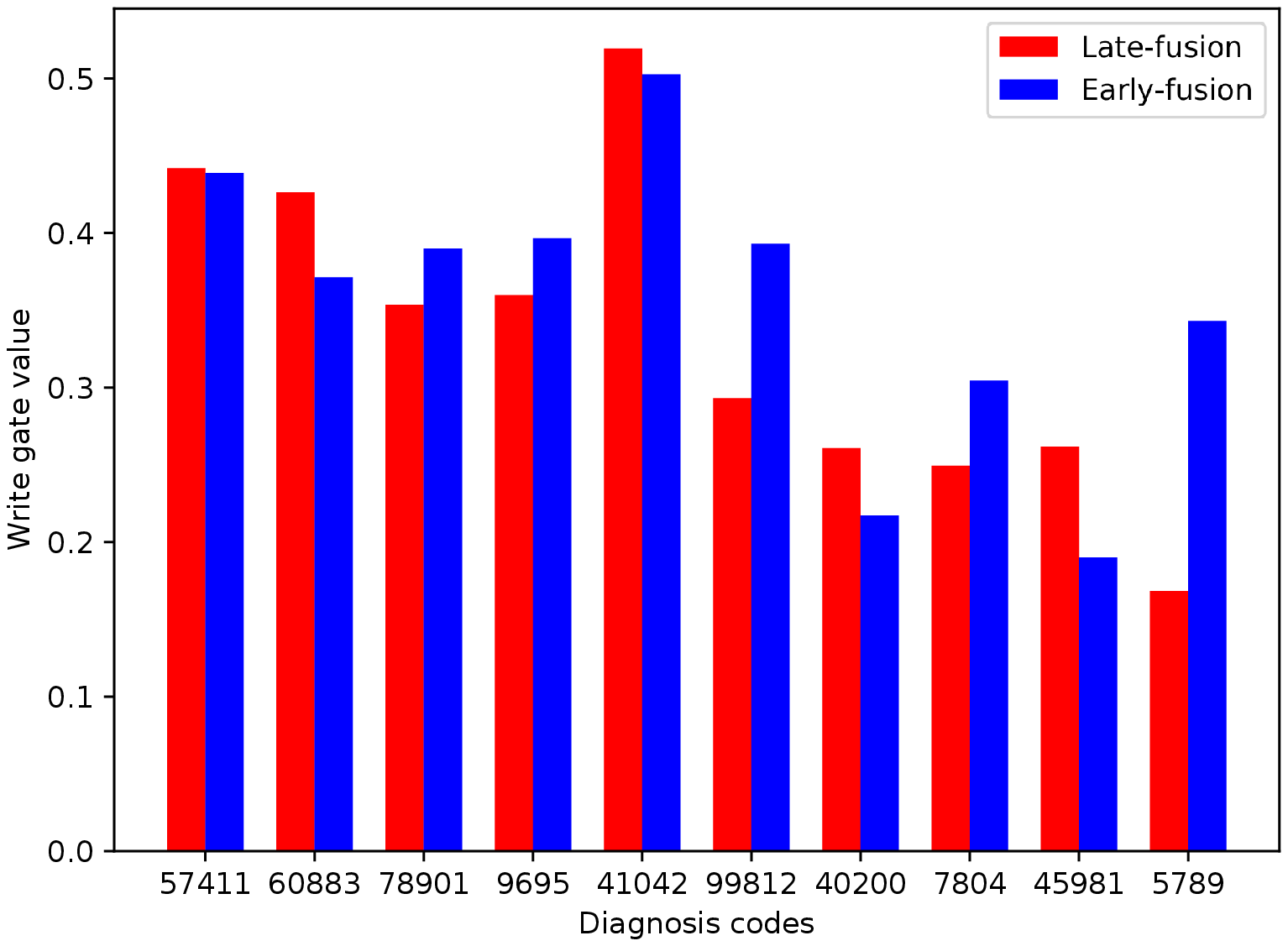}
\par\end{centering}
\caption{$M_{1}$'s $g_{t}^{w}$ over diagnoses. Diagnosis codes of a MIMIC-III
patient is listed along the x-axis (ordered by priority) with the
y-axis indicating how much the write gate allows a diagnosis to be
written to the memory $M_{1}$.\label{fig:Memory-write-gate-1}}
\end{figure}

\begin{figure}[t]
\begin{centering}
\includegraphics[width=0.8\linewidth]{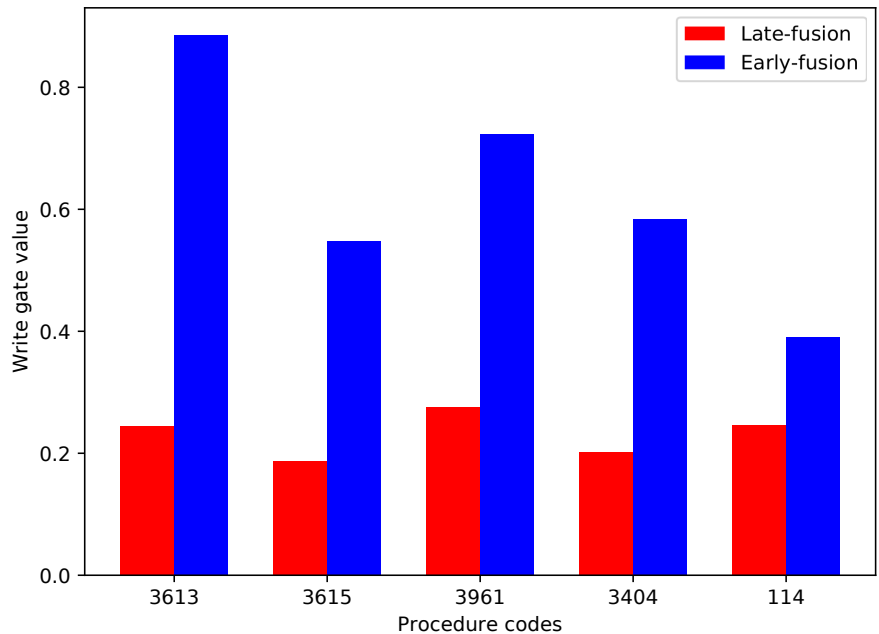}
\par\end{centering}
\caption{$M_{2}$'s $g_{t}^{w}$ over procedures. Medical procedure codes of
a MIMIC-III patient is listed along the x-axis (in the order of executions)
with the y-axis indicating how much the write gate allows a procedure
to be written to the memory $M_{2}$. \label{fig:Memory-write-gate}}
\end{figure}

\begin{table*}
\caption{Example Recommended Medications by DMNCs on MIMIC-III dataset. Bold
denotes matching against ground-truth.\label{tab:Example-Recommend-Treatments}}
~

\begin{tabular}{>{\raggedright}p{0.25\linewidth}>{\raggedright}p{0.7\linewidth}}
\hline 
Diagnoses & Calculus Of Gallbladder (57411),Vascular disorders of male genital
organs (60883), Abdominal Pain (78901), Poisoning By Other Tranquilizers
(9695), Acute Myocardial Infarction Of Other Inferior Wall (41042),
Hematoma Complicating (99812), Malignant hypertensive heart disease
40200), Dizziness and giddiness (7804), Venous (Peripheral) Insufficiency,
Unspecified (45981), Hemorrhage Of Gastrointestinal Tract (5789)\tabularnewline
\hline 
Procedures & Coronary Bypass Of Three Coronary Arteries (3613), Single Internal
Mammary Artery Bypass (3615), Extracorporeal circulation auxiliary
to open heart surgery (3961), Insertion Of Intercostal Catheter For
Drainage (3404), Operations on cornea(114)\tabularnewline
\hline 
Top 5 Ground-truth drugs (manually picked by experts) & Docusate Sodium (DOCU100L), Acetylsalicylic Acid (ASA81), Heparin
(HEPA5I), Acetaminophen (ACET325), Potassium Chloride (KCLBASE2)\tabularnewline
\hline 
Top 5 Late-fusion Recommendations & \textbf{Docusate Sodium (DOCU100L)}, Neostigmine (NEOSI), \textbf{Acetaminophen
(ACET325)}, Propofol (PROP100IG), \textbf{Potassium Chloride (KCLBASE2)}\tabularnewline
\hline 
Top 5 Early-fusion Recommendations  & \textbf{Docusate Sodium (DOCU100L)}, \textbf{Acetaminophen (ACET325)},
\textbf{Potassium Chloride (KCLBASE2)}, Dextrose (DEX50SY), \textbf{Acetylsalicylic
Acid (ASA81)}\tabularnewline
\hline 
\end{tabular}
\end{table*}

\subsection{Disease Progression Task}

Data used in this task are two chronic cohorts of diabetes and mental
EMRs collected between 2002-2013 from a large regional hospital in
Australia. Since we want to predict the next diagnoses for a patient
given his or her history of admission, we preprocessed the datasets
by removing patients with less than 2 admissions, which ends up with
53,208 and 52,049 admissions for the two cohorts. In this data set,
procedures and medicines are grouped into intervention codes, together
with diagnosis codes forming a patient's admission record. The number
of diagnosis and intervention codes are 249 and 1071, respectively.
We follow the same preprocessing steps and data split as in Pham et
al., (2017). Different from MIMIC-III, a patient record suffering
from chronic conditions often consists of multiple admissions, which
is suitable for the task of predicting disease progression. The average
number and the maximum number of admission per patient are 5.35 and
253, respectively.

\textbf{Evaluations: }For comparison, we choose the second best-runner
in our previous experiments WLAS and the current state-of-the-art
DeepCare \citet{pham2017predicting} as the two baselines. 

\textbf{Implementations:} We use the validation data set to tune the
hyper-parameters of our implementing methods and have the best embedding
and hidden dimensions are 20 and 64, respectively. The word and memory
size for DMNC are found to be 32 and 32, respectively. For performance
measurements, we use P@$k$ metric ($k=1,2,3$) to make it comparable
with DeepCare's results reported in Pham et al., (2017). 

\begin{table}
\begin{centering}
\caption{Regional hospital test results. P@K is precision at top K predictions
in \%.\label{tab:Local-hospital-test}}
\par\end{centering}
\begin{centering}
~
\par\end{centering}
\centering{}%
\begin{tabular}{ccccccc}
\hline 
\multirow{2}{*}{Model} & \multicolumn{3}{c}{Diabetes} & \multicolumn{3}{c}{Mental}\tabularnewline
\cline{2-7} \cline{3-7} \cline{4-7} \cline{5-7} \cline{6-7} \cline{7-7} 
 & P@1 & P@2 & P@3 & P@1 & P@2 & P@3\tabularnewline
\hline 
DeepCare & 66.2 & 59.6 & 53.7 & 52.7 & 46.9 & 40.2\tabularnewline
WLAS & 65.9 & 60.8 & 56.5 & 51.8 & 48.9 & 45.7\tabularnewline
\hline 
$DMNC_{l}$ & 66.5 & \textbf{61.3} & \textbf{57.0} & 52.7 & 49.4 & 46.2\tabularnewline
$DMNC_{e}$ & \textbf{67.6} & 61.2 & 56.9 & \textbf{53.6} & \textbf{50.0} & \textbf{47.1}\tabularnewline
\hline 
\end{tabular}
\end{table}

\textbf{Results: }We report the results on test data of models for
disease progression task in Table \ref{tab:Local-hospital-test}.
For both cohorts, our proposed model consistently outperforms other
methods and the performance gains become larger as the number of predictions
increase. Compared to DeepCare which uses pre-trained embeddings and
time-intervals as extra information, our methods only use raw medical
codes and perform better. This emphasises the importance of modeling
view interactions at event level. The late-fusion DMNC seems to perform
slightly better than the early-fusion DMNC in the diabetes cohort,
yet overall, the latter is the better one, which again validates its
ability to model all types of view interactions.

\section{Closing Remarks}

In this chapter, we have introduced DCw-MANN and DMNC, which are slot-based
MANNs designed for multiple processes. Under our designs, each input
sequence is assigned a neural controller to encode and store its events
to a dedicated memory. After all input sequences are stored, a decoder
will access the memories and synthesise the read contents to produce
the final output. Our methods can be generalised to sequence-to-sequence
and multi-view prediction tasks that require special handling of long-term
dependencies and view interactions. 

In summary, our main contributions are: (i) handling very long-term
dependencies and rare events in healthcare data by solving the sequence
prediction problem, (ii) proposing a novel memory-augmented architecture
that uses dual controller and write-protected mechanism (DCw-MANN)
to fit with sequence-in-sequence-out (SISO) task, (iii) proposing
a novel dual memory neural computer (DMNC) to solve the asynchronous
multi-view sequential problem and designing our architecture to model
view interactions and long-term dependencies, (iv) demonstrating the
efficacy of our proposed model on real-world medical data sets for
the problems of treatment recommendation, drug prescription and disease
progression. In particular, our models outperform other baselines
by 1-2\% across various metrics, which is significant in the domain.
More importantly, explainability is critical in healthcare and we
can somehow explain the behavior of our models by analysing the write
gate of the memory. 

We wish to emphasise that although our models are designed as predictive
model targeted to healthcare, they can be applied to other sequential
domains with similar data characteristics (i.e., sequential, long-term,
rare and multi-view) such as video understanding. Also, the current
work is limited to generating simple and deterministic output sequences.
In the next chapter, we will focus on another important problem, in
which, an external memory is necessary for holding temporal information
and composing mixture distributions in the latent space of generative
models.

\chapter{Variational Memory in Generative Models \label{chap:Variational-Memory-Encoder}}

\section{Introduction}

In the previous chapter, we have addressed the problem of encoder-decoder
architecture with long-term dependencies. However, the decoding is
an ill-posed problem, where there are many possible decoded sequences
given an input sequence. To account for such variation, we need a
method to model latent variables underlying these uncertainty. This
lends naturally to generative model of sequences.

Recent advances in generative modeling have led to exploration of
generative tasks. While generative models such as GAN \citet{goodfellow2014generative}
and VAE \citet{king14autoencoder,rezende2014stochastic} have been
applied successfully for image generation, learning generative models
for sequential discrete data is a long-standing problem. Early attempts
to generate sequences using RNNs \citet{graves2013generating} and
neural encoder-decoder models \citet{kalchbrenner2013recurrent,vinyals2015neural}
gave promising results, but the deterministic nature of these models
proves to be inadequate in many realistic settings. Tasks such as
translation, question-answering and dialog generation would benefit
from stochastic models that can produce a variety of outputs for an
input. For example, there are several ways to translate a sentence
from one language to another, multiple answers to a question and multiple
responses for an utterance in conversation. 

For tasks involving language understanding and production, handling
intrinsic uncertainty and latent variations is necessary. The choice
of words and grammars may change erratically depending on speaker
intentions, moods and previous languages used. The underlying RNNs
in neural sequential models find it hard to capture the dynamics and
their outputs are often trivial or too generic \citet{li2016diversity}.
One way to overcome these problems is to introduce variability into
these models. Unfortunately, sequential data such as speech and natural
language is a hard place to inject variability \citet{serban2017hierarchical}
since they require a coherence of grammars and semantics yet allow
freedom of word choice.

We propose a novel hybrid approach that integrates MAED (Sec. \ref{subsec:Memory-augmented-Encoder-Decoder})
and VAE, called Variational Memory Encoder-Decoder (VMED), to model
the sequential properties and inject variability in sequence generation
tasks. In this proposal, we utilise DC-MANN described in Sec. \ref{sec:Dual-Control-Architecture}
where the powerful DNC \citet{graves2016hybrid} is chosen as the
external memory. We prefer to allow writing to the memory during inference
because in this work, we focus on generating diverse output sequences,
which requires a dynamic memory for both encoding and decoding process.
Furthermore, we introduce latent random variables to model the variability
observed in the data and capture dependencies between the latent variables
across timesteps. Our assumption is that there are latent variables
governing an output at each timestep. In the conversation context,
for instance, the latent space may represent the speaker's hidden
intention and mood that dictate word choice and grammars. For a rich
latent multimodal space, we use a Mixture of Gaussians (MoG) because
a spoken word's latent intention and mood can come from different
modes, e.g., whether the speaker is asking or answering, or she/he
is happy or sad. By modeling the latent space as an MoG where each
mode associates with some memory slot, we aim to capture multiple
modes of the speaker's intention and mood when producing a word in
the response. Since the decoder in our model has multiple read heads,
the MoG can be computed directly from the content of chosen memory
slots. Our external memory plays a role as a mixture model distribution
generating the latent variables that are used to produce the output
and take part in updating the memory for future generative steps. 

To train our model, we adapt Stochastic Gradient Variational Bayes
(SGVB) framework \citet{king14autoencoder}. Instead of minimising
the $KL$ divergence directly, we resort to using its variational
approximation \citet{hershey2007approximating} to accommodate the
MoG in the latent space. We show that minimising the approximation
results in $KL$ divergence minimisation. We further derive an upper
bound on our total timestep-wise $KL$ divergence and demonstrate
that minimising the upper bound is equivalent to fitting a continuous
function by a scaled MoG. We validate the proposed model on the task
of conversational response generation. This task serves as a nice
testbed for the model because an utterance in a conversation is conditioned
on previous utterances, the intention and the mood of the speaker.
Finally, we evaluate our model on two open-domain and two closed-domain
conversational datasets. The results demonstrate our proposed VMED
gains significant improvement over state-of-the-art alternatives. 

\section{Preliminaries}

\subsection{Conditional Variational Autoencoder (CVAE) for Conversation Generation}

A dyadic conversation can be represented via three random variables:
the conversation context $x$ (all the chat before the response utterance),
the response utterance $y$ and a latent variable $z$, which is used
to capture the latent distribution over the reasonable responses.
A variational autoencoder conditioned on $x$ (CVAE\nomenclature{CVAE}{Conditional Variational Autoencoder})
is trained to maximise the conditional log likelihood of $y$ given
$x$, which involves an intractable marginalisation over the latent
variable $z$, i.e.,

\begin{equation}
p\left(y\mid x\right)=\int_{z}p\left(y,z\mid x\right)dz=\int_{z}p\left(y\mid x,z\right)p\left(z\mid x\right)dz
\end{equation}
Fortunately, CVAE can be efficiently trained with the Stochastic Gradient
Variational Bayes (SGVB) framework \citet{king14autoencoder} by maximising
the variational lower bound of the conditional log likelihood. In
a typical CVAE work, $z$ is assumed to follow multivariate Gaussian
distribution with a diagonal covariance matrix, which is conditioned
on $x$ as $p_{\phi}\left(z\mid x\right)$ and a recognition network
$q_{\theta}(z\mid x,y)$ to approximate the true posterior distribution
$p(z\mid x,y).$ The variational lower bound becomes

\begin{alignat}{1}
L\left(\phi,\theta;y,x\right)= & -KL\left(q_{\theta}\left(z\mid x,y\right)\parallel p_{\phi}\left(z\mid x\right)\right)\label{eq:L_CVAE}\\
 & +\mathbb{E}_{q_{\theta}\left(z\mid x,y\right)}\left[\log p\left(y\mid x,z\right)\right]\leq\log p\left(y\mid x\right)
\end{alignat}
where $KL$ is the Kullback\textendash Leibler divergence. With the
introduction of the neural approximator $q_{\theta}(z\mid x,y)$ and
the reparameterisation trick \citet{kingma2014semi}, we can apply
the standard back-propagation to compute the gradient of the variational
lower bound. Fig. \ref{fig:Graphical-Model-of}(a) depicts elements
of the graphical model for this approach in the case of using CVAE. 

\subsection{Related Works}

With the recent revival of recurrent neural networks (RNNs), there
has been much effort spent on learning generative models of sequences.
Early attempts include training RNN to generate the next output given
previous sequence, demonstrating RNNs' ability to generate text and
handwriting images \citet{graves2013generating}. Later, encoder-decoder
architecture \citet{sutskever2014sequence} enables generating a whole
sequence in machine translation \citet{kalchbrenner2013recurrent},
text summation \citet{nallapati2016abstractive} and conversation
generation \citet{vinyals2015neural}. Although these models have
achieved significant empirical successes, they fall short to capture
the complexity and variability of sequential processes.

These limitations have recently triggered a considerable effort on
introducing variability into the encoder-decoder architecture. Most
of the methods focus on conditional VAE (CVAE) by constructing a variational
lower bound conditioned on the context. The setting can be found in
many applications including machine translation \citet{zhang2016variational}
and dialog generation \citet{bowman2016generating,serban2017hierarchical,shen2017conditional,zhao2017learning}.
A common trick is to place a neural net between the encoder and the
decoder to compute the Gaussian prior and posterior of the CVAE. This
design is further enhanced by the use of external memory \citet{chen2018hierarchical}
and reinforcement learning \citet{wen2017latent}. In contrast to
this design, our VMED uses recurrent latent variable approach \citet{chung2015recurrent},
that is, our model requires a CVAE for each step of generation. Besides,
our external memory is used for producing the latent distribution,
which is different from the one proposed in \citet{chen2018hierarchical}
where the memory is used only for holding long-term dependencies at
sentence level. Compared to variational addressing scheme mentioned
in \citet{bornschein2017variational}, our memory uses deterministic
addressing scheme, yet the memory content itself is used to introduce
randomness to the architecture. More relevant to our work is GTMM
\citet{gemici2017generative} where memory read-outs involve in constructing
the prior and posterior at every timestep. However, this approach
uses Gaussian prior without conditional context.  

Using mixture of models instead of single Gaussian in VAE\nomenclature{VAE}{Variational Autoencoder}
framework is not a new concept. Several works proposed replacing the
Gaussian prior and posterior in VAE by MoGs for clustering and generating
image problems \citet{dilokthanakul2016deep,jiangvariational,nalisnick2016approximate}.
Others applied MoG prior to model transitions between video frames
and caption generation \citet{shu2016stochastic,wang2017diverse}.
These methods use simple feed forward network to produce Gaussian
sub-distributions independently. In our model, on the contrary, memory
slots are strongly correlated with each others, and thus modes in
our MoG work together to define the shape of the latent distributions
at specific timestep. To the best of our knowledge, our work is the
first attempt to use an external memory to induce mixture models  for
sequence generation problems. 

\section{Variational Memory Encoder-Decoder }

Built upon CVAE and partly inspired by VRNN \citet{chung2015recurrent},
we introduce a novel memory-augmented variational recurrent network
dubbed Variational Memory Encoder-Decoder (VMED\nomenclature{VMED}{Variational Memory Encoder-Decoder}).
With an external memory module, VMED explicitly models the dependencies
between latent random variables across subsequent timesteps. However,
unlike the VRNN which uses hidden values of RNN to model the latent
distribution as a Gaussian, our VMED uses read values $r$ from an
external memory $M$ as a Mixture of Gaussians (MoG) to model the
latent space. This choice of MoG\nomenclature{MoG}{Mixture of Gaussians}
also leads to new formulation for the prior $p_{\phi}$ and the posterior
$q_{\theta}$ mentioned in Eq. (\ref{eq:L_CVAE}). The graphical representation
of our model is shown in Fig. \ref{fig:Graphical-Model-of}(b).

\begin{figure}
\begin{centering}
\includegraphics[width=1\linewidth]{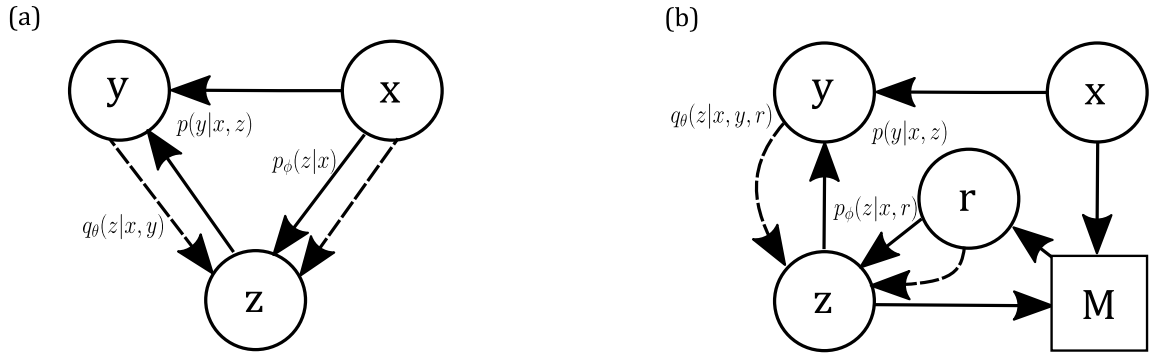}\caption{Graphical Models of the vanilla CVAE (a) and our proposed VMED (b)\label{fig:Graphical-Model-of}}
\par\end{centering}
\selectlanguage{australian}%
\selectlanguage{australian}%
\end{figure}

\subsection{Generative Process}

The VMED includes a CVAE at each time step of the decoder. These CVAEs
are conditioned on the context sequence via $K$ read values $r_{t-1}=\left[r_{t-1}^{1},r_{t-1}^{2},...,r_{t-1}^{K}\right]$
from the external memory. Since the read values are conditioned on
the previous state of the decoder $h_{t-1}^{d}$, our model takes
into account the temporal structure of the output. Unlike other designs
of CVAE where there is often only one CVAE with a Gaussian prior for
the whole decoding process, our model keeps reading the external memory
to produce the prior as a Mixture of Gaussians at every timestep.
At the $t$-th step of generating an utterance in the output sequence,
the decoder will read from the memory $K$ read values, representing
$K$ modes of the MoG. This multi-modal prior reflects the fact that
given a context $x$, there are different modes of uttering the output
word $y_{t}$, which a single mode cannot fully capture. The MoG prior
distribution is modeled as

\begin{equation}
g_{t}=p_{\phi}\left(z_{t}\mid x,r_{t-1}\right)=\stackrel[i=1]{K}{\sum}\pi_{t}^{i,x}\left(x,r_{t-1}^{i}\right)\mathcal{N}\left(z_{t};\mu_{t}^{i,x}\left(x,r_{t-1}^{i}\right),\sigma_{t}^{i,x}\left(x,r_{t-1}^{i}\right)^{2}\mathbf{I}\right)\label{eq:mog}
\end{equation}
We treat the mean $\mu_{t}^{i,x}$ and standard deviation (s.d.) $\sigma_{t}^{i,x}$
of each Gaussian distribution in the prior as neural functions of
the context sequence $x$ and read vectors from the memory. The context
is encoded into the memory by an $LSTM^{E}$ encoder. In decoding,
the decoder $LSTM^{D}$ attends to the memory and choose $K$ read
vectors. We split each read vector into two parts $r^{i,\mu}$ and
$r^{i,\sigma}$ , each of which is used to compute the mean and s.d.,
respectively: $\mu_{t}^{i,x}=r_{t-1}^{i,\mu}$, $\sigma_{t}^{i,x}=softplus\left(r_{t-1}^{i,\sigma}\right)$.
Here we use the softplus function for computing s.d. to ensure the
positiveness. The mode weight $\pi_{t}^{i,x}$ is chosen based on
the read attention weights $w_{t-1}^{i,r}$ over memory slots. Since
we use soft-attention, a read value is computed from all slots yet
the main contribution comes from the one with highest attention score.
Thus, we pick the maximum attention score in each read weight and
\foreignlanguage{australian}{normalise} to become the mode weights:
$\pi_{t}^{i,x}=\max\,w_{t-1}^{i,r}/\stackrel[i=1]{i=K}{\sum}\max\,w_{t-1}^{i,r}$. 

Armed with the prior, we follow a recurrent generative process by
alternatively using the memory to compute the MoG and using latent
variable $z$ sampled from the MoG to update the memory and produce
the output conditional distribution. The pseudo-algorithm of the generative
process is given in Algorithm \ref{alg:Generation-step-of}. 

\begin{algorithm}[t]
\begin{algorithmic}[1]
\small
\Require{Given $p_{\phi}$, $\left[r_{0}^{1},r_{0}^{2},...,r_{0}^{K}\right]$, $h_0^d$, $y_{0}^{*}$}
\For{$t=1,T$}
\State{Sampling $z_{t}\sim p_{\phi}\left(z_{t}\mid x,r_{t-1}\right)$ in Eq.($\ref{eq:mog}$)}
\State{Compute: $o_{t}^{d},h_{t}^{d}=LSTM^{D}\left(\left[y_{t-1}^{*},z_{t}\right],h_{t-1}^{d}\right)$}
\State{Compute the conditional distribution: $p\left(y_{t}\mid x,z_{\leq t}\right)=\softmax\left(W_{out}o_{t}^{d}\right)$}
\State{Update memory and read $[r_{t}^{1},r_{t}^{2},...,r_{t}^{K}]$ using $h_t^d$ as in DNC}
\State{Generate output $y_{t}^{*}=\underset{y\in Vocab}{argmax}\,p\left(y_{t}=y\mid x,z_{\leq t}\right)$}
\EndFor
\end{algorithmic} 

\caption{VMED Generation\label{alg:Generation-step-of}}

\selectlanguage{australian}%
\selectlanguage{australian}%
\end{algorithm}

\subsection{Neural Posterior Approximation}

At each step of the decoder, the true posterior $p\left(z_{t}\mid x,y\right)$
will be approximated by a neural function of $x,y$ and $r_{t-1}$,
denoted as $q_{\theta}\left(z_{t}\mid x,y,r_{t-1}\right)$ . Here,
we use a Gaussian distribution to approximate the posterior. The unimodal
posterior is chosen because given a response $y$, it is reasonable
to assume only one mode of latent space is responsible for this response.
Also, choosing a unimodel will allow the reparameterisation trick
during training and reduce the complexity of $KL$ divergence computation.
The approximated posterior is computed by the following the equation

\begin{equation}
f_{t}=q_{\theta}\left(z_{t}\mid x,y_{\leq t},r_{t-1}\right)=\mathcal{N}\left(z_{t};\mu_{t}^{x,y}\left(x,y_{\leq t},r_{t-1}\right),\sigma_{t}^{x,y}\left(x,y_{\leq t},r_{t-1}\right)^{2}\textrm{\ensuremath{\mathbf{I}}}\right)\label{eq:q_theta}
\end{equation}
 with mean $\mu_{t}^{x,y}$ and s.d. $\sigma_{t}^{x,y}$. We use
an $LSTM^{U}$ utterance encoder to model the ground truth utterance
sequence up to timestep $t$-th $y_{\leq t}$. The $t$-th hidden
value of the $LSTM^{U}$ is used to represent the given data in the
posterior: $h_{t}^{u}=LSTM^{U}\left(y_{t},h_{t-1}^{u}\right)$. The
neural posterior combines the read values $\mathbf{r}_{t}=\stackrel[i=1]{K}{\sum}\pi_{t}^{i,x}r_{t-1}^{i}$
together with the ground truth data to produce the Gaussian posterior:
$\mu_{t}^{x,y}=W_{\mu}\left[\mathbf{r}_{t},h_{t}^{u}\right]$, $\sigma_{t}^{x,y}=softplus\left(W_{\sigma}\left[\mathbf{r}_{t},h_{t}^{u}\right]\right)$.
In these equations, we use learnable matrix weights $W_{\mu}$ and
$W_{\sigma}$ as a recognition network to compute the mean and s.d.
of the posterior, ensuring that the distribution has the same dimension
as the prior. We apply the reparamterisation trick to calculate the
random variable sampled from the posterior as $z'_{t}=\mu_{t}^{x,y}+\sigma_{t}^{x,y}\odot\epsilon$,
$\epsilon\in\mathcal{N}\left(0,\mathbf{I}\right)$. Intuitively, the
reparameterisation trick bridges the gap between the generation model
and the inference model during the training. 

\subsection{Learning}

In the training phase, the neural posterior is used to produce the
latent variable $z'_{t}$. The read values from memory are used directly
as the MoG priors and the priors are trained to approximate the posterior
by reducing the $KL$ divergence. During testing, the decoder uses
the prior for generating latent variable $z_{t}$, from which the
output is computed. The training and testing diagram is illustrated
in Fig. \ref{fig:Training-of-VM3NN}. The objective function becomes
a timestep-wise variational lower bound by following similar derivation
presented in Chung et al., (2015),

\begin{equation}
\mathcal{L}\left(\theta,\phi;y,x\right)=E_{q*}\left[\stackrel[t=1]{T}{\sum}-KL\left(q_{\theta}\left(z_{t}\mid x,y_{\leq t},r_{t-1}\right)\parallel p_{\phi}\left(z_{t}\mid x,r_{t-1}\right)\right)+\log p\left(y_{t}\mid x,z_{\leq t}\right)\right]\label{eq:exloss}
\end{equation}
where $q*=q_{\theta}\left(z_{\leq T}\mid x,y_{\leq T},r_{<T}\right)$.
To maximise the objective function, we have to compute $KL$ divergence
between $f_{t}=q_{\theta}\left(z_{t}\mid x,y_{\leq t},r_{t-1}\right)$
and $g_{t}=p_{\phi}\left(z_{t}\mid x,r_{t-1}\right)$. Since there
is no closed-form for this $KL\left(f_{t}\parallel g_{t}\right)$
between Gaussian $f_{t}$ and Mixture of Gaussians $g_{t}$, we use
a closed-form approximation named $D_{var}$ \citet{hershey2007approximating}
to replace the $KL$ term in the objective function. For our case:
$KL\left(f_{t}\parallel g_{t}\right)\approx D_{var}\left(f_{t}\parallel g_{t}\right)=-\log\stackrel[i=1]{K}{\sum}\pi^{i}e^{-KL\left(f_{t}\parallel g_{t}^{i}\right)}$.
Here, $KL\left(f_{t}\parallel g_{t}^{i}\right)$ is the $KL$ divergence
between two Gaussians and $\pi^{i}$ is the mode weight of $g_{t}$.
The final objective function is

\begin{equation}
\begin{alignedat}{1}\mathcal{L}= & \stackrel[t=1]{T}{\sum}\log\stackrel[i=1]{K}{\sum}\left[\pi_{t}^{i,x}\exp\left(-KL\left(\mathcal{N}\left(\mu_{t}^{x,y},\sigma_{t}^{x,y}{}^{2}\textrm{\ensuremath{\mathbf{I}}}\right)\parallel\mathcal{N}\left(\mu_{t}^{i,x},\sigma_{t}^{i,x}{}^{2}\mathbf{I}\right)\right)\right)\right]\\
 & +\frac{1}{L}\stackrel[t=1]{T}{\sum}\stackrel[l=1]{L}{\sum}\log p\left(y_{t}\mid x,z_{\leq t}^{(l)}\right)
\end{alignedat}
\label{eq:loss}
\end{equation}

\begin{figure}
\begin{centering}
\includegraphics[width=0.9\linewidth]{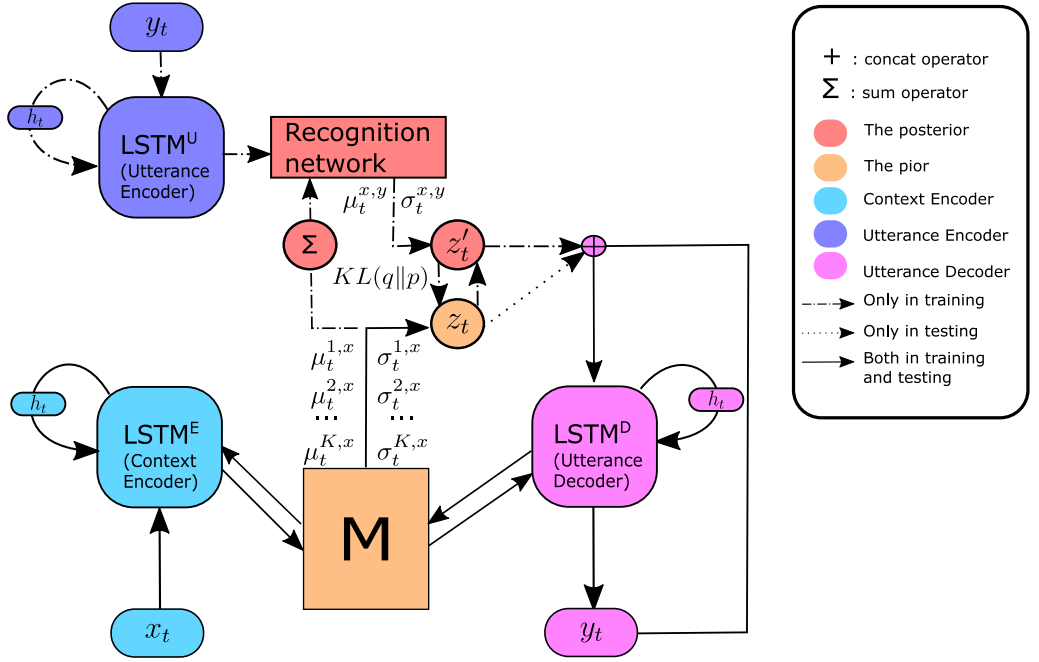}
\par\end{centering}
\caption{Training and testing of VMED\label{fig:Training-of-VM3NN}}
\end{figure}

\subsection{Theoretical Analysis }

We now show that by modeling the prior as MoG and the posterior as
Gaussian, minimising the approximation results in $KL$ divergence
minimisation.
\begin{thm}
The KL divergence between a Gaussian and a Mixture of Gaussians has
an upper bound $D_{var}$.\label{The-KL-divergence}
\end{thm}
A sketch of the proof is derived (see \ref{subsec:Proof-of-theorem-KL}
for full derivation). Let define the log-likelihood $L_{f}\left(g\right)=E_{f\left(x\right)}\left[\log g\left(x\right)\right]$,
we have

\begin{alignat*}{1}
L_{f}\left(g\right)\geq & \log\stackrel[i=1]{K}{\sum}\pi^{i}e^{-KL\left(f\parallel g^{i}\right)}+L_{f}\left(f\right)=-D_{var}+L_{f}\left(f\right)\\
\Rightarrow D_{var}\geq & L_{f}\left(f\right)-L_{f}\left(g\right)=KL\left(f\parallel g\right)
\end{alignat*}
Thus, minimising $D_{var}$ results in $KL$ divergence minimisation.
Next, we establish an upper bound on the total timestep-wise $KL$
divergence in Eq. (\ref{eq:exloss}) and show that minimising this
upper bound is equivalent to fitting a continuous function by a scaled
MoG. The total timestep-wise $KL$ divergence reads

\begin{alignat*}{1}
\stackrel[t=1]{T}{\sum}KL\left(f_{t}\parallel g_{t}\right)= & \stackrel[-\infty]{+\infty}{\int}\stackrel[t=1]{T}{\sum}f_{t}\left(x\right)\log\left[f_{t}\left(x\right)\right]dx\quad-\stackrel[-\infty]{+\infty}{\int}\stackrel[t=1]{T}{\sum}f_{t}\left(x\right)\log\left[g_{t}\left(x\right)\right]dx
\end{alignat*}
where $g_{t}=\stackrel[i=1]{K}{\sum}\pi_{t}^{i}g_{t}^{i}$ and $g_{t}^{i}$
is the $i$-th Gaussian in the MoG at timestep $t$-th. If at each
decoding step, minimising $D_{var}$ results in adequate $KL$ divergence
such that the prior is optimised close to the neural posterior, according
to Chebyshev's sum inequality, we can derive an upper bound on the
total timestep-wise $KL$ divergence as (see  Supplementary Materials
 for full derivation)

\begin{equation}
\stackrel[-\infty]{+\infty}{\int}\stackrel[t=1]{T}{\sum}f_{t}\left(x\right)\log\left[f_{t}\left(x\right)\right]dx\quad-\stackrel[-\infty]{+\infty}{\int}\frac{1}{T}\stackrel[t=1]{T}{\sum}f_{t}\left(x\right)\log\left[\stackrel[t=1]{T}{\prod}g_{t}\left(x\right)\right]dx\label{eq:fupper}
\end{equation}
The left term is sum of the entropies of $f_{t}\left(x\right)$, which
does not depend on the training parameter $\phi$ used to compute
$g_{t}$, so we can ignore that. Thus given $f$, minimising the upper
bound of the total timestep-wise $KL$ divergence is equivalent to
maximising the right term of Eq. (\ref{eq:fupper}). Since $g_{t}$
is an MoG and products of MoG is proportional to an MoG, $\stackrel[t=1]{T}{\prod}g_{t}\left(x\right)$
is a scaled MoG (see  Supplementary material  for full proof). Maximising
the right term is equivalent to fitting function $\stackrel[t=1]{T}{\sum}f_{t}\left(x\right)$,
which is sum of Gaussians and thus continuous, by a scaled MoG. This,
in theory, is possible regardless of the form of $f_{t}$ since MoG
is a universal approximator \citet{bacharoglou2010approximation,maz1996approximate}. 

\section{Experiments and Results}

\subsection{Quantitative Results}

\begin{table}
\begin{centering}
\caption{BLEU-1, 4 and A-Glove on testing datasets. B1, B4, AG are acronyms
for BLEU-1, BLEU-4, A-Glove metrics, respectively (higher is better).
\label{tab:BLEUs-and-A-Glove}}
~
\par\end{centering}
\centering{}{\small{}}%
\begin{tabular}{ccccccccccccc}
\hline 
\multirow{2}{*}{{\footnotesize{}Model}} & \multicolumn{3}{c}{{\footnotesize{}Cornell Movies}} & \multicolumn{3}{c}{{\footnotesize{}OpenSubtitle}} & \multicolumn{3}{c}{{\footnotesize{}LJ users}} & \multicolumn{3}{c}{{\footnotesize{}Reddit comments}}\tabularnewline
\cline{2-13} \cline{3-13} \cline{4-13} \cline{5-13} \cline{6-13} \cline{7-13} \cline{8-13} \cline{9-13} \cline{10-13} \cline{11-13} \cline{12-13} \cline{13-13} 
 & {\scriptsize{}B1} & {\scriptsize{}B4} & {\scriptsize{}AG} & {\scriptsize{}B1} & {\scriptsize{}B4} & {\scriptsize{}AG} & {\scriptsize{}B1} & {\scriptsize{}B4} & {\scriptsize{}AG} & {\scriptsize{}B1} & {\scriptsize{}B4} & {\scriptsize{}AG}\tabularnewline
\hline 
{\footnotesize{}Seq2Seq} & {\scriptsize{}18.4} & {\scriptsize{}9.5} & {\scriptsize{}0.52} & {\scriptsize{}11.4} & {\scriptsize{}5.4} & {\scriptsize{}0.29} & {\scriptsize{}13.1} & {\scriptsize{}6.4} & {\scriptsize{}0.45} & {\scriptsize{}7.5} & {\scriptsize{}3.3} & {\scriptsize{}0.31}\tabularnewline
{\footnotesize{}Seq2Seq-att} & {\scriptsize{}17.7} & {\scriptsize{}9.2} & {\scriptsize{}0.54} & {\scriptsize{}13.2} & {\scriptsize{}6.5} & {\scriptsize{}0.42} & {\scriptsize{}11.4} & {\scriptsize{}5.6} & {\scriptsize{}0.49} & {\scriptsize{}5.5} & {\scriptsize{}2.4} & {\scriptsize{}0.25}\tabularnewline
{\footnotesize{}DNC} & {\scriptsize{}17.6} & {\scriptsize{}9.0} & {\scriptsize{}0.51} & {\scriptsize{}14.3} & {\scriptsize{}7.2} & {\scriptsize{}0.47} & {\scriptsize{}12.4} & {\scriptsize{}6.1} & {\scriptsize{}0.47} & {\scriptsize{}7.5} & {\scriptsize{}3.4} & {\scriptsize{}0.28}\tabularnewline
{\footnotesize{}CVAE} & {\scriptsize{}16.5} & {\scriptsize{}8.5} & {\scriptsize{}0.56} & {\scriptsize{}13.5} & {\scriptsize{}6.6} & {\scriptsize{}0.45} & {\scriptsize{}12.2} & {\scriptsize{}6.0} & {\scriptsize{}0.48} & {\scriptsize{}5.3} & {\scriptsize{}2.8} & {\scriptsize{}0.39}\tabularnewline
{\footnotesize{}VLSTM} & {\scriptsize{}18.6} & {\scriptsize{}9.7} & {\scriptsize{}0.59} & {\scriptsize{}16.4} & {\scriptsize{}8.1} & {\scriptsize{}0.43} & {\scriptsize{}11.5} & {\scriptsize{}5.6} & {\scriptsize{}0.46} & {\scriptsize{}6.9} & {\scriptsize{}3.1} & {\scriptsize{}0.27}\tabularnewline
\hline 
{\footnotesize{}VMED (K=1)} & {\scriptsize{}20.7} & {\scriptsize{}10.8} & {\scriptsize{}0.57} & {\scriptsize{}12.9} & {\scriptsize{}6.2} & {\scriptsize{}0.44} & {\scriptsize{}13.7} & {\scriptsize{}6.9} & {\scriptsize{}0.47} & {\scriptsize{}9.1} & {\scriptsize{}4.3} & {\scriptsize{}0.39}\tabularnewline
{\footnotesize{}VMED (K=2)} & {\scriptsize{}22.3} & {\scriptsize{}11.9} & \textbf{\scriptsize{}0.64} & {\scriptsize{}15.3} & {\scriptsize{}8.8} & {\scriptsize{}0.49} & {\scriptsize{}15.4} & {\scriptsize{}7.9} & \textbf{\scriptsize{}0.51} & {\scriptsize{}9.2} & {\scriptsize{}4.4} & {\scriptsize{}0.38}\tabularnewline
{\footnotesize{}VMED (K=3)} & {\scriptsize{}19.4} & {\scriptsize{}10.4} & {\scriptsize{}0.63} & \textbf{\scriptsize{}24.8} & \textbf{\scriptsize{}12.9} & \textbf{\scriptsize{}0.54} & \textbf{\scriptsize{}18.1} & \textbf{\scriptsize{}9.8} & {\scriptsize{}0.49} & \textbf{\scriptsize{}12.3} & \textbf{\scriptsize{}6.4} & \textbf{\scriptsize{}0.46}\tabularnewline
{\footnotesize{}VMED (K=4)} & \textbf{\scriptsize{}23.1} & \textbf{\scriptsize{}12.3} & {\scriptsize{}0.61} & {\scriptsize{}17.9} & {\scriptsize{}9.3} & {\scriptsize{}0.52} & {\scriptsize{}14.4} & {\scriptsize{}7.5} & {\scriptsize{}0.47} & {\scriptsize{}8.6} & {\scriptsize{}4.6} & {\scriptsize{}0.41}\tabularnewline
\hline 
\end{tabular}{\small\par}
\end{table}

\textbf{Datasets and pre-processing}: We perform experiments on two
collections: The first collection includes open-domain movie transcript
datasets containing casual conversations: Cornell Movies\footnote{\url{http://www.cs.cornell.edu/~cristian/Cornell_Movie-Dialogs_Corpus.html}}
and OpenSubtitle\footnote{\url{http://opus.nlpl.eu/OpenSubtitles.php}}.
They have been used commonly in evaluating conversational agents \citet{lison2017not,vinyals2015neural}.
The second are closed-domain datasets crawled from specific domains,
which are question-answering of LiveJournal (LJ) users and Reddit
comments on movie topics. For each dataset, we use 10,000 conversations
for validating and 10,000 for testing. 

\textbf{Baselines, implementations and metrics}: We compare our model
with three deterministic baselines: the encoder-decoder neural conversational
model (Seq2Seq) \citet{vinyals2015neural} and its two variants equipped
with attention mechanism \citet{cho2014properties,Bahdanau2015a}
(Seq2Seq-att) and a DNC external memory \citet{graves2016hybrid}
(DNC). The vanilla CVAE is also included in the baselines. To build
this CVAE, we follow similar architecture introduced in \citet{zhao2017learning}
without bag-of-word loss and dialog act features\footnote{Another variant of non-memory CVAE with MoG prior is also examined.
We produce a set of MoG parameters by a feed forward network with
the input as the last encoder hidden states. However, the model is
hard to train and fails to converge with these datasets.}. A variational recurrent model without memory is also included in
the baselines. The model termed VLSTM is implemented based on LSTM
instead of RNN as in VRNN framework \citet{chung2015recurrent}. We
try our model VMED\footnote{Source code is available at \url{https://github.com/thaihungle/VMED}}
with different number of modes ($K=1,2,3,4$). It should be noted
that, when $K=1$, our model's prior is exactly a Gaussian and the
$KL$ term in Eq. (\ref{eq:nsm_loss}) is no more an approximation.
Details of dataset descriptions and model implementations are included
in  Supplementary material.

We report results using two performance metrics in order to evaluate
the system from various linguistic points of view: (i) Smoothed Sentence-level
BLEU \citet{chen2014systematic}: BLEU is a popular metric that measures
the geometric mean of modified ngram precision with a length penalty.
We use BLEU-1 to 4 as our lexical similarity. (ii) Cosine Similarly
of Sentence Embedding: a simple method to obtain sentence embedding
is to take the average of all the word embeddings in the sentences
\citet{forgues2014bootstrapping}. We follow Zhao et al., (2017) and
choose Glove \citet{levy2014neural} as the word embedding in measuring
sentence similarly (A-Glove) \citet{zhao2017learning}. To measure
stochastic models, for each input, we generate output ten times. The
metric between the ground truth and the generated output is calculated
and taken average over ten responses. 

\textbf{Metric-based Analysi}s: We report results on four test datasets
in Table \ref{tab:BLEUs-and-A-Glove}. For BLEU scores, here we only
list results for BLEU-1 and 4. Other BLEUs show similar pattern and
will be listed in  Supplementary material. As clearly seen, VMED models
outperform other baselines over all metrics across four datasets.
In general, the performance of Seq2Seq is comparable with other deterministic
methods despite its simplicity. Surprisingly, CVAE or VLSTM does not
show much advantage over deterministic models. As we shall see, although
CVAE and VLSTM responses are diverse, they are often out of context.
Amongst different modes of VMED, there is often one best fit with
the data and thus shows superior performance. The optimal number of
modes in our experiments often falls to $K=3$, indicating that increasing
modes does not mean to improve accuracy. 

It should be noted that there is inconsistency between BLEU scores
and A-Glove metrics. This is because BLEU measures lexicon matching
while A-Glove evaluates semantic similarly in the embedding space.
For example, two sentences having different words may share the same
meaning and lie close in the embedding space. In either case, compared
to others, our optimal VMED always achieves better performance.

\subsection{Qualitative Analysis}

\begin{table}
\begin{centering}
\caption{Examples of context-response pairs. /{*}/ denotes separations between
stochastic responses.\label{tab:Examples-of-context-response}}
~
\par\end{centering}
\centering{}%
\begin{tabular}{>{\raggedright}p{0.2\columnwidth}|>{\raggedright}p{0.73\columnwidth}}
\hline 
{\footnotesize{}Input context} & {\footnotesize{}Response}\tabularnewline
\hline 
\multirow{6}{0.2\columnwidth}{\textbf{\footnotesize{}Reddit comment:}{\footnotesize{} What is your
favorite scene in film history ? Mine is the restaurant scene in the
Godfather. }} & \multirow{1}{0.73\columnwidth}{\textbf{\footnotesize{}Seq2Seq:}{\footnotesize{} The scene in}}\tabularnewline
 & \textbf{\footnotesize{}Seq2Seq-att: }{\footnotesize{}The final}\tabularnewline
 & \textbf{\footnotesize{}DNC: }{\footnotesize{}The scene in}\tabularnewline
 & \textbf{\footnotesize{}CVAE:}{\footnotesize{} Inception god! Not by
a shark /{*}/ Amour great /{*}/ Pro thing you know 3 dead}\tabularnewline
 & \textbf{\footnotesize{}VLSTM: }{\footnotesize{}The scene in /{*}/
The of a dead /{*}/ The sky in scene}\tabularnewline
 & \textbf{\footnotesize{}VMED (K=3): }{\footnotesize{}The opening scene
from history movie /{*}/ The scene in a shot nights! Robin movie /{*}/
The psycho scene in fight from }\tabularnewline
\hline 
\multirow{6}{0.2\columnwidth}{\textbf{\footnotesize{}Reddit comment:}{\footnotesize{} What actor
will win an Oscar in the next 10 years ?}} & \textbf{\footnotesize{}Seq2Seq: }{\footnotesize{}Colin}\tabularnewline
 & \textbf{\footnotesize{}Seq2Seq-att: }{\footnotesize{}Liam Neeson }\tabularnewline
 & \textbf{\footnotesize{}DNC:}{\footnotesize{} Tom Gyllenhaal}\tabularnewline
 & \textbf{\footnotesize{}CVAE:}{\footnotesize{} Daryl and Aaron /{*}/
Carefully count Alfred Deniro /{*}/ Ponyo Joker possible}\tabularnewline
 & \textbf{\footnotesize{}VLSTM:}{\footnotesize{} Michael Bullock /{*}/
Michael /{*}/ Michael De}\tabularnewline
 & \textbf{\footnotesize{}VMED (K=3)}{\footnotesize{}: Edward or Leo
Dicaprio goes on /{*}/ Dicaprio will /{*}/ Dicaprio Tom has actually
in jack on road }\tabularnewline
\hline 
\multirow{6}{0.2\columnwidth}{\textbf{\footnotesize{}LJ user:}{\footnotesize{} Your very pretty
and your is very cute. May i add ya to my friends list ? }} & \textbf{\footnotesize{}Seq2Seq:}{\footnotesize{} I! Add and too you
back}\tabularnewline
 & \textbf{\footnotesize{}Seq2Seq-att: }{\footnotesize{}I did may! Away
back }\tabularnewline
 & \textbf{\footnotesize{}DNC: }{\footnotesize{}Sure you added}\tabularnewline
 & \textbf{\footnotesize{}CVAE:}{\footnotesize{} Sure ! I can add you
/{*}/ Thank cool of surely /{*}/ Thank hi of sure!}\tabularnewline
 & \textbf{\footnotesize{}VLSTM:}{\footnotesize{} Haha of do go /{*}/
Haha of sure! Yes lol /{*}/ Haha sure! Go}\tabularnewline
 & \textbf{\footnotesize{}VMED (K=3):}{\footnotesize{} Thank cool for
sure /{*}/ Add definitely! Sure surely friends /{*}/ Thank girl go
ahead}\tabularnewline
\hline 
\multirow{6}{0.2\columnwidth}{\textbf{\footnotesize{}Cornell movies: }{\footnotesize{}Nobody can
murder someone with their feelings . I mean that is crazy ?}} & \textbf{\footnotesize{}Seq2Seq:}{\footnotesize{} Yes you are}\tabularnewline
 & \textbf{\footnotesize{}Seq2Seq-att: }{\footnotesize{}Really not is
it}\tabularnewline
 & \textbf{\footnotesize{}DNC: }{\footnotesize{}Managing the truth}\tabularnewline
 & \textbf{\footnotesize{}CVAE: }{\footnotesize{}Sure not to be in here
/{*}/ Oh yes but i know! /{*}/ That to doing with here and}\tabularnewline
 & \textbf{\footnotesize{}VLSTM:}{\footnotesize{} I am dead . ! That
is ... /{*}/ I did not what way . /{*}/ I am not . But his things
with ...}\tabularnewline
 & \textbf{\footnotesize{}VMED (K=4):}{\footnotesize{} You are right!
/{*}/ That is crazy /{*}/ You can't know Jimmy }\tabularnewline
\hline 
\multirow{6}{0.2\columnwidth}{\textbf{\footnotesize{}OpenSubtitle:}{\footnotesize{} I'm obliged
for your hospitality. I appreciate it and your husband too.}} & \textbf{\footnotesize{}Seq2Seq: }{\footnotesize{}That is have got
coming about these }\tabularnewline
 & \textbf{\footnotesize{}Seq2Seq-att:}{\footnotesize{} May you not what
nothing nobody}\tabularnewline
 & \textbf{\footnotesize{}DNC: }{\footnotesize{}Yes i am taking this}\tabularnewline
 & \textbf{\footnotesize{}CVAE:}{\footnotesize{} No . that for good!
And okay /{*}/ All in the of two thing /{*}/ Sure. Is this! }\tabularnewline
 & \textbf{\footnotesize{}VLSTM:}{\footnotesize{} I ... /{*}/ I understand
/{*}/ I ! . ...}\tabularnewline
 & \textbf{\footnotesize{}VMED (K=3):}{\footnotesize{} I know. I can
afford /{*}/ I know nothing to store for you pass /{*}/ I know. Doing
anymore you father}\tabularnewline
\hline 
\end{tabular}
\end{table}

Table \ref{tab:Examples-of-context-response} represents responses
generated by experimental models in reply to different input sentences.
The replies listed are chosen randomly from 50 generated responses
whose average of metric scores over all models are highest. For stochastic
models, we generate three times for each input, resulting in three
different responses. In general, the stochastic models often yield
longer and diverse sequences as expected. For closed-domain cases,
all models responses are fairly acceptable. Compared to the rest,
our VMED's responds seem to relate more to the context and contain
meaningful information. In this experiment, the open-domain input
seems nosier and harder than the closed-domain ones, thus create a
big challenge for all models. Despite that, the quality of VMED's
responses is superior to others. Amongst deterministic models, DNC's
generated responses look more reasonable than Seq2Seq's even though
its BLEU scores are not always higher. Perhaps, the reference to external
memory at every timestep enhances the coherence between output and
input, making the response more related to the context. VMED may inherit
this feature from its external memory and thus tends to produce reasonable
responses. By contrast, although responses from CVAE and VLSTM are
not trivial, they have more grammatical errors and sometimes unrelated
to the topic. 

\section{Closing Remarks}

In this chapter, we propose a novel approach to sequence generation
called Variational Memory Encoder-Decoder (VMED) that introduces variability
into encoder-decoder architecture via the use of external memory as
mixture model. By modeling the latent temporal dependencies across
timesteps, our VMED produces a MoG representing the latent distribution.
Each mode of the MoG associates with some memory slot and thus captures
some aspect of context supporting generation process. To accommodate
the MoG, we employ a $KL$\nomenclature{KL}{Kullback\textendash Leibler}
approximation and we demonstrate that minimising this approximation
is equivalent to minimising the $KL$ divergence. We derive an upper
bound on our total timestep-wise $KL$ divergence and indicate that
the optimisation of this upper bound is equivalent to fitting a continuous
function by an scaled MoG, which is in theory possible regardless
of the function form. This forms a theoretical basis for our model
formulation using MoG prior for every step of generation. We apply
our proposed model to conversation generation problem. The results
demonstrate that VMED outperforms recent advances both quantitatively
and qualitatively. 

So far we have designed and applied MANNs to long-term sequential
modeling problems without considering the memorisation ability of
these models. Besides, we have followed common memory access patterns
without judging the effectiveness of these accesses. We will tackle
these issues in the next chapter.

\chapter{Optimal Writing Memory \label{chap:Optimal-Writing-in}}

\section{Introduction }

Recall that a core task in sequence learning is to capture long-term
dependencies amongst timesteps which demands memorisation of distant
inputs (Sec. \ref{sec:External-Memory-for}). In recurrent neural
networks (RNNs), the memorisation is implicitly executed via integrating
the input history into the state of the networks. However, learning
vanilla RNNs over long distance proves to be difficult due to the
vanishing gradient problem (Sec. \ref{subsec:Cell-memory}). One alleviation
is to introduce skip-connections along the execution path, in the
form of dilated layers \citet{van2016wavenet,chang2017dilated}, which
is not our main focus. Other solutions have been mentioned in Chapter
\ref{chap:MANN} including attention mechanisms \citet{cho2014properties,vaswani2017attention,Bahdanau2015a}
and external slot-based memory \citet{graves2014neural,graves2016hybrid}. 

Amongst all, using external memory most resembles human cognitive
architecture where we perceive the world sequentially and make decision
by consulting our memory. Recent attempts have simulated this process
by using RAM-like memory architectures that store information into
memory slots. Reading and writing are governed by neural controllers
using attention mechanisms (Sec. \ref{sec:Slot-based-Memory-Networks}). 

Despite the promising empirical results, there is no theoretical analysis
or clear understanding on optimal operations that a memory should
have to maximise its performance. To the best of our knowledge, no
solution has been proposed to help MANNs handle ultra-long sequences
given limited memory. This scenario is practical because (i) sequences
in the real-world can be very long while the computer resources are
limited and (ii) it reflects the ability to compress in human brain
to perform life-long learning. Previous attempts \citet{rae2016scaling}
try to learn ultra-long sequences by expanding the memory, which is
not always feasible and do not aim to optimise the memory by some
theoretical criterion. This chapter presents a new approach towards
finding optimal operations for MANNs that serve the purpose of learning
longer sequences with finite memory. 

More specifically, upon analysing RNN and MANN operations we first
introduce a measurement on the amount of information that a MANN holds
after encoding a sequence. This metric reflects the quality of memorisation
under the assumption that contributions from timesteps are equally
important. We then derive a generic solution to optimise the measurement.
We term this optimal solution as Uniform Writing (UW\nomenclature{UW}{Uniform Writing}),
and it is applicable for any MANN due to its generality. Crucially,
UW helps reduce significantly the computation time of MANN. Third,
to relax the assumption and enable the method to work in realistic
settings, we further propose Cached Uniform Writing (CUW\nomenclature{CUW}{Cached Uniform Writing})
as an improvement over the Uniform Writing scheme. By combining uniform
writing with local attention, CUW can learn to discriminate timesteps
while maximising local memorisation. Finally we demonstrate that our
proposed models outperform several MANNs and other state-of-the-art
methods in various synthetic and practical sequence modeling tasks. 

\section{Related Backgrounds}

Traditional recurrent models such as RNN/LSTM \citet{elman1990finding,hochreiter1997long}
exhibit some weakness that prevents them from learning really long
sequences. The reason is mainly due to the vanishing gradient problem
or to be more specific, the exponential decay of input value over
time, which is analysed in Sec. \ref{subsec:Cell-memory}. It should
be noted that although LSTM is designed to diminish the problem, it
is limited to the capacity of cell memory. In other words, if the
amount of relevant information exceeds the capacity, LSTM will forget
distant inputs since exponential decay still exists outside the cell
memory. One different approach to overcome this problem is enforcing
the exponential decay factor close to one by putting a unitary constraint
on the recurrent weight \citet{arjovsky2016unitary,wisdom2016full}.
Although this approach is theoretically motivated, it restricts the
space of learnt parameters. 

More relevant to our work, the idea of using less or adaptive computation
has been proposed by many \citet{graves2016adaptive,yu2017learning,yu2018fast,seo2018neural}.
Most of these works are based on the assumption that some of timesteps
in a sequence are unimportant and thus can be ignored to reduce the
cost of computation and increase the performance of recurrent networks.
Different form our approach, these methods lack theoretical supports
and do not directly aim to solve the problem of memorising long-term
dependencies. 

Dilated RNN \citet{chang2017dilated} is another RNN-based proposal
which improves long-term learning by stacking multiple dilated recurrent
layers with hierarchical skip-connections. This theoretically guarantees
the mean recurrent length and shares with our method the idea to construct
a measurement on memorisation capacity of the system and propose solutions
to \foreignlanguage{australian}{optimise} it. The difference is that
our system is memory-augmented neural networks while theirs is multi-layer
RNNs, which leads to totally different \foreignlanguage{australian}{optimisation}
problems. 

Recent researches recommend to replace traditional recurrent models
by other neural architectures to overcome the vanishing gradient problem.
The Transformer \citet{vaswani2017attention} attends to all timesteps
at once, which ensures instant access to distant timestep yet requires
quadratic computation and physical memory proportional to the sequence
length (see Sec. \ref{subsec:Multi-head-Attention}). Memory-augmented
neural networks (MANNs), on the other hand, learn to establish a limited-size
memory and attend to the memory only, which is scalable to any-length
sequence (see Sec. \ref{sec:Slot-based-Memory-Networks}). Compared
to others, MANNs resemble both computer architecture design and human
working memory \citet{logie2014visuo}. However, the current understanding
of the underlying mechanisms and theoretical foundations for MANNs
are still limited.

Recent works on MANNs rely almost on reasonable intuitions. Some introduce
new addressing mechanisms such as location-based \citet{graves2014neural},
least-used \citet{santoro2016meta} and order-based \citet{graves2016hybrid}.
Others focus on the scalability of MANN by using sparse memory access
to avoid attending to a large number of memory slots \citet{rae2016scaling}.
These problems are different from ours which involves MANN memorisation
capacity \foreignlanguage{australian}{optimisation}. 

Our local optimal solution to this problem is related to some known
neural caching \citet{grave2017improving,grave2017unbounded,yogatama2018memory}
in terms of storing recent hidden states for later encoding uses.
These methods either aim to create structural bias to ease the learning
process \citet{yogatama2018memory} or support large scale retrieval
\citet{grave2017unbounded}. These are different from our caching
purpose, which encourages overwriting and relaxes the equal contribution
assumption of the optimal solution. Also, the details of implementation
are different as ours uses local memory-augmented attention mechanisms. 

\section{Theoretical Analysis on Memorisation }

\subsection{Generic Memory Operations}

Memory-augmented neural networks can be viewed as an extension of
RNNs with external memory $M$. The memory supports read and write
operations based on the output $o_{t}$ of the controller, which in
turn is a function of current timestep input $x_{t}$, previous hidden
state $h_{t-1}$ and read value $r_{t-1}$ from the memory. Let assume
we are given these operators from recent MANNs such as NTM \citet{graves2014neural}
or DNC \citet{graves2016hybrid}, represented as

\begin{equation}
r_{t}=f_{r}\left(o_{t},M_{t-1}\right)\label{eq:read_R}
\end{equation}
\begin{equation}
M_{t}=f_{w}\left(o_{t},M_{t-1}\right)\label{eq:write_v}
\end{equation}

The controller output and hidden state are updated as follows,

\begin{equation}
o_{t}=f_{o}\left(h_{t-1},r_{t-1},x_{t}\right)\label{eq:update_o}
\end{equation}
\begin{equation}
h_{t}=f_{h}\left(h_{t-1},r_{t-1},x_{t}\right)\label{eq:control_h}
\end{equation}

Here, $f_{o}$ and $f_{h}$ are often implemented as RNNs while $f_{r}$
and $f_{w}$ are designed specifically for different memory types. 

Current MANNs only support regular writing by applying Eq. (\ref{eq:write_v})
every timestep. In effect, regular writing ignores the accumulated
short-term memory stored in the controller hidden states which may
well-capture the recent subsequence. We argue that the controller
does not need to write to memory continuously as its hidden state
also supports memorising. Another problem of regular writing is time
complexity. As the memory access is very expensive, reading/writing
at every timestep makes MANNs much slower than RNNs. This motivates
a irregular writing strategy to utilise the memorisation capacity
of the controller and consequently, speed up the model. In the next
sections, we first define a metric to measure the memorisation performance
of RNNs, as well as MANNs. Then, we solve the problem of finding the
best irregular writing that optimises the metric. 

\subsection{Memory Analysis of RNNs}

We first define the ability to ``remember'' of recurrent neural
networks, which is closely related to the vanishing/exploding gradient
problem \citet{pascanu2013difficulty}. In RNNs, the state transition
$h_{t}=\phi\left(h_{t-1},x_{t}\right)$ contains contributions from
not only $x_{t}$, but also previous timesteps $x_{i<t}$ embedded
in $h_{t-1}$. Thus, $h_{t}$ can be considered as a function of timestep
inputs, i.e, $h_{t}=f\left(x_{1},x_{2},...,x_{t}\right)$. One way
to measure how much an input $x_{i}$ contributes to the value of
$h_{t}$ is to calculate the norm of the gradient $\left\Vert \frac{\partial h_{t}}{\partial x_{i}}\right\Vert $.
If the norm equals zero, $h_{t}$ is constant w.r.t $x_{i}$, that
is, $h_{t}$ does not ``remember'' $x_{i}$. As a bigger $\left\Vert \frac{\partial h_{t}}{\partial x_{i}}\right\Vert $
implies more influence of $x_{i}$ on $h_{t}$, we propose using $\left\Vert \frac{\partial h_{t}}{\partial x_{i}}\right\Vert $
to measure the contribution of the $i$-th input to the $t$-th hidden
state. Let $c_{i,t}$ denotes this term, we can show that in the case
of common RNNs, $\lambda_{c}c_{i,t}\geq c_{i-1,t}$ with some $\lambda_{c}\in\mathbb{R^{+}}$(see
Appendix \ref{subsec:Derivation-on-theRNN} - \ref{subsec:Derivation-on-theLSTM}
for proof). This means further to the past, the contribution decays
(when $\lambda_{c}<1$) or grows (when $\lambda_{c}>1$) with the
rate of at least $\lambda_{c}$ .We can measure the average amount
of contributions across $T$ timesteps as follows (see Appendix \ref{subsec:Proof-of-theorem-1}
for proof):
\begin{thm}
\label{thm:The-average-amount} There exists $\lambda\in\mathbb{R^{+}}$such
that the average contribution of a sequence of length T with respect
to a RNN can be quantified as the following,

\begin{equation}
I_{\lambda}=\frac{\stackrel[t=1]{T}{\sum}c_{t,T}}{T}=c_{T,T}\frac{\stackrel[t=1]{T}{\sum}\lambda^{T-t}}{T}\label{eq:d0}
\end{equation}
\end{thm}
If $\lambda<1$, $\lambda^{T-t}\rightarrow0$ as $T-t\rightarrow\infty$.
This is closely related to vanishing gradient problem. LSTM is known
to ``remember'' long sequences better than RNN by using extra memory
gating mechanisms, which help $\lambda$ to get closer to 1. If $\lambda>1$,
the system may be unstable and suffer from the exploding gradient
problem. 

\begin{figure}
\begin{centering}
\includegraphics[width=0.9\textwidth]{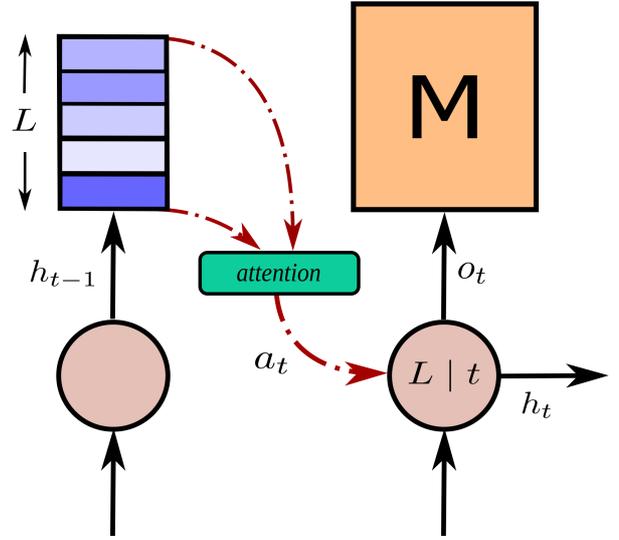}
\par\end{centering}
\caption{Writing mechanism in Cached Uniform Writing. During non-writing intervals,
the controller hidden states are pushed into the cache. When the writing
time comes, the controller attends to the cache, chooses suitable
states and accesses the memory. The cache is then emptied. \label{fig:Writing-mechanism-in}}
\end{figure}

\subsection{Memory Analysis of MANNs}

In slot-based MANNs, memory $M$ is a set of $D$ memory slots. A
write at step $t$ can be represented by the controller's hidden state
$h_{t}$, which accumulates inputs over several timesteps (i.e., $x_{1}$,
...,$x_{t}$). If another write happens at step $t+k$, the state
$h_{t+k}$'s information containing timesteps $x_{t+1}$, ...,$x_{t+k}$
is stored in the memory ($h_{t+k}$ may involves timesteps further
to the past, yet they are already stored in the previous write and
can be ignored). During writing, overwriting may happen, replacing
an old write with a new one. Thus after all, $D$ memory slots associate
with $D$ chosen writes of the controller. From these observations,
we can generalise Theorem \ref{thm:The-average-amount} to the case
of MANNs having $D$ memory slots (see Appendix \ref{subsec:Proof-of-theorem-2}
for proof).
\begin{thm}
\label{thm:Under-the-assumption}With any $D$ chosen writes at timesteps
$1\leq$$K_{1}$\textup{$<$} $K_{2}$$<$ ...$<$$K_{D}$$<T$, there
exist $\lambda,C\in\mathbb{R^{+}}$such that the lower bound on the
average contribution of a sequence of length $T$ with respect to
a MANN having $D$ memory slots can be quantified as the following,

\begin{align}
I_{\lambda} & =C\frac{\stackrel[t=1]{K_{1}}{\sum}\lambda^{K_{1}-t}+\stackrel[t=K_{1}+1]{K_{2}}{\sum}\lambda^{K_{2}-t}+...+\stackrel[t=K_{D-1}+1]{K_{D}}{\sum}\lambda^{K_{D}-t}+\stackrel[t=K_{D}+1]{T}{\sum}\lambda^{T-t}}{T}\nonumber \\
 & =\frac{C}{T}\stackrel[i=1]{D+1}{\sum}\stackrel[j=0]{l_{i}-1}{\sum}\lambda^{j}=\frac{C}{T}\stackrel[i=1]{D+1}{\sum}f_{\lambda}(l_{i})
\end{align}
where $l_{i}=\begin{cases}
K_{1} & ;i=1\\
K_{i}-K_{i-1} & ;D\geq i>1\\
T-K_{D} & ;i=D+1
\end{cases}$, $f_{\lambda}\left(x\right)=\begin{cases}
\frac{1-\lambda^{x}}{1-\lambda} & \lambda\neq1\\
x & \lambda=1
\end{cases}$, $\forall x\in\mathbb{R}^{+}$.
\end{thm}
If $\lambda\leq1$, we want to maximise $I_{\lambda}$ to keep the
information from vanishing. On the contrary, if $\lambda>1$, we may
want to minimise $I_{\lambda}$ to prevent the information explosion.
As both scenarios share the same solution (see Appendix \ref{subsec:Proof-of-theorem}),
thereafter we assume that $\lambda\leq1$ holds for other analyses.
By taking average over $T$, we are making an assumption that all
timesteps are equally important. This helps simplify the measurement
as $I_{\lambda}$ is independent of the specific position of writing.
Rather, it is a function of the interval lengths between the writes.
This turns out to be an optimisation problem whose solution is stated
in the following theorem.
\begin{thm}
\label{thm:Given-the-number}Given $D$ memory slots, a sequence with
length $T$, a decay rate $0<\lambda\leq1$, then the optimal intervals
$\left\{ l_{i}\in\mathbb{R^{+}}\right\} _{i=1}^{D+1}$ satisfying
$T=\stackrel[i=1]{D+1}{\sum}l_{i}$ such that the lower bound on the
average contribution $I_{\lambda}=\frac{C}{T}\stackrel[i=1]{D+1}{\sum}f_{\lambda}(l_{i})$
is maximised are

\begin{equation}
l_{1}=l_{2}=...=l_{D+1}=\frac{T}{D+1}
\end{equation}
\end{thm}
We name the optimal solution as Uniform Writing (UW) and refer to
the term $\frac{T}{D+1}$ and $\frac{D+1}{T}$ as the \textit{optimal
interval} and the\textit{ compression ratio}, respectively. The proof
is given in Appendix \ref{subsec:Proof-of-theorem}.

\begin{algorithm}[t]
\begin{algorithmic}[1]
\Require{a sequence $x=\left\{ x_{t}\right\} _{t=1}^{T}$, a cache $C$ sized $L$, a memory sized $D$}.
\For{$t=1,T$}
\State{$C$.append($h_{t-1}$)}
\If{$t \bmod L==0$}
\State{Use Eq.($\ref{eq:att}$) to calculate $a_t$}
\State{Execute Eq.($\ref{eq:update_o}$): $o_{t}=f_{o}\left(a_t,r_{t-1},x_{t}\right)$}
\State{Execute Eq.($\ref{eq:control_h}$): $h_{t}=f_{h}\left(a_t,r_{t-1},x_{t}\right)$}
\State{Update the memory using Eq.($\ref{eq:write_v}$)}
\State{Read $r_t$ from the memory using Eq.($\ref{eq:read_R}$)}
\State{$C$.clear()}
\Else
\State{Update the controller using Eq.($\ref{eq:control_h}$): $h_{t}=f_{h}\left(h_{t-1},r_{t-1},x_{t}\right)$}
\State{Assign $r_t=r_{t-1}$}
\EndIf
\EndFor
\end{algorithmic} 

\caption{Cached Uniform Writing\label{alg:Cached-Uniform-Writing}}
\end{algorithm}

\section{Optimal Writing for Slot-based Memory Models}

\subsection{Uniform Writing}

Uniform writing can apply to any MANNs that support writing operations.
Since the writing intervals are discrete, i.e., $l_{i}\in\mathbb{N^{+}}$,
UW is implemented as the following,

\begin{equation}
M_{t}=\begin{cases}
f_{w}\left(o_{t},M_{t-1}\right) & \mathrm{if}\:t=\left\lfloor \frac{T}{D+1}\right\rfloor k,k\in\mathbb{N^{+}}\\
M_{t-1} & \mathrm{otherwise}
\end{cases}\label{eq:uw}
\end{equation}

By following Eq. (\ref{eq:uw}), the write intervals are close to
the optimal interval defined in Theorem \ref{thm:Given-the-number}
and approximately maximise the average contribution. This writing
policy works well if timesteps are equally important and the task
is to remember all of them to produce outputs (i.e., in copy task).
However, in reality, timesteps are not created equal and a good model
may need to ignore unimportant or noisy timesteps. That is why overwriting
in MANN can be necessary. In the next section, we propose a method
that tries to balance between following the optimal strategy and employing
overwriting mechanism as in current MANNs. 

\subsection{Local Optimal Design}

To relax the assumptions of Theorem \ref{thm:Given-the-number}, we
propose two improvements of the\textit{ }Uniform Writing (UW) strategy.
First, the intervals between writes are equal with length $L$ ($1\leq L\leq\left\lfloor \frac{T}{D+1}\right\rfloor $).
If $L=1$, the strategy becomes regular writing and if $L=\left\lfloor \frac{T}{D+1}\right\rfloor $,
it becomes uniform writing. This ensures that after $\left\lfloor \frac{T}{L}\right\rfloor $
writes, all memory slots should be filled and the model has to learn
to overwrite. Meanwhile, the average kept information is still locally
maximised every $L*D$ timesteps.

Second, we introduce a cache of size $L$ to store the hidden states
of the controller during a write interval. Instead of using the hidden
state at the writing timestep to update the memory, we perform an
attention over the cache to choose the best representative hidden
state. The model will learn to assign attention weights to the elements
in the cache. This mechanism helps the model consider the importance
of each timestep input in the local interval and thus relax the equal
contribution assumption of Theorem \ref{thm:Given-the-number}. We
name the writing strategy that uses the two mentioned-above improvements
as Cached Uniform Writing (CUW). An illustration of the writing mechanism
is depicted in Fig. \ref{fig:Writing-mechanism-in}.

\subsection{Local Memory-Augmented Attention Unit}

In this subsection, we provide details of the attention mechanism
used in our CUW. To be specific, the best representative hidden state
$a_{t}$ is computed as follows,

\begin{equation}
\alpha_{tj}=softmax\left(v^{T}\tanh\left(Wh_{t-1}+Ud_{j}+Vr_{t-1}\right)\right)
\end{equation}

\begin{equation}
a_{t}=\stackrel[j=1]{L}{\sum}\alpha_{tj}d_{j}\label{eq:att}
\end{equation}

where $\alpha_{tj}$ is the attention score between the $t$-th writing
step and the $j$-th element in the cache; $W$, $U$, $V$ and $v$
are parameters; $h$ and $r$ are the hidden state of the controller
and the read-out (Eq. (\ref{eq:read_R})), respectively; $d_{j}$
is the cache element and can be implemented as the controller's hidden
state ($d_{j}=h_{t-1-L+j}$). 

The vector $a_{t}$ will be used to replace the previous hidden state
in updating the controller and memory. The whole process of performing
CUW is \foreignlanguage{australian}{summarised} in Algo. \ref{alg:Cached-Uniform-Writing}\footnote{Source code is available at \url{https://github.com/thaihungle/UW-DNC}}.

\section{Experiments and Results}

\begin{figure}
\begin{centering}
\includegraphics[width=1\linewidth]{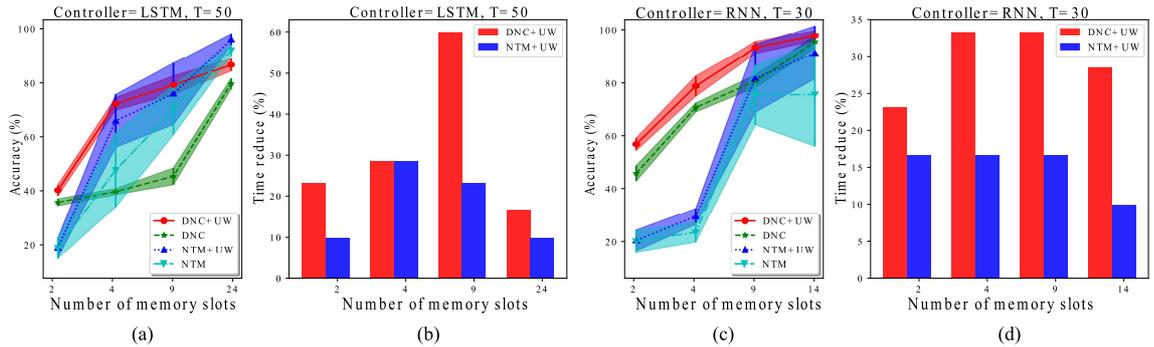}
\par\end{centering}
\caption{The accuracy (\%) and computation time reduction (\%) with different
memory types and number of memory slots. The controllers/sequence
lengths/memory sizes are chosen as LSTM/50/$\left\{ 2,4,9,24\right\} $
(a\&b) and RNN/30/$\left\{ 2,4,9,14\right\} $ (c\&d), respectively.
\label{fig:The-test-accuracy}}
\end{figure}

\subsection{An Ablation Study: Memory-Augmented Neural Networks with and without
Uniform Writing}

In this section, we study the impact of uniform writing on MANNs under
various circumstances (different controller types, memory types and
number of memory slots). We restrict the memorisation problem to the
double task in which the models must reconstruct a sequence of integers
sampled uniformly from range $\left[1,10\right]$ twice. We cast this
problem to a sequence to sequence problem with 10 possible outputs
per decoding step. The training stops after 10,000 iterations of batch
size 64. We choose DNC\footnote{Our reimplementation based on https://github.com/deepmind/dnc}
and NTM\footnote{https://github.com/MarkPKCollier/NeuralTuringMachine}
as the two MANNs in the experiment. The recurrent controllers can
be RNN or LSTM. With LSTM controller, the sequence length is set to
50. We choose sequence length of 30 to make it easier for the RNN
controller to learn the task. The number of memory slots $D$ is chosen
from the set $\left\{ 2,4,9,24\right\} $ and $\left\{ 2,4,9,14\right\} $
for LSTM and RNN controllers, respectively. More memory slots will
make UW equivalent to the regular writing scheme. For this experiment,
we use Adam optimiser \citet{kingma2014adam} with initial learning
rate and gradient clipping of $\left\{ 0.001,0.0001\right\} $ and
$\left\{ 1,5,10\right\} $, respectively. The metric used to measure
the performance is the average accuracy across decoding steps. For
each configuration of hyper-parameters, we run the experiment 5 times
and report the mean accuracy with error bars.

Figs. \ref{fig:The-test-accuracy}(a) and (c) depict the performance
of UW and regular writing under different configurations. In any case,
UW boosts the prediction accuracy of MANNs. The performance gain can
be seen clearly when the compression ratio is between $10-40\%$.
This is expected since when the compression ratio is too small or
too big, UW converges to regular writing. Interestingly, increasing
the memory size does not always improve the performance, as in the
case of NTM with RNN controllers. Perhaps, learning to attend to many
memory slots is tricky for some task given limited amount of training
data. This supports the need to apply UW to MANN with moderate memory
size. We also conduct experiments to verify the benefit of using UW
for bigger memory. The results can be found in Appendix \ref{subsec:UW-performance-on}. 

We also measure the speed-up of training time when applying UW on
DNC and NTM, which is illustrated in Figs. \ref{fig:The-test-accuracy}(b)
and (d). The result shows that with UW, the training time can drop
up to $60\%$ for DNC and $28\%$ for NTM, respectively. As DNC is
more complicated than NTM, using UW to reduce memory access demonstrates
clearer speed-up in training (similar behavior can be found for testing
time).

\subsection{Synthetic Memorisation}

Here we address a broader range of baselines on two synthetic memorisation
tasks, which are the sequence copy and reverse. In these tasks, there
is no discrimination amongst timesteps so the model's goal is to learn
to compress the input efficiently for later retrieval. We experiment
with different sequence lengths of 50 and 100 timesteps. Other details
are the same as the previous double task except that we fix the learning
rate and gradient clipping to 0.001 and 10, respectively. The standard
baselines include LSTM, NTM and DNC. All memory-augmented models have
the same memory size of $4$ slots, corresponding to compression ratio
of $10\%$ and $5\%$, respectively. We aim at this range of compression
ratio to match harsh practical requirements. UW and CUW (cache size
$L=5$) are built upon the DNC, which from our previous observations,
works best for given compression ratios. We choose different dimensions
$N_{h}$ for the hidden vector of the controllers to ensure the model
sizes are approximately equivalent. To further verify that our UW
is actually the optimal writing strategy, we design a new baseline,
which is DNC with random irregular writing strategy (RW). The write
is sampled from a binomial distribution with $p=\left(D+1\right)/T$
(equivalent to compression ratio). After sampling, we conduct the
training for that policy. The final performances of RW are taken average
from 3 different random policies' results.

The performance of the models is listed in Table \ref{tab:mem_test}.
As clearly seen, UW is the best performer for the pure memorisation
tests. This is expected from the theory as all timesteps are importantly
equivalent. Local attention mechanism in CUW does not help much in
this scenario and thus CUW finishes the task as the runner-up. Reverse
seems to be easier than copy as the models tend to ``remember''
more the last-seen timesteps whose contributions $\lambda^{T-t}$
remains significant. In both cases, other baselines including random
irregular and regular writing underperform our proposed models by
a huge margin. 

\begin{table}
\begin{centering}
\begin{tabular}{ccccccc}
\hline 
\multirow{2}{*}{Model} & \multirow{2}{*}{$N_{h}$} & \multirow{2}{*}{\# parameter} & \multicolumn{2}{c}{Copy} & \multicolumn{2}{c}{Reverse}\tabularnewline
\cline{4-7} \cline{5-7} \cline{6-7} \cline{7-7} 
 &  &  & L=50 & L=100 & L=50 & L=100\tabularnewline
\hline 
LSTM & 125 & 103,840 & 15.6 & 12.7 & 49.6 & 26,1\tabularnewline
NTM & 100 & 99,112 & 40.1 & 11.8 & 61.1 & 20.3\tabularnewline
DNC & 100 & 98,840 & 68.0 & 44.2 & 65.0 & 54.1\tabularnewline
DNC+RW & 100 & 98,840 & 47.6 & 37.0 & 70.8 & 50.1\tabularnewline
\hline 
DNC+UW & 100 & 98,840 & \textbf{97.7} & \textbf{69.3} & \textbf{100} & \textbf{79.5}\tabularnewline
DNC+CUW & 95  & 96,120 & 83.8 & 55.7 & 93.3 & 55.4\tabularnewline
\hline 
\end{tabular}
\par\end{centering}
\caption{Test accuracy (\%) on synthetic memorisation tasks. MANNs have 4 memory
slots.\label{tab:mem_test}}
\end{table}

\subsection{Synthetic Reasoning}

Tasks in the real world rarely involve just memorisation. Rather,
they require the ability to selectively remember the input data and
synthesise intermediate computations. To investigate whether our proposed
writing schemes help the memory-augmented models handle these challenges,
we conduct synthetic reasoning experiments which include add and max
tasks. In these tasks, each number in the output sequence is the sum
or the maximum of two numbers in the input sequence. The pairing is
fixed as: $y_{t}=\frac{x_{t}+x_{T-t}}{2},t=\overline{1,\left\lfloor \frac{T}{2}\right\rfloor }$
for add task and $y_{t}=\max\left(x_{2t},x_{2t+1}\right),t=\overline{1,\left\lfloor \frac{T}{2}\right\rfloor }$
for max task, respectively. The length of the output sequence is thus
half of the input sequence. A brief overview of input/output format
for these tasks can be found in Appendix \ref{subsec:Summary-of-synthetic}.
We deliberately use local (max) and distant (add) pairing rules to
test the model under different reasoning strategies. The same experimental
setting as in the previous section is applied except for the data
sample range for the max task, which is $\left[1,50\right]$\footnote{With small range like $\left[1,10\right]$, there is no much difference
in performance amongst models}. LSTM and NTM are excluded from the baselines as they fail on these
tasks. 

Table \ref{tab:reason_task} shows the testing results for the reasoning
tasks. Since the memory size is small compared to the number of events,
regular writing or random irregular writing cannot compete with the
uniform-based writing policies. Amongst all baselines, CUW demonstrates
superior performance in both tasks thanks to its local attention mechanism.
It should be noted that the timesteps should not be treated equally
in these reasoning tasks. The model should weight a timestep differently
based on either its content (max task) or location (add task) and
maintain its memory for a long time by following uniform criteria.
CUW is designed to balance the two approaches and thus it achieves
better performance. Further insights into memory operations of these
models are given in Appendix \ref{subsec:Memory-writing-behaviours}. 

\begin{table}
\begin{centering}
\begin{tabular}{ccccc}
\hline 
\multirow{2}{*}{Model} & \multicolumn{2}{c}{Add} & \multicolumn{2}{c}{Max}\tabularnewline
\cline{2-5} \cline{3-5} \cline{4-5} \cline{5-5} 
 & L=50 & L=100 & L=50 & L=100\tabularnewline
\hline 
DNC & 83.8 & 22.3 & 59.5 & 27.4\tabularnewline
DNC+RW & 83.0 & 22.7 & 59.7 & 36.5\tabularnewline
\hline 
DNC+UW & 84.8 & 50.9 & 71.7 & 66.2\tabularnewline
DNC+CUW & \textbf{94.4} & \textbf{60.1} & \textbf{82.3} & \textbf{70.7}\tabularnewline
\hline 
\end{tabular}
\par\end{centering}
\caption{Test accuracy (\%) on synthetic reasoning tasks. MANNs have 4 memory
slots.\label{tab:reason_task}}
\end{table}

\subsection{Synthetic Sinusoidal Regression }

In real-world settings, sometimes a long sequence can be captured
and fully reconstructed by memorising some of its feature points.
For example, a periodic function such as sinusoid can be well-captured
if we remember the peaks of the signal. By observing the peaks, we
can deduce the frequency, amplitude, phase and thus fully reconstructing
the function. To demonstrate that UW and CUW are useful for such scenarios,
we design a sequential continuation task, in which the input is a
sequence of sampling points across some sinusoid: $y=5+A\sin(2\pi fx+\varphi)$.
Here, $A\sim\mathcal{U}\left(1,5\right)$, $f\sim\mathcal{U}\left(10,30\right)$
and $\varphi\sim\mathcal{U}\left(0,100\right)$. After reading the
input $y=\left\{ y_{t}\right\} _{t=1}^{T}$, the model have to generate
a sequence of the following points in the sinusoid. To ensure the
sequence $y$ varies and covers at least one period of the sinusoid,
we set $x=\left\{ x_{t}\right\} _{t=1}^{T}$ where $x_{i}=\left(t+\epsilon_{1}\right)/1000$,
$\epsilon_{1}\sim\mathcal{U}\left(-1,1\right)$. The sequence length
for both input and output is fixed to $T=100$. The experimental models
are LSTM, DNC, UW and CUW (built upon DNC). For each model, optimal
hyperparameters including learning rate and clipping size are tuned
with 10,000 generated sinusoids. The memories have $4$ slots and
all baselines have similar parameter size. We also conduct the experiment
with noisy inputs by adding a noise $\epsilon_{2}\sim\mathcal{U}\left(-2,2\right)$
to the input sequence $y$. This increases the difficulty of the task.
The loss is the average of mean square error (MSE) over decoding timesteps.

We plot the mean learning curves with error bars over 5 runnings for
sinusoidal regression task under clean and noisy condition in Figs.
\ref{fig:Training-curves-of}(a) and (b), respectively. Regular writing
DNC learns fast at the beginning, yet soon saturates and approaches
the performance of LSTM ($MSE=1.05$ and $1.39$ in clean and noisy
condition, respectively). DNC performance does not improve much as
we increase the memory size to $50$, which implies the difficulty
in learning with big memory. Although UW starts slower, it ends up
with lower errors than DNC and perform slightly better than CUW in
clean condition ($MSE=0.44$ for UW and $0.61$ for CUW). CUW demonstrates
competitive performance against other baselines, approaching to better
solution than UW for noisy task where the model should discriminate
the timesteps ($MSE=0.98$ for UW and $0.55$ for CUW). More visualisations
can be found in Appendix \ref{subsec:Visualizations-of-model}. 

\begin{figure}
\begin{centering}
\includegraphics[width=1\textwidth]{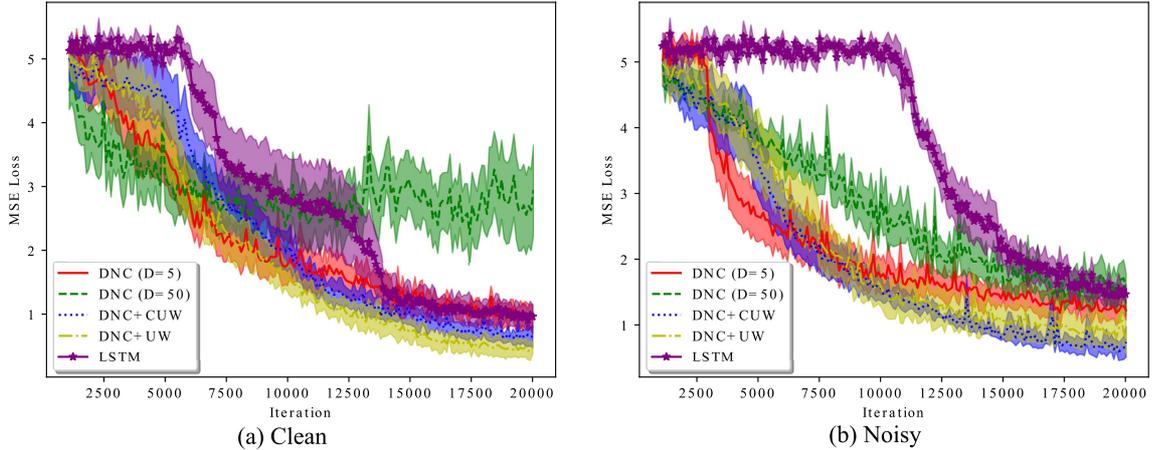}
\par\end{centering}
\caption{Learning curves of models in clean (a) and noisy (b) sinusoid regression
experiment.\label{fig:Training-curves-of}}
\end{figure}

\subsection{Flatten Image Recognition}

We want to compare our proposed models with DNC and other methods
designed to help recurrent networks learn longer sequence. The chosen
benchmark is a pixel-by-pixel image classification task on MNIST in
which pixels of each image are fed into a recurrent model sequentially
before a prediction is made. In this task, the sequence length is
fixed to 768 with highly redundant timesteps (black pixels). The training,
validation and testing sizes are 50,000, 10,000 and 10,000, respectively.
We test our models on both versions of non-permutation (MNIST) and
permutation (pMNIST) \citet{le2015simple}. More details on the task
and data can be found in Le et al., (2015). $\joinrel$For DNC, we
try with several memory slots from $\left\{ 15,30,60\right\} $ and
report the best results. For UW and CUW, memory size is fixed to 15
and cache size $L$ is set to 10. The controllers are implemented
as single layer GRU with $100$-dimensional hidden vector. To optimise
the models, we use RMSprop with initial learning rate of 0.0001. 

Table \ref{tab:mnist} shows that DNC underperforms r-LSTM, which
indicates that regular DNC with big memory finds it hard to beat LSTM-based
methods. After applying UW, the results get better and with CUW, it
shows significant improvement over r-LSTM and demonstrates competitive
performance against dilated-RNNs models. Notably, dilated-RNNs use
9 layers in their experiments compared to our singer layer controller.
Furthermore, our models exhibit more consistent performance than dilated-RNNs.
For completeness, we include comparisons between CUW and non-recurrent
methods in Appendix \ref{subsec:Comparsion-with-non-recurrent}

\begin{table}
\begin{centering}
\begin{tabular}{lcc}
\hline 
Model & MNIST & pMNIST\tabularnewline
\hline 
iRNN$^{\dagger}$  & 97.0 & 82.0\tabularnewline
uRNN$^{\circ}$ & 95.1 & 91.4\tabularnewline
r-LSTM Full BP$^{\star}$ & 98.4 & 95.2\tabularnewline
Dilated-RNN$^{\blacklozenge}$ & 95.5 & 96.1\tabularnewline
Dilated-GRU$^{\blacklozenge}$ & \textbf{99.2} & 94.6\tabularnewline
DNC & 98.1 & 94.0\tabularnewline
\hline 
DNC+UW & 98.6 & 95.6\tabularnewline
DNC+CUW & 99.1 & \textbf{96.3}\tabularnewline
\hline 
\end{tabular}
\par\end{centering}
\caption{Test accuracy (\%) on MNIST, pMNIST. Previously reported results are
from the literature \citet{le2015simple}$^{\dagger}$, \citet{arjovsky2016unitary}$^{\circ}$,
\citet{trinh2018learning}$^{\star}$, and \citet{chang2017dilated}$^{\blacklozenge}$.
\label{tab:mnist}}
\end{table}

\subsection{Document Classification}

To verify our proposed models in real-world applications, we conduct
experiments on document classification task. In the task, the input
is a sequence of words and the output is the classification label.
Following common practices \citet{yogatama2017generative,seo2018neural},
each word in the document is embedded into a $300$-dimensional vector
using Glove embedding \citet{pennington2014glove}. We use RMSprop
for optimisation, with initial learning rate of 0.0001. Early-stop
training is applied if there is no improvement after 5 epochs in the
validation set. Our UW and CUW are built upon DNC with single layer
$512$-dimensional LSTM controller and the memory size is chosen in
accordance with the average length of the document, which ensures
$10-20\%$ compression ratio. The cache size for CUW is fixed to 10.
The datasets used in this experiment are common big datasets where
the number of documents is between 120,000 and 1,400,000 with maximum
of 4,392 words per document (see Appendix \ref{subsec:Details-on-document}
for further details). The baselines are recent state-of-the-arts in
the domain, some of which are based on recurrent networks such as
D-LSTM \citet{yogatama2017generative} and Skim-LSTM \citet{seo2018neural}.
We exclude DNC from the baselines as it is inefficient to train the
model with big document datasets. 

Our results are reported in Table \ref{tab:text}. On five datasets
out of six, our models beat or match the best published results. For
IMDb dataset, our methods outperform the best recurrent model (Skim-LSTM).
The performance gain is competitive against that of the state-of-the-arts.
In most cases, CUW is better than UW, which emphasises the importance
of relaxing the timestep equality assumption in practical situations.
Details results across different runs for our methods are listed in
Appendix \ref{subsec:Document-classification-detailed}. 

\begin{table}
\begin{centering}
\begin{tabular}{lcccccc}
\hline 
Model & AG & IMDb\tablefootnote{Methods that use semi-supervised training to achieve higher accuracy
are not listed.} & Yelp P. & Yelp F. & DBP & Yah. A.\tabularnewline
\hline 
VDCNN$^{\bullet}$  & 91.3 & - & 95.7 & 64.7 & 98.7 & 73.4\tabularnewline
D-LSTM$^{\ast}$ & - & - & 92.6 & 59.6 & 98.7 & \textit{73.7}\tabularnewline
Standard LSTM$^{\ddagger}$ & 93.5 & 91.1 & - & - & - & -\tabularnewline
Skim-LSTM$^{\ddagger}$ & \textit{93.6} & 91.2 & - & - & - & -\tabularnewline
Region Embedding$^{\blacktriangle}$ & 92.8 & - & \textbf{\textit{96.4}} & \textit{64.9} & \textit{98.9} & \textit{73.7}\tabularnewline
\hline 
DNC+UW & 93.7 & \textbf{91.4} & \textbf{96.4} & 65.3 & \textbf{99.0} & 74.2\tabularnewline
DNC+CUW & \textbf{93.9} & 91.3 & \textbf{96.4} & \textbf{65.6} & \textbf{99.0} & \textbf{74.3}\tabularnewline
\hline 
\end{tabular}
\par\end{centering}
\caption{Document classification accuracy (\%) on several datasets. Previously
reported results are from the literature \citet{conneau2016very}$^{\bullet}$,
\citet{yogatama2017generative}$^{\ast}$, \citet{seo2018neural}$^{\ddagger}$
and \citet{qiao2018a}$^{\blacktriangle}$. We use italics to denote
the best published and bold the best records.\label{tab:text}}
\end{table}

\section{Closing Remarks}

We have introduced Uniform Writing (UW) and Cached Uniform Writing
(CUW) as faster solutions for longer-term memorisation in MANNs. With
a comprehensive suite of synthetic and practical experiments, we provide
strong evidences that our simple writing mechanisms are crucial to
MANNs to reduce computation complexity and achieve competitive performance
in sequence modeling tasks. In complement to the experimental results,
we have proposed a meaningful measurement on MANN memory capacity
and provided theoretical analysis showing the optimality of our methods.
In the next chapter, we shift the focus back to designing a better
MANN that is capable of learning multiple procedures at the same time
and behaving like a universal Turing machine rather than a Turing
machine.

\chapter{Neural Stored-Program Memory\label{chap:Neural-Stored-program-Memory}}

\section{Introduction}

Recurrent Neural Networks (RNNs) are Turing-complete, that is, we
can theoretically construct a Turing machine by using RNN elements
\citet{siegelmann1995computational}. However, in practice RNNs struggle
to learn simple procedures as they lack explicit memory \citet{graves2014neural,mozer1993connectionist}
. Memory Augmented Neural Networks (MANNs) as analysed in Chapter
\ref{chap:MANN}, lift these limitations as they are capable of emulating
modern computer behavior by detaching memorisation from computation
via memory and controller network, respectively. In previous chapters,
we have also demonstrated that MANNs are powerful in many settings,
ranging from encoding multi-processes, generating complex sequences,
to skip-readings. Nonetheless, MANNs have barely simulated general-purpose
computers as they miss a key concept in computer design: stored-program
memory. 

The concept of stored-program has emerged from the idea of Universal
Turing Machine (UTM) \citet{turing1937computable} and developed in
the Von Neumann Architecture (VNA) \citet{vonNeumann:1993:FDR:612487.612553}.
In UTM\nomenclature{UTM}{Universal Turing Machine}/VNA\nomenclature{VNA}{Von Neumann Architecture},
both data and programs that manipulate the data are stored in memory.
A control unit then reads the programs from the memory and executes
them with the data. This mechanism allows flexibility to perform universal
computations. Unfortunately, current MANNs such as Neural Turing Machine
(NTM) \citet{graves2014neural}, Differentiable Neural Computer (DNC)
\citet{graves2016hybrid} and Least Recently Used Access (LRUA\nomenclature{LRUA}{Least Recently Used Access})
\citet{santoro2016meta} only support memory for data and embed a
single program into the controller network, which goes against the
stored-program memory principle. 

Our goal is to advance a step further towards UTM/VNA by coupling
a MANN with an external program memory. The program memory co-exists
with the data memory in the MANN, providing more flexibility, reuseability
and modularity in learning complicated tasks. The program memory stores
the weights of the MANN's controller network, which are retrieved
quickly via a key-value attention mechanism across timesteps yet updated
slowly via backpropagation. By introducing a meta network to moderate
the operations of the program memory, our model, henceforth referred
to as Neural Stored-program Memory (NSM\nomenclature{NSM}{Neural Stored-program Memory}),
can learn to switch the programs/weights in the controller network
appropriately, adapting to different functionalities aligning with
different parts of a sequential task, or different tasks in continual
and few-shot learning.

To validate our proposal, the NTM armed with NSM, namely Neural Universal
Turing Machine (NUTM\nomenclature{NUTM}{Neural Universal Turing Machine}),
is tested on a variety of synthetic tasks including algorithmic tasks
\citet{graves2014neural}, composition of algorithmic tasks and continual
procedure learning. For these algorithmic problems, we demonstrate
clear improvements of NUTM over NTM. Further, we investigate NUTM
in few-shot learning by using LRUA as the MANN and achieve notably
better results. Finally, we expand NUTM application to linguistic
problems by equipping NUTM with DNC core and achieve competitive performances
against state-of-the-arts in the bAbI task \citet{weston2015towards}. 

Taken together, our study advances neural network simulation of Turing
Machines to neural architecture for Universal Turing Machines. This
develops a new class of MANNs that can store and query both the weights
and data of their own controllers, thereby following the stored-program
principle. A set of five diverse experiments demonstrate the computational
universality of the approach.

\section{Backgrounds\label{sec:Backgrounds}}

\subsection{Turing Machines and MANNs}

In this section, we briefly review MANN and its relations to Turing
Machines. A MANN consists of a controller network and an external
memory $\mathbf{M}\in\mathbb{R}^{N\times M}$, which is a collection
of $N$ $M$-dimensional vectors. The controller network is responsible
for accessing the memory, updating its state and optionally producing
output at each timestep. The first two functions are executed by an
interface network and a state network\footnote{Some MANNs (e.g., NTM with Feedforward Controller) neglect the state
network, only implementing the interface network and thus analogous
to one-state Turing Machine. }, respectively. Usually, the interface network is a Feedforward neural
network whose input is $c_{t}$ - the output of the state network
implemented as RNNs. Let $W^{c}$ denote the weight of the interface
network, then the state update and memory control are as follows,

\begin{equation}
h_{t},c_{t}=RNN\left(\left[x_{t},r_{t-1}\right],h_{t-1}\right)
\end{equation}
\begin{equation}
\xi_{t}=c_{t}W^{c}
\end{equation}

where $x_{t}$ and $r_{t-1}$ are data from current input and the
previous memory read, respectively. The interface vector $\xi_{t}$
then is used to read from and write to the memory $\mathbf{M}$. We
use a generic notation $memory\left(\xi_{t},\mathbf{M}\right)$ to
represent these memory operations that either update or retrieve read
value $r_{t}$ from the memory. To support multiple memory accesses
per step, there might be several interface networks to produce multiple
interfaces, also known as control heads. Readers are referred to Graves
et al., (2014, 2016) and Santoro et al., (2016) for details of memory
read/write examples. 

A deterministic one-tape Turing Machine can be defined by 4-tuple
$\left(Q,\Gamma,\delta,q_{0}\right)$, in which $Q$ is finite set
of states, $q_{0}\in Q$ is an initial state, $\Gamma$ is finite
set of symbol stored in the tape (the data) and $\delta$ is the transition
function (the program), $\delta:Q\times\Gamma\rightarrow\Gamma\times\left\{ -1,1\right\} \times Q$.
At each step, the machine performs the transition function, which
takes the current state and the read value from the tape as inputs
and outputs actions including writing new values, moving tape head
to new location (left/right) and jumping to another state. Roughly
mapping to current MANNs, $Q$, $\Gamma$ and $\delta$ map to the
set of the controller states, the read values and the controller network,
respectively. Further, the function $\delta$ can be factorised into
two sub functions: $Q\times\Gamma\rightarrow\Gamma\times\left\{ -1,1\right\} $
and $Q\times\Gamma\rightarrow Q$, which correspond to the interface
and state networks, respectively. 

By encoding a Turing Machine into the tape, one can build an UTM that
simulates the encoded machine \citet{turing1937computable}. The transition
function of the UTM queries the encoded Turing Machine that solves
the considering task. Amongst 4 tuples, $\delta$ is the most important
and hence uses most of the encoding bits. In other words, if we assume
that the space of $Q$, $\Gamma$ and $q_{0}$ are shared amongst
Turing Machines, we can simulate any Turing Machine by encoding only
its transition function $\delta$. Translating to neural language,
if we can store the controller network into a queriable memory and
make use of it, we can build a Neural Universal Turing Machine. Using
NSM is a simple way to achieve this goal, which we introduce in the
subsequent section.

\subsection{Related Approaches}

Previous investigations into MANNs mostly revolve around memory access
mechanisms. The works in Graves et al., (2014, 2016) introduce content-based,
location-based and dynamic memory reading/writing. Further, Rae et
al., (2016) scales to bigger memory by sparse access while Le et al.,
(2019a) \foreignlanguage{australian}{optimises} memory operations
with uniform writing. These works keep using memory for storing data
rather than the weights of the network and thus parallel to our approach.
Other DNC modifications \citet{csordas2018improving,W18-2606} are
also orthogonal to our work. 

Another line of related work involves modularisation of neural networks,
which is designed for visual question answering. In module networks
\citet{andreas2016neural,andreas-etal-2016-learning}, the modules
are manually aligned with predefined concepts and the order of execution
is decided by the question. Although the module in these works resembles
the program in NSM, our model is more generic and flexible with soft-attention
over programs and thus fully differentiable. Further, the motivation
of NSM does not limit to a specific application. Rather, NSM aims
to help MANN reach general-purpose computability. 

Finally, if we view NSM network as a dynamic weight generator, the
program in NSM can be linked to fast weight \citet{hinton1987using,ba2016using,munkhdalai2017meta}.
These papers share the idea of using different weights across timesteps
to enable dynamic adaptation. However, fast weights are directly generated
while our programs are interpolated from a set of slow weights. 

\section{Neural Stored-Program Memory and Neural Universal Turing Machine}

\subsection{Neural Stored-Program Memory}

A Neural Stored-program Memory (NSM) is a key-value memory $\mathbf{M}_{p}\in\mathbb{R}^{P\times(K+S)}$,
whose value is the weight of another neural network$-$the program.
$P$, $K$, and $S$ are the number of programs, the key space dimension
and the program size, respectively. This concept is a hybrid between
the traditional slow-weight and fast-weight \citet{hinton1987using}.
Like slow-weight, the weights in NSM are updated gradually by backpropagation.
However, they are dynamically recomputed on-the-fly during the processing
of a sequence, which resembles fast-weight computation. Let us denote
$\mathbf{M}_{p}\left(i\right).k$ and $\mathbf{M}_{p}\left(i\right).v$
as the key and the content of the $i$-th memory slot. At timestep
$t,$ given a query key $k_{t}^{p}$, the corresponding program is
retrieved as follows,

\begin{equation}
D\left(k_{t}^{p},\mathbf{M}_{p}(i).k\right)=\frac{k_{t}^{p}\cdot\mathbf{M}_{p}(i).k}{||k_{t}^{p}||\cdot||\mathbf{M}_{p}(i).k)||}\label{eq:d_p}
\end{equation}

\begin{equation}
p_{t}=\stackrel[i=1]{P}{\sum}\text{softmax}\left(\beta_{t}^{p}D\left(k_{t}^{p},\mathbf{M}_{p}(i).k\right)\right)\mathbf{M}_{p}\left(i\right).v\label{eq:pt}
\end{equation}
where $D\left(\cdot\right)$ is cosine similarity and $\beta_{t}^{p}$
is the scalar program strength parameter. The vector program $p_{t}$
is then reshaped to its matrix form and ready to be used in other
neural computations. 

The key-value design is essential for convenient memory access as
the size of the program stored in $\mathbf{M}_{p}$ can be millions
of dimensions and thus, direct content-based addressing \citet{graves2014neural,graves2016hybrid}
is infeasible. More importantly, we can inject external control on
the behavior of the memory by imposing constraints on the key space.
For example, program collapse will happen when the keys stored in
the memory stay close to each other. When this happens, $p_{t}$ is
a balanced mixture of all programs regardless of the query key and
thus having multiple programs is useless. We can avoid this phenomenon
by minimising a regularisation loss defined as the following,

\begin{equation}
l_{p}=\stackrel[i=1]{P}{\sum}\stackrel[j=i+1]{P}{\sum}D\left(\mathbf{M}_{p}(i).k,\mathbf{M}_{p}(j).k\right)
\end{equation}

\subsection{Neural Universal Turing Machine}

It turns out that the combination of MANN and NSM approximates an
Universal Turing Machine (Sec. \ref{sec:Backgrounds}). At each timestep,
the controller in MANN reads its state and memory to generate control
signal to the memory via the interface network $W^{c}$, then updates
its state using the state network $RNN$. Since the parameters of
$RNN$ and $W^{c}$ represent the encoding of $\delta$, we store
both into NSM to completely encode an MANN. For simplicity, in this
proposal, we only use NSM to store $W^{c}$, which is equivalent to
the Universal Turing Machine that can simulate any one-state Turing
Machine. 

In traditional MANN, $W^{c}$ is constant across timesteps and only
updated slowly during training, typically through backpropagation.
In our design, we compute $W_{t}^{c}$ from NSM for every timestep
and thus, we need a program interface network$-$the meta network
$P_{\mathscr{\mathcal{I}}}-$that generates an interface vector for
the program memory: $\xi_{t}^{p}=P_{\mathscr{\mathcal{I}}}\left(c_{t}\right)$,
where $\xi_{t}^{p}=\left[k_{t}^{p},\beta_{t}^{p}\right]$ and $P_{\mathscr{\mathcal{I}}}$
is implemented as a Feedforward neural network. The procedure for
computing $W_{t}^{c}$ is executed by following Eqs. (\ref{eq:d_p})-(\ref{eq:pt}),
hereafter referred to as $NSM\left(\xi_{t}^{p},\mathbf{M}_{p}\right)$.
Figure \ref{fig:NUTM-diagram} depicts the integration of NSM into
MANN.

\begin{figure}
\begin{centering}
\includegraphics[width=0.8\columnwidth]{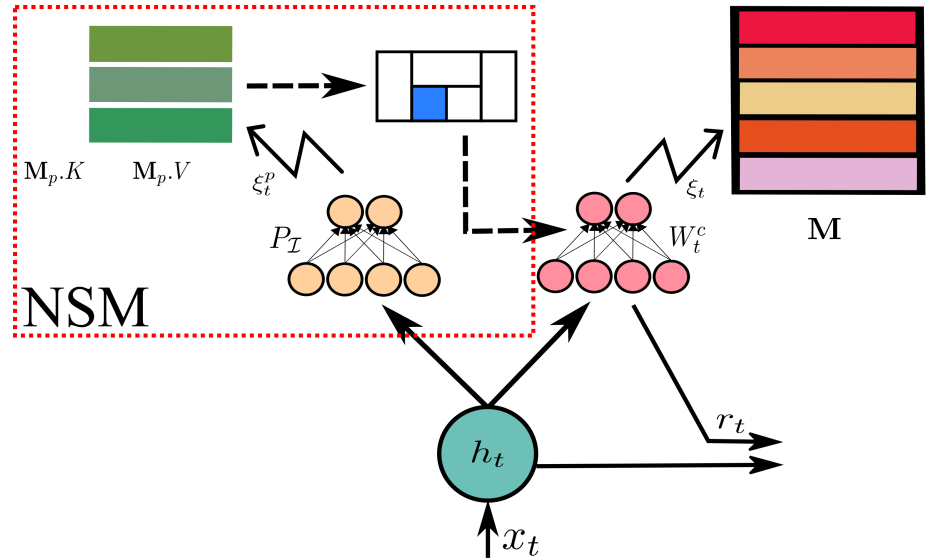}
\par\end{centering}
\caption{Introducing NSM into MANN. At each timestep, the program interface
network ($P_{\mathscr{\mathcal{I}}}$) receives input from the state
network and queries the program memory $\mathbf{M}_{p}$, acquiring
the working weight for the interface network ($W_{t}^{c}$). The interface
network then operates on the data memory $\mathbf{M}$. \label{fig:NUTM-diagram}}
\end{figure}

For the case of multi-head NTM, we implement one NSM per control head
and name this model Neural Universal Turing Machine (NUTM). Each control
head will read from (for read head) or write to (for write head) the
data memory $\mathbf{M}$ via $memory\left(\xi_{t},\mathbf{M}\right)$
as described in Graves et al., (2014). Other MANNs such as DNC \citet{graves2016hybrid}
and LRUA \citet{santoro2016meta} can be armed with NSM in this manner.
We also employ the regularisation loss $l_{p}$ to prevent the programs
from collapsing, resulting in a final loss as follows,

\begin{equation}
Loss=Loss_{pred}+\eta_{t}l_{p}\label{eq:nsm_loss}
\end{equation}
where $Loss_{pred}$ is the prediction loss and $\eta_{t}$ is annealing
factor, reducing as the training step increases. The details of NUTM
operations are presented in Algorithm \ref{alg:Neural-Uinversal-Turing}. 

\begin{algorithm}[t]
\begin{algorithmic}[1]
\Require{a sequence $x=\left\{ x_{t}\right\} _{t=1}^{T}$, a data memory $\mathbf{M}$ and $R$ program memories $\left\{ \mathbf{M}_{p,n}\right\} _{n=1}^{R}$ corresponding to $R$ control heads}
\State{Initilise $h_0$, $r_0$}
\For{$t=1,T$}
\State{$h_t,c_t=RNN([x_t,r_{t-1}],h_{t-1})$} \Comment{$RNN$ can be replaced by GRU/LSTM}
\For{$n=1,R$}
\State{Compute the program interface $\xi_{t,n}^{p}\leftarrow P_{\mathscr{\mathcal{I}},n}\left(c_{t}\right)$}
\State{Compute the program $W_{t,n}^{c}\leftarrow NSM\left(\xi_{t,n}^{p},\mathbf{M}_{p,n}\right)$}
\State{Compute the data interface $\xi_{t,n}\leftarrow c_{t}W_{t,n}^{c}$}
\State{Access/update data memory $r_{t,n}\leftarrow memory\left(\xi_{t,n},\mathbf{M}\right)$}\Comment{Write heads return $\emptyset$}
\EndFor
\State{$r_{t}\leftarrow\left[r_{t,1},...,r_{t,R}\right]$}
\EndFor
\end{algorithmic} 

\caption{Neural Universal Turing Machine\label{alg:Neural-Uinversal-Turing}}
\end{algorithm}

\subsection{On the Benefit of NSM to MANN: An Explanation from Multilevel Modeling\label{subsec:On-the-Benefit}}

Learning to access memory is a multi-dimensional regression problem.
Given the input $c_{t}$, which is derived from the state $h_{t}$
of the controller, the aim is to generate a correct interface vector
$\xi_{t}$ via optimising the interface network. Instead of searching
for one transformation that maps the whole space of $c_{t}$ to the
optimal space of $\xi_{t}$, NSM first partitions the space of $c_{t}$
into subspaces, then finds multiple transformations, each of which
covers subspace of $c_{t}$. The program interface network $P_{\mathscr{\mathcal{I}}}$
is a meta learner that routes $c_{t}$ to the appropriate transformation,
which then maps $c_{t}$ to the $\xi_{t}$ space. This is analogous
to multilevel regression in statistics \citet{andrew2006mr}. Many
practical studies have demonstrated that multilevel regression is
better than ordinary regression if the input is clustered \citet{cohen2014applied,huang2018multilevel}. 

RNNs have the capacity to learn to perform finite state computations
\citet{casey1996dynamics,tivno1998finite}. The states of a RNN must
be grouped into partitions representing the states of the generating
automation. As Turing Machine is finite state automata augmented with
an external memory tape, we expect MANN, if learnt well, will \foreignlanguage{australian}{organise}
its state space clustered in a way to reflect the states of the emulated
Turing Machine. That is, $h_{t}$ as well as $c_{t}$ should be clustered.
We realise that NSM helps NTM learn better clusterisation over this
space (see Appendix \ref{subsec:Clustering-on-The}), thereby improving
NTM's performances. 

\section{Applications}

\subsection{NTM Single Tasks\label{subsec:NTM-Single-Tasks}}

\begin{figure}
\begin{centering}
\includegraphics[width=1\linewidth]{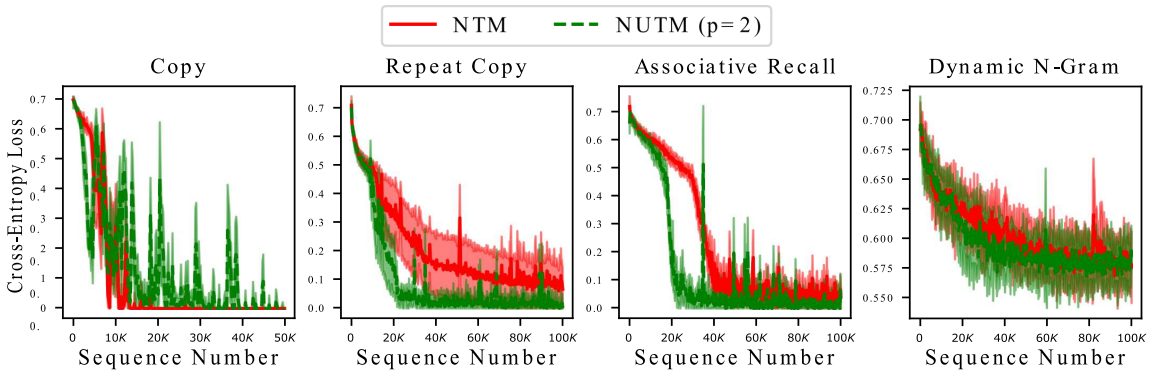}
\par\end{centering}
\caption{Learning curves on NTM tasks.\label{fig:Learning-curves-on}}
\end{figure}

\begin{table}
\begin{centering}
\begin{tabular}{c|cccccc}
\hline 
{\footnotesize{}Task} & {\footnotesize{}Copy} & {\footnotesize{}Repeat Copy} & {\footnotesize{}A. Recall} & {\footnotesize{}D. N-grams} & {\footnotesize{}Priority Sort} & {\footnotesize{}Long Copy}\tabularnewline
\hline 
{\footnotesize{}NTM} & \textbf{\footnotesize{}0.00} & {\footnotesize{}405.10} & {\footnotesize{}7.66} & {\footnotesize{}132.59} & {\footnotesize{}24.41} & {\footnotesize{}16.04}\tabularnewline
\hline 
{\footnotesize{}NUTM (p=2)} & \textbf{\footnotesize{}0.00} & \textbf{\footnotesize{}366.69} & \textbf{\footnotesize{}1.35} & \textbf{\footnotesize{}127.68} & \textbf{\footnotesize{}20.00} & \textbf{\footnotesize{}0.02}\tabularnewline
\hline 
\end{tabular}
\par\end{centering}
~

\caption{Generalisation performance of best models measured in average bit
error per sequence (lower is better). For each task, we pick a set
of 1,000 unseen sequences as test data. \label{tab:Generalisation-performance-of}}
\end{table}

In this section, we investigate the performance of NUTM on algorithmic
tasks introduced in Graves et al., (2014): Copy, Repeat Copy, Associative
Recall, Dynamic N-Gram and Priority Sort. Besides these five NTM tasks,
we add another task named Long Copy which doubles the length of training
sequences in the Copy task. In these tasks, the model will be fed
a sequence of input items and is required to infer a sequence of output
items. Each item is represented by a binary vector.

In the experiment, we compare two models: NTM\footnote{For algorithmic tasks, we choose NTM as the only baseline as NTM is
known to perform and generalise well on these tasks. If NSM can help
NTM in these tasks, it will probably help other MANNs as well.} and NUTM with two programs. Although the tasks are atomic, we argue
that there should be at least two memory manipulation schemes across
timesteps, one for encoding the inputs to the memory and another for
decoding the output from the memory. The two models are trained with
cross-entropy objective function under the same setting as in Graves
et al., (2014). For fair comparison, the controller hidden dimension
of NUTM is set smaller to make the total number of parameters of NUTM
equivalent to that of NTM (details in Appendix \ref{subsec:Details-on-Synthetic}).

We run each experiments five times and report the mean with error
bars of training losses for the first 4 tasks in Fig. \ref{fig:Learning-curves-on}.
Except for the Copy task, which is too simple, other tasks observe
convergence speed improvement of NUTM over that of NTM, thereby validating
the benefit of using two programs across timesteps even for the single
task setting. Full report is listed in Appendix \ref{subsec:Full-Learning-Curves}.
As NUTM requires fewer training samples to converge, it generalises
better to unseen sequences that are longer than training sequences.
Table \ref{tab:Generalisation-performance-of} reports the test results
of the best models chosen after five runs and confirms the outperformance
of NUTM over NTM for generalisation. 

To illustrate the program usage, we plot NUTM's program distributions
across timesteps for Repeat Copy and Priority Sort in Fig. \ref{fig:Memory-read-(a,c,d)/write}
(a) and (b), respectively. We observe two program usage patterns corresponding
to the encoding and decoding phases. For Repeat Copy, there is no
reading in encoding and thus, NUTM assigns the ``no-read'' strategy
mainly to the ``orange program''. In decoding, the sequential reading
is mostly done by the ``blue program'' with some contributions from
the ``orange program'' when resetting reading head. For Priority
Sort, while the encoding ``fitting writing'' (see Graves et al.,
(2014) for explanation on the strategy) is often executed by the ``blue
program'', the decoding writing is completely taken by the ``orange''
program (more visualisations in Appendix \ref{subsec:Program-Usage-Visualizations}). 

\begin{figure}
\begin{centering}
\includegraphics[width=0.95\linewidth]{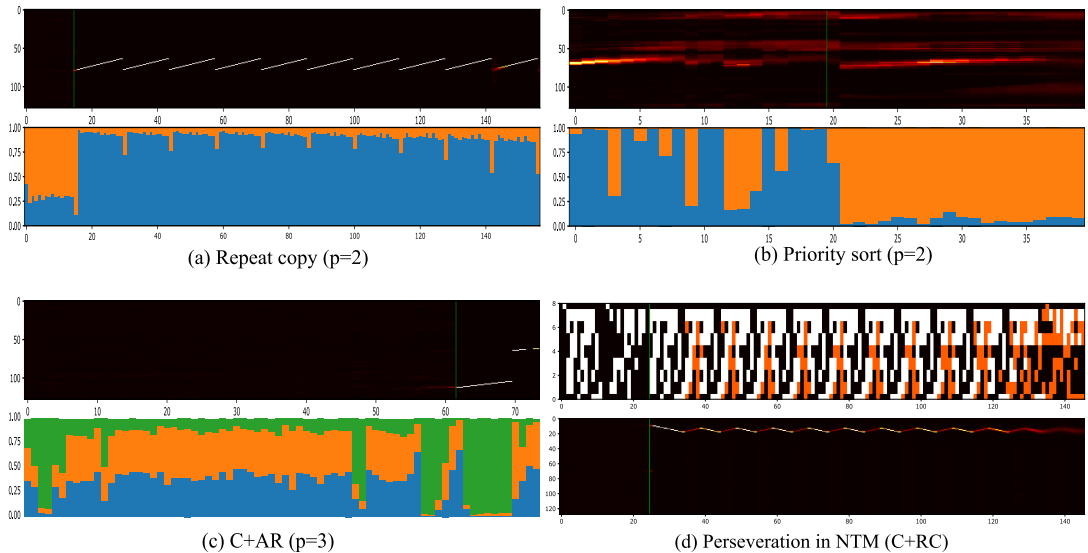}
\par\end{centering}
\caption{(a,b,c) visualises NUTM's executions in synthetic tasks: the upper
rows are memory read (left)/write (right) locations; the lower rows
are program distributions over timesteps. The green line indicates
the start of the decoding phase. (d) visualises perservation in NTM:
the upper row are input, output, predicted output with errors (orange
bits); the lower row is reading location. \label{fig:Memory-read-(a,c,d)/write}}
\end{figure}

\subsection{NTM Sequencing Tasks}

In neuroscience, sequencing tasks test the ability to remember a series
of tasks and switch tasks alternatively \citet{hal2019neur}. A dysfunctional
brain may have difficulty in changing from one task to the next and
get stuck in its preferred task (perseveration phenomenon). To analyse
this problem in neural algorithmic learners, we propose a new set
of experiments in which a task is generated by sequencing a list of
subtasks. The set of subtasks is chosen from the NTM single tasks
(excluding Dynamic N-grams for format discrepancy) and the order of
subtasks in the sequence is dictated by an indicator vector put at
the beginning of the sequence. Amongst possible combinations of subtasks,
we choose \{Copy, Repeat Copy\}(C+RC), \{Copy, Associative Recall\}
(C+AR), \{Copy, Priority Sort\} (C+PS) and all (C+RC+AC+PS)\footnote{We focus on the combinations that contain Copy as Copy is the only
task where NTM can reach NUTM's performance. If NTM fails in these
combinations, it will most likely fail in other combinations.}. The learner observes the order indicator following by a sequence
of subtasks' input items and is requested to consecutively produce
the output items of each subtasks. 

As shown in Fig. \ref{fig:Learning-curves-on-1}, some tasks such
as Copy and Associative Recall, easy to solve if trained separately,
become unsolvable by NTM when sequenced together. One reason for NTM's
poor performance is its failure to change the memory access behavior
(perseveration). For example, NTM keeps following repeat copy reading
strategy for all timesteps in C+RC task (Fig. \ref{fig:Memory-read-(a,c,d)/write}
(d)). Meanwhile, NUTM can learn to change program distribution when
a new subtask appears in the sequence and thus ensure different memory
accessing strategy per subtask (Fig. \ref{fig:Memory-read-(a,c,d)/write}
(c)).

\begin{figure}
\begin{centering}
\includegraphics[width=1\linewidth]{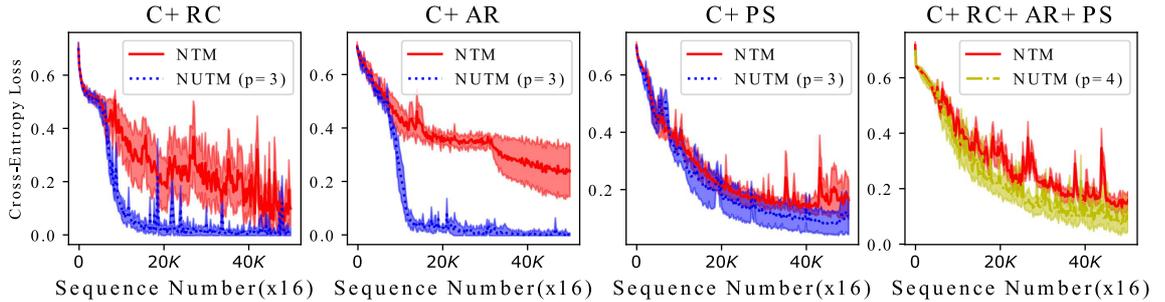}
\par\end{centering}
\caption{Learning curves on sequencing NTM tasks.\label{fig:Learning-curves-on-1}}
\end{figure}

\subsection{Continual Procedure Learning}

In continual learning, catastrophic forgetting happens when a neural
network quickly forgets previously acquired skills upon learning new
skills \citet{french1999catastrophic}. In this section, we prove
the versatility of NSM by showing that a naive application of NSM
without much modification can help NTM to mitigate catastrophic forgetting.
We design an experiment similar to the Split MNIST \citet{zenke2017continual}
to investigate whether NSM can improve NTM's performance. In our experiment,
we let the models see the training data from the 4 tasks: Copy (C),
Repeat Copy (RC), Associative Recall (AR) and Priority Sort (PS),
consecutively in this order. Each task is trained in 20,000 iterations
with batch size 16 (see Appendix \ref{subsec:Details-on-Synthetic}
for task details). To encourage NUTM to spend exactly one program
per task while freezing others, we force ``hard'' attention over
the programs by replacing the softmax function in Eq. \ref{eq:pt}
with the Gumbel-softmax \citet{jang2016categorical}. Also, to ignore
catastrophic forgetting in the state network, we use Feedforward controllers
in the two baselines.

After finishing one task, we evaluate the bit accuracy $-$measured
by $1-$(bit error per sequence/total bits per sequence)$-$over 4
tasks. As shown in Fig. \ref{fig:Mean-bit-accuracy}, NUTM outperforms
NTM by a moderate margin (10-40\% per task). Although NUTM also experiences
catastrophic forgetting, it somehow preserves some memories of previous
tasks. Especially, NUTM keeps performing perfectly on Copy even after
it learns Repeat Copy. For other dissimilar task transitions, the
performance drops significantly, which requires more effort to bring
NSM to continual learning. 

\begin{figure}
\centering{}\includegraphics[width=0.95\linewidth]{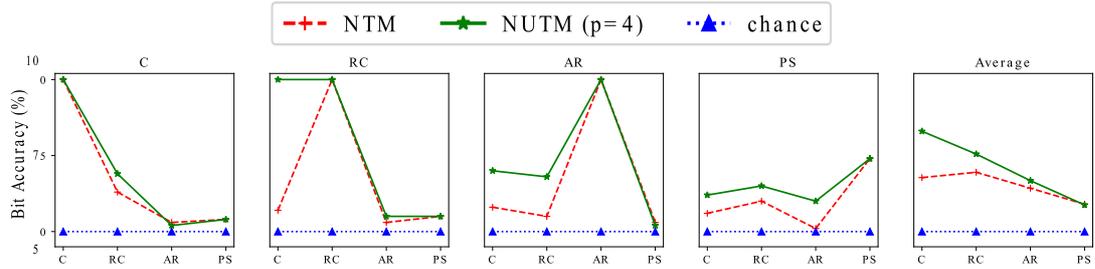}\caption{Mean bit accuracy for the continual algorithmic tasks. Each of the
first four panels show bit accuracy on four tasks after finishing
a task. The rightmost shows the average accuracy.\label{fig:Mean-bit-accuracy}}
\end{figure}

\subsection{Few-Shot Learning}

Few-shot learning or meta learning tests the ability to rapidly adapt
within a task while gradually capturing the way the task structure
varies \citet{thrun1998lifelong}. By storing sample-class bindings,
MANNs are capable of classifying new data after seeing only few samples
 \citet{santoro2016meta}. As NSM gives flexible memory controls,
it makes MANN more adaptive to changes and thus perform better in
this setting. To verify that, we apply NSM to the LRUA memory and
follow the experiments introduced in Santoro et al., (2016), using
the Omniglot dataset to measure few-shot classification accuracy.
The dataset includes images of 1623 characters, with 20 examples of
each character. During training, a sequence (episode) of images are
randomly selected from $C$ classes of characters in the training
set (1200 characters), where $C=5,10$ corresponding to sequence length
of 50, 75, respectively. Each class is assigned a random label which
shuffles between episodes and is revealed to the models after each
prediction. After 100,000 episodes of training, the models are tested
with unseen images from the testing set (423 characters). The two
baselines are MANN and NUTM (both use LRUA core). For NUTM, we only
tune $p$ and pick the best values: $p=2$ and $p=3$ for 5 classes
and 10 classes, respectively. 

Table \ref{tab:meta} reports the classification accuracy when the
models see characters for the second, third and fifth time. NUTM generally
achieves better results than MANN, especially when the number of classes
increases, demanding more adaptation within an episode. For the persistent
memory mode, which demands fast forgetting old experiences in previous
episodes, NUTM outperforms MANN significantly (10-20\%). 

\begin{table}
\begin{centering}
\begin{tabular}{lccccccc}
\hline 
\multirow{2}{*}{Model} & Persistent & \multicolumn{3}{c}{5 classes} & \multicolumn{3}{c}{10 classes}\tabularnewline
\cline{3-8} \cline{4-8} \cline{5-8} \cline{6-8} \cline{7-8} \cline{8-8} 
 & memory\tablefootnote{If the memory is not artificially erased between episodes, it is called
persistent. This mode is hard for the case of 5 classes \citet{santoro2016meta} } & $2^{nd}$ & $3^{rd}$ & $5^{th}$ & $2^{nd}$ & $3^{rd}$ & $5^{th}$\tabularnewline
\hline 
MANN (LRUA){*} & No & 82.8 & 91.0 & 94.9 & - & - & -\tabularnewline
MANN (LRUA) & No & 82.3 & 88.7 & 92.3 & 52.7 & 60.6 & 64.7\tabularnewline
NUTM (LRUA) & No & \textbf{85.7} & \textbf{91.3} & \textbf{95.5} & \textbf{68.0} & \textbf{78.1} & \textbf{82.8}\tabularnewline
\hline 
MANN (LRUA) & Yes & 66.2 & 73.4 & 81.0 & 51.3 & 59.2 & 63.3\tabularnewline
NUTM (LRUA) & Yes & \textbf{77.8} & \textbf{85.8} & \textbf{89.8} & \textbf{69.0} & \textbf{77.9} & \textbf{82.7}\tabularnewline
\hline 
\end{tabular}
\par\end{centering}
~

\caption{Test-set classification accuracy (\%) on the Omniglot dataset after
100,000 episodes of training. {*} denotes available results from Santoro
et al., (2016). See Appendix \ref{subsec:Details-on-Few-shot} for
more details. \label{tab:meta}}
\end{table}

\subsection{Text Question Answering }

Reading comprehension typically involves an iterative process of multiple
actions such as reading the story, reading the question, outputting
the answers and other implicit reasoning steps \citet{weston2015towards}.
We apply NUTM to the question answering domain by replacing the NTM
core with DNC \citet{graves2016hybrid}. Compared to NTM's sequential
addressing, dynamic memory addressing in DNC is more powerful and
thus suitable for NSM integration to solve non-algorithmic problems
such as question answering. Following previous works of DNC, we use
bAbI dataset \citet{weston2015towards} to measure the performance
of the NUTM with DNC core (two variants $p=2$ and $p=4$). In the
dataset, each story is followed by a series of questions and the network
reads all word by word, then predicts the answers. Although synthetically
generated, bAbI is a good benchmark that tests 20 aspects of natural
language reasoning including complex skills such as induction, counting
and path finding, 

We found that NUTM with 4 programs, after 50 epochs jointly trained
on all 20 question types, can achieve a mean test error rate of 3.3\%
and manages to solve 19/20 tasks (a task is considered solved if its
error <5\%). The mean and s.d. across 10 runs are also compared with
other results reported by recent works (see Table \ref{tab:Mean-bAbI-error}).
Excluding baselines under different setups, our result is the best
reported mean result on bAbI that we are aware of. More details are
described in Appendix \ref{subsec:Details-on-bAbI}. 

\begin{table}
\begin{centering}
\begin{tabular}{lc}
\hline 
Model & Error\tabularnewline
\hline 
DNC\citet{graves2016hybrid} & 16.7 \textpm{} 7.6 \tabularnewline
SDNC\citet{rae2016scaling} & 6.4 \textpm{} 2.5 \tabularnewline
ADNC\citet{W18-2606} & 6.3 \textpm{} 2.7 \tabularnewline
DNC-MD\citet{csordas2018improving} & 9.5 \textpm{} 1.6\tabularnewline
\hline 
NUTM (DNC core) $p=2$ & 7.5 \textpm{} 1.6\tabularnewline
NUTM (DNC core) $p=4$ & \textbf{5.6 \textpm{} 1.9}\tabularnewline
\hline 
\end{tabular}
\par\end{centering}
~

\caption{Mean and s.d. for bAbI error ($\%$).\label{tab:Mean-bAbI-error}}
\end{table}

\section{Closing Remarks}

This chapter introduces the Neural Stored-program Memory (NSM), a
new type of external memory for neural networks. The memory, which
takes inspirations from the stored-program memory in computer architecture,
gives memory-augmented neural networks (MANNs) flexibility to change
their control programs through time while maintaining differentiability.
The mechanism simulates modern computer behavior, potential making
MANNs truly neural computers. Our experiments demonstrated that when
coupled with our model, the Neural Turing Machine learns algorithms
better and adapts faster to new tasks at both sequence and sample
levels. When used in few-shot learning, our method helps MANN as well.
We also applied the NSM to the Differentiable Neural Computer and
observed a significant improvement, reaching the state-of-the-arts
in the bAbI task.

\chapter{Conclusions \label{chap:Conclusions}}

\section{Summary}

In this thesis, we have presented several types of memory for neural
networks in general and Recurrent Neural Networks (RNNs) in particular.
We emphasise the notion of the memory as an external storage for RNNs
in which, RNNs can learn to read from and write to the external memory
to support their working memory (Chapter \ref{chap:Tax}). We reviewed
advancements to solve the difficulties in training RNNs such as gating
and attention mechanisms and especially we focused on slot-based MANNs,
which is based on sparse distributed memory model\textendash the main
theme of new models we propose in the thesis (Chapter \ref{chap:MANN}).
Our main contributions are four-fold. First, we extended MANNs as
a multi-process multi-view model to handle complex problems such as
sequence-to-sequence mapping and multi-view sequential learning (Chapter
\ref{chap:multiple}). We further extended MANNs as a model for discrete
sequence generation on conversational data where variability and coherence
are required (Chapter \ref{chap:Variational-Memory-Encoder}). We
also shed light into memory operations to estimate MANN capacity and
proposed new writing schemes to maximise the capacity (Chapter \ref{chap:Optimal-Writing-in}).
Finally, we introduced a new class of MANNs that follows stored-program
memory principle and can switch the controller's programs to execute
different functionalities through time. 

In Chapter \ref{chap:multiple}, we presented our first contributions:
Dual Controller Write-Protected Memory Augmented Neural Network (DCw-MANN),
an extension of MANN to model sequence to sequence mapping, and Dual
Memory Neural Computer (DMNC) that can capture the correlations between
two views by exploiting two external memory units. We demonstrated
these models on predicting next disease stages and recommending medicines
in healthcare. The results are competitive against contemporary state-of-the-arts.

Chapter \ref{chap:Variational-Memory-Encoder} presented Variational
Memory Encoder-Decoder, a novel memory-augmented generation framework.
Our external memory not only holds long-term context but also constructs
mixture model distribution generating the latent variables. We derived
theoretical analysis to validate our training protocol using Stochastic
Gradient Variational Bayes framework by minimising variational approximation
of KL divergence. We evaluated our model on two open-domain and two
closed-domain conversational datasets and outperformed other baselines
by a large margin. 

Chapter \ref{chap:Optimal-Writing-in} focused more on theoretical
analysis of a meaningful measurement on MANN's memory capacity. We
proposed solution dubbed Uniform Writing is optimal in terms of maximising
the memory capacity. To encourage forgetting when necessary, we introduce
modifications to the original solution, resulting in a new solution
termed Cached Uniform Writing. This method aims to balance between
memorising and forgetting via allowing overwriting mechanism. We conducted
a set of diverse experiments on six ultra-long sequential learning
problems given a limited number of memory slots to demonstrate the
superiority of our methods.

Chapter \ref{chap:Neural-Stored-program-Memory} addressed the simulation
capacity of MANNs. To make current MANNs truly simulate modern computers,
a design of Neural Stored-program Memory (NSM) was proposed to implement
stored-program principle, which resulted in new MANN architectures
that materialise Universal Turing Machines. 

\section{Future Directions}

There are possible extensions of the work proposed in this thesis
for further investigations. First, in Chapter \ref{chap:multiple},
we can extend the DCw-MANN to handle multiple healthcare tasks by
developing new capabilities for medical question answering, which
can be cast to a sequence-to-sequence mapping problem. Moreover, DMNC
can be generalised to multi-input multi-output settings and extending
the range of applications to bigger problems such as multi-media and
multi-agent systems. For VMED (Chapter \ref{chap:Variational-Memory-Encoder}),
future explorations may involve implementing a dynamic number of modes
that enable learning of the optimal $K$ for each timestep. Another
aspect would be multi-person dialog setting, where our memory as mixture
model may be useful to capture more complex modes of speaking in the
dialog. Future investigations for Chapter \ref{chap:Optimal-Writing-in}
will focus on tightening the measurement bound. Last but not least,
we emphasise that apart from MANNs, other neural networks can also
reap benefits from the NSM proposed in Chapter \ref{chap:Neural-Stored-program-Memory}.
More effort will be spent on integrating NSM into different neural
networks to make them more flexible and modular.

\selectlanguage{australian}%
\newpage{}

\chapter*{Appendix}

\selectlanguage{english}%
\markboth{}{Appendix}

\selectlanguage{australian}%
\addcontentsline{toc}{chapter}{Appendix}

\counterwithin{thm}{section} 

\counterwithin{lem}{section} 

\counterwithin{table}{section} 

\counterwithin{figure}{section} 

\selectlanguage{english}%
\renewcommand\thesection{\Alph{section}}

\selectlanguage{australian}%
\selectlanguage{english}%

\section{Supplementary for Chapter 5}

\subsection{Proof of Theorem \ref{The-KL-divergence}\label{subsec:Proof-of-theorem-KL}}
\begin{proof}
$D_{var}$$\left(f\parallel g\right)$ \citet{hershey2007approximating}
is an approximation of $KL$ divergence between two Mixture of Gaussians
(MoG), which is defined as the following,

\begin{alignat}{1}
D_{var}\left(f\parallel g\right)= & \underset{j}{\sum}\pi_{j}^{f}\log\frac{\underset{j'}{\sum}\pi_{j'}^{f}e^{-KL\left(f_{j}\parallel f_{j'}\right)}}{\underset{i}{\sum}\pi_{i}^{g}e^{-KL\left(f_{j}\parallel g_{i}\right)}}\label{eq:ori_dvar}
\end{alignat}
In our case, $f$ is a Gaussian, a special case of MoG where the number
of mode equals one. Then, Eq. (\ref{eq:ori_dvar}) becomes

\[
D_{var}\left(f\parallel g\right)=\log\frac{1}{\stackrel[i=1]{K}{\sum}\pi_{i}^{g}e^{-KL\left(f\parallel g^{i}\right)}}=-\log\stackrel[i=1]{K}{\sum}\pi^{i}e^{-KL\left(f\parallel g^{i}\right)}
\]
Let define the log-likelihood $L_{f}\left(g\right)=E_{f\left(x\right)}\left[\log g\left(x\right)\right]$,
the lower bound for $L_{f}\left(g\right)$ can be also be derived,
using variational parameters as follows,

\begin{alignat*}{1}
L_{f}\left(g\right)= & E_{f}\left[\log\left(\stackrel[i=1]{K}{\sum}\pi^{i}g^{i}\left(x\right)\right)\right]\\
= & \stackrel[-\infty]{+\infty}{\int}f\left(x\right)\log\left(\stackrel[i=1]{K}{\sum}\beta^{i}\pi^{i}\frac{g^{i}\left(x\right)}{\beta^{i}}\right)dx\\
\geq & \stackrel[i=1]{K}{\sum}\beta^{i}\stackrel[-\infty]{+\infty}{\int}f\left(x\right)\log\left(\pi^{i}\frac{g^{i}\left(x\right)}{\beta^{i}}\right)dx
\end{alignat*}
where $\beta^{i}\geq0$ and $\stackrel[i=1]{K}{\sum}\beta^{i}=1$.
According to Durrieu et al., (2012), maximising the RHS of the above
inequality with respect to $\beta^{i}$ provides a lower bound for
$L_{f}\left(g\right)$ \citet{durrieu2012lower},

\begin{alignat*}{1}
L_{f}\left(g\right)\geq & \log\stackrel[i=1]{K}{\sum}\pi^{i}e^{-KL\left(f\parallel g^{i}\right)}+L_{f}\left(f\right)\\
= & -D_{var}+L_{f}\left(f\right)\\
\Rightarrow D_{var}\geq & L_{f}\left(f\right)-L_{f}\left(g\right)\\
= & KL\left(f\parallel g\right)
\end{alignat*}
Therefore, the $KL$ divergence has an upper bound: $D_{var}$.
\end{proof}

\subsection{Derivation of the Upper Bound on the Total Timestep-Wise $KL$ Divergence}
\begin{lem}
Chebyshev's sum inequality: \label{lem:Chebyshev's-sum-inequality:}\\
if 

\[
a_{1}\geq a_{2}\geq...\geq a_{n}
\]
and
\end{lem}
\[
b_{1}\geq b_{2}\geq...\geq b_{n}
\]
then\\
\[
\frac{1}{n}\stackrel[k=1]{n}{\sum}a_{k}b_{k}\geq\left(\frac{1}{n}\stackrel[k=1]{n}{\sum}a_{k}\right)\left(\frac{1}{n}\stackrel[k=1]{n}{\sum}b_{k}\right)
\]

\begin{proof}
Consider the sum

\[
S=\stackrel[j=1]{n}{\sum}\stackrel[k=1]{n}{\sum}\left(a_{j}-a_{k}\right)\left(b_{j}-b_{k}\right)
\]
The two sequences are non-increasing, therefore $a_{j}-a_{k}$ and
$b_{j}-b_{k}$ have the same sign for any $j,k$. Hence $S\ge0$.
Opening the brackets, we deduce

\[
{\displaystyle 0\leq2n\sum_{j=1}^{n}a_{j}b_{j}-2\sum_{j=1}^{n}a_{j}\,\sum_{k=1}^{n}b_{k}}
\]
whence

\[
{\displaystyle {\frac{1}{n}}\sum_{j=1}^{n}a_{j}b_{j}\geq\left({\frac{1}{n}}\sum_{j=1}^{n}a_{j}\right)\,\left({\frac{1}{n}}\sum_{k=1}^{n}b_{k}\right)}
\]
\end{proof}
In our problem, $a_{i}=f_{i}\left(x\right)$ and $b_{i}=\log\left[g_{i}\left(x\right)\right]$,
$i=\overline{1,T}$. Under the assumption that at each step, thanks
to minimising $D_{var}$, the approximation between the MoG and the
Gaussian is adequate to preserve the order of these values, that is,
if $f_{i}\left(x\right)\leq f_{j}\left(x\right)$, then $g_{i}\left(x\right)\leq g_{j}\left(x\right)$
and $\log\left[g_{i}\left(x\right)\right]\leq\log\left[g_{j}\left(x\right)\right]$.
Without loss of generality, we hypothesise that $f_{1}\left(x\right)\leq f_{2}\left(x\right)\leq...\leq f_{T}\left(x\right)$,
then we have $\log\left[g_{1}\left(x\right)\right]\leq\log\left[g_{2}\left(x\right)\right]\leq...\leq\log\left[g_{T}\left(x\right)\right]$.
Thus, applying Lemma \ref{lem:Chebyshev's-sum-inequality:}, we have

\begin{alignat*}{1}
\frac{1}{T}\stackrel[t=1]{T}{\sum}f_{t}\left(x\right)\log\left[g_{t}\left(x\right)\right]dx\geq & \frac{1}{T}\stackrel[t=1]{T}{\sum}f_{t}\left(x\right)\frac{1}{T}\stackrel[t=1]{T}{\sum}\log\left[g_{t}\left(x\right)\right]dx\\
\Rightarrow\stackrel[-\infty]{+\infty}{\int}\stackrel[t=1]{T}{\sum}f_{t}\left(x\right)\log\left[g_{t}\left(x\right)\right]dx\geq & \stackrel[-\infty]{+\infty}{\int}\frac{1}{T}\stackrel[t=1]{T}{\sum}f_{t}\left(x\right)\stackrel[t=1]{T}{\sum}\log\left[g_{t}\left(x\right)\right]dx\\
\Rightarrow\stackrel[-\infty]{+\infty}{\int}\stackrel[t=1]{T}{\sum}f_{t}\left(x\right)\log\left[g_{t}\left(x\right)\right]dx\geq & \stackrel[-\infty]{+\infty}{\int}\frac{1}{T}\stackrel[t=1]{T}{\sum}f_{t}\left(x\right)\log\left[\stackrel[t=1]{T}{\prod}g_{t}\left(x\right)\right]dx
\end{alignat*}
Thus, the upper bound on the total timestep-wise $KL$ divergence
reads

\[
\stackrel[-\infty]{+\infty}{\int}\stackrel[t=1]{T}{\sum}f_{t}\left(x\right)\log\left[f_{t}\left(x\right)\right]dx\quad-\stackrel[-\infty]{+\infty}{\int}\frac{1}{T}\stackrel[t=1]{T}{\sum}f_{t}\left(x\right)\log\left[\stackrel[t=1]{T}{\prod}g_{t}\left(x\right)\right]dx
\]

\subsection{Proof $\stackrel[t=1]{T}{\prod}g_{t}\left(x\right)=\stackrel[t=1]{T}{\prod}\stackrel[i=1]{K}{\sum}\pi_{t}^{i}g_{t}^{i}\left(x\right)$
Is a Scaled MoG}
\begin{lem}
Product of two Gaussians is a scaled Gaussian. \label{lem:Product-of-two-1}
\end{lem}
\begin{proof}
Let $\mathcal{N}_{x}\left(\mu,\Sigma\right)$ denote a density of
$x$, then

\[
\mathcal{N}_{x}\left(\mu_{1},\Sigma_{1}\right)\cdot\mathcal{N}_{x}\left(\mu_{2},\Sigma_{2}\right)=c_{c}\mathcal{N}_{x}\left(\mu_{c},\Sigma_{c}\right)
\]
where

\begin{alignat*}{1}
c_{c}= & \frac{1}{\sqrt{\det\left(2\pi\left(\Sigma_{1}+\Sigma_{2}\right)\right)}}\exp\left(-\frac{1}{2}\left(m_{1}-m_{2}\right)^{T}\left(\Sigma_{1}+\Sigma_{2}\right)^{-1}\left(m_{1}-m_{2}\right)\right)\\
m_{c}= & \left(\Sigma_{1}^{-1}+\Sigma_{2}^{-1}\right)^{-1}\left(\Sigma_{1}^{-1}m_{1}+\Sigma_{2}^{-1}m_{2}\right)\\
\Sigma_{c}= & \left(\Sigma_{1}^{-1}+\Sigma_{2}^{-1}\right)
\end{alignat*}
\end{proof}
\begin{lem}
Product of two MoGs is proportional to an MoG. \label{lem:Product-of-two}
\end{lem}
\begin{proof}
Let $g_{1}\left(x\right)=\stackrel[i=1]{K_{1}}{\sum}\pi_{1,i}\mathcal{N}_{x}\left(\mu_{1,i},\Sigma_{1,i}\right)$
and $g_{2}\left(x\right)=\stackrel[j=1]{K_{2}}{\sum}\pi_{2,j}\mathcal{N}_{x}\left(\mu_{2,j},\Sigma_{2,j}\right)$
are two Mixtures of Gaussians. We have

\begin{align}
g_{1}\left(x\right)\cdot g_{2}\left(x\right)= & \stackrel[i=1]{K_{1}}{\sum}\pi_{1,i}\mathcal{N}_{x}\left(\mu_{1,i},\Sigma_{1,i}\right)\cdot\stackrel[j=1]{K_{2}}{\sum}\pi_{2,j}\mathcal{N}_{x}\left(\mu_{2,j},\Sigma_{2,j}\right)\nonumber \\
= & \stackrel[i=1]{K_{1}}{\sum}\stackrel[,j=1]{K_{2}}{\sum}\pi_{1,i}\pi_{2,j}\mathcal{N}_{x}\left(\mu_{1,i},\Sigma_{1,i}\right)\cdot\mathcal{N}_{x}\left(\mu_{2,j},\Sigma_{2,j}\right)\label{eq:big_mix}
\end{align}
By applying Lemma \ref{lem:Product-of-two-1} to Eq. (\ref{eq:big_mix}),
we have

\begin{alignat}{1}
g_{1}\left(x\right)\cdot g_{2}\left(x\right)= & \stackrel[i=1]{K_{1}}{\sum}\stackrel[,j=1]{K_{2}}{\sum}\pi_{1,i}\pi_{2,j}c_{ij}\mathcal{N}_{x}\left(\mu_{ij},\Sigma_{ij}\right)\nonumber \\
= & \:C\stackrel[i=1]{K_{1}}{\sum}\stackrel[,j=1]{K_{2}}{\sum}\frac{\pi_{1,i}\pi_{2,j}c_{ij}}{C}\mathcal{N}_{x}\left(\mu_{ij},\Sigma_{ij}\right)\label{eq:mixprop}
\end{alignat}
where $C=\stackrel[i=1]{K_{1}}{\sum}\stackrel[,j=1]{K_{2}}{\sum}\pi_{1,i}\pi_{2,j}c_{ij}$.
Clearly, Eq. (\ref{eq:mixprop}) is proportional to an MoG with $K_{1}\cdot K_{2}$
modes 
\end{proof}
\begin{thm}
$\stackrel[t=1]{T}{\prod}g_{t}\left(x\right)=\stackrel[t=1]{T}{\prod}\stackrel[i=1]{K}{\sum}\pi_{t}^{i}g_{t}^{i}\left(x\right)$
is a scaled MoG.
\end{thm}
\begin{proof}
By induction from Lemma \ref{lem:Product-of-two}, we can easily show
that product of $T$ MoGs is also proportional to an MoG. That means
$\stackrel[t=1]{T}{\prod}g_{t}\left(x\right)$ equals to a scaled
MoG.
\end{proof}

\subsection{Details of Data Descriptions and Model Implementations}

Here we list all datasets used in our experiments:
\begin{itemize}
\item Open-domain datasets: 
\begin{itemize}
\item Cornell movie dialog: This corpus contains a large metadata-rich collection
of fictional conversations extracted from 617 raw movies with 220,579
conversational exchanges between 10,292 pairs of movie characters.
For each dialog, we preprocess the data by limiting the context length
and the utterance output length to 20 and 10, respectively. The vocabulary
is kept to top 20,000 frequently-used words in the dataset.
\item OpenSubtitles: This dataset consists of movie conversations in XML
format. It also contains sentences uttered by characters in movies,
yet it is much bigger and noisier than Cornell dataset. After preprocessing
as above, there are more than 1.6 million pairs of contexts and utterance
with chosen vocabulary of 40,000 words. 
\end{itemize}
\item Closed-domain datasets:: 
\begin{itemize}
\item Live Journal (LJ) user question-answering dataset: question-answer
dialog by LJ users who are members of anxiety, arthritis, asthma,
autism, depression, diabetes, and obesity LJ communities\footnote{\url{https://www.livejournal.com/}}.
After preprocessing as above, we get a dataset of more than 112,000
conversations. We limit the vocabulary size to 20,000 most common
words. 
\item Reddit comments dataset: This dataset consists of posts and comments
about movies in Reddit website\footnote{\url{https://www.reddit.com/r/movies/}}.
A single post may have multiple comments constituting a multi-people
dialog amongst the poster and commentors, which makes this dataset
the most challenging one. We crawl over four millions posts from Reddit
website and after preprocessing by retaining conversations whose utterance's
length are less than 20, we have a dataset of nearly 200 thousand
conversations with a vocabulary of more than 16 thousand words. 
\end{itemize}
\end{itemize}
We trained with the following hyperparameters (according to the performance
on the validate dataset): word embedding has size 96 and is shared
across everywhere. We initialise the word embedding from Google's
Word2Vec \citet{mikolov2013distributed} pretrained word vectors.
The hidden dimension of LSTM in all controllers is set to 768 for
all datasets except the big OpenSubtitles whose LSTM dimension is
1024. The number of LSTM layers for every controller is set to 3.
All the initial weights are sampled from a normal distribution with
mean $0$, standard deviation 0.$1$. The mini-batch size is chosen
as 256. The models are trained end-to-end using the Adam optimiser
\citet{kingma2014adam} with a learning rate of 0.001 and gradient
clipping at 10. For models using memory, we set the number and the
size of memory slots to 16 and 64, respectively. As indicated previously,
it is not trivial to optimise VAE with RNN-like decoder due to the
vanishing latent variable problem \citet{bowman2016generating}. Hence,
to make the variational models in our experiments converge we have
to use the $KL$ annealing trick by adding to the $KL$ loss term
an annealing coefficient $\alpha$ starts with a very small value
and gradually increase up to 1.

\subsection{Full Reports on Model Performance}

\begin{table}[H]
\begin{centering}
\begin{tabular}{cccccc}
\hline 
Model & BLEU-1 & BLEU-2 & BLEU-3 & BLEU-4 & A-glove\tabularnewline
\hline 
Seq2Seq & 18.4 & 14.5 & 12.1 & 9.5 & 0.52\tabularnewline
Seq2Seq-att & 17.7 & 14.0 & 11.7 & 9.2 & 0.54\tabularnewline
DNC & 17.6 & 13.9 & 11.5 & 9.0 & 0.51\tabularnewline
CVAE & 16.5 & 13.0 & 10.9 & 8.5 & 0.56\tabularnewline
VLSTM & 18.6 & 14.8 & 12.4 & 9.7 & 0.59\tabularnewline
\hline 
VMED (K=1) & 20.7 & 16.5 & 13.8 & 10.8 & 0.57\tabularnewline
VMED (K=2) & 22.3 & 18.0 & 15.2 & 11.9 & \textbf{0.64}\tabularnewline
VMED (K=3) & 19.4 & 15.6 & 13.2 & 10.4 & 0.63\tabularnewline
VMED (K=4) & \textbf{23.1} & \textbf{18.5} & \textbf{15.5} & \textbf{12.3} & 0.61\tabularnewline
\hline 
\end{tabular}
\par\end{centering}
\caption{Results on Cornell Movies}
\end{table}

\begin{table}[H]
\begin{centering}
\begin{tabular}{cccccc}
\hline 
Model & BLEU-1 & BLEU-2 & BLEU-3 & BLEU-4 & A-glove\tabularnewline
\hline 
Seq2Seq & 11.4 & 8.7 & 7.1 & 5.4 & 0.29\tabularnewline
Seq2Seq-att & 13.2 & 10.2 & 8.4 & 6.5 & 0.42\tabularnewline
DNC & 14.3 & 11.2 & 9.3 & 7.2 & 0.47\tabularnewline
CVAE & 13.5 & 10.2 & 8.4 & 6.6 & 0.45\tabularnewline
VLSTM & 16.4 & 12.7 & 10.4 & 8.1 & 0.43\tabularnewline
\hline 
VMED (K=1) & 12.9 & 9.5 & 7.5 & 6.2 & 0.44\tabularnewline
VMED (K=2) & 15.3 & 13.8 & 10.4 & 8.8 & 0.49\tabularnewline
VMED (K=3) & \textbf{24.8} & \textbf{19.7} & \textbf{16.4} & \textbf{12.9} & \textbf{0.54}\tabularnewline
VMED (K=4) & 17.9 & 14.2 & 11.8 & 9.3 & 0.52\tabularnewline
\hline 
\end{tabular}
\par\end{centering}
\caption{Results on OpenSubtitles}
\end{table}

\begin{table}[H]
\begin{centering}
\begin{tabular}{cccccc}
\hline 
Model & BLEU-1 & BLEU-2 & BLEU-3 & BLEU-4 & A-glove\tabularnewline
\hline 
Seq2Seq & 13.1 & 10.1 & 8.3 & 6.4 & 0.45\tabularnewline
Seq2Seq-att & 11.4 & 8.7 & 7.1 & 5.6 & 0.49\tabularnewline
DNC & 12.4 & 9.6 & 7.8 & 6.1 & 0.47\tabularnewline
CVAE & 12.2 & 9.4 & 7.7 & 6.0 & 0.48\tabularnewline
VLSTM & 11.5 & 8.8 & 7.3 & 5.6 & 0.46\tabularnewline
\hline 
VMED (K=1) & 13.7 & 10.7 & 8.9 & 6.9 & 0.47\tabularnewline
VMED (K=2) & 15.4 & 12.2 & 10.1 & 7.9 & \textbf{0.51}\tabularnewline
VMED (K=3) & \textbf{18.1} & \textbf{14.8} & \textbf{12.4} & \textbf{9.8} & 0.49\tabularnewline
VMED (K=4) & 14.4 & 11.4 & 9.5 & 7.5 & 0.47\tabularnewline
\hline 
\end{tabular}
\par\end{centering}
\caption{Results on LJ users question-answering}
\end{table}

\begin{table}[H]
\begin{centering}
\begin{tabular}{cccccc}
\hline 
Model & BLEU-1 & BLEU-2 & BLEU-3 & BLEU-4 & A-glove\tabularnewline
\hline 
Seq2Seq & 7.5 & 5.5 & 4.4 & 3.3 & 0.31\tabularnewline
Seq2Seq-att & 5.5 & 4.0 & 3.1 & 2.4 & 0.25\tabularnewline
DNC & 7.5 & 5.6 & 4.5 & 3.4 & 0.28\tabularnewline
CVAE & 5.3 & 4.3 & 3.6 & 2.8 & 0.39\tabularnewline
VLSTM & 6.9 & 5.1 & 4.1 & 3.1 & 0.27\tabularnewline
\hline 
VMED (K=1) & 9.1 & 6.8 & 5.5 & 4.3 & 0.39\tabularnewline
VMED (K=2) & 9.2 & 7.0 & 5.7 & 4.4 & 0.38\tabularnewline
VMED (K=3) & \textbf{12.3} & \textbf{9.7} & \textbf{8.1} & \textbf{6.4} & \textbf{0.46}\tabularnewline
VMED (K=4) & 8.6 & 6.9 & 5.9 & 4.6 & 0.41\tabularnewline
\hline 
\end{tabular}
\par\end{centering}
\caption{Results on Reddit comments}
\end{table}

\section{Supplementary for Chapter 6}

\subsection{Derivation on the Bound Inequality in Linear Dynamic System\label{subsec:Derivation-on-theRNN}}

The linear dynamic system hidden state is described by the following
recursive equation

\[
h_{t}=Wx_{t}+Uh_{t-1}+b
\]
By induction,

\[
h_{t}=\stackrel[i=1]{t}{\sum}U^{t-i}Wx_{i}+C
\]
where $C$ is some constant with respect to $x_{i}$. In this case,
$\frac{\partial h_{t}}{\partial x_{i}}=U^{t-i}W$. By applying norm
sub-multiplicativity\footnote{If not explicitly stated otherwise, norm refers to any consistent
matrix norm which satisfies sub-multiplicativity. }, 

\begin{align*}
c_{i-1,t} & =\left\Vert U^{t-i+1}W\right\Vert \\
 & \leq\left\Vert U\right\Vert \left\Vert U^{t-i}W\right\Vert \\
 & =\left\Vert U\right\Vert c_{i,t}
\end{align*}
That is, $\lambda_{c}=\left\Vert U\right\Vert $.

\subsection{Derivation on the Bound Inequality in Standard RNN\label{subsec:Derivation-on-theRNN-1}}

The standard RNN hidden state is described by the following recursive
equation

\[
h_{t}=\tanh\left(Wx_{t}+Uh_{t-1}+b\right)
\]
From $\frac{\partial h_{t}}{\partial x_{i}}=\frac{\partial h_{t}}{\partial h_{t-1}}\frac{\partial h_{t-1}}{\partial x_{i}}$,
by induction, 

\[
\frac{\partial h_{t}}{\partial x_{i}}=\left(\stackrel[j=i+1]{t}{\prod}\frac{\partial h_{j}}{\partial h_{j-1}}\right)\frac{\partial h_{i}}{\partial x_{i}}=\left(\stackrel[j=i]{t}{\prod}diag\left(\tanh^{\prime}\left(a_{j}\right)\right)\right)U^{t-i}W
\]
where $a_{j}=Wx_{j}+Uh_{j-1}+b$ and $diag\left(\cdot\right)$ converts
a vector into a diagonal matrix. As $0\leq\tanh^{\prime}\left(x\right)=1-\tanh\left(x\right)^{2}\leq1$,
$\left\Vert diag\left(\tanh^{\prime}\left(X\right)\right)\right\Vert $
is bounded by some value $B$. By applying norm sub-multiplicativity, 

\begin{align*}
c_{i-1,t} & =\left\Vert \left[\stackrel[j=i]{t}{\prod}diag\left(\tanh^{\prime}\left(a_{j}\right)\right)\right]diag\left(\tanh^{\prime}\left(a_{i-1}\right)\right)U^{t-i+1}W\right\Vert \\
 & \leq\left\Vert U\right\Vert \left\Vert \stackrel[j=i]{t}{\prod}diag\left(\tanh^{\prime}\left(a_{j}\right)\right)U^{t-i}W\right\Vert \left\Vert diag\left(\tanh^{\prime}\left(a_{i-1}\right)\right)\right\Vert \\
 & =B\left\Vert U\right\Vert c_{i,t}
\end{align*}
That is, $\lambda_{c}=B\left\Vert U\right\Vert $.

\subsection{Derivation on the Bound Inequality in LSTM\label{subsec:Derivation-on-theLSTM}}

For the case of LSTM, the recursive equation reads
\begin{align*}
c_{t} & =\sigma\left(U_{f}x_{t}+W_{f}h_{t-1}+b_{f}\right)\odot c_{t-1}\\
 & +\sigma\left(U_{i}x_{t}+W_{i}h_{t-1}+b_{i}\right)\odot\tanh\left(U_{z}x_{t}+W_{z}h_{t-1}+b_{z}\right)\\
h_{t} & =\sigma\left(U_{o}x_{t}+W_{o}h_{t-1}+b_{o}\right)\odot\tanh\left(c_{t}\right)
\end{align*}
Taking derivatives,

\begin{align*}
\frac{\partial h_{j}}{\partial h_{j-1}} & =\sigma^{\prime}\left(o_{j}\right)\tanh\left(c_{j}\right)W_{o}+\sigma\left(o_{j}\right)\tanh^{\prime}\left(c_{j}\right)\sigma^{\prime}\left(f_{j}\right)c_{j-1}W_{f}\\
 & +\sigma\left(o_{j}\right)\tanh^{\prime}\left(c_{j}\right)\sigma^{\prime}\left(i_{j}\right)\tanh\left(z_{j}\right)W_{i}+\sigma\left(o_{j}\right)\tanh^{\prime}\left(c_{j}\right)\sigma\left(i_{j}\right)\tanh^{\prime}\left(z_{j}\right)W_{z}\\
\frac{\partial h_{j-1}}{\partial x_{j-1}} & =\sigma^{\prime}\left(o_{j-1}\right)\tanh\left(c_{j-1}\right)U_{o}+\sigma\left(o_{j-1}\right)\tanh^{\prime}\left(c_{j-1}\right)\sigma^{\prime}\left(f_{j-1}\right)c_{j-2}U_{f}\\
 & +\sigma\left(o_{j-1}\right)\tanh^{\prime}\left(c_{j-1}\right)\sigma^{\prime}\left(i_{j-1}\right)\tanh\left(z_{j-1}\right)U_{i}\\
 & +\sigma\left(o_{j-1}\right)\tanh^{\prime}\left(c_{j-1}\right)\sigma\left(i_{j-1}\right)\tanh^{\prime}\left(z_{j-1}\right)U_{z}\\
\frac{\partial h_{j}}{\partial x_{j}} & =\sigma^{\prime}\left(o_{j}\right)\tanh\left(c_{j}\right)U_{o}+\sigma\left(o_{j}\right)\tanh^{\prime}\left(c_{j}\right)\sigma^{\prime}\left(f_{j}\right)c_{j-1}U_{f}\\
 & +\sigma\left(o_{j}\right)\tanh^{\prime}\left(c_{j}\right)\sigma^{\prime}\left(i_{j}\right)\tanh\left(z_{j}\right)U_{i}+\sigma\left(o_{j}\right)\tanh^{\prime}\left(c_{j}\right)\sigma\left(i_{j}\right)\tanh^{\prime}\left(z_{j}\right)U_{z}
\end{align*}
where $o_{j}$ denotes the value in the output gate at $j$-th timestep
(similar notations are used for input gate ($i_{j}$), forget gate
($f_{j}$) and cell value ($z_{j}$)) and ``non-matrix'' terms actually
represent diagonal matrices corresponding to these terms. Under the
assumption that $h_{0}$=0, we then make use of the results stating
that $\left\Vert c_{t}\right\Vert _{\infty}$ is bounded for all $t$
\citet{miller2018recurrent} . By applying $l_{\infty}$-norm sub-multiplicativity
and triangle inequality, we can show that

\[
\frac{\partial h_{j}}{\partial h_{j-1}}\frac{\partial h_{j-1}}{\partial x_{j-1}}=M\frac{\partial h_{j}}{\partial x_{j}}+N
\]
with

\begin{align*}
\left\Vert M\right\Vert _{\infty} & \leq1/4\left\Vert W_{o}\right\Vert _{\infty}+1/4\left\Vert W_{f}\right\Vert _{\infty}\left\Vert c_{j}\right\Vert _{\infty}+1/4\left\Vert W_{i}\right\Vert _{\infty}+\left\Vert W_{z}\right\Vert _{\infty}=B_{m}\\
\left\Vert N\right\Vert _{\infty} & \leq1/16\left\Vert W_{o}U_{i}\right\Vert _{\infty}+1/16\left\Vert W_{o}U_{f}\right\Vert _{\infty}\left(\left\Vert c_{j}\right\Vert _{\infty}+\left\Vert c_{j-1}\right\Vert _{\infty}\right)+1/4\left\Vert W_{o}U_{z}\right\Vert _{\infty}\\
 & +1/16\left\Vert W_{i}U_{o}\right\Vert _{\infty}+1/16\left\Vert W_{i}U_{f}\right\Vert _{\infty}\left(\left\Vert c_{j}\right\Vert _{\infty}+\left\Vert c_{j-1}\right\Vert _{\infty}\right)+1/4\left\Vert W_{i}U_{z}\right\Vert _{\infty}\\
 & +1/16\left(\left\Vert W_{f}U_{o}\right\Vert _{\infty}+1/16\left\Vert W_{f}U_{i}\right\Vert _{\infty}+1/4\left\Vert W_{f}U_{z}\right\Vert _{\infty}\right)\left(\left\Vert c_{j}\right\Vert _{\infty}+\left\Vert c_{j-1}\right\Vert _{\infty}\right)\\
 & +1/4\left\Vert W_{z}U_{o}\right\Vert _{\infty}+1/4\left\Vert W_{z}U_{f}\right\Vert _{\infty}\left(\left\Vert c_{j}\right\Vert _{\infty}+\left\Vert c_{j-1}\right\Vert _{\infty}\right)+1/4\left\Vert W_{z}U_{i}\right\Vert _{\infty}\\
 & =B_{n}
\end{align*}
By applying $l_{\infty}$-norm sub-multiplicativity and triangle inequality,

\begin{align*}
c_{i-1,t} & =\left\Vert \stackrel[j=i+1]{t}{\prod}\frac{\partial h_{j}}{\partial h_{j-1}}\frac{\partial h_{i}}{\partial h_{i-1}}\frac{\partial h_{i-1}}{\partial x_{i-1}}\right\Vert _{\infty}\\
 & =\left\Vert \stackrel[j=i+1]{t}{\prod}\frac{\partial h_{j}}{\partial h_{j-1}}\left(\frac{\partial h_{i}}{\partial x_{i}}m+n\right)\right\Vert _{\infty}\\
 & \leq B_{m}c_{i,t}+B_{n}\stackrel[j=i+1]{t}{\prod}\left\Vert \frac{\partial h_{j}}{\partial h_{j-1}}\right\Vert _{\infty}
\end{align*}

As LSTM is $\lambda$-contractive with $\lambda<1$ in the $l_{\infty}$-norm
(readers are recommended to refer to Miller and Hardt, (2018) for
proof), which implies $\left\Vert \frac{\partial h_{j}}{\partial h_{j-1}}\right\Vert _{\infty}<1$,
$B_{n}\stackrel[j=i+1]{t}{\prod}\left\Vert \frac{\partial h_{j}}{\partial h_{j-1}}\right\Vert _{\infty}\rightarrow0$
as $t-i\rightarrow\infty$. For $t-i<\infty$, under the assumption
that $\frac{\partial h_{j}}{\partial x_{j}}\neq0$, we can always
find some value $B<\infty$ such that $c_{i-1,t}\leq Bc_{i,t}$. For
$t-i\rightarrow\infty$, $\lambda_{c}\rightarrow B_{m}$. That is,
$\lambda_{c}=\max\left(B_{m},B\right)$.

\subsection{Proof of Theorem \ref{thm:The-average-amount}\label{subsec:Proof-of-theorem-1}}
\begin{proof}
Given that $\lambda_{c}c_{i,t}\geq c_{i-1,t}$ with some $\lambda_{c}\in\mathbb{R^{+}}$,
we can use $c_{t,t}\lambda_{c}^{t-i}$ as the upper bound on $c_{i,t}$
with $i=\overline{1,t}$, respectively. Therefore,

\[
f(0)\leq\stackrel[t=1]{T}{\sum}c_{t,T}\leq c_{T,T}\stackrel[t=1]{T}{\sum}\lambda_{c}^{T-t}=f\left(\lambda_{c}\right)
\]
where $f(\lambda)=c_{T,T}\stackrel[t=1]{T}{\sum}\lambda^{T-t}$ is
continuous on $\mathbb{R^{+}}$. According to intermediate value theorem,
there exists $\lambda\in\left(0,\lambda_{c}\right]$ such that $c_{T,T}\stackrel[t=1]{T}{\sum}\lambda^{T-t}=\stackrel[t=1]{T}{\sum}c_{t,T}$.
\end{proof}

\subsection{Proof of Theorem \ref{thm:Under-the-assumption}\label{subsec:Proof-of-theorem-2}}
\begin{proof}
According to Theorem \ref{thm:The-average-amount}, there exists some
$\lambda_{i}\in\mathbb{R^{+}}$such that the summation of contribution
stored between $K_{i}$ and $K_{i+1}$ can be quantified as $c_{K_{i+1},K_{i+1}}\stackrel[t=K_{i}]{K_{i+1}}{\sum}\lambda_{i}^{K_{i+1}-t}$
(after ignoring contributions before $K_{i}$-th timestep for simplicity).
Let denote $P(\lambda)=\stackrel[t=K_{i}]{K_{i+1}}{\sum}\lambda^{K_{i+1}-t}$,
we have $P^{\prime}\left(\lambda\right)>0,\mathbb{\forall\lambda\in R^{+}}$.
Therefore, $P(\lambda_{i})\geq P\left(\underset{i}{\min}\left(\lambda_{i}\right)\right)$.
Let $C=\underset{i}{\min}\left(c_{i,i}\right)$ and $\lambda=\underset{i}{\min}\left(\lambda_{i}\right)$,
the average contribution stored in a MANN has a lower bound quantified
as $I_{\lambda}$, where

\begin{equation}
I_{\lambda}=C\frac{\stackrel[t=1]{K_{1}}{\sum}\lambda^{K_{1}-t}+\stackrel[t=K_{1}+1]{K_{2}}{\sum}\lambda^{K_{2}-t}+...+\stackrel[t=K_{D-1}+1]{K_{D}}{\sum}\lambda^{K_{D}-t}+\stackrel[t=K_{D}+1]{T}{\sum}\lambda^{T-t}}{T}
\end{equation}
.
\end{proof}

\subsection{Proof of Theorem \ref{thm:Given-the-number}\label{subsec:Proof-of-theorem}}
\begin{proof}
The second-order derivative of $f_{\lambda}\left(x\right)$ reads

\begin{equation}
f_{\lambda}^{\prime\prime}\left(x\right)=-\frac{\left(\ln\lambda\right)^{2}}{1-\lambda}\lambda^{x}
\end{equation}

We have $f_{\lambda,}^{\prime\prime}\left(x\right)\leq0$ with $\forall x\in\mathbb{R^{+}}$
and $1>\lambda>0$, so $f_{\lambda}\left(x\right)$ is a concave function.
Thus, we can apply Jensen inequality as follows: 
\begin{equation}
\frac{1}{D+1}\stackrel[i=1]{D+1}{\sum}f_{\lambda}(l_{i})\leq f_{\lambda}\left(\stackrel[i=1]{D+1}{\sum}\frac{1}{D+1}l_{i}\right)=f_{\lambda}\left(\frac{T}{D+1}\right)\label{eq:jensen}
\end{equation}
Equality holds if and only if $l_{1}=l_{2}=...=l_{D+1}=\frac{T}{D+1}$.
We refer to this as\textit{ Uniform Writing} strategy. By plugging
the optimal values of $l_{i}$, we can derive the maximised average
contribution as  follows,

\begin{equation}
I_{\lambda}max\equiv g_{\lambda}\left(T,D\right)=\frac{C\left(D+1\right)}{T}\left(\frac{1-\lambda^{\frac{T}{D+1}}}{1-\lambda}\right)
\end{equation}
When $\lambda=1$, $I_{\lambda}=\frac{C}{T}\stackrel[i=1]{D+1}{\sum}l_{i}=C$.
This is true for all writing strategies. Thus, Uniform Writing is
optimal for $0<\lambda\leq1$.
\end{proof}
We can show that this solution is also optimal for the case $\lambda>1$.
As $f_{\lambda}^{\prime\prime}\left(x\right)>0$ with $\forall x\in\mathbb{R^{+}};\lambda>1$,
$f_{\lambda}\left(x\right)$ is a convex function and Eq. (\ref{eq:jensen})
flips the inequality sign. Thus, $I_{\lambda}$ reaches its minimum
with Uniform Writing. For $\lambda>1$, minimising $I_{\lambda}$
is desirable to prevent the system from diverging. 

We can derive some properties of function $g$. Let $x=\frac{D+1}{L},g_{\lambda}\left(L,D\right)=g_{\lambda}\left(x\right)=Cx(\frac{\lambda^{\frac{1}{x}}-1}{\lambda-1})$.
We have $g_{\lambda}^{\prime}\left(x\right)=C\lambda\left(1-\lambda^{\frac{1}{x}}\right)\left(x-\ln\lambda\right)>0$
with $0<\lambda\leq1,\forall x\geq0$, so $g_{\lambda}\left(T,D\right)$
is an increasing function if we fix $T$ and let $D$ vary. That explains
why having more memory slots helps improve memorisation capacity.
If $D=0$, $g_{\lambda}\left(T,0\right)$ becomes E.q (\ref{eq:d0}).
In this case, MANNs memorisation capacity converges to that of recurrent
networks.

\subsection{Summary of Synthetic Discrete Task Format\label{subsec:Summary-of-synthetic}}

\begin{table}[H]
\begin{centering}
\begin{tabular}{ccc}
\hline 
Task & Input & Output\tabularnewline
\hline 
Double & $x_{1}x_{2}...x_{T}$ & $x_{1}x_{2}...x_{T}x_{1}x_{2}...x_{T}$\tabularnewline
Copy & $x_{1}x_{2}...x_{T}$ & $x_{1}x_{2}...x_{T}$\tabularnewline
Reverse & $x_{1}x_{2}...x_{T}$ & $x_{T}x_{T-1}...x_{1}$\tabularnewline
Add & $x_{1}x_{2}...x_{T}$ & $\frac{x_{1}+x_{T-1}}{2}\frac{x_{2}+x_{T-2}}{2}...\frac{x_{\left\lfloor T/2\right\rfloor }+x_{\left\lceil T/2\right\rceil }}{2}$\tabularnewline
Max & $x_{1}x_{2}...x_{T}$ & $\max\left(x_{1},x_{2}\right)\max\left(x_{3},x_{4}\right)...\max\left(x_{T-1},x_{T}\right)$\tabularnewline
\hline 
\end{tabular}
\par\end{centering}
\caption{Synthetic discrete task's input-output formats. $T$ is the sequence
length.}
\end{table}

\subsection{UW Performance on Bigger Memory\label{subsec:UW-performance-on}}

\begin{table}[H]
\begin{centering}
\begin{tabular}{ccc}
\hline 
Model & $N_{h}$ & Copy (L=500)\tabularnewline
\hline 
DNC & 128 & 24.19\%\tabularnewline
\hline 
DNC+UW & 128 & 81.45\%\tabularnewline
\hline 
\end{tabular}
\par\end{centering}
\caption{Test accuracy (\%) on synthetic copy task. MANNs have 50 memory slots.
Both models are trained with 100,000 mini-batches of size 32.}
\end{table}

\subsection{Memory Operating Behaviors on Synthetic Tasks\label{subsec:Memory-writing-behaviours}}

In this section, we pick three models (DNC, DNC+UW and DNC+CUW) to
analyse their memory operating behaviors. Fig. \ref{fig:Memory-operations-on}
visualises the values of the write weights and read weights for the
copy task during encoding input and decoding output sequence, respectively.
In the copy task, as the sequence length is 50 while the memory size
is 4, one memory slot should contain the accumulation of multiple
timesteps. This principle is reflected in the decoding process in
three models, in which one memory slot is read repeatedly across several
timesteps. Notably, the number of timesteps consecutively spent for
one slot is close to $10$-the optimal interval, even for DNC ( Fig.
\ref{fig:Memory-operations-on}(a)), which implies that the ultimate
rule would be the uniform rule. As UW and CUW are equipped with uniform
writing, their writing patterns follow the rule perfectly. Interestingly,
UW chooses the first written location for the final write (corresponding
to the <eos> token) while CUW picks the last written location. As
indicated in Figs. \ref{fig:Memory-operations-on}(b) and (c), both
of them can learn the corresponding reading pattern for decoding process,
which leads to good performances. On the other hand, regular DNC fails
to learn a perfect writing strategy. Except for the timesteps at the
end of the sequence, the timesteps are distributed to several memory
slots while the reading phase attends to one memory slot repeatedly.
This explains why regular DNC cannot compete with the other two proposed
methods in this task. 

For the max task, Fig. \ref{fig:Memory-operations-on-1} displays
similar visualisation with an addition of write gate during encoding
phase. The write gate indicates how much the model should write the
input at some timestep to the memory. A zero write gate means there
is no writing. For this task, a good model should discriminate between
timesteps and prefer writing the greater ones. As clearly seen in
Fig. \ref{fig:Memory-operations-on-1}(a), DNC suffers the same problem
as in copy task, unable to synchronise encoding writing with decoding
reading. Also, DNC's write gate pattern does not show reasonable discrimination.
For UW (Fig. \ref{fig:Memory-operations-on-1}(b)), it tends to write
every timestep and relies on uniform writing principle to achieve
write/read accordance and thus better results than DNC. Amongst all,
CUW is able to ignore irrelevant timesteps and follows uniform writing
at the same time (see Fig. \ref{fig:Memory-operations-on-1}(c)).

\begin{figure}
\begin{centering}
\includegraphics[width=1\textwidth]{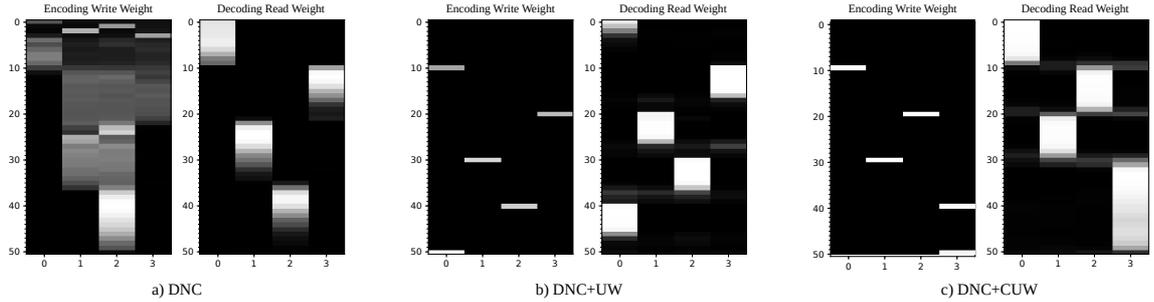}
\par\end{centering}
\caption{Memory operations on copy task in DNC (a), DNC+UW (b) and DNC+CUW(c).
Each row is a timestep and each column is a memory slot.\label{fig:Memory-operations-on}}
\end{figure}

\begin{figure}
\begin{centering}
\includegraphics[width=1\textwidth]{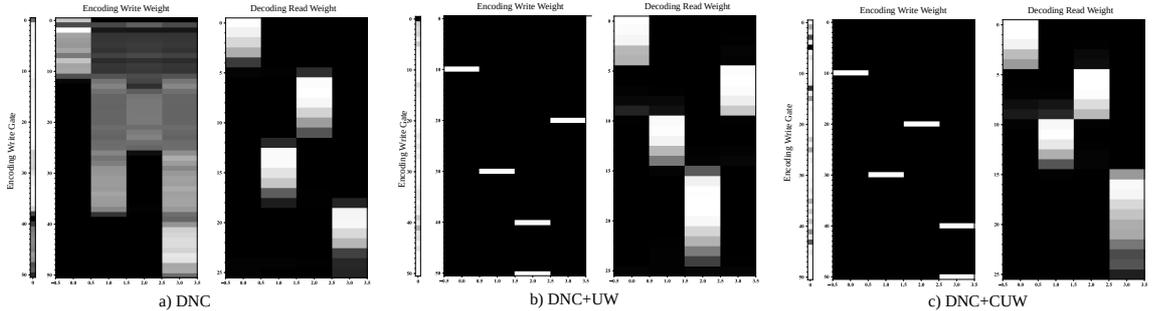}
\par\end{centering}
\caption{Memory operations on max task in DNC (a), DNC+UW (b) and DNC+CUW(c).
Each row is a timestep and each column is a memory slot.\label{fig:Memory-operations-on-1}}
\end{figure}

\subsection{Visualisations of Model Performance on Sinusoidal Regression Tasks\label{subsec:Visualizations-of-model}}

We pick randomly 3 input sequences and plot the output sequences produced
by DNC, UW and CUW in Figs. \ref{fig:Sinusoid} (clean) and \ref{fig:Sinusoid-noisy}
(noisy). In each plot, the first and last 100 timesteps correspond
to the given input and generated output, respectively. The ground
truth sequence is plotted in red while the predicted in blue. We also
visualise the values of MANN write gates through time in the bottom
of each plots. In irregular writing encoding phase, the write gate
is computed even when there is no write as it reflects how much weight
the controller puts on the timesteps. In decoding, we let MANNs write
to memory at every timestep to allow instant update of memory during
inference. 

Under clean condition, all models seem to attend more to late timesteps
during encoding, which makes sense as focusing on late periods of
sine wave is enough for later reconstruction. However, this pattern
is not clear in DNC and UW as in CUW. During decoding, the write gates
tend to oscillate in the shape of sine wave, which is also a good
strategy as this directly reflects the amplitude of generation target.
In this case, both UW and CUW demonstrate this behavior clearer than
DNC. 

Under noisy condition, DNC and CUW try to follow sine-shape writing
strategy. However, only CUW can learn the pattern and assign write
values in accordance with the signal period, which helps CUW decoding
achieve highest accuracy. On the other hand, UW choose to assign write
value equally and relies only on its maximisation of timestep contribution.
Although it achieves better results than DNC, it underperforms CUW. 

\begin{figure}
\begin{centering}
\noindent\begin{minipage}[t]{1\columnwidth}%
\includegraphics[width=0.3\linewidth]{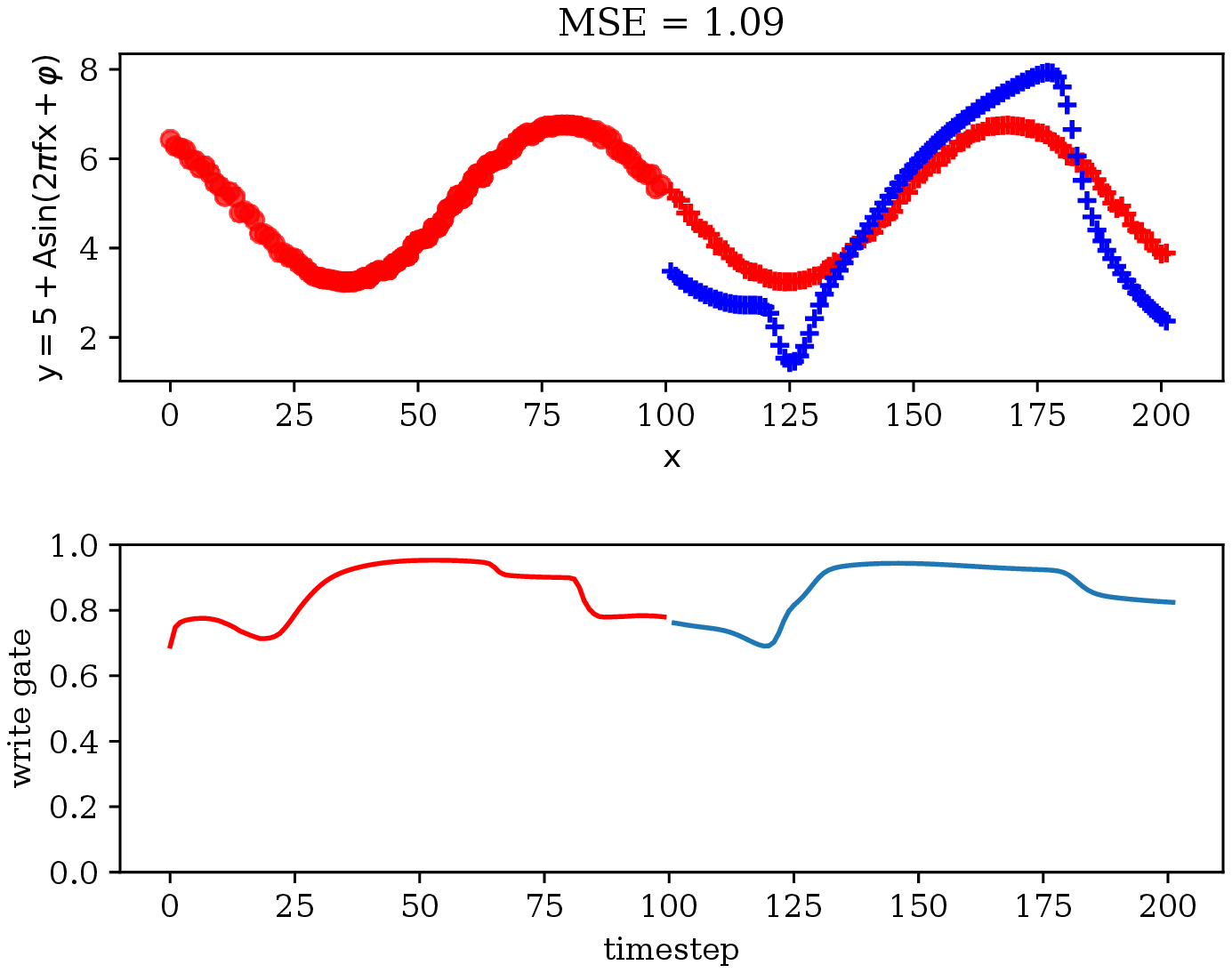}\includegraphics[width=0.3\linewidth]{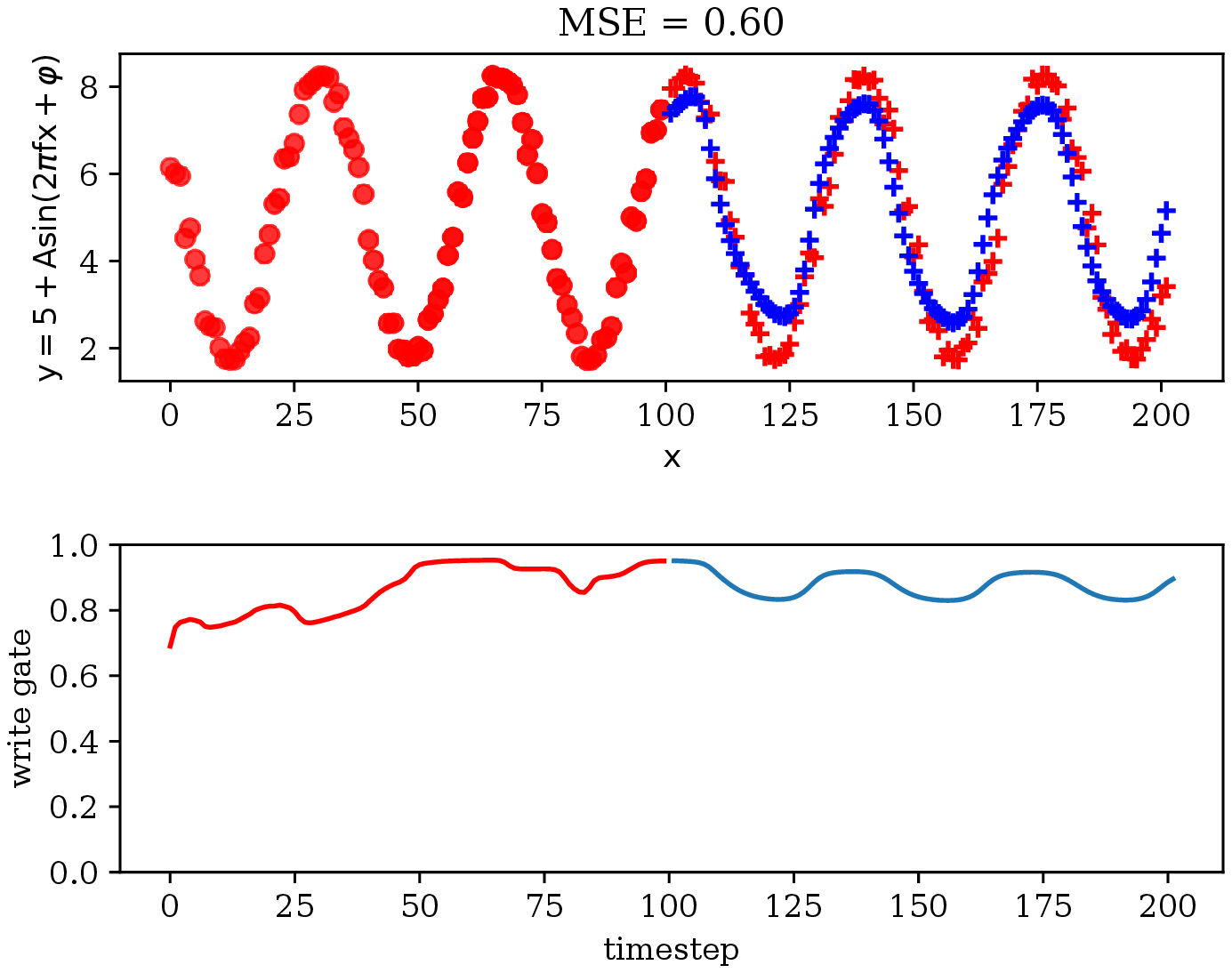}\includegraphics[width=0.3\linewidth]{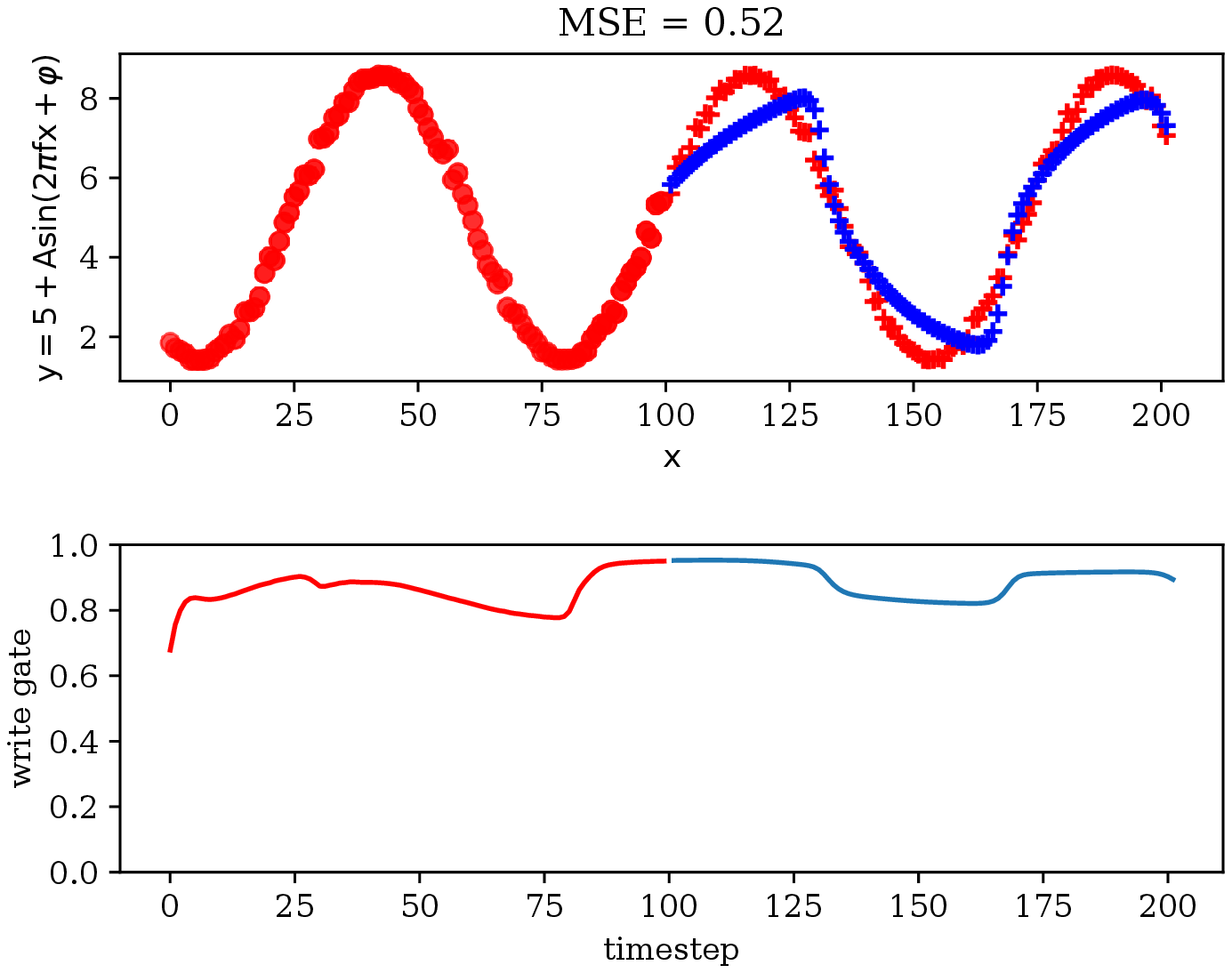}%
\end{minipage}
\par\end{centering}
\begin{centering}
\noindent\begin{minipage}[t]{1\columnwidth}%
\includegraphics[width=0.3\linewidth]{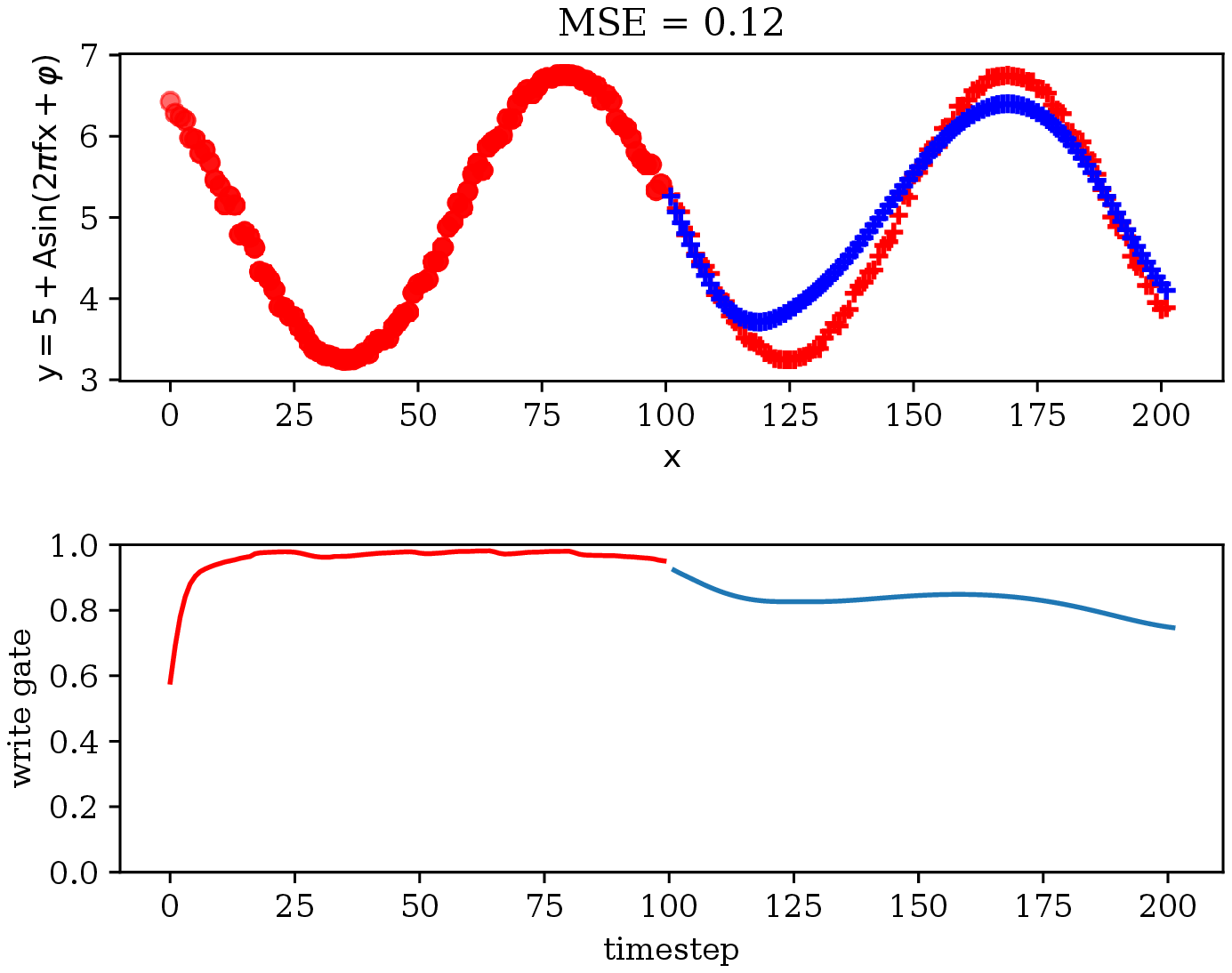}\includegraphics[width=0.3\linewidth]{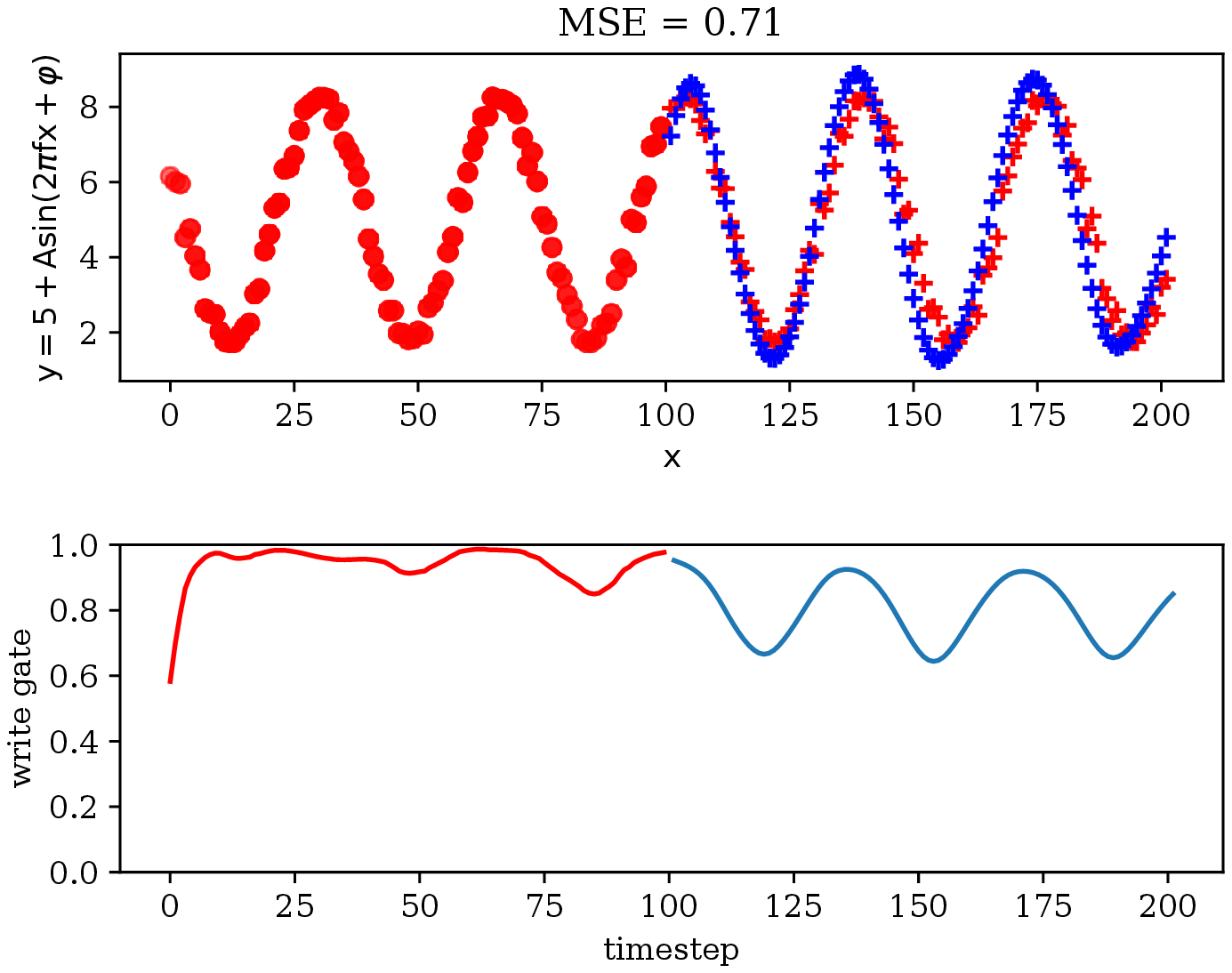}\includegraphics[width=0.3\linewidth]{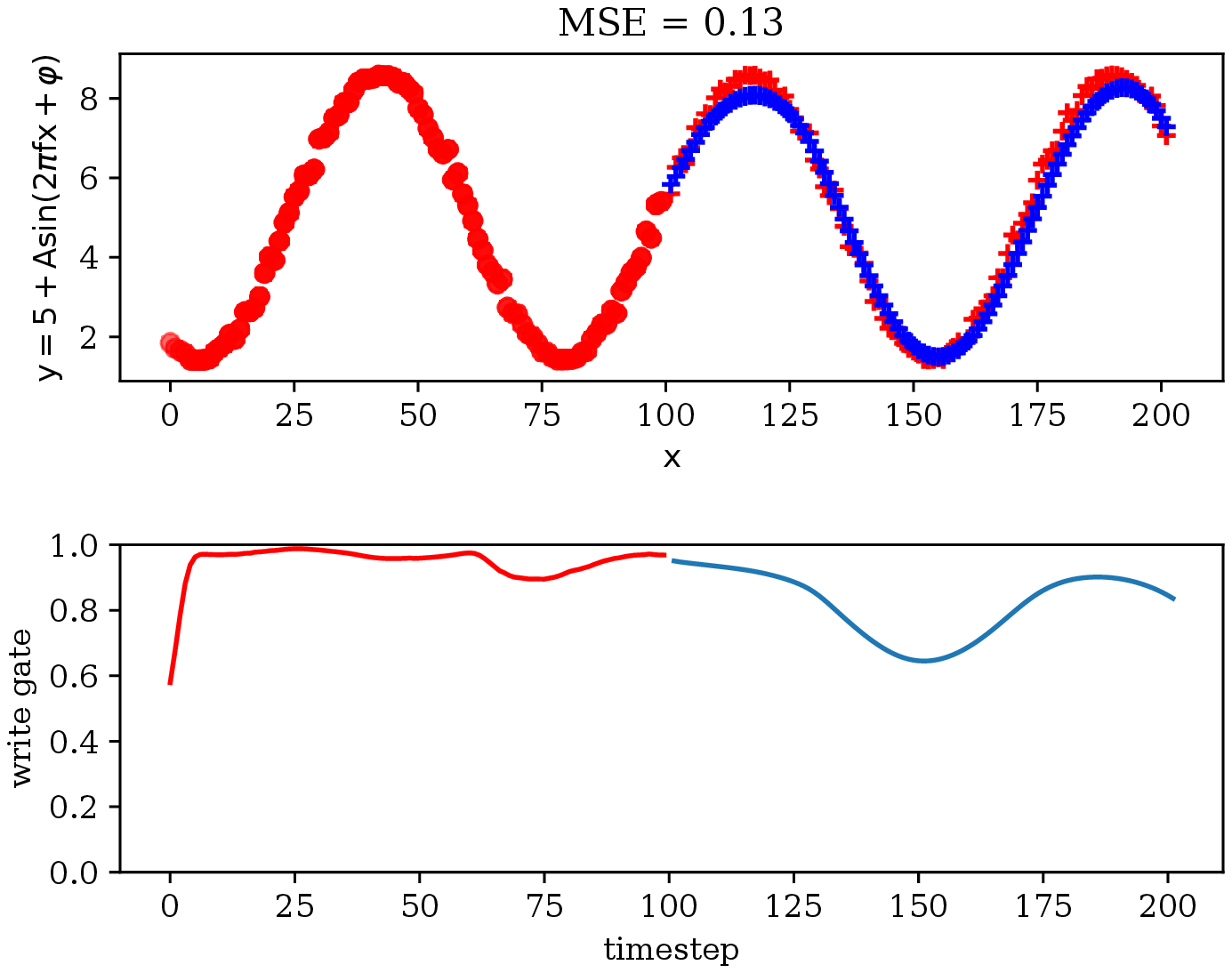}%
\end{minipage}
\par\end{centering}
\begin{centering}
\noindent\begin{minipage}[t]{1\columnwidth}%
\includegraphics[width=0.3\linewidth]{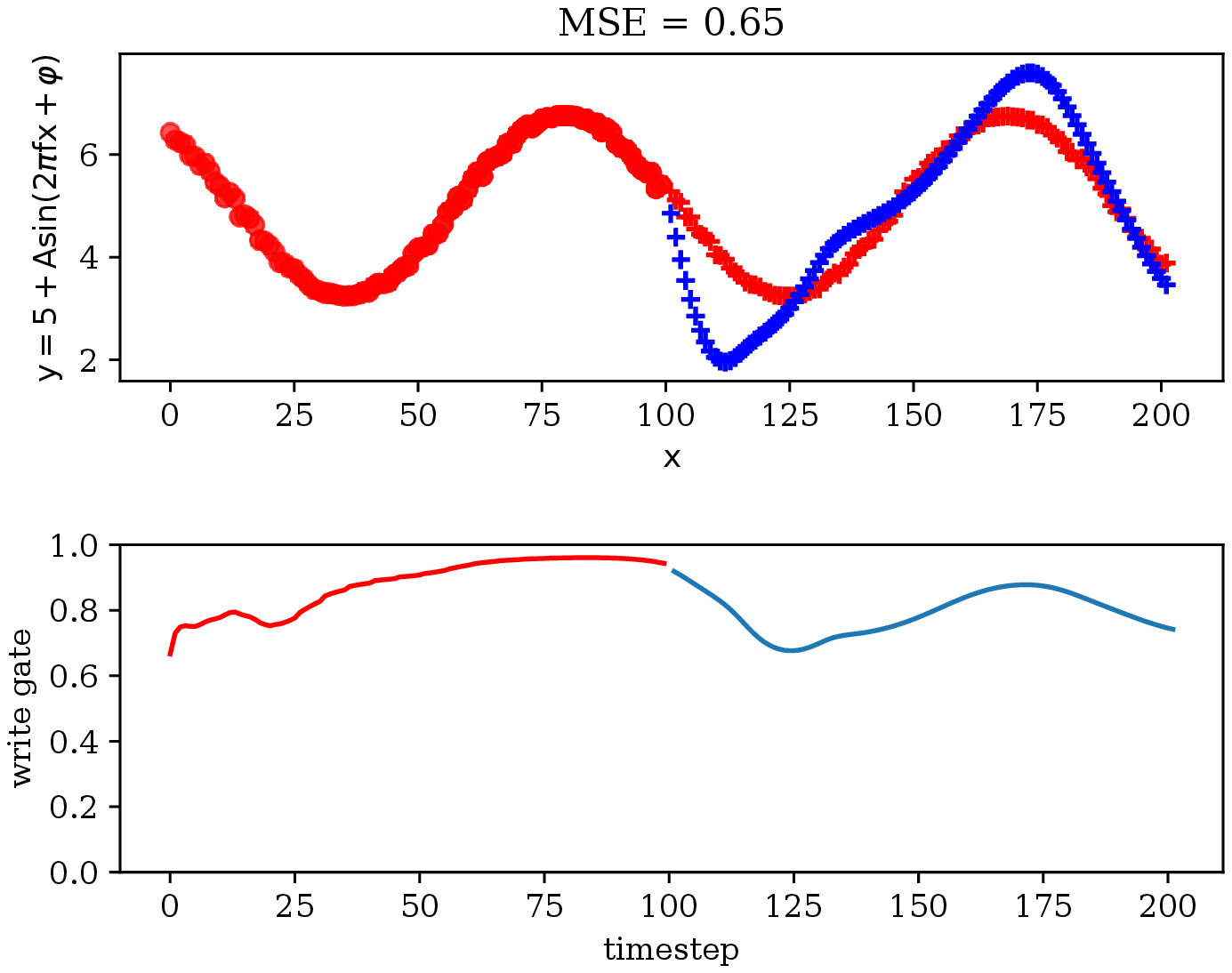}\includegraphics[width=0.3\linewidth]{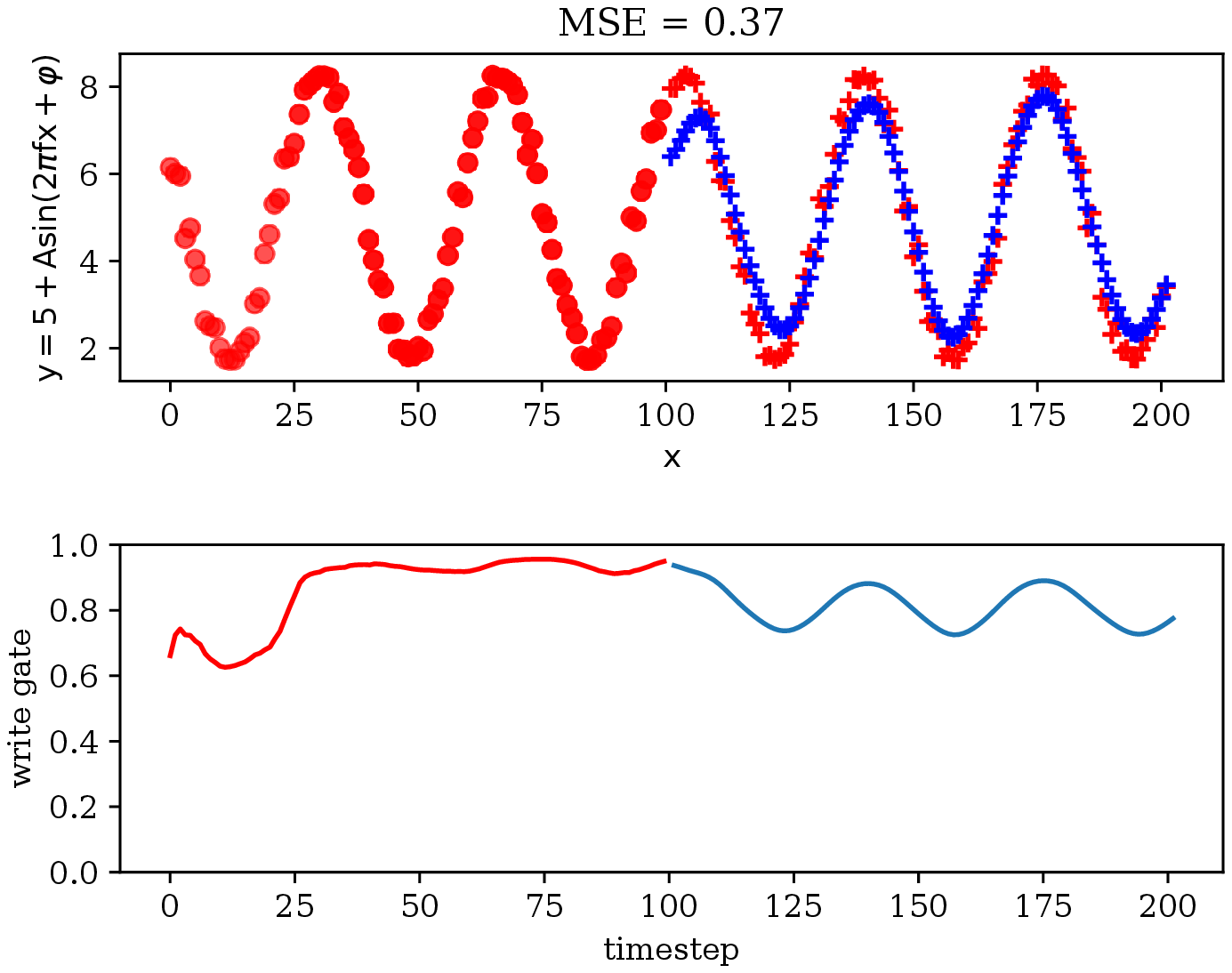}\includegraphics[width=0.3\linewidth]{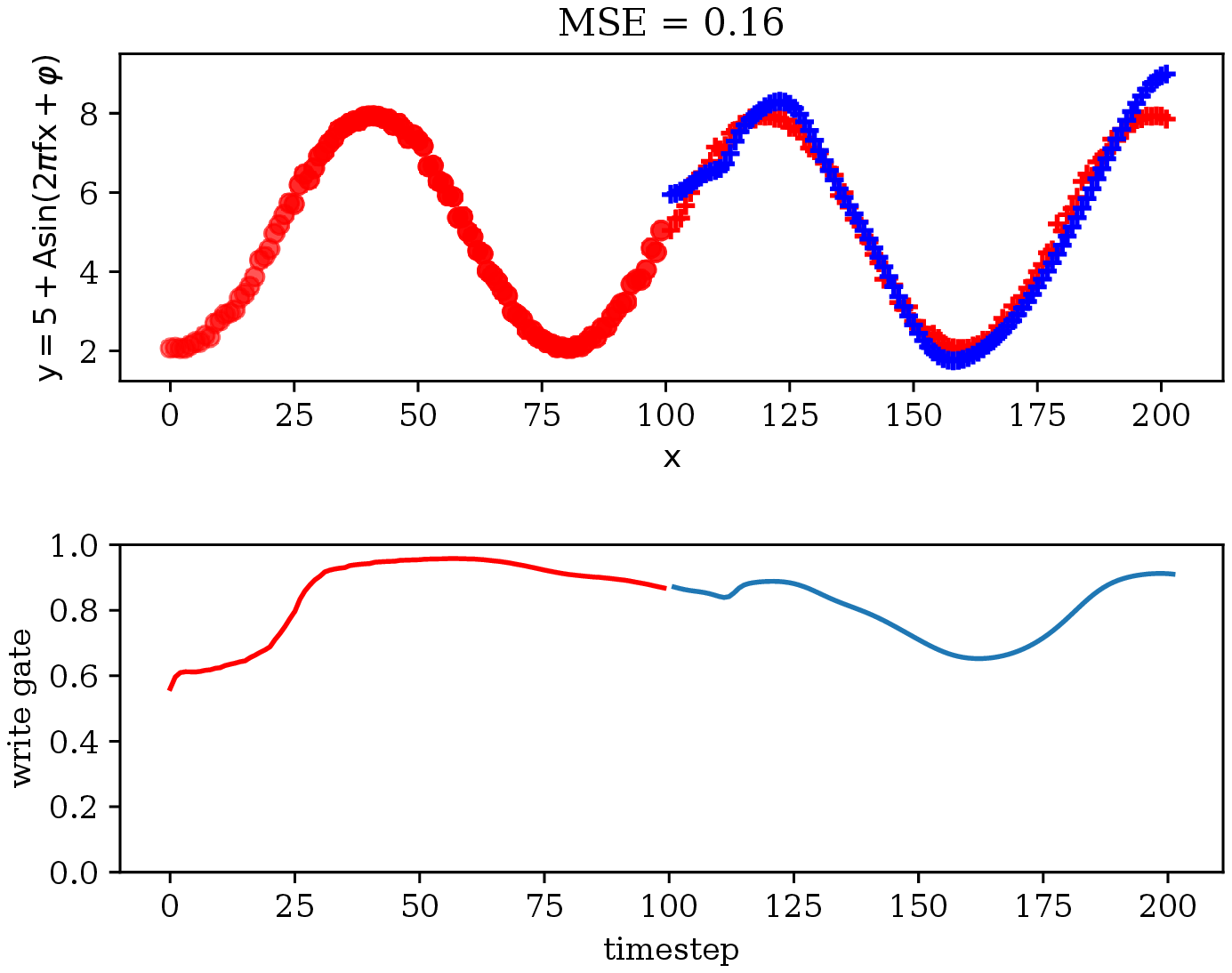}%
\end{minipage}
\par\end{centering}
\caption{Sinusoidal generation with clean input sequence for DNC, UW and CUW
in top-down order. \label{fig:Sinusoid}}
\end{figure}

\begin{figure}
\begin{centering}
\noindent\begin{minipage}[t]{1\columnwidth}%
\includegraphics[width=0.3\linewidth]{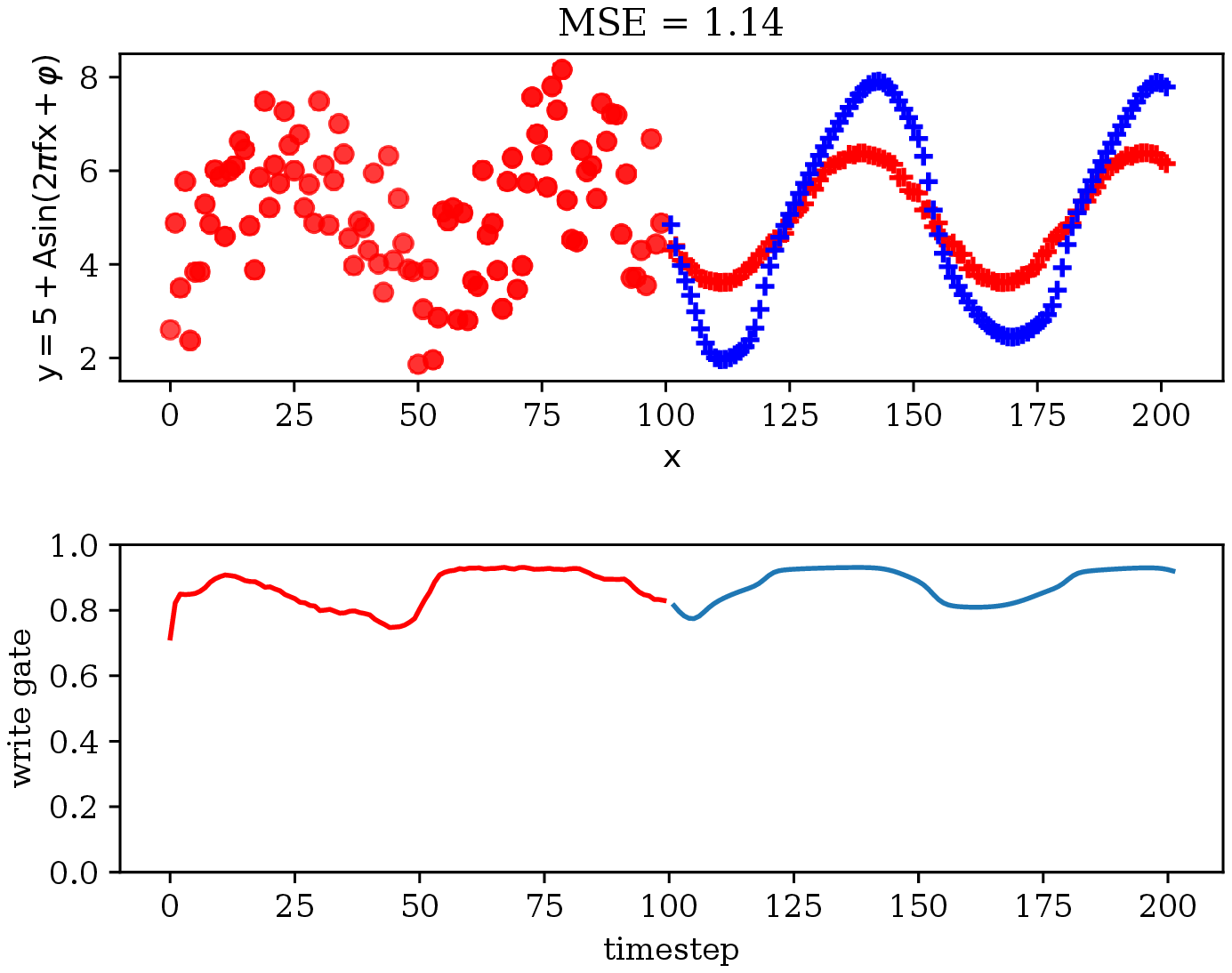}\includegraphics[width=0.3\linewidth]{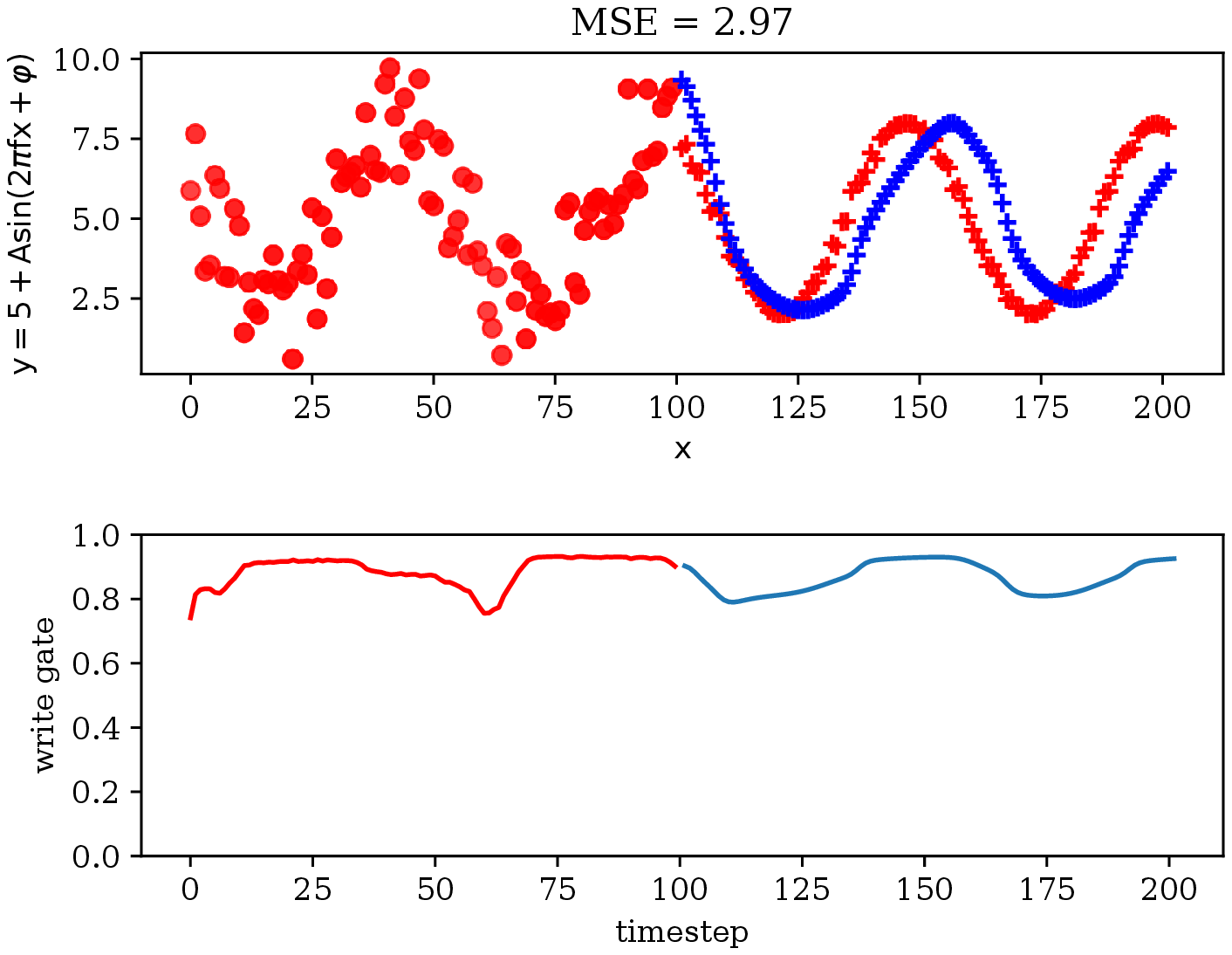}\includegraphics[width=0.3\linewidth]{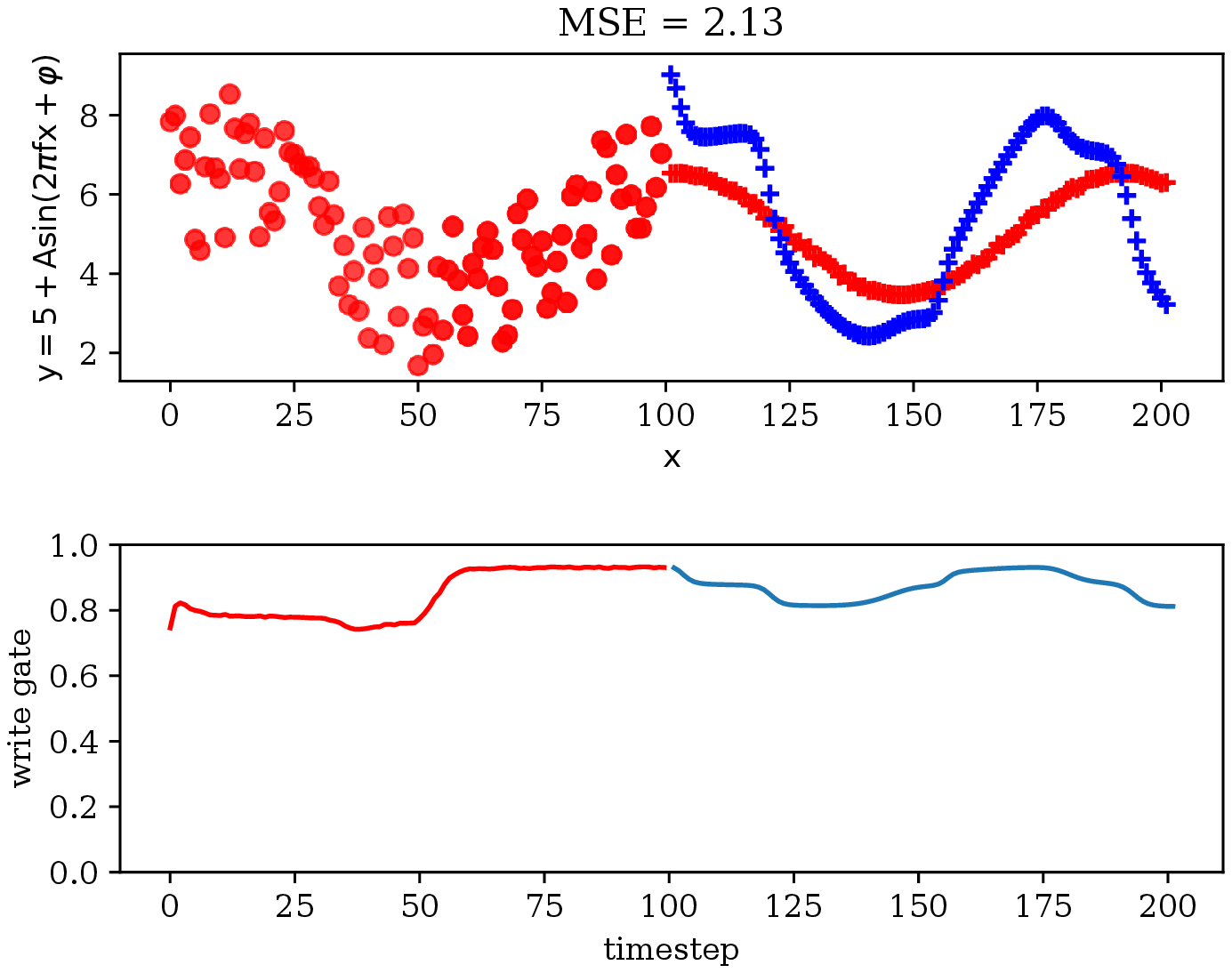}%
\end{minipage}
\par\end{centering}
\begin{centering}
\noindent\begin{minipage}[t]{1\columnwidth}%
\includegraphics[width=0.3\linewidth]{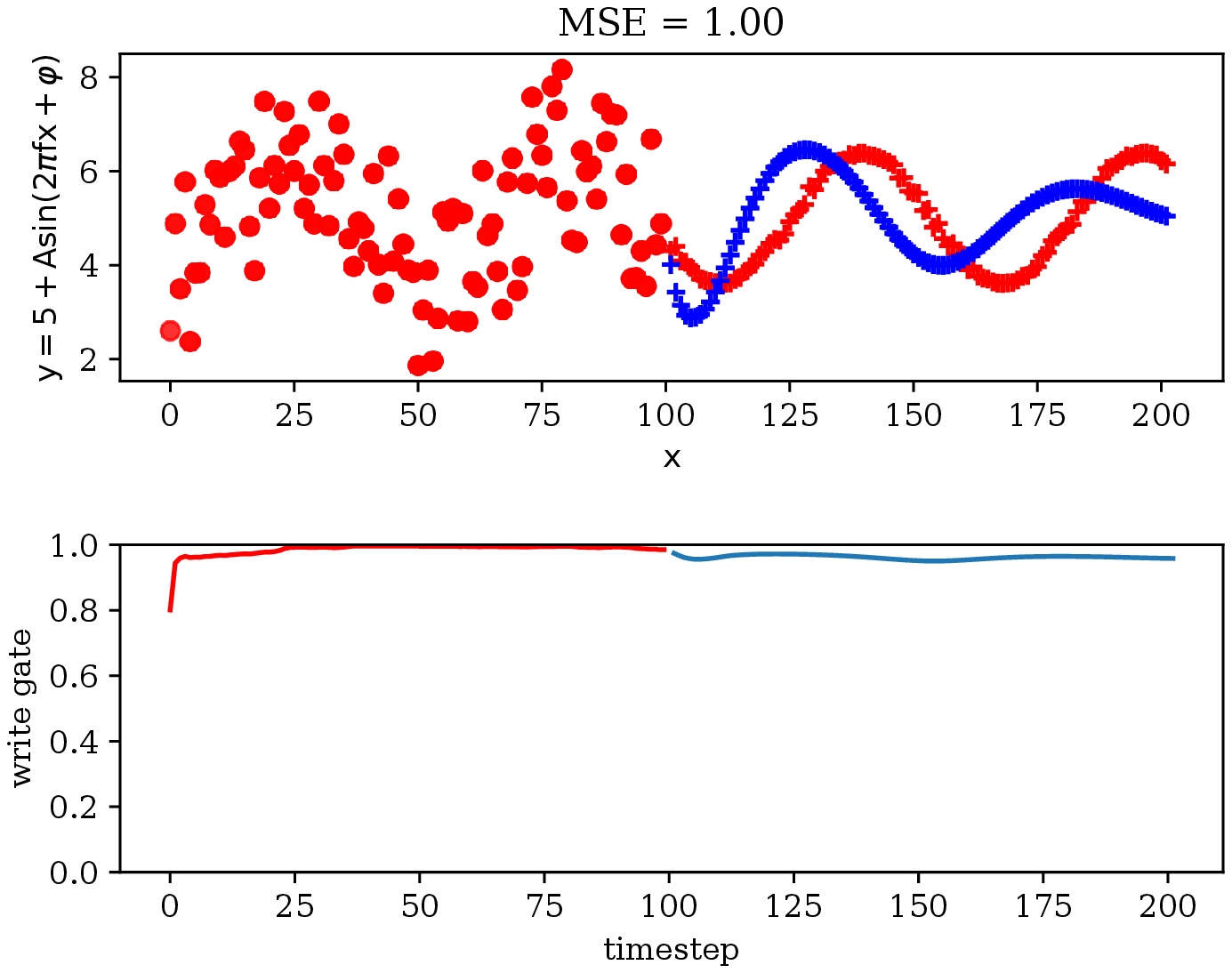}\includegraphics[width=0.3\linewidth]{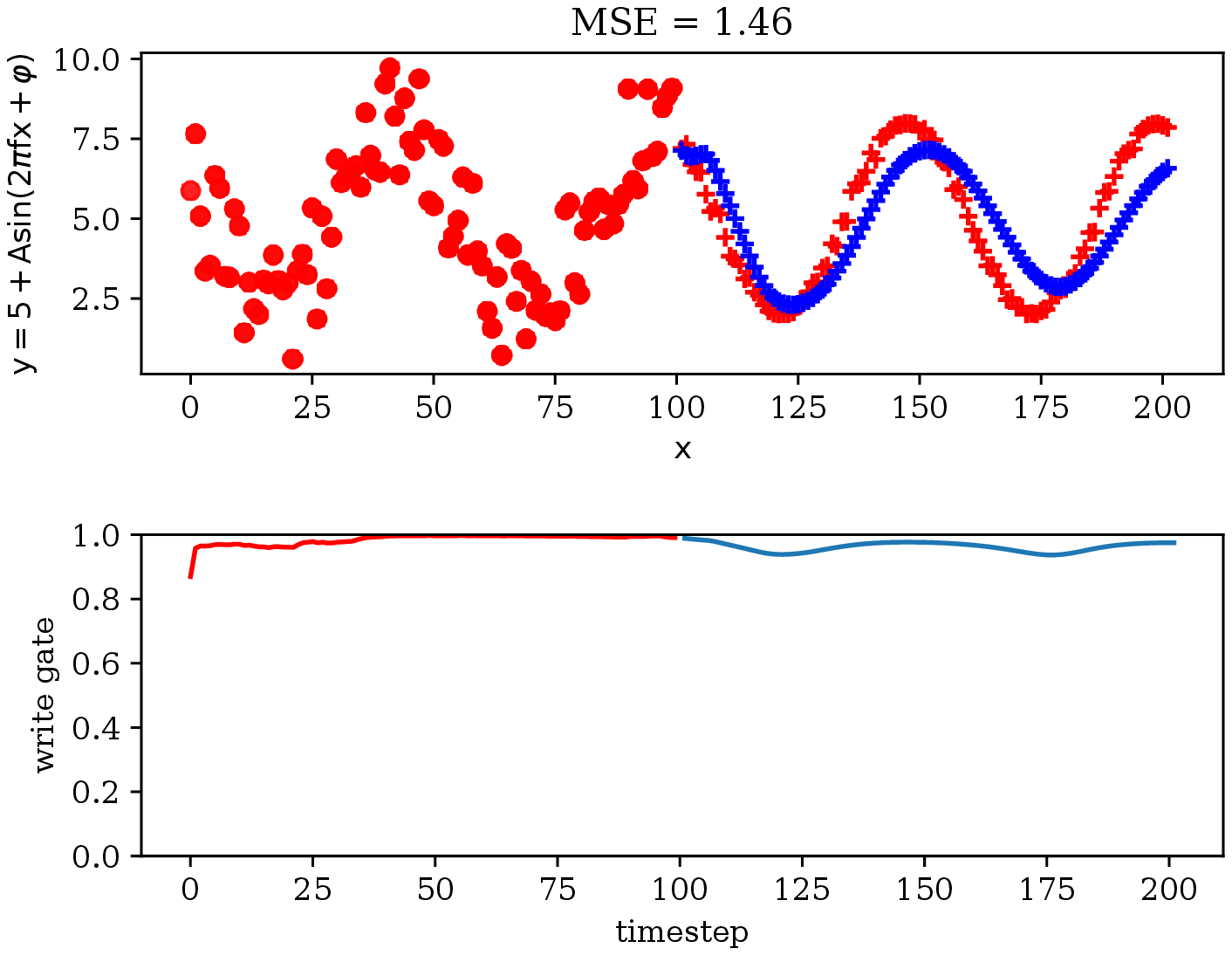}\includegraphics[width=0.3\linewidth]{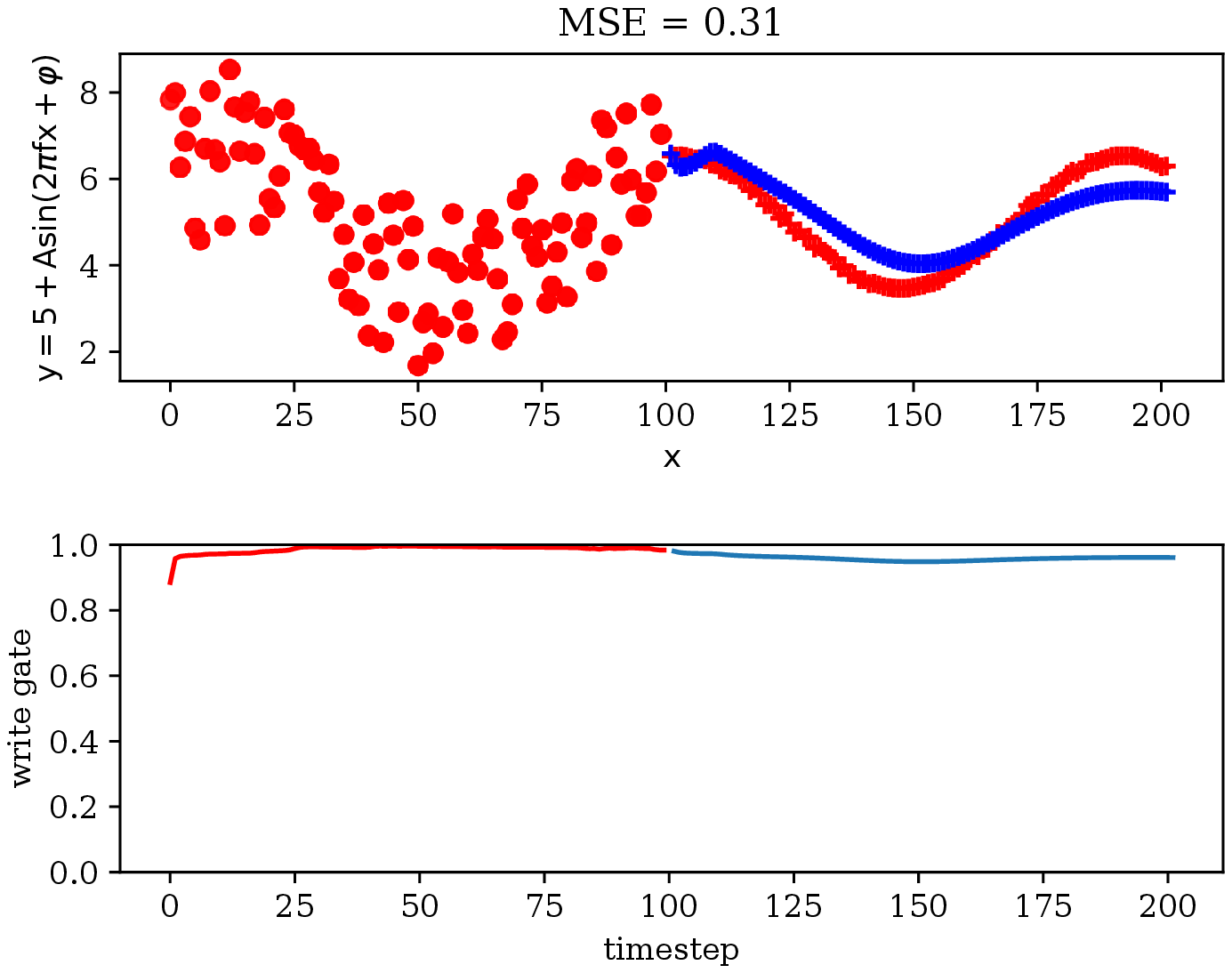}%
\end{minipage}
\par\end{centering}
\begin{centering}
\noindent\begin{minipage}[t]{1\columnwidth}%
\includegraphics[width=0.3\linewidth]{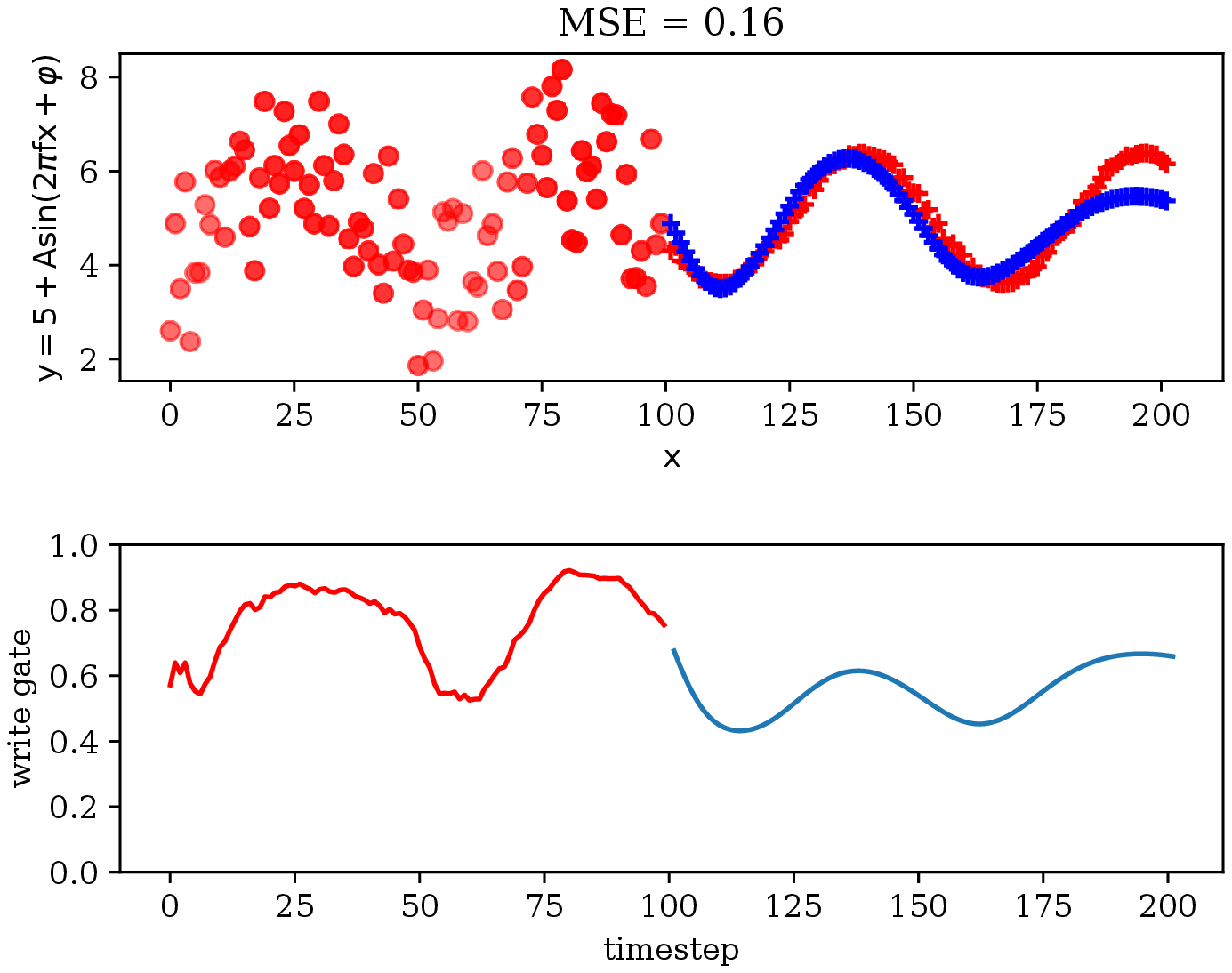}\includegraphics[width=0.3\linewidth]{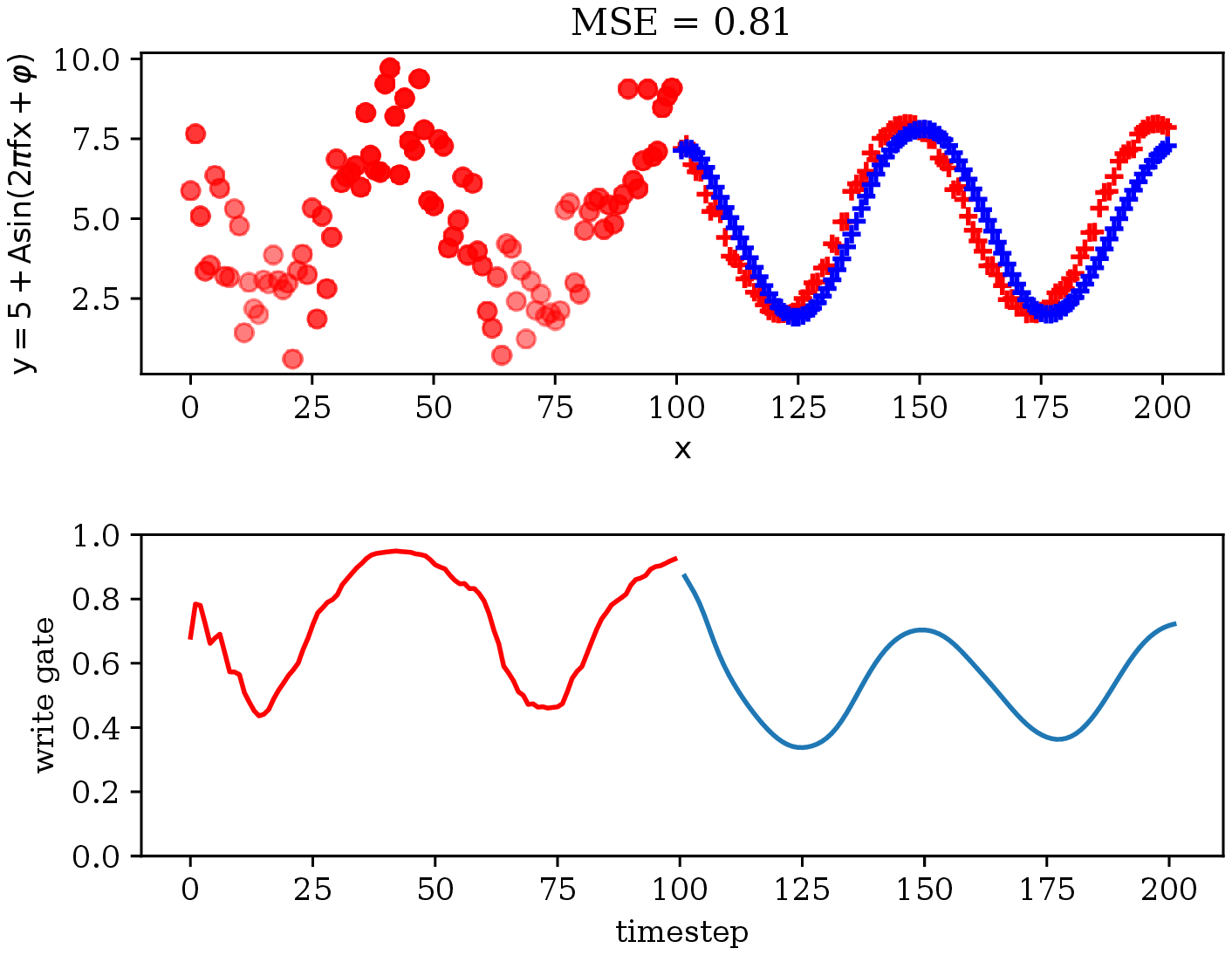}\includegraphics[width=0.3\linewidth]{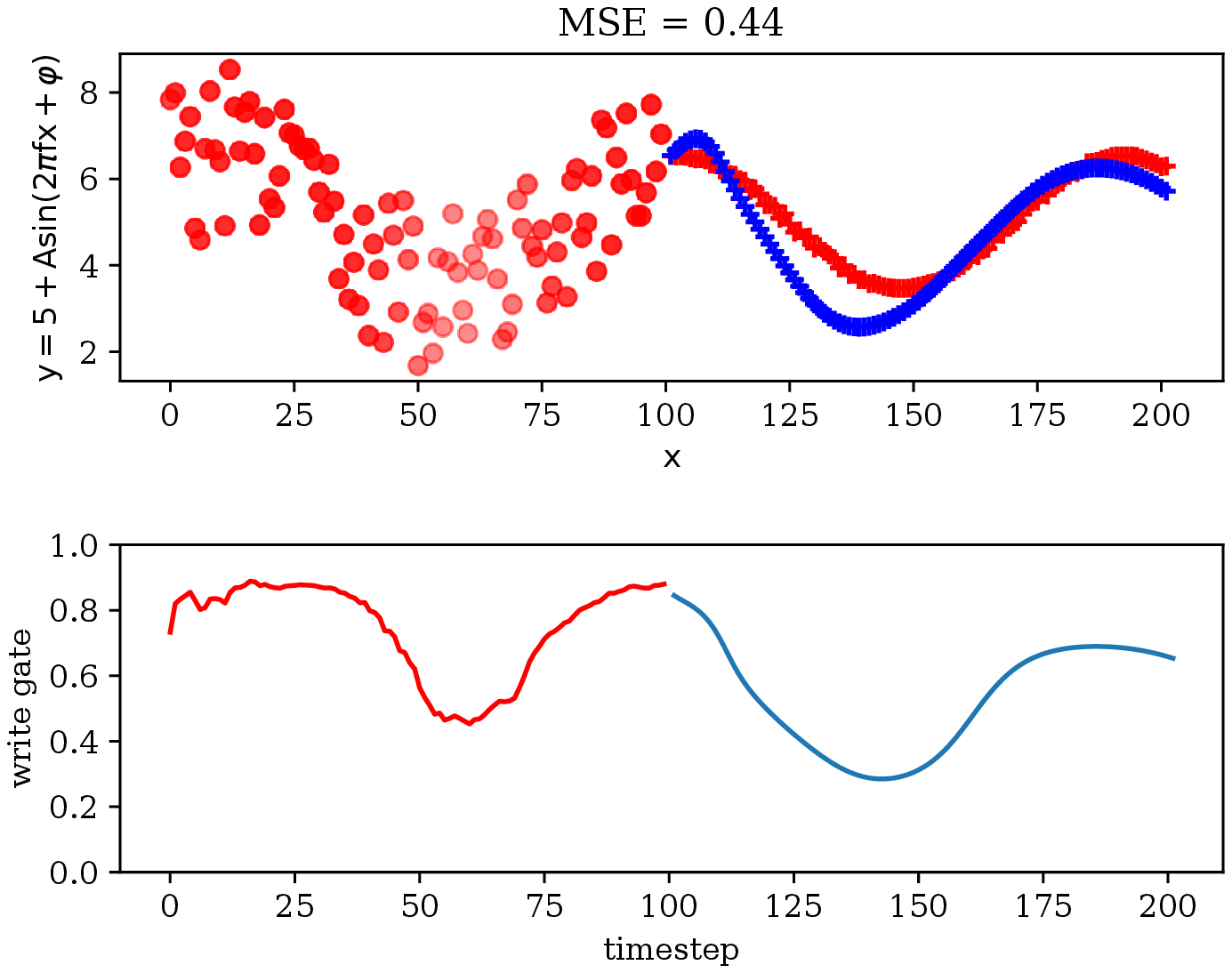}%
\end{minipage}
\par\end{centering}
\caption{Sinusoidal generation with noisy input sequence for DNC, UW and CUW
in top-down order. \label{fig:Sinusoid-noisy}}
\end{figure}

\subsection{Comparison with Non-Recurrent Methods in Flatten Image Classification
Task\label{subsec:Comparsion-with-non-recurrent}}

\begin{table}[H]
\begin{centering}
\begin{tabular}{lcc}
\hline 
Model & MNIST & pMNIST\tabularnewline
\hline 
The Transformer$^{\star}$ & 98.9 & 97.9\tabularnewline
Dilated CNN$^{\blacklozenge}$ & 98.3 & 96.7\tabularnewline
\hline 
DNC+CUW & 99.1 & 96.3\tabularnewline
\hline 
\end{tabular}
\par\end{centering}
\caption{Test accuracy (\%) on MNIST, pMNIST. Previously reported results are
from Vaswani et al., (2017)$^{\star}$ and Chang et al., (2017)$^{\blacklozenge}$.
\label{tab:mnist-1}}
\end{table}

\subsection{Details on Document Classification Datasets\label{subsec:Details-on-document}}

\begin{table}[H]
\begin{centering}
\begin{tabular}{cc>{\centering}p{0.1\linewidth}>{\centering}p{0.1\linewidth}>{\centering}p{0.1\linewidth}>{\centering}p{0.1\linewidth}}
\hline 
Dataset & Classes & Average lengths & Max lengths & Train samples & Test samples\tabularnewline
\hline 
IMDb & 2 & 282 & 2,783 & 25,000 & 25,000\tabularnewline
Yelp Review Polarity (Yelp P.) & 2 & 156 & 1,381 & 560,000 & 38,000\tabularnewline
Yelp Review Full (Yelp F.) & 5 & 158 & 1,381 & 650,000 & 50,000\tabularnewline
AG's News (AG) & 4 & 44 & 221 & 120,000 & 7,600\tabularnewline
DBPedia (DBP) & 14 & 55 & 1,602 & 560,000 & 70,000\tabularnewline
Yahoo! Answers (Yah. A.) & 10 & 112 & 4,392 & 1,400,000 & 60,000\tabularnewline
\hline 
\end{tabular}
\par\end{centering}
\caption{Statistics on several big document classification datasets}
\end{table}

\subsection{Document Classification Detailed Records\label{subsec:Document-classification-detailed}}

\begin{table}[H]
\begin{centering}
\begin{tabular}{cccccc}
\hline 
\multicolumn{2}{c}{Model} & AG & IMDb & Yelp P. & Yelp F.\tabularnewline
\hline 
\multirow{4}{*}{UW} & 1 & 93.42 & \textbf{91.39} & \textbf{96.39} & 64.89\tabularnewline
 & 2 & 93.52 & 91.30 & 96.31 & 64.97\tabularnewline
 & 3 & \textbf{93.69} & 91.25 & 96.39 & \textbf{65.26}\tabularnewline
\cline{2-6} \cline{3-6} \cline{4-6} \cline{5-6} \cline{6-6} 
 & Mean/Std & 93.54$\pm$0.08 & 91.32$\pm$0.04 & 96.36$\pm$0.03 & 65.04$\pm$0.11\tabularnewline
\hline 
\multirow{4}{*}{CUW} & 1 & 93.61 & 91.26 & \textbf{96.42} & \textbf{65.63}\tabularnewline
 & 2 & \textbf{93.87} & 91.18 & 96.29 & 65.05\tabularnewline
 & 3 & 93.70 & \textbf{91.32} & 96.36 & 64.80\tabularnewline
\cline{2-6} \cline{3-6} \cline{4-6} \cline{5-6} \cline{6-6} 
 & Mean/Std & 93.73$\pm$0.08 & 91.25$\pm$0.04 & 96.36$\pm$0.04 & 65.16$\pm$0.24\tabularnewline
\hline 
\end{tabular}
\par\end{centering}
\caption{Document classification accuracy (\%) on several datasets reported
for 3 different runs. Bold denotes the best records.}
\end{table}

\section{Supplementary for Chapter 7}

\subsection{Full Learning Curves on Single NTM Tasks\label{subsec:Full-Learning-Curves}}

\begin{figure}[H]
\begin{centering}
\includegraphics[width=1\textwidth]{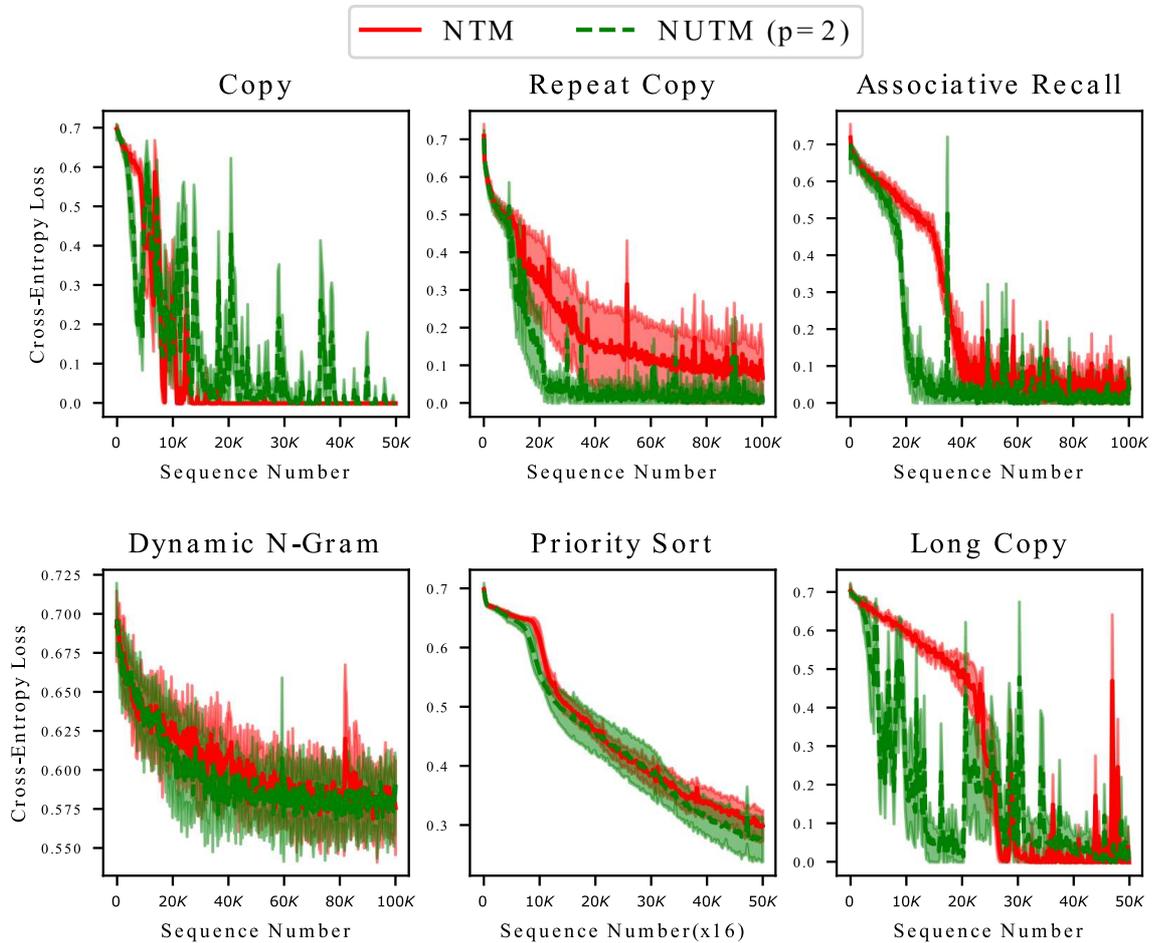}
\par\end{centering}
\caption{Learning curves on NTM tasks.\label{fig:Learning-curves-on}}
\end{figure}

\subsection{Clustering on The Latent Space\label{subsec:Clustering-on-The}}

As previously mentioned in Sec. 3.3, MANN should let its states form
clusters to well-simulate Turing Machine. Fig. \ref{fig:Visualisation-of-the}
(a) and (c) show NTM actually \foreignlanguage{australian}{organises}
its $c_{t}$ space into clusters corresponding to processing states
(e.g, encoding and decoding). NUTM, which explicitly partitions this
space, clearly learn better clusters of $c_{t}$ (see Fig. \ref{fig:Visualisation-of-the}
(b) and (d)). This contributes to NUTM's outperformance over NTM. 

\begin{figure}[H]
\begin{centering}
\includegraphics[width=0.95\linewidth]{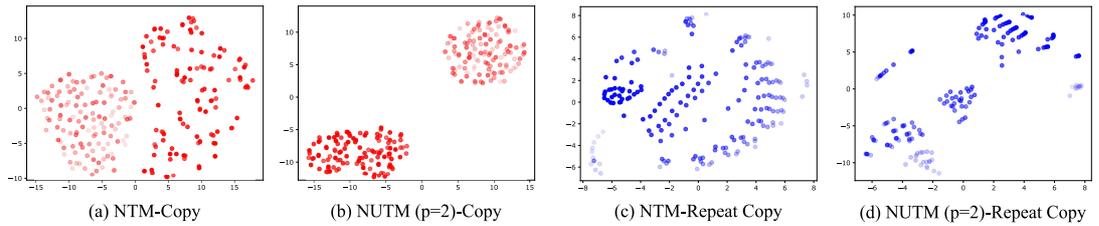}
\par\end{centering}
\caption{Visualisation of the first two principal components of $c_{t}$ space
in NTM (a,c) and NUTM (b,d) for Copy (red) and Repeat Copy (blue).
Fader color denotes lower timestep in a sequence. Both can learn clusters
of hidden states yet NUTM exhibits clearer partition. \label{fig:Visualisation-of-the} }
\end{figure}

\subsection{Program Usage Visualisations \label{subsec:Program-Usage-Visualizations}}

\ref{subsec:Visualization-on-program} and \ref{subsec:Visualization-on-program-1}
visualise the best inferences of NUTM on test data from single and
sequencing tasks. Each plot starts with the input sequence and the
predicted output sequence with error bits in the first row. The second
and fourth rows depict the read and write locations on data memory,
respectively. The third and fifth rows depict the program distribution
of the read head and write head, respectively. \ref{subsec:Perseveration-phenomenon-in}
visualises random failed predictions of NTM on sequencing tasks. The
plots follow previous pattern except for the program distribution
rows. 

\subsubsection{Visualisation on Program Distribution across Timesteps (Single Tasks)\label{subsec:Visualization-on-program}}

\begin{figure}[H]
\begin{centering}
\includegraphics[width=1\linewidth]{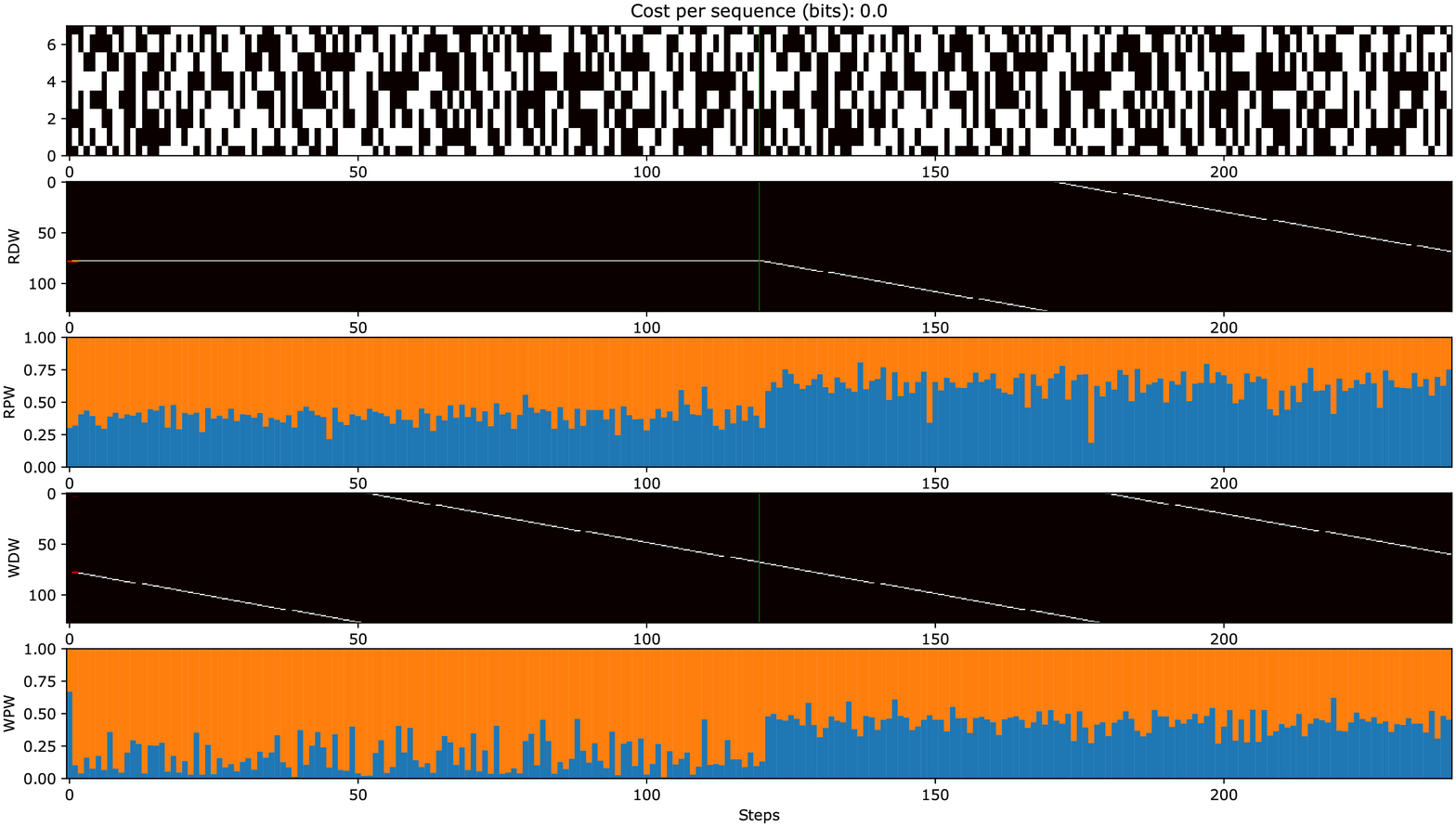}
\par\end{centering}
\caption{Copy (p=2).}
\end{figure}

\begin{figure}[H]
\begin{centering}
\includegraphics[width=1\linewidth]{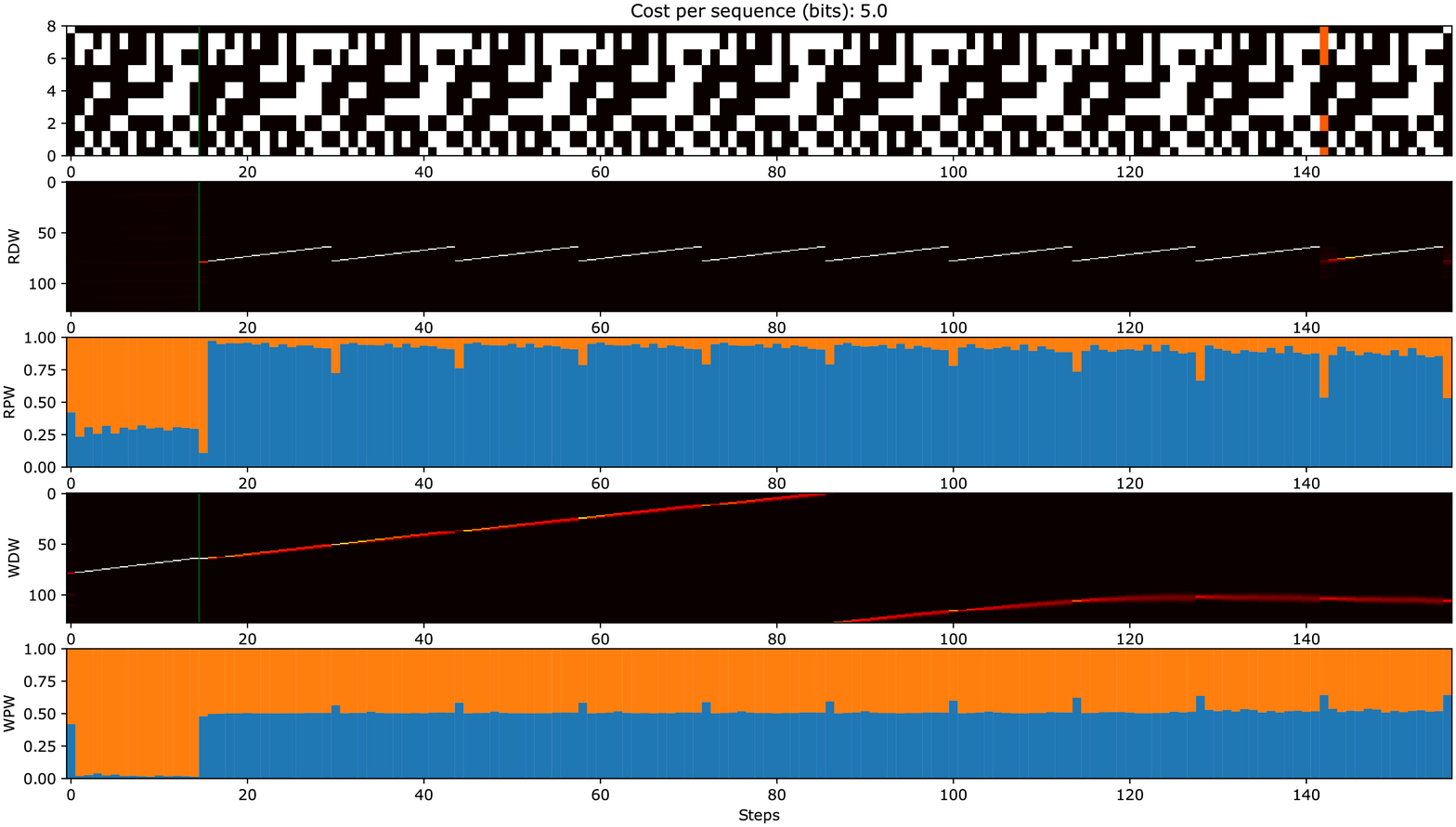}
\par\end{centering}
\caption{Repeat Copy (p=2).}
\end{figure}

\begin{figure}[H]
\begin{centering}
\includegraphics[width=1\linewidth]{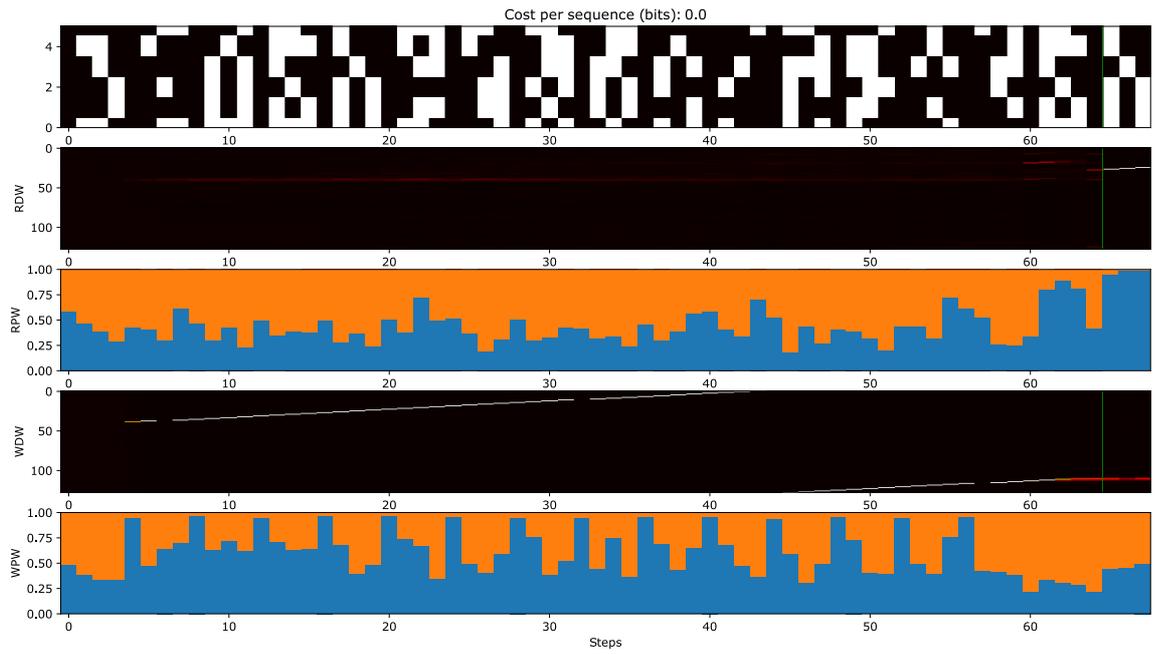}
\par\end{centering}
\caption{Associative Recall (p=2).}
\end{figure}

\begin{figure}[H]
\begin{centering}
\includegraphics[width=1\linewidth]{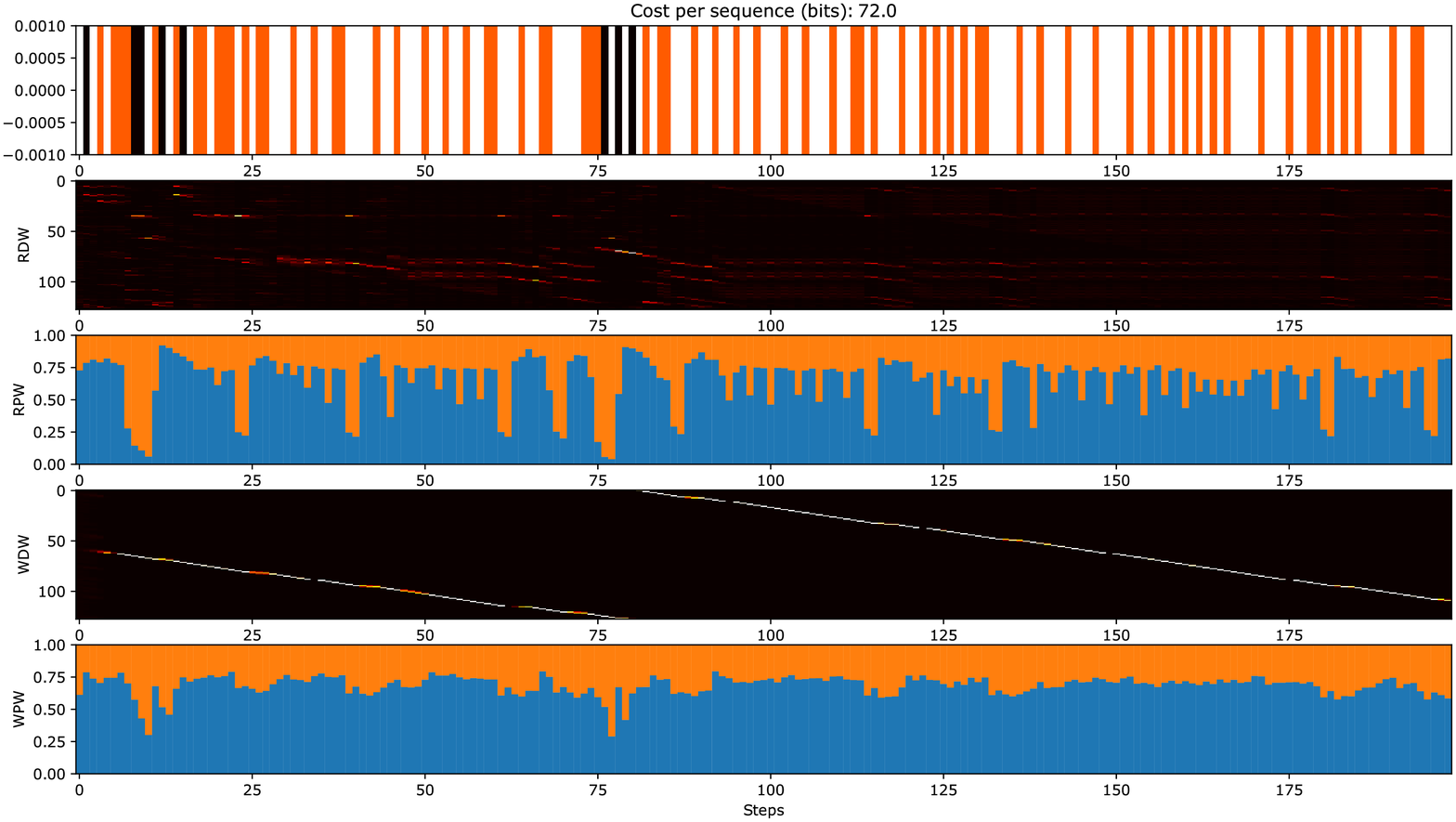}
\par\end{centering}
\caption{Dynamic N-grams (p=2).}
\end{figure}

\begin{figure}[H]
\begin{centering}
\includegraphics[width=1\linewidth]{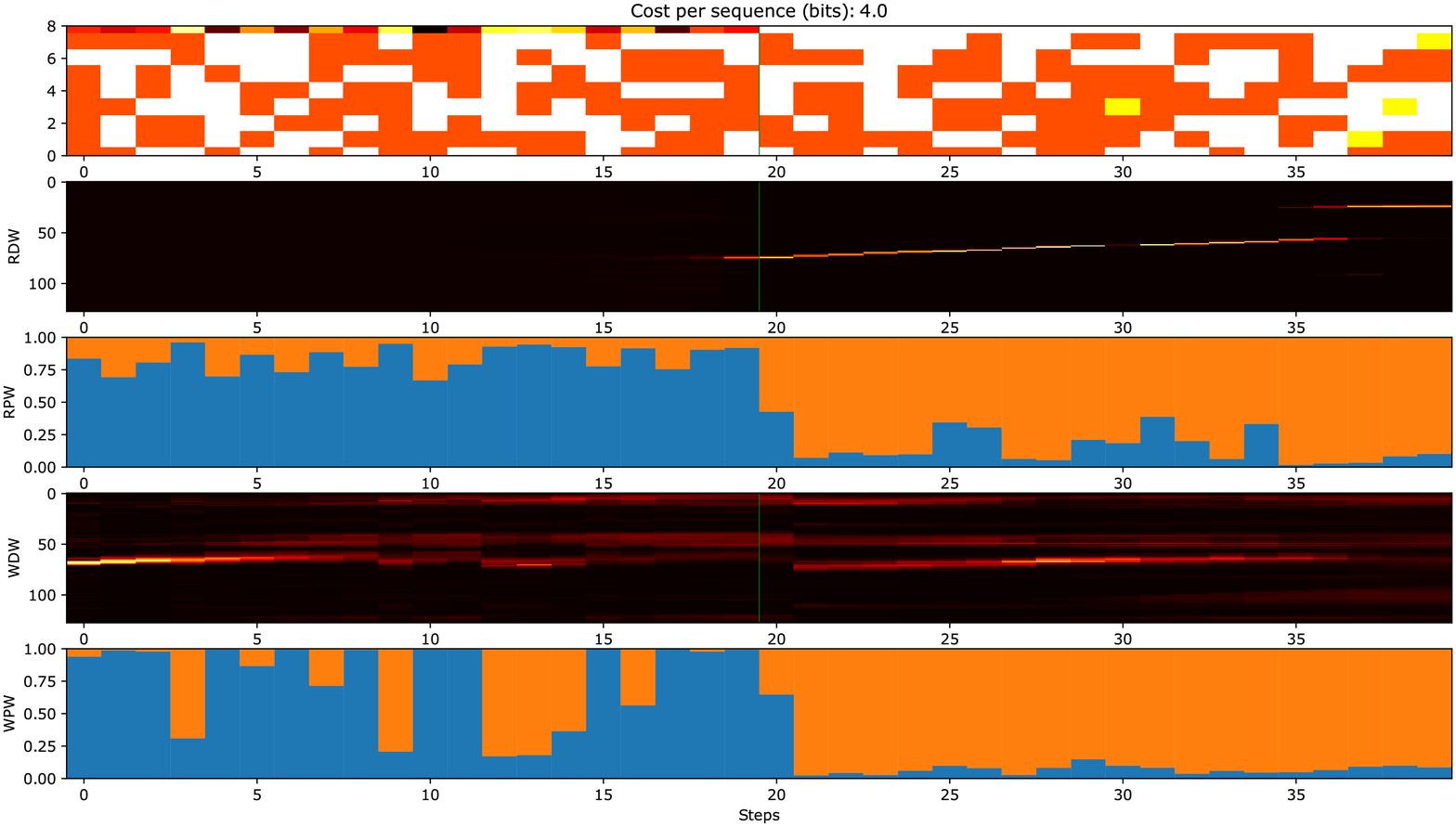}
\par\end{centering}
\caption{Priority Sort (p=2).}
\end{figure}

\begin{figure}[H]
\begin{centering}
\includegraphics[width=1\linewidth]{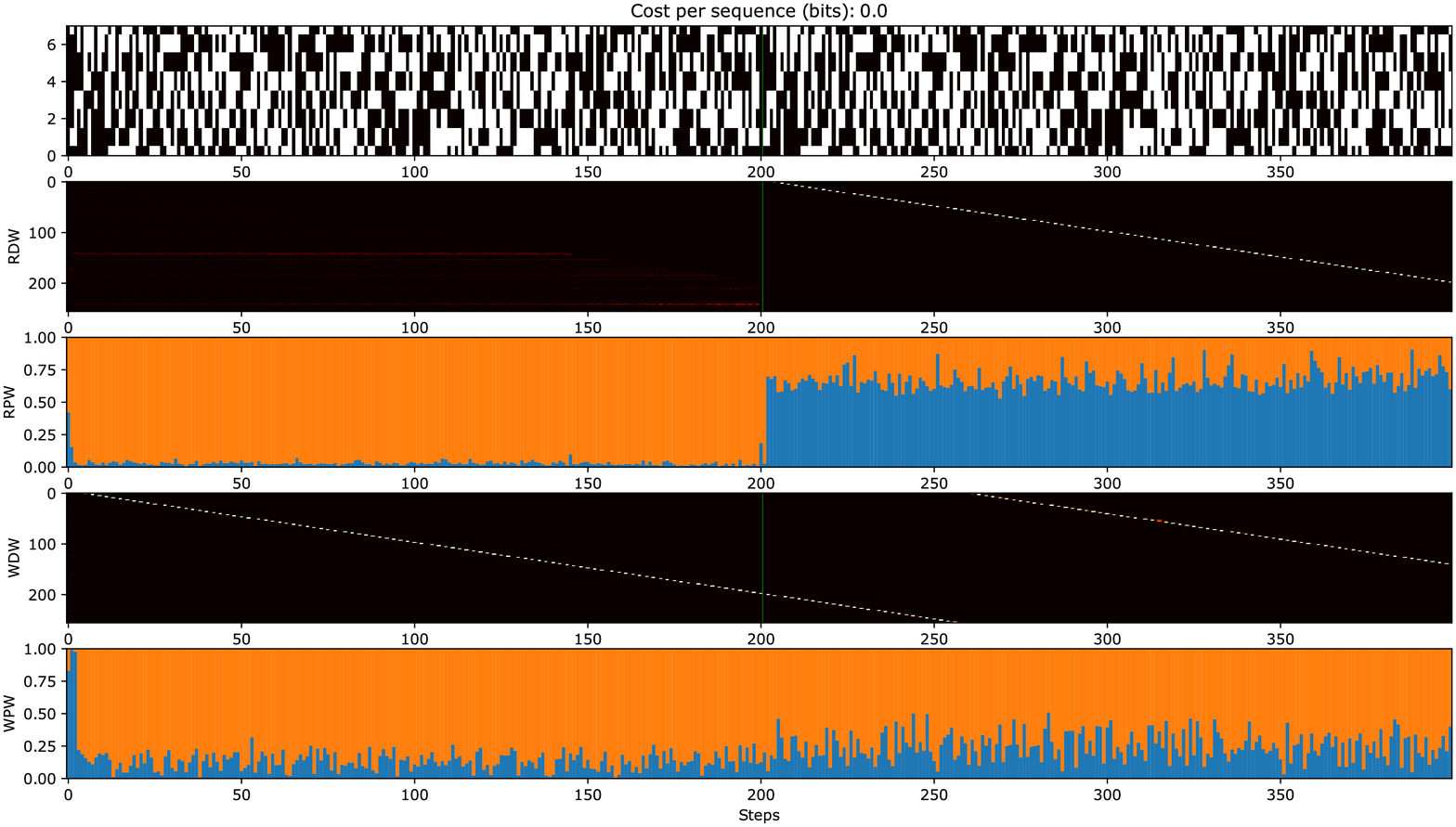}
\par\end{centering}
\caption{Long Copy (p=2).}
\end{figure}

\subsubsection{Visualisation on Program Distribution across Timesteps (Sequencing
Tasks)\label{subsec:Visualization-on-program-1}}

\begin{figure}[H]
\begin{centering}
\includegraphics[width=1\linewidth]{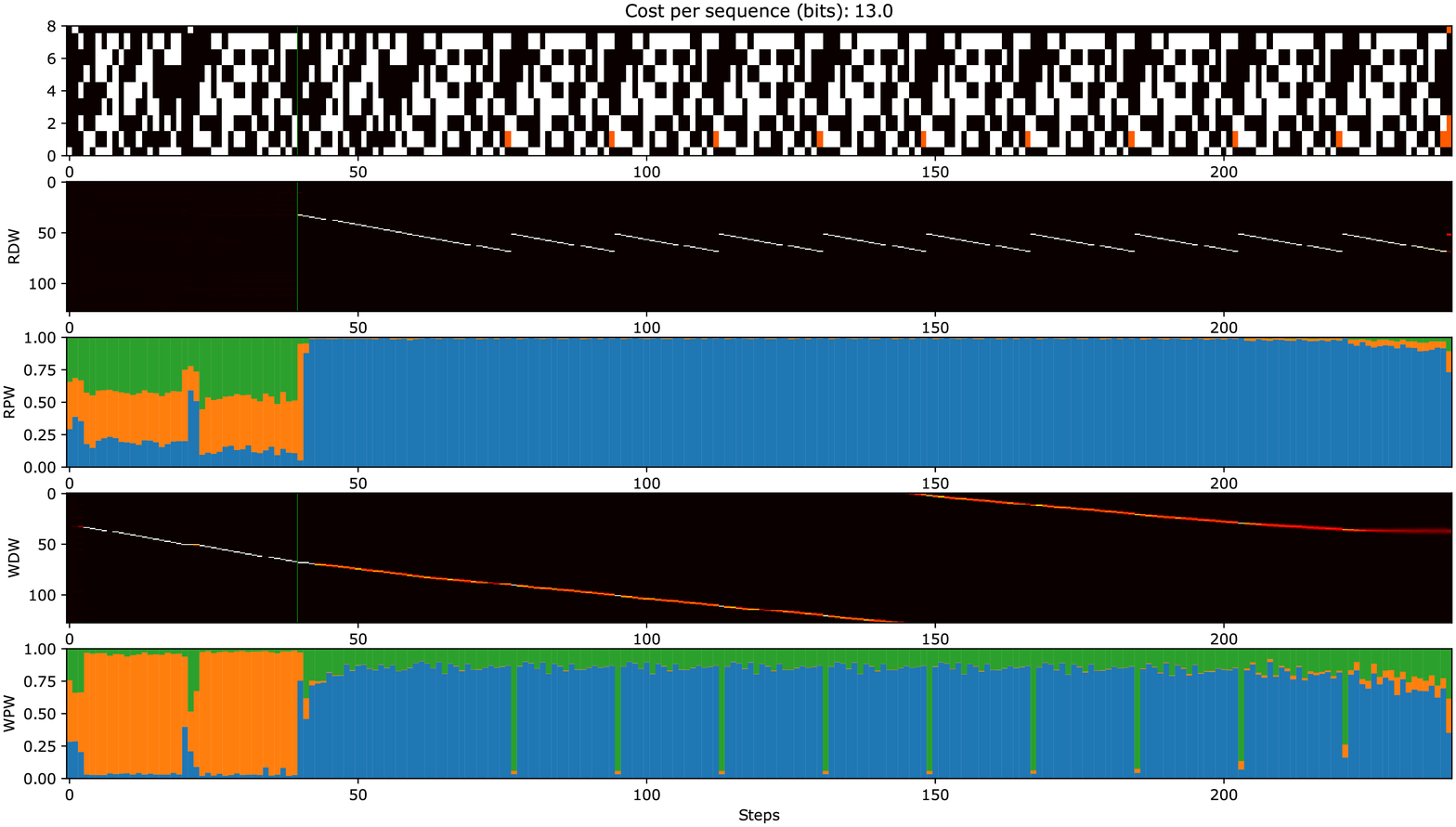}
\par\end{centering}
\caption{Copy+Repeat Copy (p=3).}
\end{figure}

\begin{figure}[H]
\begin{centering}
\includegraphics[width=1\linewidth]{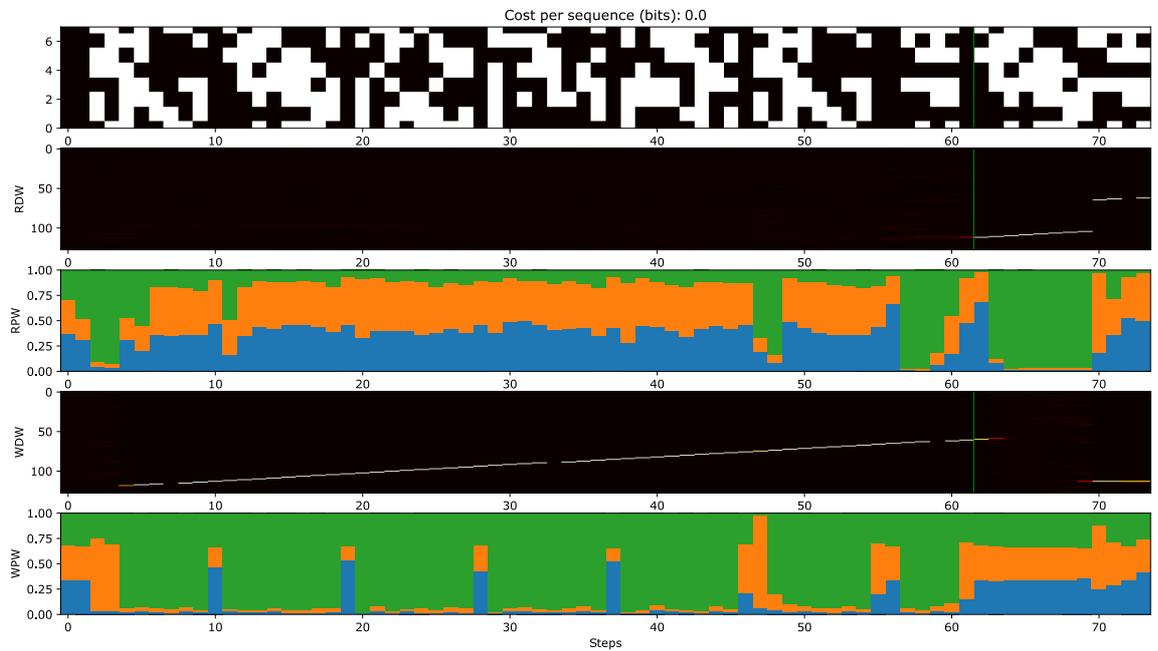}
\par\end{centering}
\caption{Copy+Associative Recall (p=3).}
\end{figure}

\begin{figure}[H]
\begin{centering}
\includegraphics[width=1\linewidth]{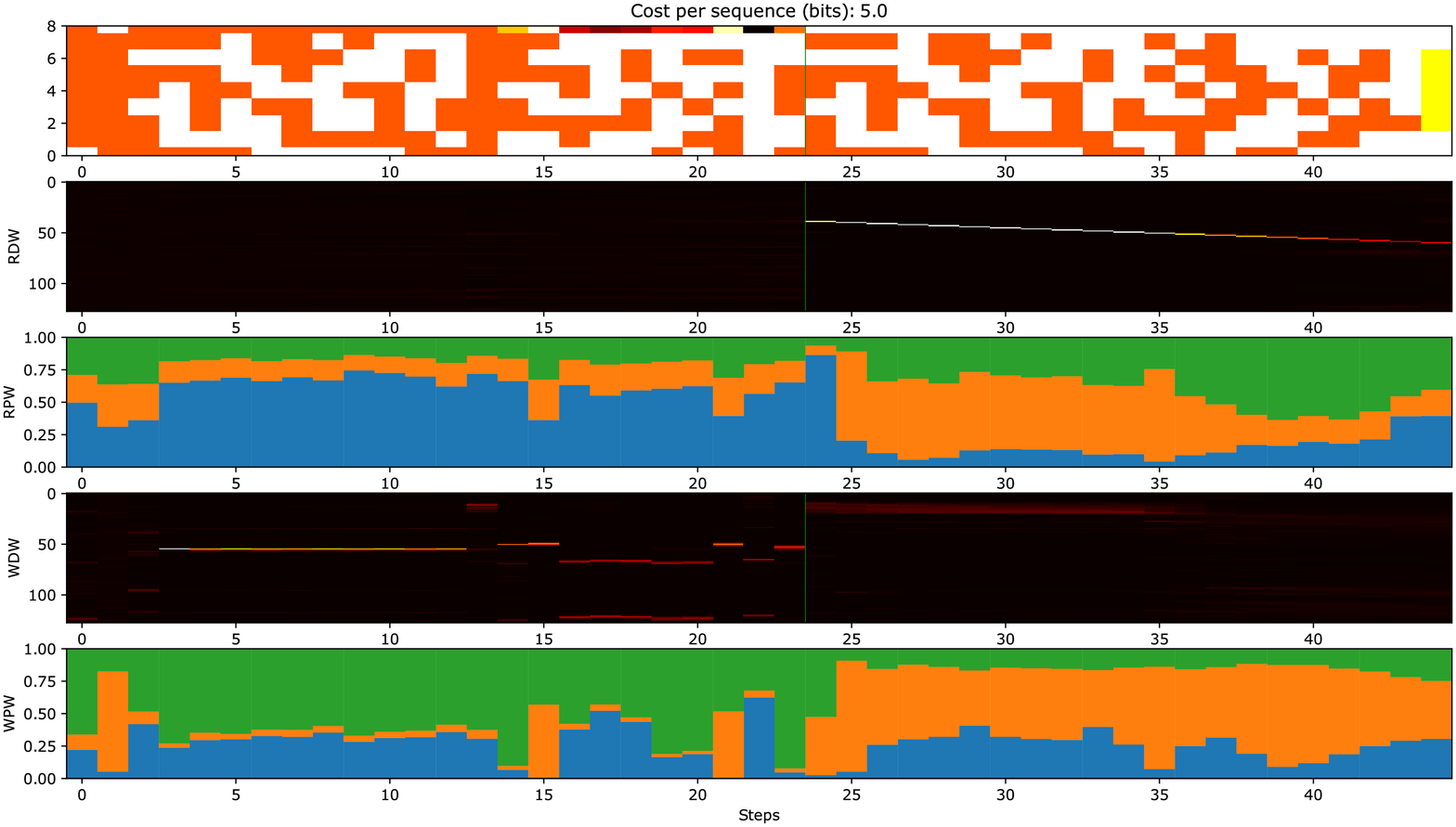}
\par\end{centering}
\caption{Copy+Priority Sort (p=3).}
\end{figure}

\begin{figure}[H]
\begin{centering}
\includegraphics[width=1\linewidth]{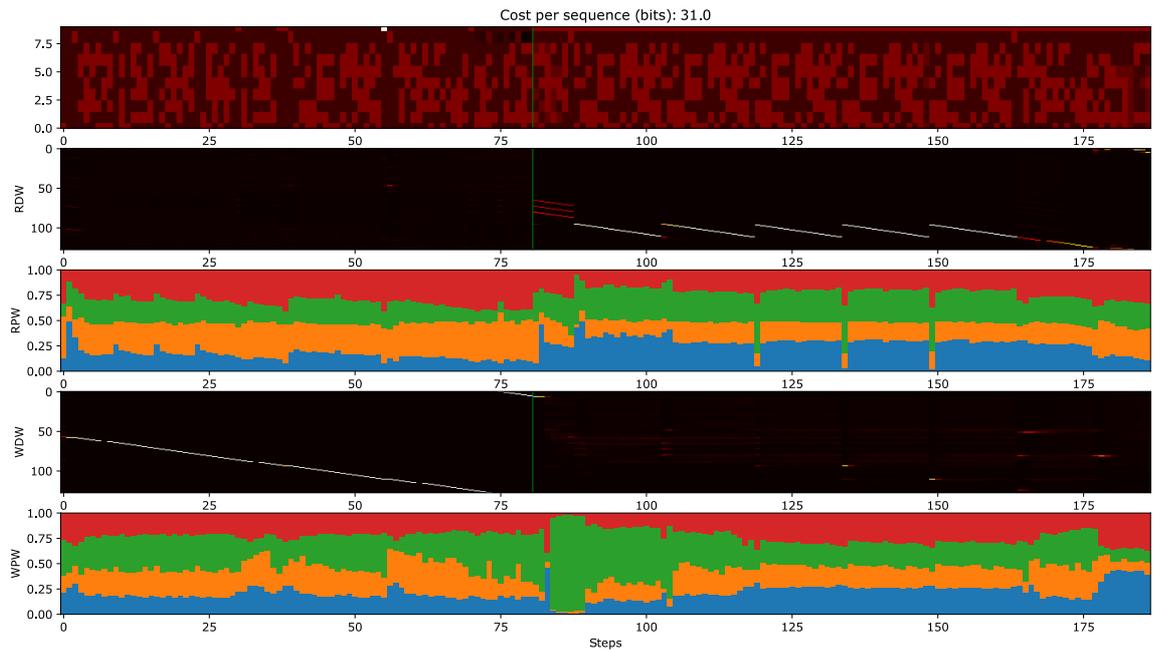}
\par\end{centering}
\caption{Copy+Repeat Copy+Associative Recall+Priority Sort (p=4).}
\end{figure}

\subsubsection{Perseveration Phenomenon in NTM (Sequencing Tasks)\label{subsec:Perseveration-phenomenon-in}}

\begin{figure}[H]
\begin{centering}
\includegraphics[width=1\linewidth]{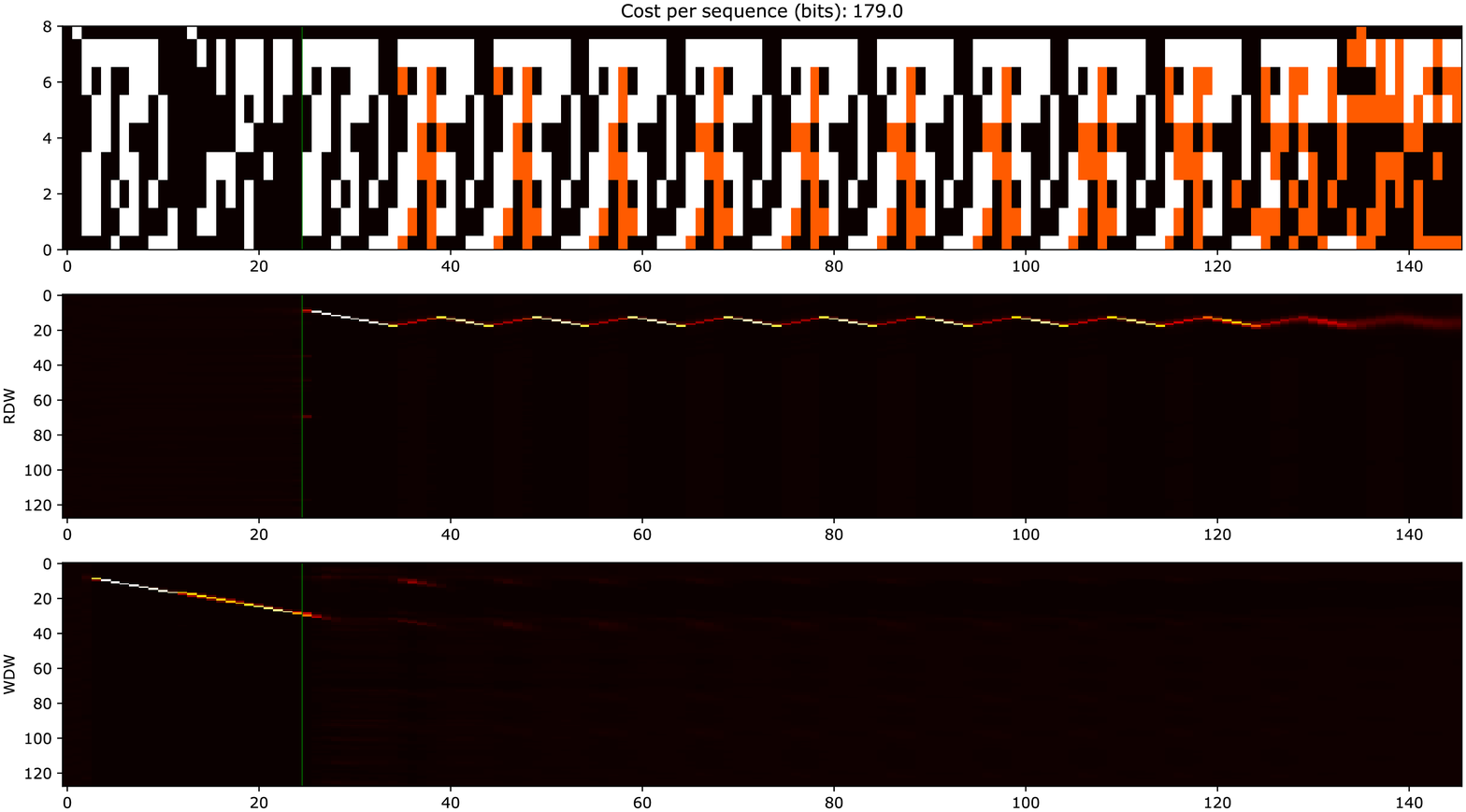}
\par\end{centering}
\caption{Copy+Repeat Copy perseveration (only Repeat Copy).}
\end{figure}

\begin{figure}[H]
\begin{centering}
\includegraphics[width=1\linewidth]{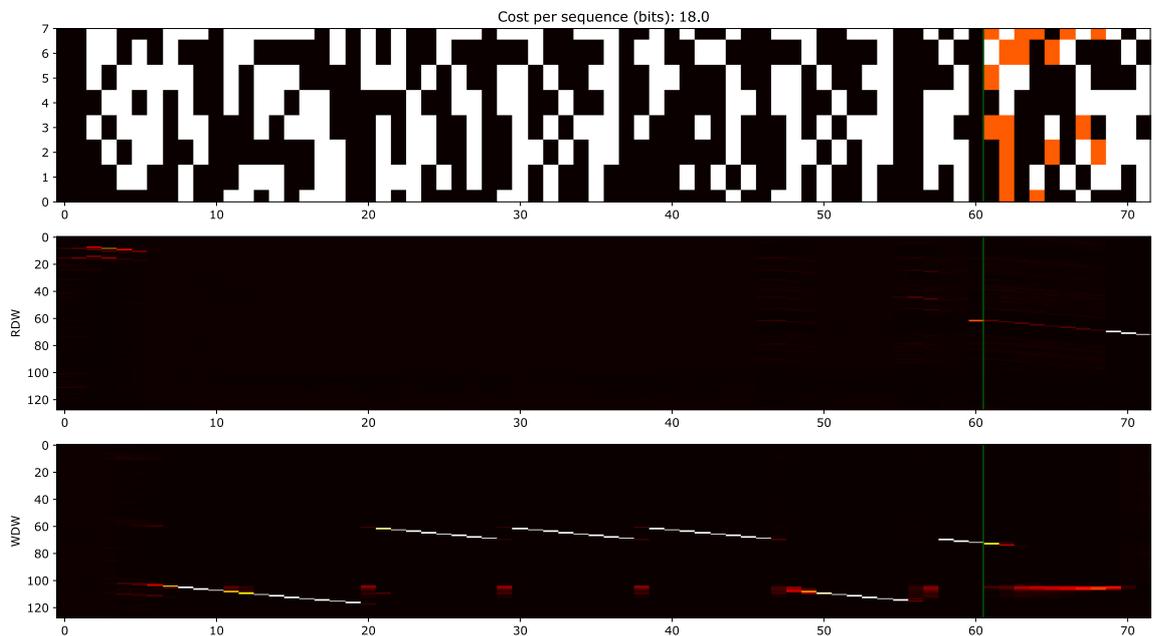}
\par\end{centering}
\caption{Copy+Associative Recall perseveration (only Copy).}
\end{figure}

\begin{figure}[H]
\begin{centering}
\includegraphics[width=1\linewidth]{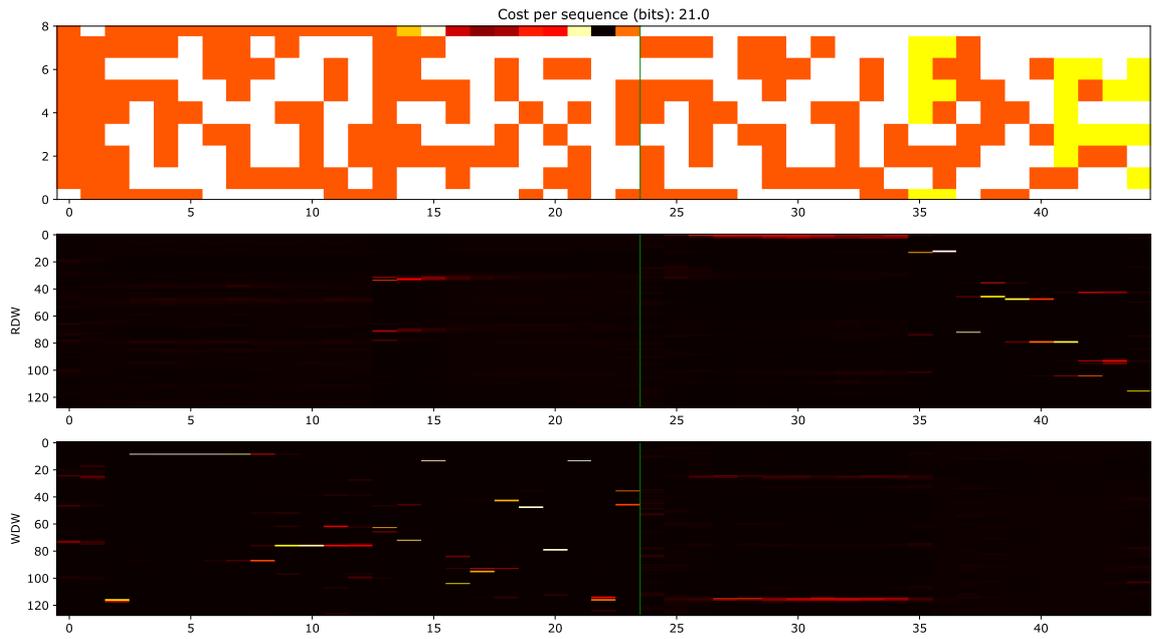}
\par\end{centering}
\caption{Copy+Priority Sort perseveration (only Copy).}
\end{figure}

\begin{figure}[H]
\begin{centering}
\includegraphics[width=1\linewidth]{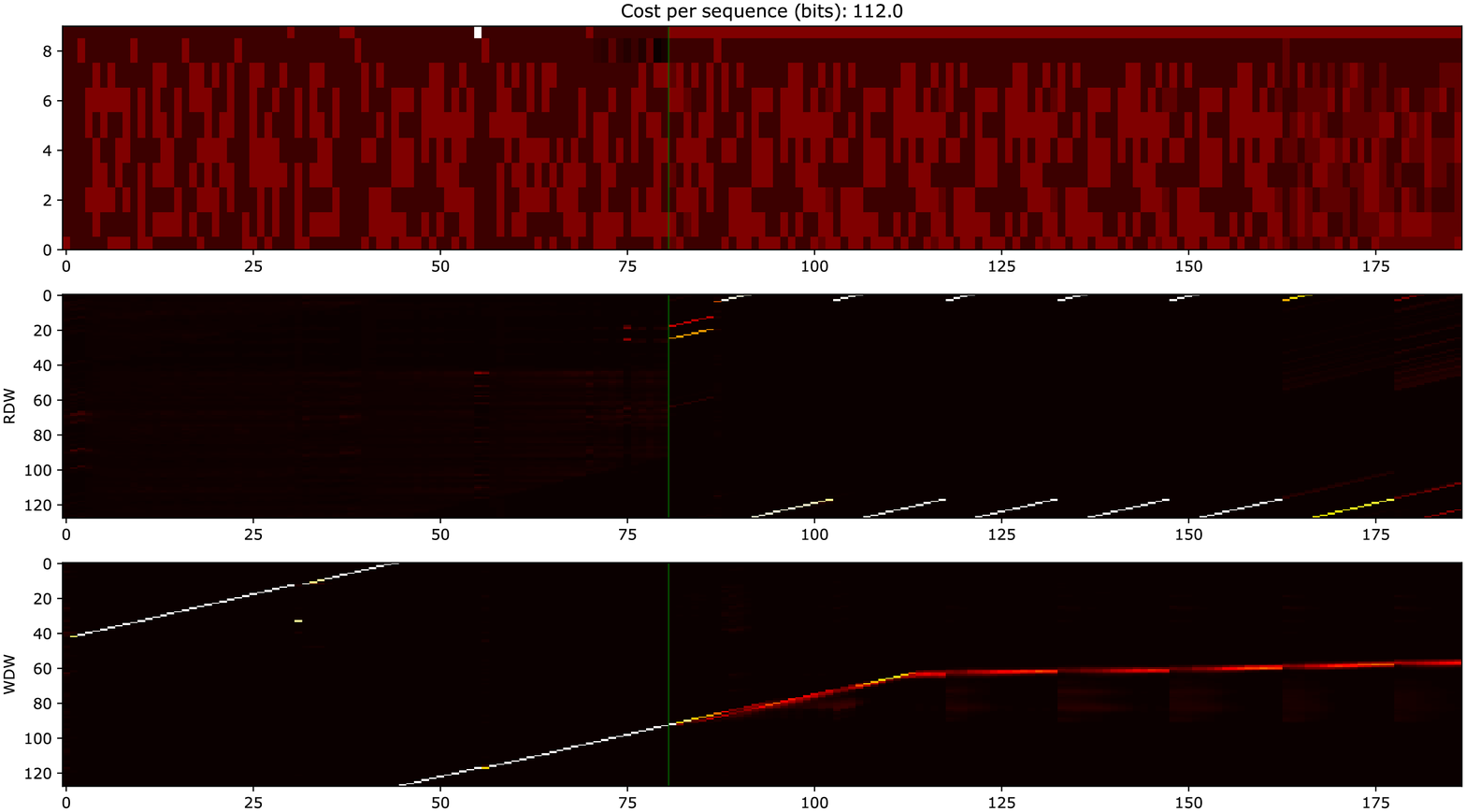}
\par\end{centering}
\caption{Copy+Repeat Copy+Associative Recall+Priority Sort perseveration (only
Repeat Copy).}
\end{figure}

\subsection{Details on Synthetic Tasks\label{subsec:Details-on-Synthetic}}

\subsubsection{NTM Single Tasks}

\begin{table}[H]
\begin{centering}
\begin{tabular}{ccccccccc}
\hline 
\multirow{2}{*}{{\footnotesize{}Tasks}} & \multicolumn{2}{c}{{\footnotesize{}\#Head}} & \multicolumn{2}{c}{{\footnotesize{}Controller Size}} & \multicolumn{2}{c}{{\footnotesize{}Memory Size}} & \multicolumn{2}{c}{{\footnotesize{}\#Parameters}}\tabularnewline
\cline{2-9} \cline{3-9} \cline{4-9} \cline{5-9} \cline{6-9} \cline{7-9} \cline{8-9} \cline{9-9} 
 & {\footnotesize{}NTM} & {\footnotesize{}NUTM} & {\footnotesize{}NTM} & {\footnotesize{}NUTM} & {\footnotesize{}NTM} & {\footnotesize{}NUTM} & {\footnotesize{}NTM} & {\footnotesize{}NUTM}\tabularnewline
\hline 
{\footnotesize{}Copy} & {\footnotesize{}1} & {\footnotesize{}1} & {\footnotesize{}100} & {\footnotesize{}80} & {\footnotesize{}128} & {\footnotesize{}128} & {\footnotesize{}63,260} & {\footnotesize{}52,206}\tabularnewline
\hline 
{\footnotesize{}Repeat Copy} & {\footnotesize{}1} & {\footnotesize{}1} & {\footnotesize{}100} & {\footnotesize{}80} & {\footnotesize{}128} & {\footnotesize{}128} & {\footnotesize{}63,381} & {\footnotesize{}52,307}\tabularnewline
\hline 
{\footnotesize{}Associative Recall} & {\footnotesize{}1} & {\footnotesize{}1} & {\footnotesize{}100} & {\footnotesize{}80} & {\footnotesize{}128} & {\footnotesize{}128} & {\footnotesize{}62,218} & {\footnotesize{}51,364}\tabularnewline
\hline 
{\footnotesize{}Dynamic N-grams} & {\footnotesize{}1} & {\footnotesize{}1} & {\footnotesize{}100} & {\footnotesize{}80} & {\footnotesize{}128} & {\footnotesize{}128} & {\footnotesize{}58,813} & {\footnotesize{}48,619}\tabularnewline
\hline 
{\footnotesize{}Priority Sort} & {\footnotesize{}5} & {\footnotesize{}5} & {\footnotesize{}200} & {\footnotesize{}150} & {\footnotesize{}128} & {\footnotesize{}128} & {\footnotesize{}344,068} & {\footnotesize{}302,398}\tabularnewline
\hline 
{\footnotesize{}Long Copy} & {\footnotesize{}1} & {\footnotesize{}1} & {\footnotesize{}100} & {\footnotesize{}80} & {\footnotesize{}256} & {\footnotesize{}256} & {\footnotesize{}63,260} & {\footnotesize{}52,206}\tabularnewline
\hline 
\end{tabular}
\par\end{centering}
~

\caption{Model hyper-parameters (single tasks).}
\end{table}

\begin{table}[H]
\begin{centering}
{\footnotesize{}}%
\begin{tabular}{lll}
\hline 
\multirow{1}{*}{{\small{}Tasks}} & {\small{}Training} & {\small{}Testing}\tabularnewline
\hline 
{\small{}Copy} & {\small{}Sequence length range: {[}1, 20{]}} & {\small{}Sequence length: 120}\tabularnewline
\hline 
\multirow{2}{*}{{\small{}Repeat Copy}} & {\small{}Sequence length range: {[}1, 10{]}} & {\small{}Sequence length range: {[}10, 20{]}}\tabularnewline
 & {\small{}\#Repeat range: {[}1, 10{]}} & {\small{}\#Repeat range: {[}10, 20{]}}\tabularnewline
\hline 
\multirow{3}{*}{{\small{}Associative Recall}} & {\small{}Sequence length: 3} & {\small{}Sequence length: 3}\tabularnewline
 & {\small{}\#Item range: {[}2, 6{]}} & {\small{}\#Item range: {[}6, 20{]}}\tabularnewline
 & {\small{}Item length: 3} & {\small{}Item length: 3}\tabularnewline
\hline 
{\small{}Dynamic N-grams} & {\small{}Sequence length: 50} & {\small{}Sequence length: 200}\tabularnewline
\hline 
\multirow{2}{*}{{\small{}Priority Sort}} & {\small{}\#Item: 20} & {\small{}\#Item: 20}\tabularnewline
 & {\small{}\#Sorted Item: 16} & {\small{}\#Sorted Item: 20}\tabularnewline
\hline 
{\small{}Long Copy} & {\small{}Sequence length range: {[}1, 40{]}} & {\small{}Sequence length: 200}\tabularnewline
\hline 
\end{tabular}{\footnotesize\par}
\par\end{centering}
~

\caption{Task settings (single tasks).}
\end{table}

\subsubsection{NTM Sequencing Tasks}

\begin{table}[H]
\begin{centering}
\begin{tabular}{ccccccccc}
\hline 
\multirow{2}{*}{{\footnotesize{}Tasks}} & \multicolumn{2}{c}{{\footnotesize{}\#Head}} & \multicolumn{2}{c}{{\footnotesize{}Controller Size}} & \multicolumn{2}{c}{{\footnotesize{}Memory Size}} & \multicolumn{2}{c}{{\footnotesize{}\#Parameters}}\tabularnewline
\cline{2-9} \cline{3-9} \cline{4-9} \cline{5-9} \cline{6-9} \cline{7-9} \cline{8-9} \cline{9-9} 
 & {\footnotesize{}NTM} & {\footnotesize{}NUTM} & {\footnotesize{}NTM} & {\footnotesize{}NUTM} & {\footnotesize{}NTM} & {\footnotesize{}NUTM} & {\footnotesize{}NTM} & {\footnotesize{}NUTM}\tabularnewline
\hline 
{\footnotesize{}C+RC} & {\footnotesize{}1} & {\footnotesize{}1} & {\footnotesize{}200} & {\footnotesize{}150} & {\footnotesize{}128} & {\footnotesize{}128} & {\footnotesize{}206,481} & {\footnotesize{}153,941}\tabularnewline
\hline 
{\footnotesize{}C+AR} & {\footnotesize{}1} & {\footnotesize{}1} & {\footnotesize{}200} & {\footnotesize{}150} & {\footnotesize{}128} & {\footnotesize{}128} & {\footnotesize{}206,260} & {\footnotesize{}153,770}\tabularnewline
\hline 
{\footnotesize{}C+PS} & {\footnotesize{}3} & {\footnotesize{}3} & {\footnotesize{}200} & {\footnotesize{}150} & {\footnotesize{}128} & {\footnotesize{}128} & {\footnotesize{}275,564} & {\footnotesize{}263,894}\tabularnewline
\hline 
{\footnotesize{}C+RC+AR+PS} & {\footnotesize{}3} & {\footnotesize{}3} & {\footnotesize{}250} & {\footnotesize{}200} & {\footnotesize{}128} & {\footnotesize{}128} & {\footnotesize{}394,575} & {\footnotesize{}448,379}\tabularnewline
\hline 
\end{tabular}
\par\end{centering}
~

\caption{Model hyper-parameters (sequencing tasks).}
\end{table}

\begin{table}[H]
\begin{centering}
{\footnotesize{}}%
\begin{tabular}{lll}
\hline 
\multirow{1}{*}{{\small{}Tasks}} & {\small{}Training} & {\small{}Testing}\tabularnewline
\hline 
\multirow{2}{*}{{\small{}C+RC}} & {\small{}Sequence length range: {[}1, 10{]}} & {\small{}Sequence length range: {[}10, 20{]}}\tabularnewline
 & {\small{}\#Repeat range: {[}1, 10{]}} & {\small{}\#Repeat range: {[}10, 15{]}}\tabularnewline
\hline 
\multirow{3}{*}{{\small{}C+AR}} & {\small{}Sequence length range: {[}1, 10{]}} & {\small{}Sequence length range: {[}10, 20{]}}\tabularnewline
 & {\small{}\#Item range: {[}2, 4{]}} & {\small{}\#Item range: {[}4, 6{]}}\tabularnewline
 & {\small{}Item length: 8} & {\small{}Item length: 8}\tabularnewline
\hline 
\multirow{3}{*}{{\small{}C+PS}} & {\small{}Sequence length range: {[}1, 10{]}} & {\small{}Sequence length range: {[}10, 20{]}}\tabularnewline
 & {\small{}\#Item: 10} & {\small{}\#Item: 10}\tabularnewline
 & {\small{}\#Sorted Item: 8} & {\small{}\#Sorted Item: 10}\tabularnewline
\hline 
\multirow{6}{*}{{\small{}C+RC+AR+PS}} & {\small{}Sequence length range: {[}1, 10{]}} & {\small{}Sequence length range: {[}10, 20{]}}\tabularnewline
 & {\small{}\#Repeat range: {[}1, 5{]}} & {\small{}\#Repeat: 6}\tabularnewline
 & {\small{}\#Item range: {[}2, 4{]}} & {\small{}\#Item: 5}\tabularnewline
 & {\small{}Item length: 6} & {\small{}Item length: 6}\tabularnewline
 & {\small{}\#Item: 10} & {\small{}\#Item: 10}\tabularnewline
 & {\small{}\#Sorted Item: 8} & {\small{}\#Sorted Item: 10}\tabularnewline
\hline 
\end{tabular}{\footnotesize\par}
\par\end{centering}
~

\caption{Task settings (sequencing tasks).}
\end{table}

\subsubsection{Continual Procedure Learning Tasks}

\begin{table}[H]
\begin{centering}
{\small{}}%
\begin{tabular}{lll}
\hline 
\multirow{1}{*}{{\small{}Tasks}} & {\small{}Training} & {\small{}Testing}\tabularnewline
\hline 
{\small{}Copy} & {\small{}Sequence length range: {[}1, 10{]}} & {\small{}Sequence length range: {[}1, 10{]}}\tabularnewline
\hline 
\multirow{2}{*}{{\small{}Repeat Copy}} & {\small{}Sequence length range: {[}1, 5{]}} & {\small{}Sequence length range: {[}1, 5{]}}\tabularnewline
 & {\small{}\#Repeat range: {[}1, 5{]}} & {\small{}\#Repeat range: {[}1, 5{]}}\tabularnewline
\hline 
\multirow{3}{*}{{\small{}Associative Recall}} & {\small{}Sequence length: 3} & {\small{}Sequence length: 3}\tabularnewline
 & {\small{}\#Item range: {[}2, 3{]}} & {\small{}\#Item range: {[}2, 3{]}}\tabularnewline
 & {\small{}Item length: 3} & {\small{}Item length: 3}\tabularnewline
\hline 
\multirow{2}{*}{{\small{}Priority Sort}} & {\small{}\#Item: 10} & {\small{}\#Item: 10}\tabularnewline
 & {\small{}\#Sorted Item: 8} & {\small{}\#Sorted Item: 8}\tabularnewline
\hline 
\end{tabular}{\small\par}
\par\end{centering}
~

\caption{Task settings (continual procedure learning tasks).}
\end{table}

\subsection{Details on Few-Shot Learning Task\label{subsec:Details-on-Few-shot}}

We use similar hyper-parameters as in \citet{santoro2016meta}, which
are reported in Tab. \ref{tab:NUTM-hyper-parameters-for-1}.

\begin{table}[H]
\begin{centering}
\begin{tabular}{ccccccccc}
\hline 
{\scriptsize{}Model} & {\scriptsize{}$p$} & {\scriptsize{}\#Head} & {\scriptsize{}Controller Size} & {\scriptsize{}$N$} & {\scriptsize{}$M$} & {\scriptsize{}$\mathbf{M}_{p}.K$ Size} & {\scriptsize{}Optimiser} & {\scriptsize{}Learning Rate}\tabularnewline
\hline 
{\scriptsize{}MANN (LRUA)} & {\scriptsize{}1} & {\scriptsize{}4} & {\scriptsize{}200} & {\scriptsize{}128} & {\scriptsize{}40} & {\scriptsize{}0} & {\scriptsize{}RMSprop} & {\scriptsize{}$10^{-4}$}\tabularnewline
\hline 
{\scriptsize{}NUTM (LRUA)} & {\scriptsize{}2} & {\scriptsize{}4} & {\scriptsize{}180} & {\scriptsize{}128} & {\scriptsize{}40} & {\scriptsize{}2} & {\scriptsize{}RMSprop} & {\scriptsize{}$10^{-4}$}\tabularnewline
{\scriptsize{}NUTM (LRUA)} & {\scriptsize{}3} & {\scriptsize{}4} & {\scriptsize{}150} & {\scriptsize{}128} & {\scriptsize{}40} & {\scriptsize{}3} & {\scriptsize{}RMSprop} & {\scriptsize{}$10^{-4}$}\tabularnewline
\hline 
\end{tabular}
\par\end{centering}
~

\caption{Hyper-parameters for few-shot learning.\label{tab:NUTM-hyper-parameters-for-1}}
\end{table}

Testing accuracy through time is listed below,

\begin{figure}[H]
\begin{centering}
\includegraphics[width=1\linewidth]{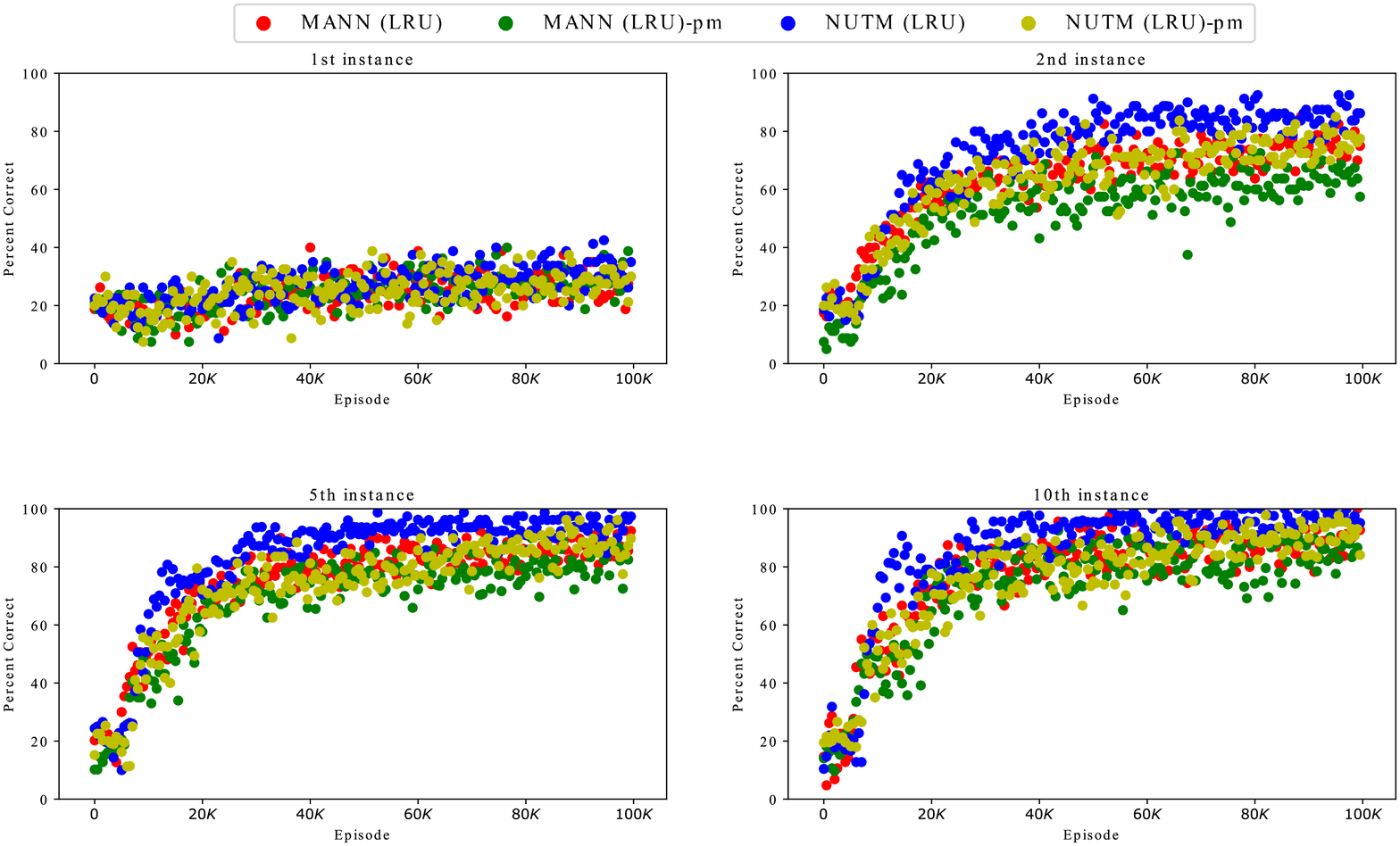}
\par\end{centering}
\caption{Testing accuracy during training (five random classes/episode, one-hot
vector labels, of length 50).}
\end{figure}

\begin{figure}[H]
\begin{centering}
\includegraphics[width=1\linewidth]{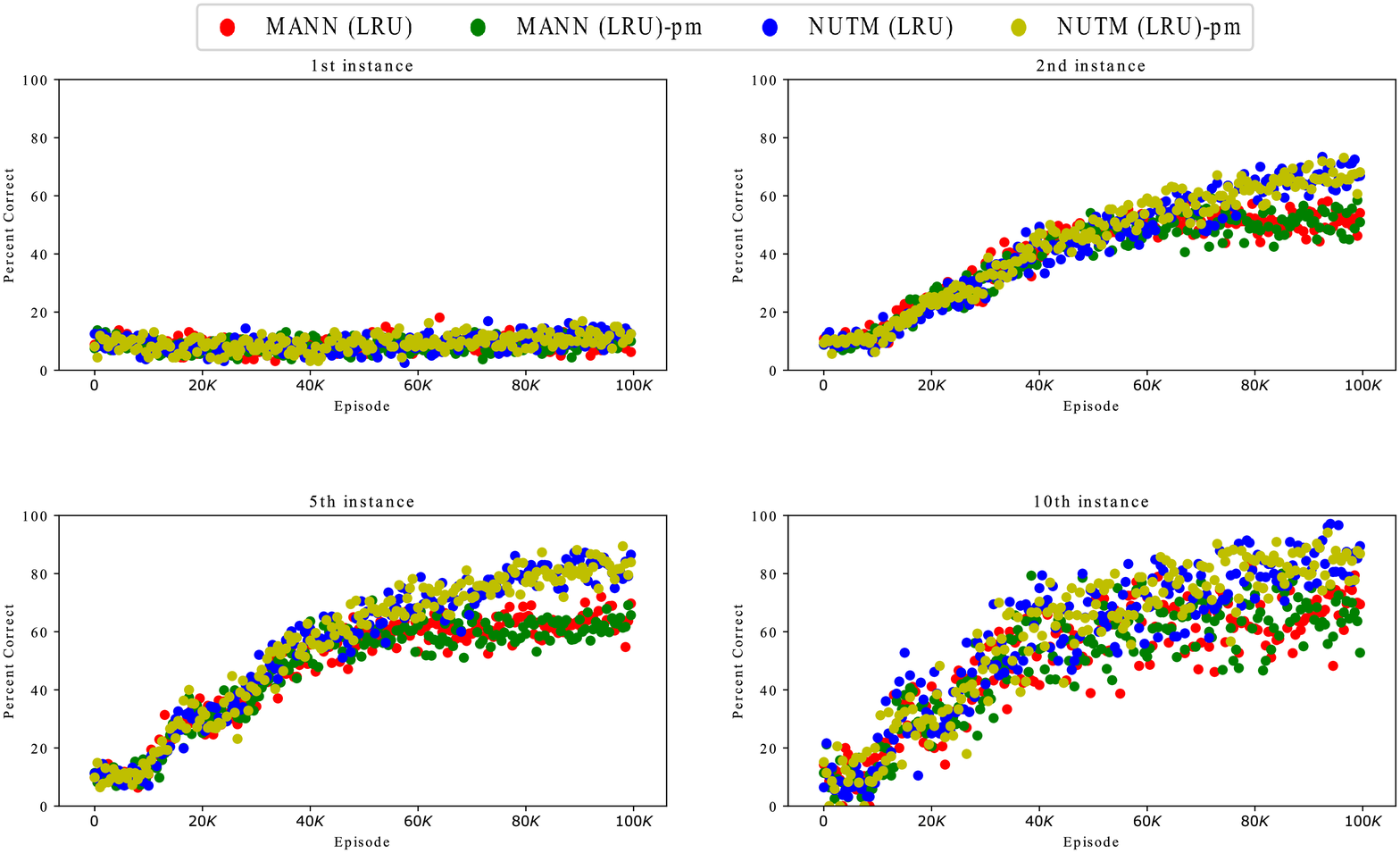}
\par\end{centering}
\caption{Testing accuracy during training (ten random classes/episode, one-hot
vector labels, of length 75).}
\end{figure}

Final testing accuracy is listed as follows,

\begin{table}
\begin{centering}
\begin{tabular}{lccccccc}
\hline 
\multirow{2}{*}{Model} & Persistent & \multicolumn{3}{c}{5 classes} & \multicolumn{3}{c}{10 classes}\tabularnewline
\cline{3-8} \cline{4-8} \cline{5-8} \cline{6-8} \cline{7-8} \cline{8-8} 
 & memory\tablefootnote{If the memory is not artificially erased between episodes, it is called
persistent. This mode is hard for the case of 5 classes \citet{santoro2016meta} } & $2^{nd}$ & $3^{rd}$ & $5^{th}$ & $2^{nd}$ & $3^{rd}$ & $5^{th}$\tabularnewline
\hline 
MANN (LRUA){*} & No & 82.8 & 91.0 & 94.9 & - & - & -\tabularnewline
MANN (LRUA) & No & 82.3 & 88.7 & 92.3 & 52.7 & 60.6 & 64.7\tabularnewline
NUTM (LRUA) & No & \textbf{85.7} & \textbf{91.3} & \textbf{95.5} & \textbf{68.0} & \textbf{78.1} & \textbf{82.8}\tabularnewline
\hline 
Human{*} & Yes & 57.3 & 70.1 & 81.4 & - & - & -\tabularnewline
MANN (LRUA){*} & Yes & $\approx58.0$ & - & $\approx75.0$ & $\approx60.0$ & - & $\approx80.0$\tabularnewline
MANN (LRUA) & Yes & 66.2 & 73.4 & 81.0 & 51.3 & 59.2 & 63.3\tabularnewline
NUTM (LRUA) & Yes & \textbf{77.8} & \textbf{85.8} & \textbf{89.8} & \textbf{69.0} & \textbf{77.9} & \textbf{82.7}\tabularnewline
\hline 
\end{tabular}
\par\end{centering}
~

\caption{Test-set classification accuracy (\%) on the Omniglot dataset after
100,000 episodes of training. {*} denotes available results from Santoro
et al., (2016) (some are estimated from plotted figures).}
\end{table}

It should be noted that our goal was not to achieve state of the art
performance on this dataset. It was to exhibit the benefit of NSM
to MANN. Compared to current methods, the MANN and NUTM used in our
experiments do not use CNN to extract visual features, thus achieve
lower accuracy. 

\subsection{Details on bAbI Task\label{subsec:Details-on-bAbI}}

We train the models using RMSprop optimiser with fixed learning rate
of $10^{-4}$ and momentum of 0.9. The batch size is 32 and we adopt
layer \foreignlanguage{australian}{normalisation} \citet{lei2016layer}
to DNC's layers. Following common practices \citet{W18-2606}, we
also remove temporal linkage for faster training. The details of hyper-parameters
are listed in Table \ref{tab:NUTM-hyper-parameters-for}. Full NUTM
($p=4$) results are reported in Table \ref{tab:babifull}.

\begin{table}[H]
\begin{centering}
\begin{tabular}{ccccccc}
\hline 
\#Head & Controller Size & $N$ & $M$ & $p$ & $\mathbf{M}_{p}.K$ Size & \#Parameters\tabularnewline
\hline 
4 & 172 & 196 & 64 & 4 & 4 & 794,773\tabularnewline
4 & 200 & 196 & 64 & 2 & 2 & 934,787\tabularnewline
\hline 
\end{tabular}
\par\end{centering}
~

\caption{NUTM hyper-parameters for bAbI.\label{tab:NUTM-hyper-parameters-for}}
\end{table}

\begin{table}[H]
\begin{centering}
\begin{tabular}{lcc}
\hline 
Task & bAbI Best Results & bAbI Mean Results\tabularnewline
\hline 
1: 1 supporting fact & 0.0 & 0.0 $\pm$ 0.0\tabularnewline
2: 2 supporting facts  & 0.2 & 0.6 $\pm$ 0.3\tabularnewline
3: 3 supporting facts  & 4.0 & 7.6 $\pm$ 3.9\tabularnewline
4: 2 argument relations  & 0.0 & 0.0 $\pm$ 0.0\tabularnewline
5: 3 argument relations  & 0.4 & 1.0 $\pm$ 0.4\tabularnewline
6: yes/no questions  & 0.0 & 0.0 $\pm$ 0.0\tabularnewline
7: counting  & 1.9 & 1.5 $\pm$ 0.8\tabularnewline
8: lists/sets  & 0.6 & 0.3 $\pm$ 0.2\tabularnewline
9: simple negation  & 0.0 & 0.0 $\pm$ 0.0\tabularnewline
10: indefinite knowledge  & 0.1 & 0.1 $\pm$ 0.0\tabularnewline
11: basic coreference  & 0.0 & 0.0 $\pm$ 0.0\tabularnewline
12: conjunction  & 0.0 & 0.0 $\pm$ 0.0\tabularnewline
13: compound coreference  & 0.1 & 0.0 $\pm$ 0.0\tabularnewline
14: time reasoning  & 0.3 & 1.6 $\pm$ 2.2\tabularnewline
15: basic deduction  & 0.0 & 2.6 $\pm$ 8.3\tabularnewline
16: basic induction  & 49.3 & 52.0 $\pm$ 1.7\tabularnewline
17: positional reasoning  & 4.7 & 18.4 $\pm$ 12.7\tabularnewline
18: size reasoning  & 0.4 & 1.6 $\pm$ 1.1\tabularnewline
19: path finding  & 4.3 & 23.7 $\pm$ 32.2\tabularnewline
20: agent\textquoteright s motivation  & 0.0 & 0.0 $\pm$ 0.0\tabularnewline
\hline 
Mean Error (\%) & 3.3 & 5.6 $\pm$ 1.9\tabularnewline
\hline 
Failed (Err. >5\%) & 1 & 3 $\pm$ 1.2\tabularnewline
\hline 
\end{tabular}
\par\end{centering}
~

\caption{NUTM ($p=4$) bAbI best and mean errors (\%).\label{tab:babifull}}
\end{table}

\subsection{Others}

If we deliberately set the key dimension equal to the number of programs,
we can even place an orthogonal basis constraint on the key space
of NSM by minimising the following loss, 

\begin{equation}
l_{p_{2}}=\left\Vert \mathbf{M}_{p}.K\mathbf{M}_{p}.K^{T}-\mathbf{I}\right\Vert 
\end{equation}
where $\mathbf{M}_{p}.K$ and $\mathbf{I}$ denote the key part in
NSM and the identity matrix, respectively.

For all tasks, $\eta_{t}$ is fixed to $0.1$, reducing with decay
rate of $0.9$. \selectlanguage{australian}%

\clearpage{}

\bibliographystyle{plainnat}
\addcontentsline{toc}{chapter}{\bibname}\bibliography{thesis}

\textbf{\textemdash \textemdash \textemdash \textemdash \textemdash \textemdash \textemdash \textemdash \textemdash \textendash }\textbf{\Large{}}\\
\textbf{Every reasonable effort has been made to acknowledge the owners
of copyright material. I would be pleased to hear from any copyright
owner who has been omitted or incorrectly acknowledged.}

\end{document}